\documentclass[phd,oneside]{ubcthesis}

\usepackage[USenglish]{babel}

\newif\ifcolorversion
\colorversionfalse %

\usepackage{xcolor}

\ifcolorversion
    \newcommand{\orange}[1]{\textcolor{orange}{#1}}
    \newcommand{\cyan}[1]{\textcolor{cyan}{#1}}
    
    \newcommand{\red}[1]{\textcolor{red}{#1}}
    \newcommand{\blue}[1]{\textcolor{blue}{#1}}
    \newcommand{\purple}[1]{\textcolor{purple}{#1}}
\else
    \newcommand{\orange}[1]{\textcolor{black}{#1}}
    \newcommand{\cyan}[1]{\textcolor{black}{#1}}
    
    \newcommand{\red}[1]{\textcolor{black}{#1}}
    \newcommand{\blue}[1]{\textcolor{black}{#1}}
    \newcommand{\purple}[1]{\textcolor{black}{#1}}
\fi

\usepackage[unicode=true,
  linktocpage,
  linkbordercolor={0.5 0.5 1},
  citebordercolor={0.5 1 0.5},
  linkcolor=blue]{hyperref}
\usepackage{amsmath,amsfonts,amssymb,amsthm}
\usepackage{graphicx}
\usepackage{afterpage}
\usepackage{float}
\usepackage{longtable}
\usepackage{booktabs}
\makeatletter
\def\@captype{table}
\usepackage{multirow}
\usepackage{colortbl}
\usepackage{pdflscape}
\usepackage{psfrag}
\usepackage{parskip}
\usepackage{caption}
\usepackage{subcaption}
\usepackage{wrapfig}

\usepackage{thmtools}
\usepackage{mathtools}
\declaretheorem[name=Hypothesis,refname={Hypothesis,Hypotheses}]{hyp}
\declaretheorem[name=Proposition,refname={Proposition,Propositions}]{prop}

\usepackage{algorithmic}
\usepackage[linesnumbered,ruled]{algorithm2e}
\makeatletter
\newenvironment{Ualgorithm}[1][htpb]{\def\@algocf@post@ruled{\kern\interspacealgoruled\hrule  height\algoheightrule\kern3pt\relax}%
\def\@algocf@capt@ruled{under}%
\begin{algorithm}[#1]}
{\end{algorithm}}
\makeatother

\usepackage[capitalize,noabbrev]{cleveref}
\usepackage{bm}
\usepackage{bbm}
\usepackage{tcolorbox}
\newcommand{\bmf}[1]{\bm{\mathsf{#1}}}
\DeclareMathOperator*{\E}{\mathbb{E}}
\DeclareMathOperator*{\Var}{Var}

\newcommand{\loss}{\mathcal{L}}

\newcommand{\R}{\mathbb{R}}
\newcommand{\tar}{\text{tar}}
\newcommand{\tp}{^\mathsf{T}}

\newcommand{\ve}{\bmf{e}}
\newcommand{\va}{\bmf{a}}

\newcommand{\vG}{\bmf{G}}
\newcommand{\vO}{\bmf{O}}

\newcommand{\vh}{\bmf{h}}
\newcommand{\vB}{\bmf{B}}
\newcommand{\vb}{\bmf{b}}
\newcommand{\vc}{\bmf{c}}

\newcommand{\vL}{\bmf{L}}
\newcommand{\vp}{\bmf{p}}
\newcommand{\vq}{\bmf{q}}
\newcommand{\vw}{\bmf{w}}
\newcommand{\vx}{\bmf{x}}
\newcommand{\vy}{\bmf{y}}

\newcommand{\vz}{\bmf{z}}
\newcommand{\vu}{\bmf{u}}

\newcommand{\vchi}{\bmf{\chi}}
\definecolor{cgpt}{HTML}{00E079}
\definecolor{cclaude}{HTML}{9B4400}

\newcommand{\cvxu}{\cyan{\bmf{x}_u}}
\newcommand{\cvyu}{\cyan{\bmf{y}_u}}
\newcommand{\cxu}{\cyan{x_u}}
\newcommand{\cyu}{\cyan{y_u}}

\newcommand{\cvxo}{\orange{\bmf{x}_o}}

\newcommand{\cxo}{\orange{x_o}}
\newcommand{\cyo}{\orange{y_o}}

\newcommand{\cvchiu}{\cyan{\bmf{\chi}_u}}
\newcommand{\cvchio}{\orange{\bmf{\chi}_o}}

\DeclareMathOperator*{\argmax}{argmax}
\DeclareMathOperator{\bigO}{\mathcal{O}}
\newcommand\indep{\protect\mathpalette{\protect\independenT}{\perp}}
\def\independenT#1#2{\mathrel{\rlap{$#1#2$}\mkern2mu{#1#2}}}

\usepackage{enumitem}

\DeclareMathOperator{\Softmax}{\mathsf{Softmax}}
\DeclareMathOperator{\Softmaxcol}{\mathsf{Softmax\_column}}

\usepackage{csquotes}
\usepackage[style=authoryear,sortcites,sorting=ynt,maxcitenames=2,maxbibnames=50,abbreviate=false,useprefix,uniquelist=minyear,natbib]{biblatex}
\DeclareDelimFormat{nameyeardelim}{\space}  %
\renewbibmacro{in:}{}
\DeclareFieldFormat[unpublished,misc]{title}{\mkbibquote{#1}}  %
\DeclareNameAlias{sortname}{given-family}

\usepackage[disable]{todonotes}
\newcommand{\danica}[2][noinline]{\todo[color=violet!20,#1]{Danica: #2}}
\newcommand{\committeeMember}[4]{#1, #2, #3\\\textit{#4}}

\def \theAuthor {Yi (Joshua) Ren}
\def \theTitle {Learning Dynamics of Deep Learning}
\def \theSubtitle {--- Force Analysis of Deep Neural Networks}
\def \theDegree {Doctor of Philosophy} %
\def \theProgramTitle {Computer Science}
\def \theCopyrightyear {2025}
\def \theSubmitdate {September 2025}
\institution{The University Of British Columbia}

\faculty{The Faculty of Graduate and Postdoctoral Studies}
\institutionaddress{Vancouver}

\previousdegree{M.S., University of Edinburgh, UK, 2019}
\previousdegree{B.A., Tongji University, China, 2013}
\title{\theTitle}
 \subtitle{\theSubtitle}
\author{\theAuthor}
\copyrightyear{\theCopyrightyear}
\submitdate{\theSubmitdate}
\program{\theProgramTitle}

\renewcommand\thepart         {\Roman{part}}
\renewcommand\thechapter      {\arabic{chapter}}
\renewcommand\thesection      {\thechapter.\arabic{section}}
\renewcommand\thesubsection   {\thesection.\arabic{subsection}}
\renewcommand\thesubsubsection{\thesubsection.\arabic{subsubsection}}
\renewcommand\theparagraph    {\thesubsubsection.\arabic{paragraph}}

\floatstyle{ruled}
\newfloat{Program}{htbp}{lop}[chapter]

\begin{document}

\frontmatter

\maketitle                      %
\chapter*{}
\vskip -0.5 cm
The following individuals certify that they have read, and recommend to the Faculty of Graduate and Postdoctoral Studies for acceptance, the thesis entitled: 

\textbf{Learning Dynamics of Deep Learning -- Force Analysis of Deep Neural Networks},

submitted by \textbf{\theAuthor} in partial fulfillment of the requirements for the degree of \textbf{\theDegree} in \textbf{\theProgramTitle}.

\textbf{Examining Committee:}

\committeeMember{Danica J. Sutherland}{Assistant Professor}{Department of Computer Science, UBC} {Supervisor}

\committeeMember{Mark Schmidt}{Professor}{Department of Computer Science, UBC} {Supervisory Committee Member}

\committeeMember{Trevor Campbell}{Associate Professor}{Department of Statistics, UBC} {Supervisory Committee Member}

\committeeMember{Leonid Sigal}{Professor}{Department of Computer Science, UBC} {University Examiner}

\committeeMember{Muhammad Abdul-Mageed}{Associate Professor}{School of Information and Department of Linguistics, UBC} {University Examiner}

\committeeMember{Quanquan Gu}{Associate Professor}{Department of Computer Science, University of California, Los Angeles}{External Examiner}

\textbf{Additional Supervisory Committee Members:}

\committeeMember{Jeff Clune}{Professor}{Department of Computer Science, UBC} {Supervisory Committee Member}

\committeeMember{Aaron Courville}{Professor}{Department of Computer Science and Operations Research, Université de Montréal and Mila} {Supervisory Committee Member}

\begin{abstract}                %
\danica{Max 350 words, this is a little too long}
This thesis investigates the learning dynamics of deep learning systems through a local, physics-inspired analytical lens.
Motivated by the need for fine-grained insights into model behavior, we begin with the step-wise influence that a single training example exerts on a specific observing example during learning. 
Central to our approach is the proposed $\mathcal{AKG}$ decomposition, which dissects this influence into three interpretable components: similarity ($\mathcal{K}$), normalization ($\mathcal{A}$), and prediction gap ($\mathcal{G}$). 
This decomposition enables an analogy with classical force analysis: the force originates from $\mathcal{G}$, is shaped by $\mathcal{AK}$, and is ultimately applied to the target object, e.g., to the model confidence, output, hidden representations, or parameters.

Building upon this foundation, we gradually scale the analysis from individual interactions to cumulative effects over time, akin to tracking an object’s motion under multiple forces.
We apply it to the following problems.
\begin{itemize}
    \item \textbf{Supervised classification:} We study the learning trajectories of examples with varying difficulty and reveal an interesting ``zig-zag'' pattern that emerges during optimization. Our analysis explains this behavior and inspires a novel knowledge distillation method, Filter-KD, which improves supervision signals for student models.
    \item \textbf{Large language model (LLM) finetuning:} We extend the framework to account for the autoregressive nature of LLMs and the presence of negative gradients. The unified perspective explains behaviors across finetuning methods such as SFT, DPO, and GRPO. We also highlight the critical role of negative gradients. In particular, we identify the ``squeezing effect'': a counterintuitive phenomenon caused by improperly applied gradient ascent.
    \item \textbf{Representation learning:} We explore the dynamics of hidden features, revealing how adaptation energy and directions influence the feature drift. Our analysis leads to a provable pattern of feature adaptation in a head-probing then finetuning pipeline, offering insights and inspiring several practical strategies.
    \item \textbf{Simplicity bias and compositional learning:} Revisiting foundational questions about why structured representations are learned faster, we apply our framework to a compositional learning setting. Our findings align with principles such as Occam’s Razor and the idea of ``compression for AGI,'' offering a novel dynamical explanation rooted in compression and learning speed.
\end{itemize}

\end{abstract}

\chapter{Lay Summary}
This thesis explores how deep learning models learn over time, using ideas inspired by force analysis.
Specifically, we zoom in on the model's training procedure to see how one training example affects another during learning, like analyzing how forces move objects. 
We break this influence into two parts: how similar the two examples are, and how strong the updating force is. 
This framework helps us understand a wide range of the model's behaviors in different real systems. 
For example, it explains why certain examples have non-trivial learning paths, why (and why not) some LLM finetuning methods work, and why simpler, more structured patterns tend to be learned more easily. 
We apply this approach to various learning tasks and uncover new strategies for improving model training.
While the method is still developing, it offers a new way to interpret models' behaviors systematically.    %
\chapter{Preface}
The work presented in this dissertation has been included in several publications. We list them as below:

\textit{1. ``Better Supervisory Signals by Observing Learning Paths.''}\\
\underline{Yi Ren}, Shangmin Guo, Danica J. Sutherland\\
{International Conference on Learning Representations (ICLR), 2022}\\
\citep{renbetter}

Contributions: We formulated the problem together. Danica introduced the NTK into the framework, which made the decomposition more theoretically solid and interpretable. We all contributed to the writing.

\textit{2. ``How to Prepare Your Task Head for Finetuning.''}\\
\underline{Yi Ren}, Shangmin Guo, Wonho Bae, Danica J. Sutherland\\
{International Conference on Learning Representations (ICLR), 2023}\\
\citep{ren2023howto}

Contributions: We formulated the problem together. Shangmin helped with coding. Wonho did the image segmentation part. Danica checked the theoretical derivations carefully. We all contributed to the writing.

\textit{3. ``Improving Compositional Generalization using Iterated Learning and Simplicial Embeddings.''}\\
\underline{Yi Ren}, Samuel Lavoie, Mikhail Galkin, Danica J. Sutherland, Aaron Courville\\
Conference on Neural Information Processing Systems (NeurIPS), 2023\\
\citep{ren2023improving}

Contributions: We formulated the problem together. Samuel helped with the SEM part. Mikhail helped with the graph network part. Danica and Aaron checked the theory and results. We all contributed to the writing.

\textit{4. ``lpNTK: Better Generalisation with Less Data via Sample Interaction During Learning''}\\
Shangmin Guo, \underline{Yi Ren}, Stefano V Albrecht, Kenny Smith\\
International Conference on Learning Representations (ICLR), 2024\\
\citep{guo2024lpntk}

Contributions: We formulated the problem together. Shangmin did most of the work; I helped with the theory part. We all contributed to the writing.

\textit{5. ``Bias Amplification in Language Model Evolution: An Iterated Learning Perspective.''}\\
\underline{Yi Ren}, Shangmin Guo, Linlu Qiu, Bailin Wang, Danica J. Sutherland\\
Conference on Neural Information Processing Systems (NeurIPS), 2024\\
\citep{ren:iicl}

Contributions: We formulated the problem together. Linlu helped with the game design. Danica helped with the theoretical and experimental checks. We all contributed to the writing.

\textit{6. ``Learning Dynamics of LLM Finetuning.''}\\
\underline{Yi Ren}, Danica J. Sutherland\\
International Conference on Learning Representations (ICLR), 2025\\
\citep{ren2025learning_dynamics_LLM}

Contributions: We formulated the problem together. Danica helped with the theory and reorganization of the paper. We all contributed to the writing.

\textit{7. ``On the Effect of Negative Gradient in Group Relative Deep Reinforcement Optimization.''}\\
Wenlong Deng, \underline{Yi Ren}, Muchen Li, Danica J Sutherland, Xiaoxiao Li, Christos Thrampoulidis\\
Conference on Neural Information Processing Systems (NeurIPS), 2025\\
\citep{deng2025effect}

Contributions: We formulated the problem together. Wenlong did most of the work; I helped with the theory part. We all contributed to the writing.

\textit{8. ``Understanding Simplicity Bias towards Compositional Mappings via Learning Dynamics.''}\\
\underline{Yi Ren}, Danica J. Sutherland\\
Workshop on Compositional Learning at NeurIPS, 2024\\
\citep{ren2024understanding}

Contributions: We formulated the problem together. Danica helped with the theory and reorganization of the paper. We all contributed to the writing.

\vskip 0.5 cm

\cref{sec:fundamentalLD} is largely based on papers 1 and 4; \cref{sec:case1} is largely based on papers 1; \cref{sec:case2} is largely based on papers 5, 6, and 7; \cref{sec:case3} is largely based on papers 2; \cref{sec:case4} is largely based on papers 8 and 3.

\vskip 0.5 cm

Disclosure of Generative AI Usage: A generative AI tool, specifically
\texttt{ChatGPT}, was used solely for grammar checking and rephrasing. 
No AI-generated ideas were included in the research or analysis presented in this dissertation.

\tableofcontents                %
\listoftables                   %
\listoffigures                  %
\chapter{List of Abbreviations}

\def\arraystretch{1.5}
\begin{longtable}[l]{l l}
    \textbf{AGI} &  Artificial General Intelligence\\
    \textbf{AIE} & Average Initial Energy\\
    \textbf{CNN}& Convolutional Neural Network\\
    \textbf{CE}& Cross Entropy\\
    \textbf{DL}& Deep Learning\\
    \textbf{DPO}& Direct Preference Optimization\\
    \textbf{ERM}& Empirical Risk Minimization\\
    \textbf{EMA}& Exponential Moving Average\\
    \textbf{FT}& Finetuning (Full-tuning in Chapter 6)\\
    \textbf{GAN}& Generative Adversarial Networks \\
    \textbf{(S)GD}& (Stochastic) Gradient Descent\\
    \textbf{GNN}& Graph Neural Network\\
    \textbf{GRPO}& Group Relative Policy Optimization\\
    \textbf{HP}& Head Probing\\
    \textbf{KD}& Knowledge Distillation\\
    \textbf{K-complexity}& Kolmogorov Complexity\\
    \textbf{KL-distance}& Kullback–Leibler divergence\\
    \textbf{LLM}& Large Language Models\\
    \textbf{LLDisp}& Lazy Likelihood Displacement\\
    \textbf{MSE}& Mean Square Error\\
    \textbf{MLP}& Multilayer Perceptron\\
    \textbf{NTHR}& Negative Token Hidden Reward\\
    \textbf{OOD}& Out of Distribution\\
    \textbf{OPM}& Overparameterized Model\\
    \textbf{PT}& Pretraining\\
    \textbf{PAC}& Probably Approximately Correct\\
    \textbf{PPO}& Proximal Policy Optimization\\
    \textbf{RL}& Reinforcement Learning \\
    \textbf{RLHF}& Reinforcement Learning with Human Feedback\\
    \textbf{RLVR}& Reinforcement Learning with Verifiable Rewards\\
    \textbf{SFT}& Supervised Finetuning\\
    \textbf{TRPO}& Trust Region Policy Optimization\\
    \textbf{UFM}& Unconstrained Feature Model\\
    \textbf{VC-dim}& Vapnik-Chervonenkis Dimension\\
    \textbf{ViT}& Vision Transformer\\
\end{longtable}
\def\arraystretch{1}
\chapter{List of Symbols}

\def\arraystretch{1.5}
\begin{longtable}[l]{l l}
    $\mathcal{A}$ & normalizing term in our decomposition \\
    $A_n, A_{n,l}$ & the advantage used in GRPO\\
    $\|\va\|_2$ & L2-norm of $\va$\\
    $\|\va\|_\text{op}$ & operation-norm of $\va$\\
    $\|\va\|_F$ & Frobenius Norm of $\va$\\
    $\langle\va,\vb\rangle$ & inner product or tensors multiplication\\
    $\va^\top$ & matrix or vector transpose\\
    $c * \va$  & scalar multiplication \\
    $\va\cdot \vb$ & inner product or matrix multiplication \\
    $\vb\odot\va$ & element-wise multiplication\\
    $\vb\oslash\va$ & element-wise division\\
    $\vB$ & the backbone parameters in OPM, the role of $\theta$ in $h_\theta$\\
    $\vb$ & the column-wise reparameterization of $\vB$, a long vector\\
    $c,\vc$ & usually a constant, or a constant vector\\
    $\mathcal{D}$ & dataset\\
    $\bar{d}$ & usually the estimated distance\\
    $\mathbb{E}$ & expectation\\
    $\ve_y$ & one-hot label\\
    $f_\theta:\mathcal{X}\rightarrow\mathcal{Y}$ & a function we want to learn\\
    $\mathcal{G}$ & updating term in our decomposition \\
    $\vG$ & the ground truth semantical generating factors\\    
    $\vh$ & the feature vector in UFM or OPM\\    
    $\mathcal{I}^+, \mathcal{I}^-$ & the index of positive (negative) responses in one roll-out\\
    $\mathcal{K}$ & similarity term in our decomposition, eNTK \\
    $\mathcal{K}_{uo}$ & the simplified notation for $\mathcal{K}(\vx_u,\vx_o)$ \\
    $L$ & the length of the sequence, or the number of labels\\
    $\mathcal{L}$ & loss function\\
    $\mathcal{M}:\mathcal{G}\rightarrow \mathcal{Z}$ & a mapping from the ground truth factors to representations\\
    $m, v$ & number of factors, the index of the attributes\\
    $N^+, N^- $ & the number of positive (negative) responses in one roll-out\\
    $\mathcal{N}$ & Gaussian distribution\\
    $\vO$ & the task-irrelavant factors\\
    $\vp=\sigma(\vz)$ & probability predictions\\
    $\vp_i \text{ or } p_i$ & the $i$-th element of this vector\\
    $\vq, q$ & another notation for probability distribution\\
    $\hat{\mathcal{R}}_n$ & the empirical risk\\
    $\mathbb{R}$ & real number\\   
    $\mathcal{R}$ & the expected risk\\ 
    $S_{v}$ & a symmetric group with $v$ elements\\    
    $t, T$ & usually the training time, or the stopped training time\\
    $\tau$ & sometimes the training time, sometimes the temperature\\
    $\vu$ & a uniform distribution vector\\
    $\text{Var}$ & variance\\
    $V$ & usually the vocabulary size or the number of classes\\
    $\vw$ & the readout layer in UFM or OPM, the role of $\phi$ in $g_\phi$\\
    $x,\vx\in\mathcal{X}$ & the input data or features\\
    ${\vx_o}$ & we generally use subscript ``o'' for observing elements\\    
    ${\vx_u}$ & we generally use subscript ``u'' for updating elements\\
    $\vchi=[\vx,\vy]$ & the concatenated input used in LLM part\\
    $\vchi^+_u, \vchi^-_u$ & the positive (negative) sequences, they share $\vx_u$ \\
    $y, \vy\in\mathcal{Y}$ & the supervision label or labels\\
    $\vz$ & logits of the output\\

\quad\quad \\
    $\hat{\alpha}_{n,l}^+, \hat{\alpha}_{n,l}^-$ & NTHR score (with positive or negative advantage) for GRPO\\
    $\beta$ & regularity strength in DPO, sometimes a pre-defined ratio\\
    $\gamma_{n,l}$ & the ratio of the $l$-th token in the $n$-th response in GRPO\\
    $\kappa$ & the eNTK of all data samples \\
    $\pi(\vy\mid\vx)$ & predicting probability in the LLM part, replace $\vp$\\
    $\rho^+, \rho^-$ & the ratio of positive (negative) responses in one roll-out\\   
    $\mathbf{\mu}$ & mean of the Gaussian distribution\\
    $\sigma \circ g_\phi \circ h_\theta(\cdot)$ & a general decomposition of our function $f$\\
    $\sigma(\cdot)$ & $\mathsf{Softmax}$ or $\mathsf{Sigmod}$ function\\
    $\mathcal{\sigma}$ & variance of the Gaussian distribution\\
    $\theta$ & parameters of the backbone\\
    $\phi$ & parameters of the task-head\\
    $\eta$ &  learning rate, sometimes a pre-defined ratio\\

\end{longtable}
\def\arraystretch{1}
\chapter{Acknowledgements}
This thesis is the result of years of curiosity, frustration, discovery, and reflection. 
But more than that, it is a testament to the support, guidance, and kindness I’ve received from many people along the way. 
I am deeply grateful to everyone who has walked this path with me.

First and foremost, I would like to express my deepest gratitude to my supervisor, Professor Danica J. Sutherland, for her invaluable guidance, unwavering support, and insightful feedback throughout my Ph.D. journey. 
Her constant encouragement, the cheerful and open atmosphere she fostered, and the intellectual freedom she generously offered have been essential to the development of my research and my growth as an independent researcher.

I would also like to thank the members of my committee, Prof.\ Mark Schmidt, Prof.\ Jeff Clune, Prof.\ Trevor Campbell, and Prof.\ Aaron Courville, for their constructive comments and inspiring discussions.

I am grateful to my labmates and friends in Sutherlab. 
Special thanks to Dr. Wonho Bae, the first Ph.D. graduate in our lab.
Whether warm start or cold start, you have already embarked on a new and exciting start at your new job!
Thanks to Hamed Shirzad, the ``social expert'' in our lab --- wishing you all the best in tackling important biological problems with your beloved graph models.
I also thank Dr.\ Bingshan Hu, Zheng He, Nathaniel Xu, Aaron Wei, Kuang Yao, and Mohamad Bazzi --- wishing you all enjoyable times ahead in the lab (perhaps in a newly upgraded room!).

Special thanks go to my collaborators and mentors across the world. 
To Shangmin Guo, who first introduced me to the intersection of cognitive science and deep learning. 
To Wenlong Deng, your passion for reinforcement learning and LLMs made working with you a real joy and a fruitful process. 
To Prof.\ Xiaoxiao Li and Prof.\ Christos Thrampoulidis, your instructions and discussions were truly enlightening. 
I also thank Linlu Qiu and Dr.\ Bailin Wang for our exciting Bayesian iterated learning project.
Thanks to Qi Yan, Muchen Li, and other UBC peers; discussion with you is inspiring.

I am grateful to my mentors at Borealis AI, Dr.\ Tristan Sylvain and Dr.\ David Evans --- thank you for your guidance and for welcoming me so warmly to the group. 
I am especially thankful for the opportunity to visit and work with Prof.\ Aaron Courville’s group at Mila.
Their expertise and open-minded research environment were deeply inspiring. 
Thanks to my collaborators there, Samuel Lavoie, Mikhail Galkin, and Prof.\ Courville, for your critical insights and ideas, which were key to our work’s success.

I also wish to thank Prof.\ Shay Cohen and Prof.\ Simon Kirby.
Your teachings ignited my interest in natural language processing and iterated learning --- interests that ultimately led me into machine learning.

Thanks to my dear friends and small group members --- Kevin, Oleg, and Ray.
Your support helped carry me through many difficult times during this journey.
I am also deeply grateful to the other members of our home group, Jungsun, Ricky, Mitsi, Joe, Sophia, and others, for creating such a warm and welcoming space.
Every time we gathered, you made me feel at home.

Thanks as well to pickleball and video games, my unexpected companions on this journey. 
They comforted me whenever I found ``my machine refused to learn.''
When experiments failed and proofs got stuck, at least I knew I was improving at something: either my pickleball skills or my World of Warcraft gear. :-)

To my parents and other family members, thank you for your endless love, patience, and belief in me, even in times when I doubted myself. Though we had many misunderstandings about my career path and future plans, we endured a long, harsh winter together. 
Thankfully, we made it through, and things have begun to bloom again.
Your unwavering support has been my anchor throughout this journey.

And finally, my deepest and most heartfelt thanks go to my best friend of the past several years, my closest confidante, the rock of our family, and the love of my life --- my wife, Yezi (Charlotte) Xia. 
You were there from the very beginning, when I left my first job to chase my dreams.
We've traveled the world together, seen friends come and go, but through it all, we’ve always had each other. 
You’ve stood by me in every anxious, helpless, or painful moment, and celebrated with me in every joyful one. 
You’ve witnessed every milestone of the past several years and shared every quiet detail of our daily life. 
Thank you for making me a better person, for building a shared life full of memories with me on this winding road, and for making my life exciting and beautiful. 
I look forward to a lifetime of adventures together.
\chapter{Dedication}
To my family.

\clearpage
\thispagestyle{empty}
\textcolor{white}{.}

\vskip 5.5 cm

To close this chapter, I would like to share a few poetic lines that capture what I feel about this journey --- about what the world has given me:

\vskip 1.5 cm

\begin{quote}
\itshape
            The world bestowed on me a life, and the grace to love it back; \\
            A tiny window to the mystery, and a gentle guide along the track. \\
            It gave me a name to grow into, and years that softly rise; \\
            A quiet poem in my hands, unfolding calm and wise.
\end{quote}
\vskip 0.7 cm
\hfill \textemdash{} adapted from the lyrics of \emph{“The Treasures Presented by the World”}

\mainmatter

\chapter{Introduction}
\label{sec:intro}

Deep learning (DL) has experienced rapid growth over the past decade, revolutionizing a wide range of fields including computer vision, natural language processing, robotics, and scientific discovery. 
Its success can largely be attributed to the remarkable ability of deep neural networks to automatically extract useful representations from high-dimensional data, enabling breakthroughs in tasks that were once considered intractable. 
From powering superhuman-level agents like AlphaGo \citep{silver2016mastering} to driving recent advances in generative models and large language models (LLMs), deep learning has become a central force in modern artificial intelligence.

Despite its empirical success, our theoretical understanding of DL remains limited. 
The behaviors of these models sometimes exceed what traditional learning theory would predict (in both beneficial and harmful ways), raising fundamental questions about generalization, optimization, and the nature of the learned representations. 
As deep learning systems grow increasingly complex and widely deployed, gaining deeper insights into their learning behavior becomes not only intellectually intriguing but also practically important for improving robustness, efficiency, and interpretability.

This section begins by contrasting two major paradigms of deep learning theory: the ``mathematical style,'' which emphasizes formal guarantees and proofs, and the ``physics style,'' which focuses on behaviors and average trends. 
We then use the evolution of reinforcement learning (RL) inspired LLM finetuning algorithms as a case study to illustrate why rapidly evolving systems demand physics-like theories.

A brief overview of different chapters is also provided. 
In essence, this work aims to advance our understanding of deep learning systems by studying their learning dynamics.
Through a combination of theoretical insights and empirical analyses, we investigate how learning different data samples influences the model's behavior and how these influences are accumulated to shape both the model’s generalization behavior and its internal representational structures.
We argue that the proposed \textbf{force analysis} and \textbf{learning dynamics} framework offers a valuable and flexible toolkit for interpreting and monitoring the behavior of modern deep learning systems.
It has the potential to generate novel insights and guide the development of more principled and efficient algorithms.

\section{Math-like and Physics-like DL Theory}
\label{sec:intro_motivation}

In the study of deep learning, researchers often adopt two distinct theoretical paradigms: math-like theory and physics-like theory. 
Although sharing many overlaps and similarities, these paradigms differ in their goals, methodologies, and views on inevitable assumptions.

Math-like theory aims for a rigorous and formal understanding. 
It focuses on providing provable guarantees, typically under well-defined assumptions, regarding generalization, optimization, or expressivity.
These results are often elegant and general, but may rely on idealized conditions that do not fully capture the complexity of real-world deep learning systems. 
For example, convergence proofs for gradient descent or bounds on VC dimension fall under this category \citep{blumer1989learnability}. 
The strength of this approach lies in its clarity and reliability.
However, its scope is sometimes limited by the gap between theoretical assumptions and practical realities.

In contrast, physics-like theory draws inspiration from empirical sciences. 
Rather than seeking absolute guarantees, it emphasizes building intuitive, often approximate models that help explain and predict observed behaviors of neural networks. 
This approach treats deep learning systems as complex dynamical entities, akin to physical systems, and tries to identify governing principles, such as implicit bias, learning dynamics, or phase transitions, through a combination of empirical evidence, analytical approximations, and simplified models. 
Examples include the information bottleneck \citep{tishby2000information, tishby2015deep}, which combines information theory and a general phase transition of deep models' training, the study of local elastics on neural networks \citep{elasticity},
loss landscapes, circuit hypothesis of the Transformer \citep{olsson2022context}, and so on.

The physics-like approach accepts that a perfect model may be unattainable, but a good enough approximation can still yield powerful insight and guide practice. 
While this comes at the cost of mathematical rigor, it allows theory to stay closer to the empirical frontier, particularly for large-scale models where traditional assumptions no longer apply.

In summary, math-like theory prioritizes analytical expressions and rigor,
while physics-like theory values intuition and empirical alignment.
Both are essential: one provides a solid foundation, and the other helps us navigate the messy reality of deep learning systems.
Their interplay continues to shape the evolving landscape of theoretical machine learning.
The frameworks studied in this thesis lean more towards the physics-like theory.

\section{Rapidly Evolving Algorithms Call for Physics-like Theories}
\label{sec:intro_contributions}

\begin{figure}[h]
\vspace{-0 in }
    \centering
    \includegraphics[width=0.65\linewidth]{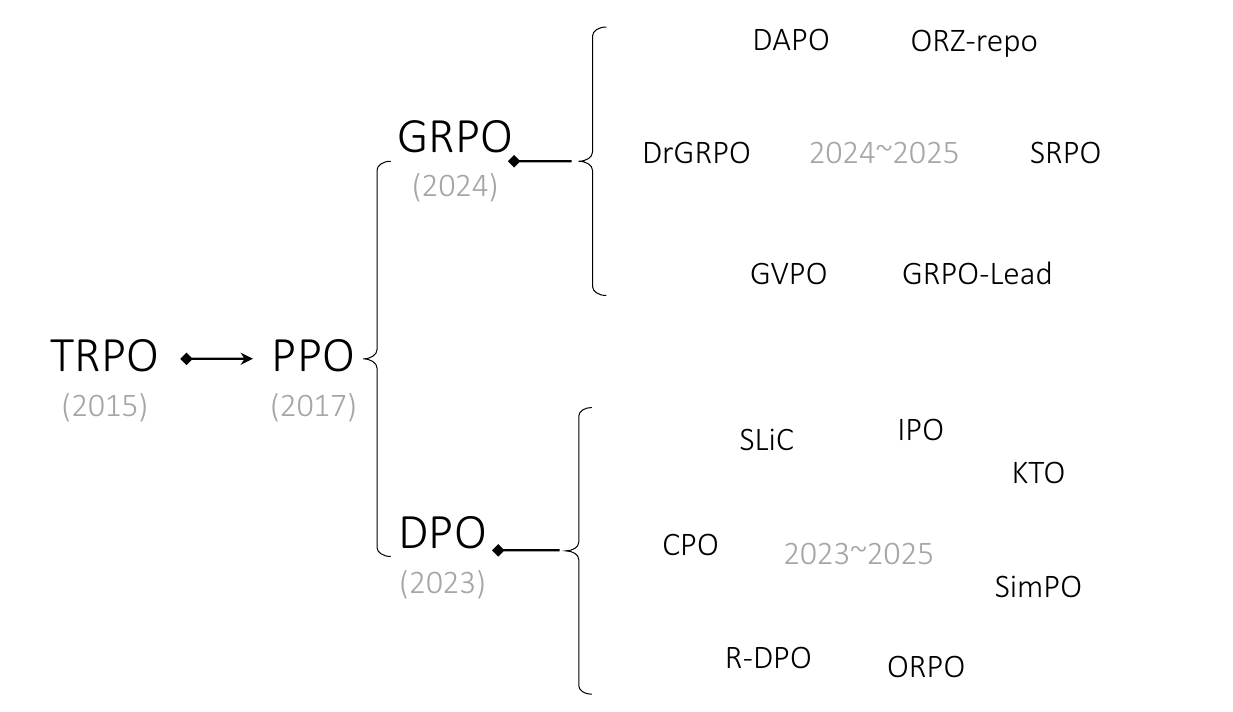}
    \caption{The evolution of xPO methods. The losses of DPO-like methods can be found in Table 3 of the SimPO paper \citep{meng2024simpo}, and the GRPO-like summary is in Table 9 of the GPG paper \citep{chu2025gpg}.}
    \label{fig:chap1_rl_llm}
\end{figure}

Many deep learning systems are highly complex in different ways, which makes math-like theories hard to track.
For instance, many such theories focus on convergence to a global optimum or assume full training to completion. 
In practice, however, models are often early-stopped based on performance on a held-out validation set, and convergence to a local optimum (or even stopping at the vicinity of a good local optimum) is usually sufficient for deployment.
Similarly, while theoretical analyses often assume datasets are drawn from well-behaved distributions, real-world models might encounter some outlier examples, which can strongly influence the model's behavior: the actual dataset might be way smaller than the required size for the theory's assumption.
Moreover, the architectures of modern systems like LLMs are often too complicated for a precise analytical description.
In short, while we aspire for theory to be rigorous and comprehensive, practical algorithms typically operate under far-from-ideal conditions.

This gap between theory and practice is further enlarged by the rapid evolution of the focusing tasks and settings of the algorithms.
For example, the Trust Region Policy Optimization (TRPO, \cite{schulman2015trust}) algorithm provides strong theoretical guarantees within the classical reinforcement learning (RL) framework.
Its successor, Proximal Policy Optimization (PPO, \cite{schulman2017proximal}), retains similar theoretical backing and significantly improves its efficiency.

However, when applying these RL algorithms to the field of LLM finetuning, maintaining strict theoretical guarantees has become increasingly difficult.
For example, while Direct Preference Optimization (DPO, \cite{rafailov2024direct}) offers valuable theoretical insights in its equivalence to PPO in the ideal setting, its extension to LLM finetuning requires more considerations, e.g., whether modeling context-token as a state-action pair is appropriate, or whether the samples are generated in an on-policy way.
Due to its simplicity and empirical effectiveness, DPO inspired numerous follow-up methods.
These variants prioritize \textit{theoretical intuition} more over \textit{theoretical guarantees}, as the scenarios they consider are more practical and subtle.

For example, SimPO points out the inaccuracy of the token-level reward and hence designs a new reward using the mean log-probability \citep{meng2024simpo}.
SLiC calibrates the likelihood of ranked sequences and applies a threshold to the margin \citep{zhao2023slic}. 
Other extensions, e.g., KTO \citep{ethayarajh2024kto}, CPO \citep{xu2024contrastive}, IPO \citep{azar2024general}, R-DPO \citep{park2024disentangling}, and ORPO \citep{hong2024orpo}, are also proposed following a similar methodology. 
Although these methods sacrifice some theoretical rigor compared to PPO and DPO, they demonstrate practical effectiveness across diverse settings: understanding why the algorithms work, although many assumptions are violated, is interesting and important.

A parallel trend is seen in Group Relative Policy Optimization (GRPO, \cite{shao2024deepseekmath}) and its derivatives. 
For instance, DAPO independently adjusts GRPO's dual thresholds to balance exploration and exploitation \citep{yu2025dapo}; DrGRPO removes the normalization term when computing advantages \citep{liu2025understanding}; and GPG omits both the clipping term and KL regularization entirely \citep{chu2025gpg}. 
Other variants, such as ORZ-Repo \citep{hu2025openreasonerzeroopensourceapproach}, GVPO \citep{zhang2025gvpo}, GRPO-Lead \citep{zhang2025grpo}, and SRPO \citep{zhang2025srpo}, further deviate from the original theoretical assumptions while still achieving strong empirical results under different settings.

Also striking is the pace of the development of the algorithms: all these variants emerged between 2024 and 2025, within a single year of DPO and GRPO's first introduction.
Such a rapid growth of different practical methods focusing on distinct scenarios is also common in other deep learning fields, e.g., computer vision, vision language models, graph neural networks, etc.
This rapid evolution inevitably restrains the exploration of those strict guarantees and simultaneously underscores the need for fast-adapting physics-like theories: for more on capturing qualitative patterns and dynamics rather than on deriving strict guarantees, and if possible, a unified view for most of them.

\section{Publications and Related Chapters}
This subsection presents the set of research papers that form the foundation of this thesis, to which I have contributed.
The papers are listed in chronological order to illustrate how the central ideas, particularly learning dynamics and force analysis in deep neural networks, have gradually developed over time.
For each paper, I include a brief summary of its main content, key takeaways, and specific contributions of each author.

As this thesis spans a broad range of topics, some conceptual connections may not be immediately clear upon first reading. 
I recommend revisiting this part after completing the thesis to gain a more comprehensive understanding of how learning dynamics serve as a unifying framework across diverse deep learning systems. 
The relationships between these works and their respective roles in different chapters are also illustrated in \cref{fig:chap0_outline}.

\vspace{0.15 in}

\textit{0. ``Compositional Languages Emerge in a Neural Iterated Learning Model.''}\\
\underline{Yi Ren}, Shangmin Guo, Matthieu Labeau, Shay B. Cohen, Simon Kirby\\
{International Conference on Learning Representations (ICLR), 2020

This paper was completed during my Master's studies, but I believe it deserves inclusion here, as it directly inspired the development of the learning dynamics framework central to this thesis. 
The work extends the iterated learning framework, originally studied in evolutionary linguistics \citep{kirby2015compression}, to the domain of deep learning. 
In essence, \citet{kirby2015compression} argue that the evolution of human language toward compositionality can be explained by the simultaneous pressures of expressivity (language should be useful and unambiguous) and compressibility (language should be as efficient as possible), which are then gradually amplified over generations.

In our study, we adapt this idea to a simplified emergent communication setting, modeled as a two-agent reinforcement learning scenario involving a sender and a receiver. 
We observe that the simplicity bias favoring compositional mappings is manifested as a learning speed advantage.

To understand the origin of this simplicity bias, we initially considered formalizing it using the Vapnik-Chervonenkis dimension (VC-dim,~\cite{blumer1989learnability}) within the Probably Approximately Correct (PAC,~\cite{valiant1984theory}) framework. 
However, we found that such classical generalization theories struggled to adequately capture the dynamics we observed. 
This limitation led us to shift our focus to a more fine-grained, micro-level perspective, i.e., examining how learning one updating example influences the model’s behavior on another observing example. 
This line of thinking laid the foundation for the force analysis framework later developed and expanded in this thesis.
This paper is covered by \cref{sec:fundamentalLD,sec:case4}.}

\vspace{0.15 in}

\textit{1. ``Better Supervisory Signals by Observing Learning Paths.''}\\
\underline{Yi Ren}, Shangmin Guo, Danica J. Sutherland\\
{International Conference on Learning Representations (ICLR), 2022}

In this work, we begin to formalize the interaction between different training examples within a fundamental supervised learning setting. 
We introduce the $\mathcal{AKG}$ decomposition, where $\mathcal{G}(\cvxu)$ quantifies the distributional gap from the model’s prediction to the supervisory distribution of the updating example $\cvxu$; and $\mathcal{K}(\cvxo, \cvxu)$ captures how the learning of $\cvxu$ influences the model’s confidence on an observing example $\cvxo$. 
Notably, we find that $\mathcal{K}$ corresponds exactly to the empirical Neural Tangent Kernel (eNTK, \cite{NTK}) of the backbone network, allowing us to integrate a broad range of existing NTK-related theoretical results into our analysis.

The most interesting contribution of this paper lies in applying the proposed force analysis framework to explain the ``zig-zag'' learning trajectories observed in certain difficult examples. 
These dynamics arise from the interplay between the similarity structure encoded in $\mathcal{K}$ and the evolving gradient magnitude governed by $\mathcal{G}$. 
This analysis offers a mechanistic understanding of how the model's predictions gradually evolve.

As a practical byproduct of our framework, we also propose ``Filter-KD,'' a variant of knowledge distillation designed to improve robustness, particularly in scenarios where labels may be noisy.
By selectively filtering examples based on their learning dynamics, this method helps stabilize training and enhance generalization.
This paper is covered by \cref{sec:fundamentalLD,sec:case1}.

\vspace{0.15 in}

\textit{2. ``How to Prepare Your Task Head for Finetuning.''}\\
\underline{Yi Ren}, Shangmin Guo, Wonho Bae, Danica J. Sutherland\\
{International Conference on Learning Representations (ICLR), 2023}

This paper extends our learning dynamics framework to a broader representation learning setting, where we model the adaptation of hidden features rather than just the model's output. 
The initial motivation stemmed from the idea of generalizing the iterated learning framework to deep networks. 
Since iterated learning typically involves a sender and a receiver, we reinterpret a standard neural network by splitting it at the hidden representations: the backbone acts as the sender, the task head as the receiver, and the hidden representation as the message passed between them.

Although this parallelism directly hints at a multi-generation training for representation learning, this paper does not pursue that direction in depth.
Instead, we uncover a particularly interesting dynamical interaction between the adaptation of features and the initialization of the task head. 
As a result, the paper shifts focus to understanding how to choose a task head that best facilitates feature adaptation. 
The analysis of feature learning dynamics thus becomes a secondary yet insightful by-product of this investigation.
We will elaborate on this work in \cref{sec:case3}.

\vspace{0.15 in}

\textit{3. ``Improving Compositional Generalization using Iterated Learning and Simplicial Embeddings.''}\\
\underline{Yi Ren}, Samuel Lavoie, Mikhail Galkin, Danica J. Sutherland, Aaron Courville\\
Conference on Neural Information Processing Systems (NeurIPS), 2023

This work successfully extends the iterated learning framework to more general representation learning problems and can be viewed as a direct follow-up to our ICLR 2020 paper. 
We generalize the framework to address the systematic generalization problem across different domains, including both visual inputs and molecular graph inputs.

In this study, force analysis and learning dynamics are employed as explanatory tools to account for the observation that mappings containing more correct structural guesses are learned more efficiently. 
Thus, the paper serves not only as a practical application of iterated learning and force-based analysis to real-world problems, but also as a bridge linking learning dynamics to another fundamental theme in deep learning: simplicity bias.
This connection is further explored and elaborated in \cref{sec:case4} of the thesis.

\vspace{0.15 in}

\textit{4. ``lpNTK: Better Generalisation with Less Data via Sample Interaction During Learning.''}\\
Shangmin Guo, \underline{Yi Ren}, Stefano V Albrecht, Kenny Smith\\
International Conference on Learning Representations (ICLR), 2024

This paper is a joint side project with Shangmin. 
Here, we focus on the $\mathcal{K}$ term in our $\mathcal{AKG}$ decomposition. 
Since $\mathcal{K}$ can be interpreted as a similarity measurement between examples, we explore its application to the core-set selection problem in supervised learning.

Our analysis reveals that the relationships between different training examples can be broadly categorized into three types: 1.) Interchangeable pairs, where the examples contain highly similar information, such that retaining only one of them is sufficient; 2.) Orthogonal pairs, where the updates from each example have minimal influence on the other; 3.) Contradictory pairs, where learning from one example negatively impacts the model’s confidence on the other, and vice versa.

These insights provide a richer understanding of the interactions among training examples, offering valuable perspectives on phenomena such as generalization, overfitting, and data redundancy. 
This work demonstrates how the $\mathcal{K}$ term can serve as a powerful tool for data-efficient learning.

\vspace{0.15 in}

\textit{5. ``Bias Amplification in Language Model Evolution: An Iterated Learning Perspective.''}\\
\underline{Yi Ren}, Shangmin Guo, Linlu Qiu, Bailin Wang, Danica J. Sutherland\\
Conference on Neural Information Processing Systems (NeurIPS), 2024

This work represents our initial attempt to extend the iterated learning framework to large language models (LLMs).
Motivated by the growing popularity of self-interaction among LLM agents, we observe that their iterative behaviors can often be well-characterized by the Bayesian iterated learning framework. 
While our learning dynamics framework does not play a central role in this study, the observed phenomenon of self-amplifying behavior in LLM self-improvement strongly motivates a deeper investigation into their evolution during finetuning.
This direction ultimately leads to our subsequent project—Learning Dynamics of LLM Finetuning, which serves as the central focus and highlight of this thesis.

\vspace{0.15 in}

\textit{6. ``Learning Dynamics of LLM Finetuning.''}\\
\underline{Yi Ren}, Danica J. Sutherland\\
International Conference on Learning Representations (ICLR), 2025

In this work, we extend our previously developed learning dynamics framework to the finetuning of LLMs. 
The $\mathcal{AKG}$ decomposition is also adapted to the next-token prediction setting, with the help of some standard assumptions commonly considered in LLM training.

By applying force analysis to the model’s training, we offer a unified mechanistic explanation for several intriguing behaviors observed across different finetuning methods.
These include the unexpected increase in confidence for less-preferred answers during Supervised Finetuning (SFT), the emergence of specific hallucinations, the exacerbated ``repeater'' phenomenon, and the simultaneous reduction in confidence for both chosen and rejected responses during Direct Preference Optimization (DPO, \cite{rafailov2024direct}).

A particularly notable and counter-intuitive phenomenon revealed by our analysis is the ``squeezing effect,'' which arises when negative gradients are inappropriately imposed in the loss function.
This effect offers insight into how confidence mass can be unintentionally redistributed during finetuning, sometimes harming the model’s alignment and reasoning capabilities.

The key results and analysis from this paper are reorganized into a more accessible structure in \cref{sec:case2} of this thesis. Readers may also find the foundational background provided in \cref{sec:fundamentalLD,sec:case1} helpful for understanding the full context of this work.
Some new findings like ``positive pressure scatters, negative pressure unites'' are also proposed in this thesis.

\vspace{0.15 in}

\textit{7. ``On the Effect of Negative Gradient in Group Relative Deep Reinforcement Optimization.''}\\
Wenlong Deng, \underline{Yi Ren}, Muchen Li, Danica J Sutherland, Xiaoxiao Li, Christos Thrampoulidis\\
Conference on Neural Information Processing Systems (NeurIPS), 2025

This is a recent project led by Wenlong, in which I participated actively. 
In this work, we extend our learning dynamics analysis to a more complex reinforcement learning-based LLM finetuning algorithm, namely Group Relative Policy Optimization (GRPO, \cite{shao2024deepseekmath}). 
Specifically, we adapt the $\mathcal{AKG}$ decomposition to a token-wise formulation, allowing for fine-grained analysis of gradient dynamics during training.

Our study reveals that certain inappropriately applied negative gradients can still degrade model performance, even in this on-policy setup. 
Through empirical analysis, we find that such harmful negative gradients are often imposed on semantically correct tokens (i.e., the correct part in an incorrect response), which can undermine the model’s intended behavior.

To address this, we introduce the concept of Negative Token Hidden Reward (NTHR), a metric designed to assess whether the negative gradient on a specific token is constructive or detrimental. 
By selectively masking out tokens with high NTHR during training, we are able to stabilize GRPO’s training process and enhance the model's reasoning ability.
This paper is also discussed in \cref{sec:case2}.

\vspace{0.15 in}

\textit{8. ``Understanding Simplicity Bias towards Compositional Mappings via Learning Dynamics.''}\\
\underline{Yi Ren}, Danica J. Sutherland\\
Workshop on Compositional Learning at NeurIPS, 2024

This is the final paper discussed in this thesis, which returns to the foundational question that originally motivated our study of learning dynamics: Can the interactions among training examples help explain the emergence of simplicity bias? 
This question lies at the heart of representation learning, and connects to broader philosophical and theoretical themes such as Occam’s Razor and the pursuit of Artificial General Intelligence (AGI).

This study represents an initial step toward understanding this problem. 
We show that representations that are better aligned with the underlying ground-truth factors tend to be learned faster, providing empirical evidence for a learning speed bias favoring simpler, more structured solutions.

These findings resonate with several contemporary discussions, including the Platonic Representation Hypothesis \citep{huh2024position} and Compression for AGI \citep{compression_for_AGI}. 
We hope that our proposed framework offers a new lens to understand this fundamental phenomenon, and that it can inspire more practical methods for improving deep learning systems by leveraging insights from learning dynamics.
We discuss this paper in \cref{sec:case4}.

\section{Thesis Outline}
\label{sec:intro_outline}

Recognizing the need for a physics-like theory to understand various deep learning systems, we adopt the following perspective: instead of focusing on global or asymptotic behavior, we emphasize a local and microscopic view of model dynamics. 
Specifically, we begin by analyzing the model’s one-step behavior after learning a single training example.
Once this foundational behavior is understood, we carefully scale up our analysis to consider the accumulated influence over multiple updates or across different stages of training. 
This progressive methodology offers a clearer picture of how complex behaviors and emergent phenomena arise in deep learning systems.

Our main analytical tool is a form of ``force analysis'' applied to the model’s output or hidden parameters, theoretically grounded in an influence function decomposition. 
The central topic studied here is the learning dynamics of the model, i.e., how its outputs and internal states evolve over time under the influence of training signals. 
This approach resonates with classical mechanics, where understanding the force acting on an object reveals its future trajectory. 
For this reason, the thesis is titled ``Learning Dynamics of Deep Learning.''
The remainder of this thesis is organized as follows:
\begin{itemize}[leftmargin=0.4cm]
    \item \cref{sec:overview} reviews the historical development and recent trends in theoretical deep learning, highlighting a renewed interest in physics-like approaches. This context motivates our core contribution: a microscopic and dynamical analytical framework centered on learning dynamics.
    
    \item \cref{sec:fundamentalLD} introduces the foundational elements of our framework.
    Starting from a basic supervised classification setting, we present key concepts such as the $\mathcal{AKG}$ decomposition, the one-step, and accumulated influences, and the dynamical evolution of the force subsequently. 
    The chapter concludes by drawing parallels between our framework and force analysis in Newtonian mechanics to help readers intuitively grasp the gist and implications of our framework.

    \item \cref{sec:case1} applies the framework to a knowledge distillation (KD) setting. 
    We show how better supervisory signals can emerge during training, and propose a practical method called Filter-KD, which is especially robust under noisy labeling conditions.

    \item \cref{sec:case2} is the highlight of this thesis. 
    Here, we extend our force analysis framework to LLM finetuning, offering novel insights into several counterintuitive behaviors observed across different settings. 
    Topics include hallucination, model degeneration, misalignment, confidence decay in preferred responses, and the underestimated influence of negative gradients. 
    We study a range of algorithms, including SFT, DPO, and GRPO, and emphasize that our framework is broadly applicable to other methods that rely on gradients of model log probabilities. 
    We hope this framework offers clearer interpretations of the pros and cons of different algorithms, and hence inspires more effective future algorithm designs.

    \item \cref{sec:case3} expands the framework by changing the observed targets from outputs to hidden representations. 
    In the context of feature adaptation, we theoretically predict how internal features evolve under different initialization conditions. 
    We also highlight core physical concepts such as force, direction, and adaptation energy.
    Inspired by the analysis, we propose strategies to guide adaptation in a controlled and efficient manner.

    \item \cref{sec:case4} switches to a more fundamental fact in machine learning: understanding the simplicity bias in deep learning, i.e., answering why models tend to favor simpler mappings in accordance with Occam’s Razor. 
    Using force analysis, we connect this bias to mutual influence between training examples across different mappings, offering a perspective aligned with several prominent theories in the machine learning literature.

    \item \cref{sec:conclusion} summarizes the contributions of the thesis. The limitations of this framework and potential future directions are also discussed.
    The associated publications for each chapter are listed in \cref{fig:chap0_outline}.
\end{itemize}

\begin{figure}[h]
    \vspace {-0.1 in}
    \centering
    \includegraphics[width=1\linewidth]{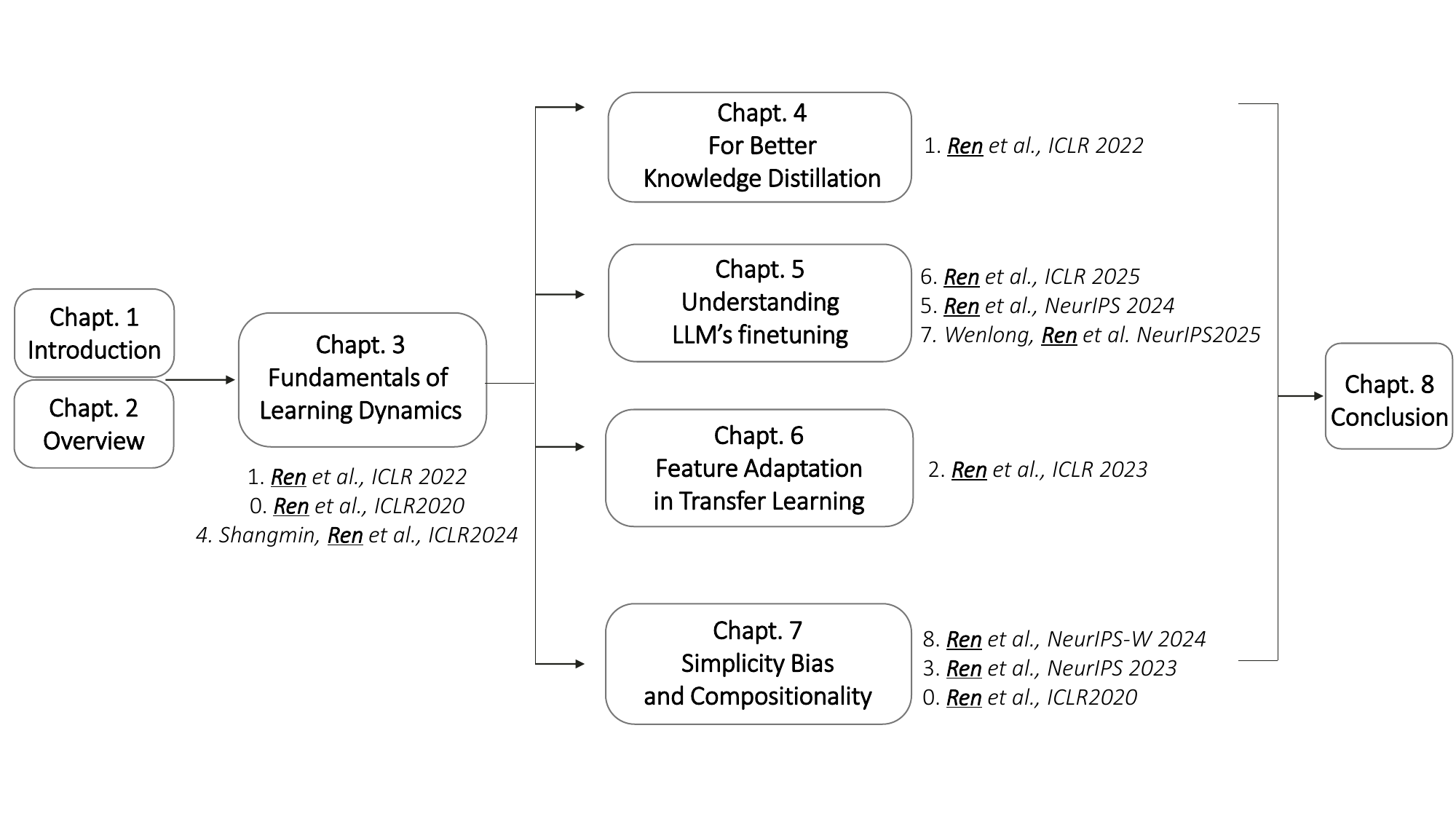}
    \caption{Outline of the thesis and papers covered in each part. Chapters 1 to 3 provide the necessary fundamentals for all the following four applications. Hence, it is fine to focus only on Chapters 3 and 5 if the reader only feels interested in LLM finetuning.}
    \label{fig:chap0_outline}
\end{figure}

\chapter{An Overview of Progress in Deep Learning}
\label{sec:overview}

\section{History of Deep Learning}
\label{sec:overview_01}

The history of deep learning traces back to the early development of artificial neural networks in the 1940s, beginning with the McCulloch-Pitts neuron model, which laid the foundation for the usage of simple binary threshold units \citep{mcculloch1943logical}. 
This was followed by the perceptron algorithm introduced by \citet{perceptron}, which demonstrated the feasibility of learning from labeled examples, but was soon critiqued for its inability to solve non-linearly separable tasks such as \texttt{XOR} \citep{minsky2017perceptrons}. 
The 1980s witnessed a revival of interest with the development of backpropagation, a gradient-based optimization algorithm enabling the training of multi-layer networks (MLP, \cite{rumelhart1986learning}). 
This period also saw the introduction of convolutional neural networks (CNNs) for image recognition tasks, notably the Neocognitron \citep{fukushima1980neocognitron} and later LeNet \citep{lecun1989handwritten}.

Despite these advances, neural networks remained relatively niche for decades, in part due to computational limitations and the dominance of alternative machine learning models such as support vector machines (SVMs) and decision trees. 
The resurgence of deep learning in the 2000s and 2010s was enabled by three key factors: access to large labeled datasets, advances in computing hardware (especially GPUs), and algorithmic innovations. 
A major turning point came in 2012 with the success of AlexNet \citep{krizhevsky2012imagenet} in the ImageNet Large Scale Visual Recognition Challenge (ILSVRC, \cite{deng2009imagenet}), which dramatically reduced the error rate and demonstrated the power of scale.

Since then, deep learning has rapidly expanded across domains, including speech recognition \citep{hinton2012deep}, natural language processing (NLP, \cite{sutskever2014sequence}), and reinforcement learning (RL, \cite{mnih2015human}). 
Transformer-based architectures \citep{vaswani2017attention} in particular have become foundational to modern NLP systems, culminating in large-scale models like BERT \citep{devlin2019bert} and GPT \citep{radford2018improving}.

Throughout the history of deep learning, both the models and the problems they address have grown increasingly complex along multiple dimensions.
One notable example is model size, typically measured by the number of parameters. 
For instance, the state-of-the-art sequence-to-sequence (seq2seq) model for natural language processing in 2014, based on LSTMs \citep{hochreiter1997long} and introduced by \citet{sutskever2014sequence}, contained approximately 100 million parameters. 
This number rose to 340 million in 2018 with the introduction of BERT-large.
Today, with the rise of \texttt{ChatGPT} and the validation of scaling laws for large language models (LLMs), the scale of modern deep learning systems has expanded dramatically, reaching hundreds of billions of parameters. 
The estimated parameter count for \texttt{GPT-4}, for example, is around 1 trillion (1000B), which is 10,000 times larger than the early seq2seq model.
This rapid increase in model size has not only significantly enhanced the capabilities of LLMs but also made the analysis of their behavior vastly more challenging.

Another dimension of increasing complexity lies in the data and tasks modern models are expected to handle.
As the capabilities of language models have advanced significantly, the focus of natural language processing has shifted: from relatively simple tasks like grammatical parsing to more challenging ones such as reading comprehension and mathematical reasoning.

Moreover, the widespread adoption of Transformer architectures across different modalities has made multimodal tasks not only feasible but increasingly common. 
In these settings, models must integrate and reason over heterogeneous inputs, such as text, images, and audio, posing a much greater challenge than single-modality tasks.

These developments have not only redefined state-of-the-art performance across tasks but also prompted a renewed focus on understanding the learning behavior and generalization properties of deep models. 
As deep learning systems continue to grow in scale and application, there is an increasing need for theoretical frameworks that can capture their dynamics, interpretability, and robustness, motivating the exploration of a unified and effective explanatory framework pursued in this thesis.

\section{More Local and More Microscopic}
\label{sec:overview_02}
As models continue to grow rapidly in size and capability, we are witnessing a subtle but significant shift in the direction of deep learning theory to keep pace with this remarkable progress.
In particular, there is a growing interest in behavioral analysis and physics-inspired theoretical frameworks.
The focus of such theories is also shifted from a global and macroscopic perspective to a more \textbf{local and microscopic} viewpoint.
Below, we present several illustrative examples.

For example, classical generalization theory has traditionally relied on global capacity measures, including VC-dimension, Rademacher complexity, and uniform convergence bounds across the entire hypothesis space. 
However, recent research emphasizes the local geometry of decision boundaries around data points. 
The norm-based and margin-based generalization bounds for deep networks introduced by \citet{neyshabur2017exploring} and \citet{bartlett2017spectrally}, for instance, analyze local quantities such as the Lipschitz constant near training points. These localized analyses help explain why overparameterized networks can still generalize well, despite their large hypothesis classes.

A similar shift is observed in the study of the loss of landscapes. 
Earlier work focused on understanding the global non-convex optimization landscape, such as identifying the number and nature of critical points or bad local minima. 
In contrast, more recent analyses emphasize local curvature properties, such as the Hessian near the learned solution or throughout training. 
A notable example is the flat minima hypothesis proposed by \citet{keskar2016large}, which suggests that flatter local regions in the loss landscape are associated with better generalization.

The neural tangent kernel (NTK, \cite{NTK}) hypothesis, a central theoretical tool in this thesis, can also be interpreted as an instance of local and microscopic analysis.
Its derivation relies heavily on repeated applications of the first-order Taylor expansion, which not only provides an effective linear approximation of the model's behavior but also shifts the analytical focus to higher-level phenomena, such as model architecture, interactions among different components of the model, or phase transitions during training.

This theory has also been applied to analyze generalization \citep{arora2019fine} and convergence rates \citep{zou2020gradient}, some typical questions in classical learning theory.
Furthermore, NTK-based analysis has been extended to different network architectures, e.g.,  convolutional networks \citep{arora2019exact}, residual networks \citep{lee2019wide}, and even graph neural networks \citep{du2019graph}. 
Despite its limitations in capturing feature learning in finite-width settings, NTK remains a powerful tool for probing the local geometry of the learning process and motivating further theoretical developments.

This local perspective also extends to the temporal dimension of training.
Different from the classical theoretical analysis, which focuses more on the asymptotic convergence behavior of models, more work about understanding the training dynamics over time, particularly phenomena such as phase transitions and emergent behaviors \citep{wei2022emergent}, has emerged.

Collectively, these developments reflect a growing recognition that to truly understand how deep models learn, adapt, and generalize, it is crucial to study the dynamics of local updates, gradient flows, and pointwise behaviors, because the global, static analyses often miss these nuances.

In line with this paradigm shift, the learning dynamics framework proposed in this thesis can be viewed as a local extension of the influence function framework introduced by \citet{koh2017understanding}.
While that work traces a model’s test-time predictions back to influential training examples over the entire training trajectory, our framework zooms in further: we begin by analyzing the one-step influence between pairs of examples, and then gradually scale up to accumulated influence over time.
Instead of using the gradient flow, which assumes the training time $t$ to be infinitely small and focuses on the change of the observing example, we use the one-step gradient update, which focuses more on the mutual influence between the updating and observing examples.
However, we believe that these two framing frameworks are inherently identical under mild assumptions.
\chapter{Fundamentals of Learning Dynamics}
\label{sec:fundamentalLD}

As discussed in earlier chapters, the growing complexity of models and tasks in modern deep learning has made it increasingly difficult for traditional analytical approaches, which usually study the system as a whole, to provide accurate global behavioral guarantees. 
In response, this thesis adopts a bottom-to-up methodology, i.e., the one that emphasizes localized analysis and aligns more closely with a physics-inspired theoretical style.
Specifically, we narrow down our scope in the following two aspects:
\begin{itemize}
    \item Time-level: rather than the global convergence analysis, which considers training a randomly initialized model to its convergence, we start from the model's behavioral change after \textit{one (or few) update(s)};
    \item Sample-level: rather than analyzing the model's performance on a full dataset, we start from the model's behavioral change on \textit{one example} after learning \textit{one training sample}.
\end{itemize}

Combining the two principles, we define the learning dynamics here as
\begin{center}
    \textit{Measuring how the model's confidence on an \orange{observing example}\\ changes after learning an \cyan{updating example}.}
\end{center}

In the remainder of this thesis, we will define the problem more formally under corresponding settings. 
For now, we only need to know that the model parameters are updated using the gradient from an ``updating example,'' and we track the resulting confidence change in one or more ``observing examples.''
Our analysis centers on the proposed $\mathcal{AKG}$ decomposition, which breaks down this one-step change into three interpretable terms. 
We also extend the analysis to accumulated influence over time, illustrated by a typical ``zig-zag'' learning trajectory. 
The chapter concludes with a summary and discussion of limitations.

\section{Warm-up with Supervised Learning}
\label{sec:fundamentalLD_01}

\subsection{Problem: Start from Supervised Classification}
Supervised learning is one of the most established and widely used paradigms in machine learning. 
It refers to the task of learning a function $f: \mathcal{X} \rightarrow \mathcal{Y}$ from a labeled dataset of input-output pairs $\{(x_i, y_i)\}_{i=1}^n$, where $x_i \in \mathcal{X}$ denotes the input (e.g., a feature vector) and $y_i \in \mathcal{Y}$ denotes the corresponding label or target for the model to learn.

The goal is to find a hypothesis $f \in \mathcal{F}$, from a chosen function class $\mathcal{F}$ (e.g., neural networks, decision trees, etc.), that generalizes well to new data drawn from the same (but unknown) data-generating distribution $\mathcal{D}(x, y)$.

The learning process aims to minimize the expected risk:
\[
\mathcal{R}(f) = \mathbb{E}_{(x, y) \sim \mathcal{D}}[\mathcal{L}(f(x), y)]
\]
where $\mathcal{L}: \mathcal{Y} \times \mathcal{Y} \rightarrow \mathbb{R}$ is a loss function that quantifies the error between prediction and ground truth (e.g., 0-1 loss, mean squared error, cross-entropy). 
Since $\mathcal{D}$ is usually unknown, we minimize the empirical risk instead:
\[
\hat{\mathcal{R}}_n(f) = \frac{1}{n} \sum_{i=1}^n \mathcal{L}(f(x_i), y_i).
\]
This leads to the Empirical Risk Minimization (ERM) principle, a central concept in statistical learning theory.

For most of the deep learning systems considered in this thesis, the output would be a probability distribution over all possible classes.
We hence decompose the function $f$ into several parts and parameterize them separately.
In general, we have $f(x) = \sigma  \circ g_\phi\circ h_\theta  (x)$,
where $h_\theta(\cdot)$ is usually the backbone of the network parameterized by $\theta$, and $g_\phi(\cdot)$ is the task head parameterized by $\phi$.
We usually call the output of the backbone part $\vh=h_\theta(x)$ ``hidden embeddings'' or ``hidden features.''

Then, we define the logits $\vz\in\mathbb{R}^V$ as $\vz=g_\phi \circ h_\theta (x)$,
which is a length-$V$ vector, and $V$ is the number of all possible classes.
Since we usually want our model to give a probability distribution prediction over all possible $V$,
a softmax layer $\sigma(\cdot)$ is usually attached to the logits $\vz$, i.e.,
\[ 
    \sigma_i(\vz) \triangleq \mathsf{Softmax}_i(\vz)=\frac{e^{z_i}}{\sum_{j=1}^V e^{z_j}}.
\]

Note that the separation of $h_\theta$ and $g_\phi$ can be arbitrary, depending on what hidden features we want to observe (as will be discussed in detail in \cref{sec:case3}).
As a warm-up example, we simplify the model by considering $g$ as an identity function.
In other words, we have
\[
    \vp(x) =f(x) = \sigma(h_\theta(x)) = \sigma(\vz) ,
\]
where $\vp(x)=[p(y=1\mid x), \dots, p(y=V\mid x)]$ is the vector form of the predicted distribution.
Please refer to \cref{fig:chap3_z_hg_x} for an example of the notations and their roles in a deep learning system.
We will largely stick to this definition throughout this thesis.

\begin{figure}[t]
\vskip -0in
    \begin{center}
    \centerline{\includegraphics[width=0.9\columnwidth,trim=0 0 0 0, clip]{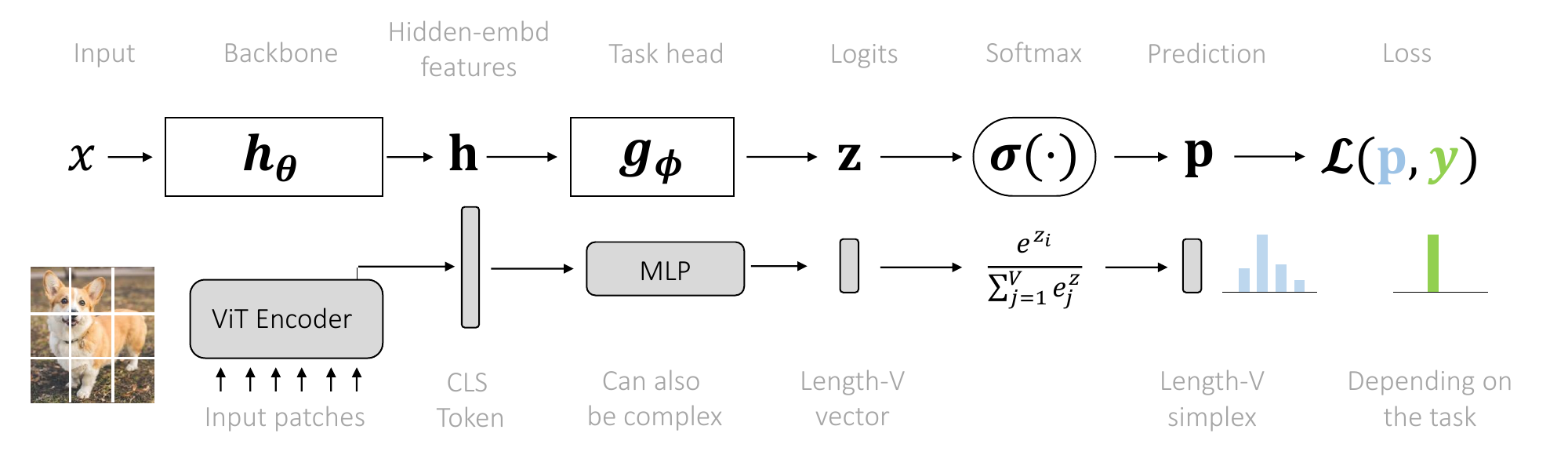}}
    \caption{A general pipeline and notation for a typical supervised classification task using Vision Transformer (ViT, \citealp{vit}).}
    \label{fig:chap3_z_hg_x}
    \end{center}
\vskip -0.2 in
\end{figure}

We then discuss the target we want to optimize.
Depending on the problem we are facing, the loss functions can be diverse.
Throughout this thesis, we mainly focus on the following two types of losses:

\begin{itemize}
    \item \textbf{Cross-Entropy Loss (CE, for classification with probabilities)}:
  \begin{equation}
      \mathcal{L}_{ce} \left(f(x), y\right) = - \sum_{v=1}^{V} y_v \log(p(y=v\mid x)) = -\ve_y^\top  \log \vp(x),
      \label{eq:sec3:ce_loss}
  \end{equation}
  
  where $y_v$ is the indicator of ground-truth label and $\ve_y$ is its vector form (i.e., $y_v$ is one for the correct label $v$ and is zero otherwise).
  $\vp(x) \triangleq \left[p(y=1\mid x), \dots, p(y=V\mid x)\right]$ is the vector form of the model's prediction.
  If the supervisory signal is not one-hot but a distribution,
  e.g., in the knowledge distillation setting we will discuss in \cref{sec:case1},
  we can have a cross-entropy loss like $\mathcal{L}_{ce} = -\vq^\top\log\vp$,
  where $\vq$ is another probability vector with the same size as $\vp$.

  \item \textbf{Mean Squared Error (MSE, usually for regression)}:
  \begin{equation}
      \mathcal{L}_{mse} \left(f(x), y\right) = \left(f(x) - y\right)^2,
      \label{eq:sec3:mse_loss}
  \end{equation}
  where both $f(x)$ and $\mathcal{Y}\in\mathbb{R}$ are scalars.
  Note that it is also possible for us to use MSE loss for classification problems.
  For example, if our model generates probabilistic predictions $\vp(x)$ and we still want to use MSE loss, we can use $\mathcal{L}_{mse} = \|\vp - \ve_y\|_2^2$,
  where $\|\cdot\|_2$ is the L2-norm of a vector.
  \citet{ce_vs_mse} provides more discussions about the pros and cons between MSE and CE when applied to a classification task.
\end{itemize}

\subsection{Learning: Gradient Descent}
Recall our goal is to find a good function $f$ that performs well on the data distribution $\mathcal{D}$.
Since the ground truth $\mathcal{D}$ is usually intractable,
we instead minimize the empirical risk $\hat{\mathcal{R}}_n$ using a training dataset sampled from $\mathcal{D}$.
Depending on the task we are facing,
an appropriate loss function is chosen to measure how well our parameterized model $f_\theta$ performs on this dataset.
Then, the next question is: how could we find the best $f_{\theta^*}$ with the help of a training set $\mathcal{D}_\text{train}$?

Although numerous closed-form solutions for $f_{\theta^*}$ exist under various assumptions, optimization-based methods have become the primary workhorse in the deep learning era, owing to their simplicity and scalability.
Among these, gradient descent (GD, see e.g.\ \citealp{gd_sgd}) is particularly popular, as it only requires the existence of first-order gradients and is applicable across deep learning systems of varying scales.

Specifically, we want to solve the following optimization problem:
\[
    \theta^* = \mathop{\arg\min}\limits_{\theta\in\Theta} \ \hat{\mathcal{R}}_n(f).
\]
The idea of GD is to iteratively update the parameters $\theta$ in the direction of the gradient of the function, which points toward the direction that provides the steepest loss descent:
\[
    \theta_{t+1} = \theta_t - \eta\nabla_{\theta} \hat{\mathcal{R}}_n(f),
\]
where $\eta>0$ is the learning rate, which is usually very small, and $t$ denotes the time (or step) during training.
$\nabla_{\theta} \hat{\mathcal{R}}_n(f)$ is the gradient of the estimated risk on all training data,
which can also be written as $\frac{1}{n}\sum_{i=1}^n\nabla_{\theta} \mathcal{L}(f(x_i),y_i)$.

However, the number of training examples $n$ is usually huge, so a variant called stochastic gradient descent (SGD) is mostly applied in practice.
Specifically, in SGD, the model updates its parameters using only a small subset (or even just one example) of the training data at each step, such as
\begin{equation}
    \theta_{t+1} = \theta_t - \eta\nabla_{\theta} \mathcal{L}(f(x_i),y_i),\quad (x_i,y_i)\sim\mathcal{D}_\text{train}.
\end{equation}

Although each update in SGD is noisy and does not guarantee a decrease in the total loss, each step is far more computationally efficient and the noise can even help escape shallow local minima and plateaus in non-convex optimization landscapes. 
Over time, averaging over many such noisy updates typically leads to convergence to a good solution, especially when combined with techniques like:
\begin{itemize}
    \item Learning rate schedules (e.g., warm-up or cosine scheduling in \cite{loshchilov2016sgdr}),
    \item Momentum of historical updates (e.g., Adam in \cite{kingma2014adam}),
    \item Regularization (e.g., weight decay in \cite{loshchilov2017decoupled}).
\end{itemize}
SGD is not just a practical algorithm; it also plays a central role in understanding the learning dynamics of deep model training.
Its behavior has been linked to generalization performance (including in \citealp{renbetter} and also in \cref{sec:case1}), implicit regularization \citep{barrett2020implicit}, and even simplicity bias (in \citealp{ren2024understanding} and also in \cref{sec:case4}).

\subsection{Evaluation: Loss, Accuracy, and More}
Now, we can build our model and train it using SGD.
At the same time, we want to check whether our model performs well or not, which needs evaluation from different perspectives.
Evaluating deep models goes beyond just checking their accuracy or training loss.
Depending on our ultimate aims, i.e., whether we aim for performance, interpretability, robustness, or understanding learning dynamics, different evaluation techniques are used.

In general, we wish to cover log-probability of the model's prediction, expected calibration error, norm or other quantities of hidden embeddings, out-of-distribution probing accuracies, noise sensitivity, 2D-projection of features, representation similarity, etc.
The theme is to provide a more thorough understanding of how the model's different parts gradually evolve by tracking important quantities during training.
Compared with traditional learning theory such as the PAC learning framework \citep{valiant1984theory},
which usually evaluates the model's behavior when fully trained, we focus more on the dynamical process \textit{during training} using force analysis.
The results from some classical training dynamics literature, e.g., those on SGD \citep{mandt2017stochastic, arora2018optimization, xie2020diffusion}, those on infinite-width approximation like Neural Tangent Kernel \citep{NTK}, and those on loss landscapes \citep{foret2020sharpness}, are usually considered as the practical assumptions or lemmas during our force analysis.

\danica{Should mention here and elsewhere that there are more ``mathematical-style'' analyses of training dynamics, and cite a few}

\section{One-Step Influence}
\label{sec:fundamentalLD_02}

\subsection{The $\mathcal{AKG}$ Decomposition}
Recall that we want to track how the model's confidence in an observing example \orange{$x_o$} changes after learning an updating example \cyan{$x_u$}.
To simplify the analysis in the CE loss setting, we track the logarithmic probability change:
\[
    \Delta^t(\cxo) \triangleq \log p_{\theta^{t+1}}(y\mid \cxo) - \log p_{\theta^{t}}(y\mid \cxo).
\]
To calculate this, we first approximate $\log p_{\theta^{t+1}}(y\mid \cxo)$ using a first-order Taylor expansion. (We use $p^t$ to represent $p_{\theta^t}$ and use $p(x)$ or $\vp(x)$ interchangeably to represent $p(y\mid x)$, for notational conciseness.) This gives
\[
    \log \vp^{t+1}(\cxo) = \log \vp^{t}(\cxo) + \
    \langle \nabla \log \vp^{t}(\cxo),\ \theta^{t+1}-\theta^t \rangle + O(\| \theta^{t+1}-\theta^t \|^2)
.\]
Rearranging the terms and rewriting the inner product leads to
\[
    \Delta^t(\cxo) = \underbrace{
        \nabla_{\theta}\log \vp^t(\cxo)|_{\theta^t}
    }_{V \times d}
    \underbrace{ \big( \theta^{t+1}-\theta^t \big) }_{d \times 1}
    {} + O\big(\lVert \theta^{t+1}-\theta^t \rVert^2\big),
\]
where $d$ is the number of parameters of the model.
If the model updates its parameters using SGD with a single ``updating example,''
we have
\[
    \theta^{t+1}-\theta^t = -\eta\nabla_{\theta}\mathcal{L}(\vp(\cxu),\cyu).
\]

To evaluate the leading term, we plug in the definition of SGD
and repeatedly use the chain rule, recalling that $\vp=\sigma(\vz)$ and $\vz=h_\theta(x)$:
\begin{align}\nonumber
    \Delta^t(\cxo) &= \underbrace{
        \nabla_{\theta} \log \vp^t(\cxo)|_{\theta^t}
    }_{V \times d}
    \underbrace{ \big( \theta^{t+1}-\theta^t \big) }_{d \times 1} &\\\nonumber
    &= \underbrace{
        \nabla_{\theta} \log \vp^t(\cxo)|_{\theta^t}
    }_{V \times d} \big(-\eta
        \underbrace{
        \nabla_{\theta} \mathcal{L} (\cxu,\cyu)|_{\theta^t}
        }_{1 \times d}
        \big)\tp \\\nonumber
    &=  \big(
        \underbrace{
            \nabla_{\vz} \log \vp^t(\cxo)|_{\vz^t}
        }_{V \times V}
        \underbrace{
            \nabla_{\theta}\vz^t(\cxo)|_{\theta^t}
        }_{V \times d}
        \big)
        \big(-\eta
        \underbrace{
        \nabla_{\theta} \mathcal{L} (\cxu,\cyu)|_{\theta^t}
        }_{1 \times d}
        \big)\tp
\\\nonumber
    &=
    \underbrace{
        \nabla_{\vz} \log \vp^t(\cxo)|_{\vz^t}
    }_{V \times V}
    \underbrace{
        \nabla_{\theta}\vz^t(\cxo)|_{\theta^t}
    }_{V \times d}
    \big(
    \underbrace{
        -\eta\nabla_{\vz} \mathcal{L} (\cxu,\cyu)|_{\vz^t}
    }_{1 \times V}
    \underbrace{
        \nabla_{\theta}\vz^t(\cxu)|_{\theta^t}
    }_{V \times d}
    \big)\tp
\\\nonumber
    &= -\eta
      \underbrace{\nabla_{\vz} \log \vp^t(\cxo)|_{\vz^t}}_{V \times V}
      \big[
        \underbrace{\nabla_{\theta}\vz^t(\cxo)|_{\theta^t}}_{V \times d}
        \underbrace{\left(\nabla_{\theta}\vz^t(\cxu)|_{\theta^t}\right)\tp}_{d \times V}
      \big]
      \underbrace{
      \big(
        \nabla_{\vz} \mathcal{L} (\cxu,\cyu)|_{\vz^t}
      \big)\tp
      }_{V \times 1}
    \\
    &= -\eta \, \mathcal{A}^t(\cxo) \, 
\mathcal{K}^t(\cxo,\cxu) \, \mathcal{G}^t(\cxu,\cyu),\label{eq:sec3:akg}
\end{align}
which is the core $\mathcal{AKG}$ decomposition of our one-step influence function.

For the higher-order term, $O\big(\lVert \theta^{t+1}-\theta^t \rVert^2\big)$, we know that
\[
    \theta^{t+1} - \theta^t
    = 
    -\eta \left(
    \nabla_{\theta}\vz^t(\cxu)|_{\theta^t}
    \right)^\top
    \mathcal{G}^t(\cxu,\cyu).
\]
The residual term $\mathcal{G}^t$, which is the norm of the gap between two probabilities, is obviously bounded.
In practice, the gradient of $\vz$ to $\theta$ is also usually bounded, both because the training process does not diverge, and particularly when considering the widely applied practice of gradient clipping.
We hence have
\[
    O\left( \| \theta^{t+1} - \theta^{t} \|^2\right)
    = O\left( \eta^2
    \| \nabla_{\theta} \vz^t(\cxu) \|_{\text{op}}^2 \,
    \| \mathcal{G}^t(\cxu,\cyu) \|^2
    \right),
\]
which is typically $O(\eta^2)$ and can be neglected when $\eta$ is small.

The decomposition above holds under mild assumptions.
For instance, the function $f$ must be differentiable around $\theta$ so that the first-order Taylor expansion exists, and the network’s gradient should be bounded to ensure that higher-order terms remain controlled. These are very common assumptions applied in many existing works, and practitioners also design many mechanisms to stabilize the training, e.g., clipping gradients, smoothing the activation function (using smoothed variants of ReLU), etc.

Another concern is the role played by those higher-order terms in the Taylor series. While higher-order terms can play an important role over the long run, as discussed by \citet{damian2022self}, we argue that in one-step or short-horizon analyses, and assuming the learning rate is relatively small, the first-order component dominates the model’s behavior. A more detailed exploration of higher-order effects is left to future work.

\subsection{Role of $\mathcal{A}$}
Recalling \cref{eq:sec3:akg}, we have $\mathcal{A}^t(\cxo) = \nabla_{\vz}\log \vp^t(\cxo)$, which only depends on the observing example $\cxo$.
If we use $\{p_1, \dots,p_V\}$ to represent the model's prediction on different dimensions, we can write $\mathcal{A}^t$ as:
\[
    \mathcal{A}^t(\cxo)
    = I-\bmf{1} \vp_{\theta^t}^\top
    = \left[
    \begin{matrix}
        1- p_1  &- p_2    &\cdots     &- p_V \\
        - p_1   &1- p_2   &\cdots     &- p_V \\
        \dots   & \dots   &\ddots     &\dots  \\
        - p_1   &- p_2    &\cdots     &1- p_V
    \end{matrix}
    \right].
\]
This matrix is a \textit{centering} operator based on $\vp_{\theta^t}$.
For any length-$V$ vector $\vb$, we have $\mathcal{A}\vb = (I-\bmf{1} \vp_{\theta^t}^\top) \vb=\vb-\bar{b}\bmf{1}$, where $\bar{b}=\vp_{\theta^t}^\top \vb$ is the weighted sum of all elements in $\vb$.

Imagine, in one time step, that the influence imposed by the term $\mathcal{KG}$ is this vector $\vb$.
Such an influence will then be normalized by $\mathcal{A}$ term depending on how well this $\vb$ aligns with current $\vp^t(\cxo)$: better alignment leads to a larger $\bar{b}$, which means the knowledge provided by $\vb$ is not surprising.
Hence, the resulting influence is penalized by $\bar{b}$ for every dimension and vice versa.
In other words, $\mathcal{A}$ is self-stabilizing the model's confidence change based on its current prediction, and vice versa.

\subsection{Role of $\mathcal{K}$}
We then check $\mathcal{K}^t(\cxo, \cxu) = (\nabla_{\theta}\vz(\cxo)|_{\theta^t})(\nabla_{\theta}\vz(\cxu)|_{\theta^t})^\top$,
which is the product of the model's gradients with respect to $\cxo$ and $\cxu$.
It is also the empirical neural tangent kernel (eNTK, \cite{NTK}) of the backbone network $h_\theta$ (recall \cref{fig:chap3_z_hg_x}, and we assume $g_\phi$ to be an identity function here).
Most of the consistent trends observed in this thesis implicitly rely on the following assumption (although the equations above always hold even without it):
\begin{center}
    \textit{During training, the relative influence of learning $\cxu$ \\ on other different $\cxo$ is relatively stable (e.g., w.r.t. ranking).}
\end{center}
This is much less strict than the common ``lazy eNTK'' assumption discussed by \citet{arora2019exact}; that is a sufficient but not necessary condition for this thesis.
We will elaborate on this assumption later under different problems.

Intuitively, we can also interpret $\mathcal{K}^t$ as a model-specific similarity measurement between different input samples.
The relative strength can be approximated by the Frobenius norm\footnote{$\|A\|_F=\sqrt{\sum^m_{i=1}\sum_{j=1}^n|a_{ij}|^2}=\sqrt{\mathsf{Tr}(AA^\top)}$ is similar to the 2-norm of a vector.} of this matrix: larger $\|\mathcal{K}^t\|_F$ means that the update of $\cxu$ likely influences the model's prediction on $\cxo$ more.

\subsection{Role of $\mathcal{G}$}
Finally, we check the role of $\mathcal{G}^t(\cxu,\cyu)=\nabla_{\vz}\mathcal{L}(\cxu,\cyu)|_{\vz^t}$, which could be considered as the ``gap'' term between the model's current prediction and the supervisory signal on $(\cxu,\cyu)$.
Considering a popular cross-entropy loss in \cref{eq:sec3:ce_loss}, we have
\[
    \nabla_{\vz}\mathcal{L}(\cxu,\cyu)=
    -\nabla_{\vz}\ve_{\cyu}^\top \log \sigma(\vz)=
    -\nabla_{\vz}\log \sigma_i(\vz)=
    -\frac{1}{\sigma_i(\vz)}\nabla_{\vz} \sigma_i(\vz).
\]
For each dimension of this gradient, we have the standard result that
\[
    \frac{\partial\sigma_i}{\partial z_k} = \begin{cases}
        \sigma_i(1-\sigma_k)\quad \text{if } i=k \\
        \sigma_i(0-\sigma_k)\quad \text{if } i\neq k.
    \end{cases}
\]
Substituting this back leads to the vector form of $\mathcal{G}^t$ under CE-loss:
\begin{equation}
    \mathcal{G}^t_{ce}(\cxu,\cyu)=\vp^t(\cxu)-\ve_{\cyu},
\end{equation}
which is a length-$V$ vector.
Using a similar method, we can calculate the gap term for the MSE-loss ($\mathcal{L}_{mse}=\|\vp(\cxu)-\ve_{\cyu}\|^2_2$) as:
\begin{equation}
    \mathcal{G}^t_{mse}(\cxu,\cyu)=2\vc\odot (\vp^t(\cxu)-\ve_{\cyu}),
\end{equation}
where $\odot$ is element-wise multiplication and $\vc_i=\sigma_i(1-\sigma_i)$.

We find that $\vp^t(\cxu)-\ve_{\cyu}$ can be interpreted as the ``force'' imposed by learning $(\cxu,\cyu)$,
which starts from the current prediction and points to a one-hot distribution, as demonstrated in \cref{fig:chap3_mnist_akg01}.
Note that both the \textit{strength} and the \textit{direction} of this term matter in our analysis.

\subsection{Demonstration using an MNIST Example}

\begin{figure}[t]
\vskip -0in
    \begin{center}
    \centerline{\includegraphics[width=1\columnwidth,trim=0 0 0 0, clip]{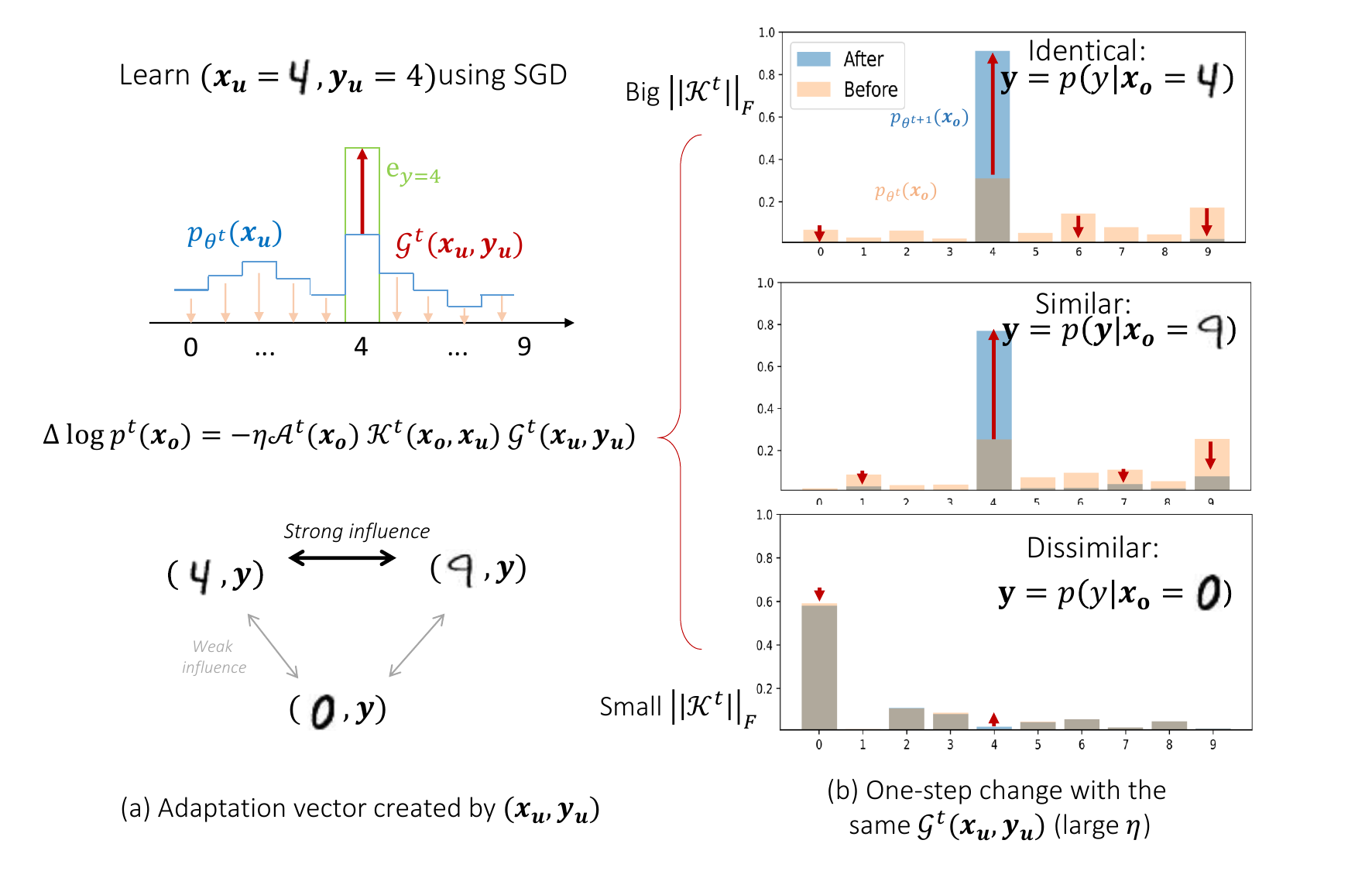}}
    \caption{The per-step influence in an MNIST experiment. A large learning rate is applied to make the difference more obvious.}
    \label{fig:chap3_mnist_akg01}
    \end{center}
\vskip -0 in
\end{figure}

Now we have the formal definition of the $\mathcal{AKG}$ decomposition.
The MNIST classification example provided in \cref{fig:chap3_mnist_akg01} and \ref{fig:chap3_mnist_akg02} can help us understand them better.
As illustrated in the upper panel in \cref{fig:chap3_mnist_akg01}-(a), our model learns a number ``4'' in this update.
$\mathcal{G}$ is a vector starting from the current prediction and pointing to the one-hot distribution.
The red arrow (on the label's dimension) is the pressure we \textit{intentionally} imposed, and the orange arrows come from the fact that the probability mass must sum to one.

Now, the learning dynamics provide us a way to do ``force analysis'' on the log-probability of $\cxo$.
The force originates from the $\mathcal{G}$ term and is then projected by the $\mathcal{K}$ term.
Note that $\mathcal{K}$ measures the similarity between $\cxu$ and $\cxo$ from the model's gradient perspective.
Perhaps surprisingly, we find that this similarity usually aligns with humans' perception.
For example, since a number ``4'' looks more similar to a number ``9'' than a number ``0'', then learning a number ``4'' influence $p(y\mid x_o=9)$ more than $p(y\mid x_o=0)$.
\cref{fig:chap3_mnist_akg01}-(b) verifies this relative strength by showing the probability change of different examples after the same one-step update of learning a number ``4''.
First, in the identical case, the 4th dimension is pulled up, as expected.
Then, the learning of this ``4'' pulls up the 4th dimension rather than the 9th dimension when a number ``9'' is given, which introduces an interesting pairing effect that is hard to interpret by analyzing the global optimum.
Last, since a number ``0'' is not similar to ``4'',
this update almost does not influence $p(y\mid x_o=0)$.

\section{Accumulated Influence}
\label{sec:fundamentalLD_03}

\subsection{Relative Stable $\mathcal{K}$ and Pairing Effect}

\begin{figure}[t]
\vskip -0.1in
    \begin{center}
    \centerline{\includegraphics[width=1\columnwidth,trim=0 0 0 0, clip]{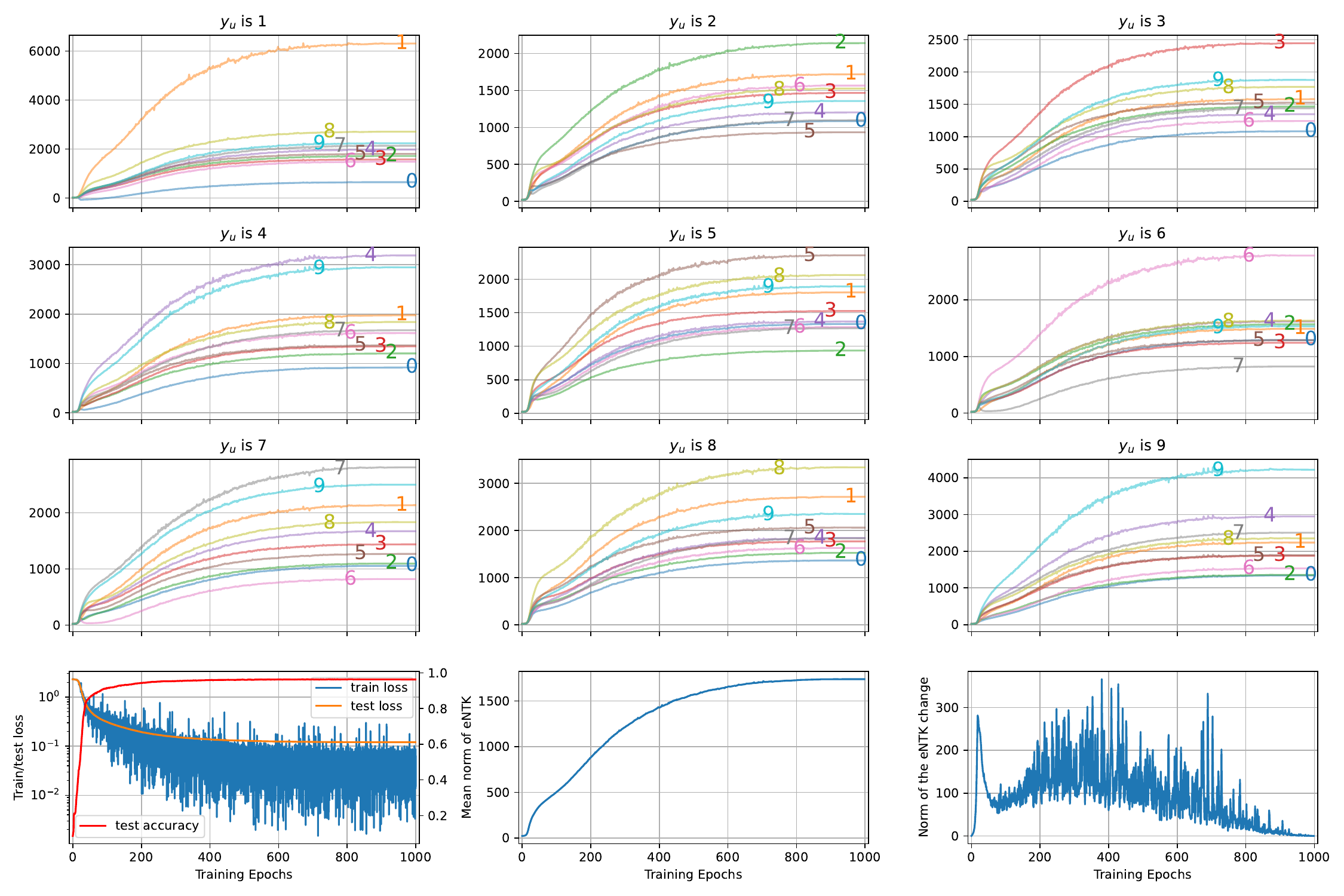}}
    \caption{First three rows: Results showing the relative stability of $\|\mathcal{K}^t_{uo}\|_F$ for learning a specific number (title of each panel) on other numbers (labeled by the colorful digits near the lines).
            Last row: change of train/test loss, mean of eNTK, and $\mathbb{E}_{u,o}\left[\|\mathcal{K}^t_{uo}-\mathcal{K}^{t-1}_{uo}\|_F\right]$. ($\mathcal{K}^t_{uo}$ is just $\mathcal{K}^t(x_o,x_u)$ for notation conciseness.)}
    \label{fig:chap3_ntk_adapt}
    \end{center}
\vskip -0.2in
\end{figure}

The previous section demonstrates the one-step influence of learning a single example. 
We now turn to discuss the cumulative influence throughout the training process.
The first thing to check is whether the relatively stable $\mathcal{K}^t$ assumption holds.
Hence in \cref{fig:chap3_ntk_adapt}, we directly visualize the average $\|\mathcal{K}^t\|_F$ throughout the training (we also use $\mathcal{K}^t_{uo}$ to represent $\mathcal{K}^t(x_o,x_u)$ to make the notation concise).
The key findings are:

\begin{figure}[t]
\vskip -0.1in
    \begin{center}
    \centerline{\includegraphics[width=0.75\columnwidth,trim=0 0 0 0, clip]{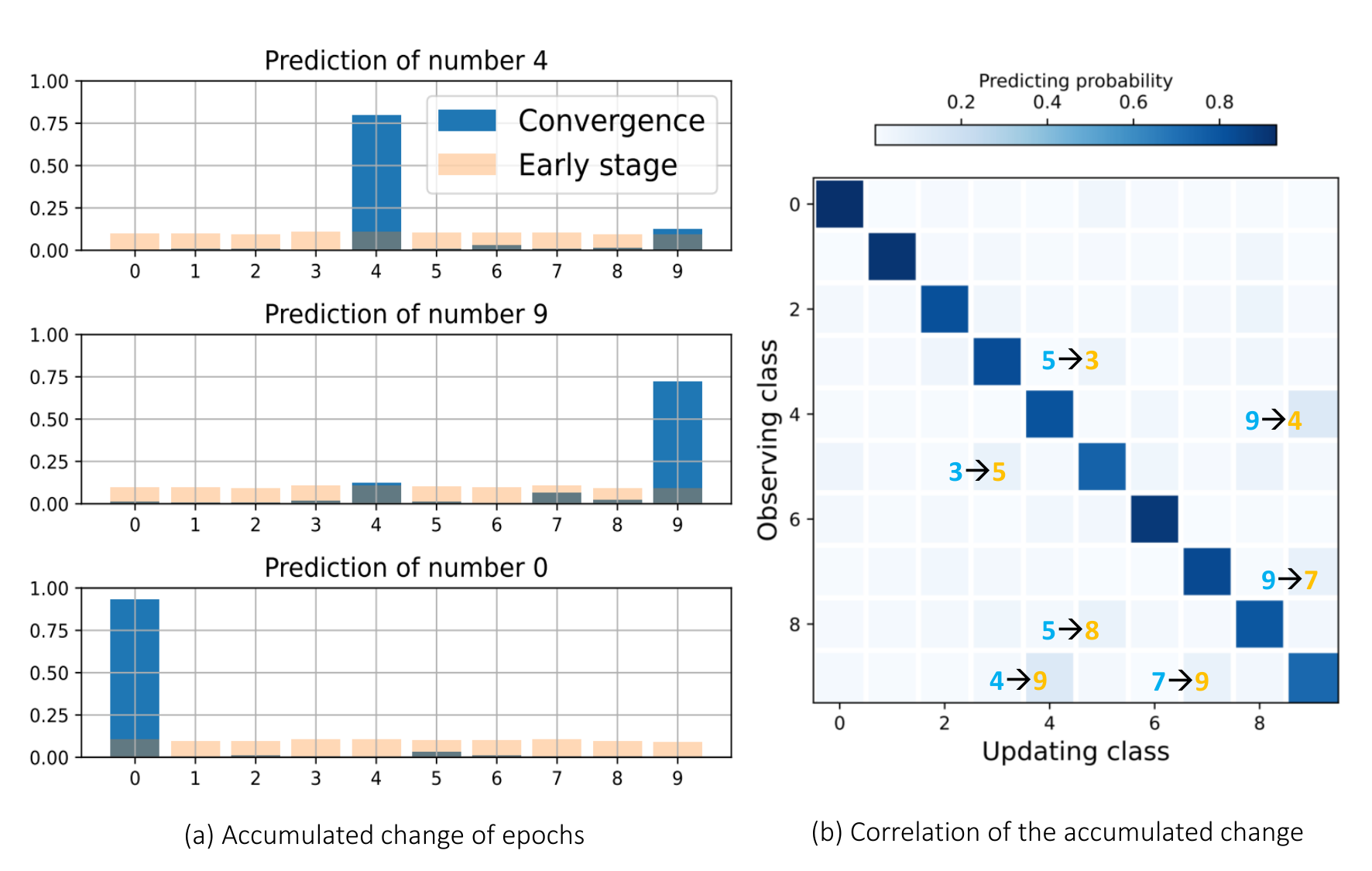}}
    \caption{Accumulated influence (pairing effect) in an MNIST experiment.}
    \label{fig:chap3_mnist_akg02}
    \end{center}
\vskip -0.2in
\end{figure}

\begin{itemize}
    \item[1. ] The last row roughly indicates different phases throughout the training: the first several epochs ($0\sim 30$) are a bit messy, and the last several epochs ($800\sim 1000$) behave similarly to the finetuning stage;
    \item[2. ] Although the norm of eNTK ($\mathbb{E}_{u,o}\left[\|\mathcal{K}^t_{uo}\|_F\right]$) and the norm of eNTK's adaptation ($\mathbb{E}_{u,o}\left[\|\mathcal{K}^t_{uo}-\mathcal{K}^{t-1}_{uo}\|_F\right]$) changes a lot after 30 epochs,
    the ranking between $\|K^t_{uo}\|_F$ on different observing classes are relatively stable, as demonstrated by the upper nine panels;
    \item[3. ] The pairing effect between the ``similar'' inputs is clear, e.g., ``4'' and ``9'', ``5'' and ``8'', etc;
    \item[4. ] The pairing effect between the ``dis-similar'' inputs are also clear, e.g., ``6'' and ``7'', ``2'' and ``5'', etc.
    \item[5. ] The pairing effect mentioned previously is not strictly symmetric, which is because of the inconsistent $\mathcal{A}$ and $\mathcal{G}$ terms.
\end{itemize}
In summary, throughout the majority of the training process, especially when the test accuracy is relatively high, the similarity between different classes tends to remain relatively stable.
Hence, the pairing effect mentioned in one-step influence can be accumulated.
As in \cref{fig:chap3_mnist_akg02}, if we keep training the model, interesting pairing effects emerge in their predictions.
Note that this pattern gradually vanishes as the model approaches convergence. This supports our view that analyzing a model’s dynamic behavior during training provides a valuable complement to traditional learning theory, which primarily focuses on asymptotic convergence.

\subsection{Evolution of $\mathcal{G}$ and Zig-zag Learning Path}

\begin{figure}[t]
\vskip -0in
    \begin{center}
    \centerline{\includegraphics[width=1\columnwidth,trim=0 0 0 0, clip]{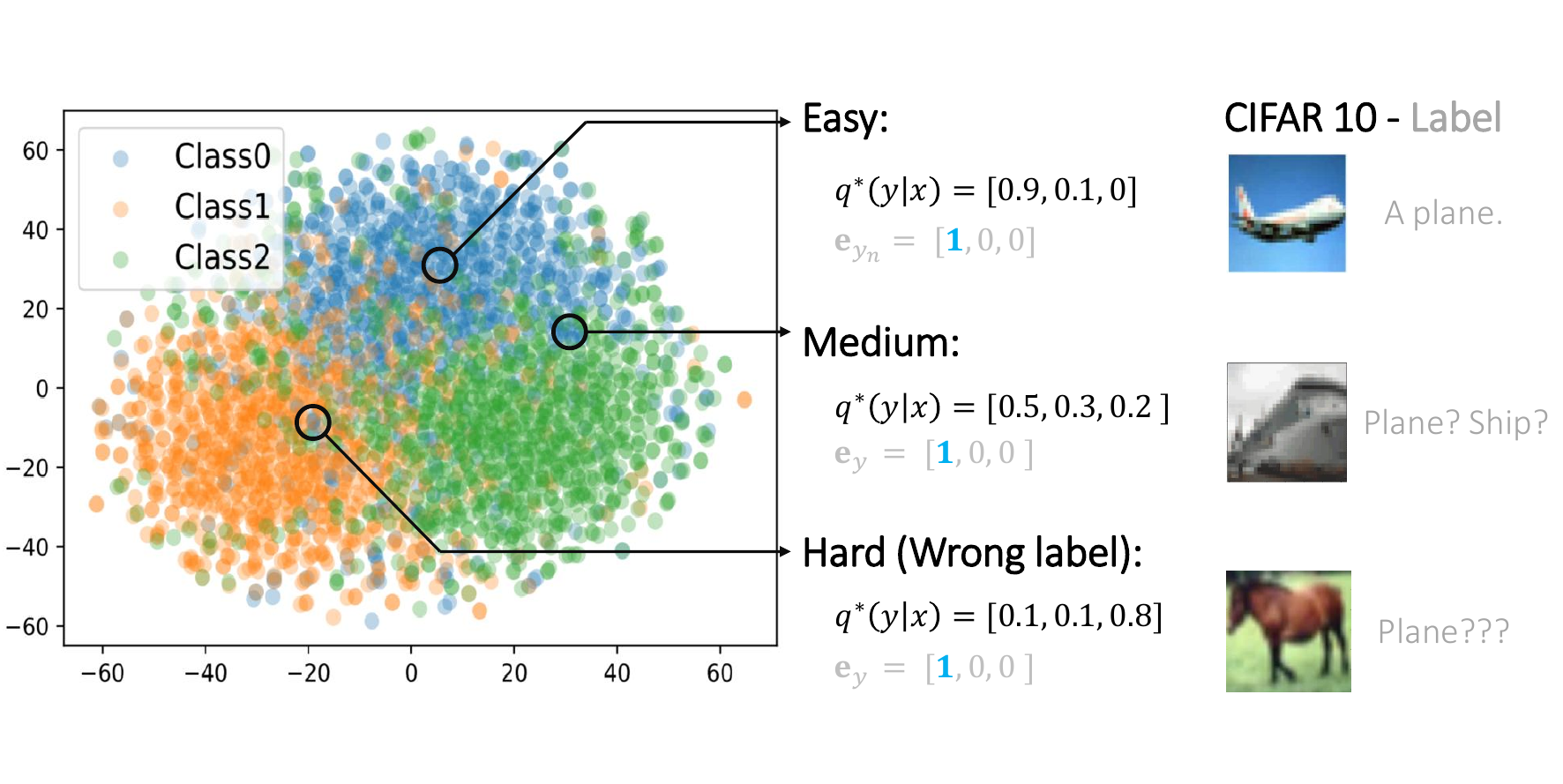}}
    \caption{Left: a multi-variate Gaussian dataset with $V=3$. Right: three typical sample groups in the \texttt{Toy-Gaussian} and noisy-CIFAR-10 dataset \citep{krizhevsky2009learning}.}
    \label{fig:chap3_toy_gaussian}
    \end{center}
\vskip -0in
\end{figure}

In contrast to $\mathcal{K}^t$, which usually changes slowly with different $t$,
$\mathcal{G}^t$ typically undergoes significant changes during training.
Fortunately, such changes are often predictable owing to their simplicity, i.e., they resemble the form $\vp^t(\cxu)-\ve_{\cyu}$. 
As a result, one can perform a straightforward ``force analysis'' by examining its norm and direction over the course of training.
We then use the following toy example to give a better demonstration of this process.

Consider a multivariate Gaussian dataset (\texttt{Toy-Gaussian} for short), as illustrated in \cref{fig:chap3_toy_gaussian}.
The dataset contains $N$ samples, where each sample is a 3-tuple ($x, y, \vq^*)$.
We can use the following three steps to get one:
\begin{itemize}
    \item Sample label $y$: we first select the label $y\in[V]$ following a uniform distribution over all $V$ classes.
    \item Sample input $x$: we sample the input signal $x|_{y=v}\sim\mathcal{N}(\bmf{\mu}_v,\sigma^2I)$, where $\sigma$ is the noisy level for all the samples. $\bmf{\mu}_v$ is the mean vector for all the samples in class $v$. Here, each $\bmf{\mu}_v$ is a 30-dim vector, in which each dimension is randomly selected from $\{-\delta_\mu,0,\delta_\mu\}$. This process is similar to selecting 30 different features for each class.
    \item Calculate ground truth Bayesian probability $q^*(y\mid x)$: To do this, we use the fact that $q^*(y\mid x)\propto q(x\mid y)p(y)$. 
    As $y$ follows an uniform distribution, we have $q^*(y\mid x)=\frac{q(x\mid y=v)}{\sum_{j\neq v}q(x\mid y=j)}$. 
    Following $q(x\mid y=v)\sim\mathcal{N}(\bmf{\mu}_v,\sigma^2I)$, we find $q^*(y\mid x)$ should have a Softmax form. Specifically, we have:
    \begin{equation}
        q^*(y=v\mid x)=\frac{\ve^{s_v}}{\sum_{j\neq v}\ve^{s_j}};\quad s_i = -\frac{1}{2\sigma^2} \| x -\bmf{\mu}_i\|^2_2.\nonumber
    \end{equation}
\end{itemize}
For a finer-grained understanding of the model's behavior, we roughly group all the samples into the following three groups, based on a \textit{base difficulty} defined as $\|\vq^*(x)-\ve_y\|_2^2$ (please also refer to \cref{fig:chap3_toy_gaussian}):
\begin{itemize}
    \item Easy: small base difficulty, where $\vq^*$ is similar to $\ve_y$ ($x$ locates close to the cluster center), e.g., an image of a plane that is easy to classify;
    \item Medium: medium base difficulty, where the peak of $\vq^*$ might still align with $y$ ($x$ locates near the boundary of different clusters), but is already dissimilar to $\ve_y$, e.g., an image which is hard to tell whether it is a plane or a ship;
    \item Hard (e.g., wrong labeled): high base difficulty, where the peak of $\vq^*$ might be different from $y$ ($x$ locates close to another cluster center), e.g., an image of a horse but labeled as a plane.
\end{itemize}

With the preparations above, we can now do a very interesting experiment to demonstrate the evolution of $\mathcal{G}^t$.
Specifically, we randomly initialize a 4-layer MLP and train it on the \texttt{Toy-Gaussian} data using SGD (it converges roughly at the 80th epoch).
During the training, we track the model's confidence $\vp(x_o)$ on four examples with different difficulties.
Since $V=3$ and each $\vp(x_o)$ is a 2-simplex, we can use a barycentric coordinate visualization.\footnote{In this triangular plane, each corner is a one-hot distribution and the center is $[\frac{1}{3},\frac{1}{3},\frac{1}{3}].$}

\begin{figure}[h]
\vskip -0 in
    \begin{center}
    \centerline{\includegraphics[width=1\columnwidth,trim=0 0 0 0, clip]{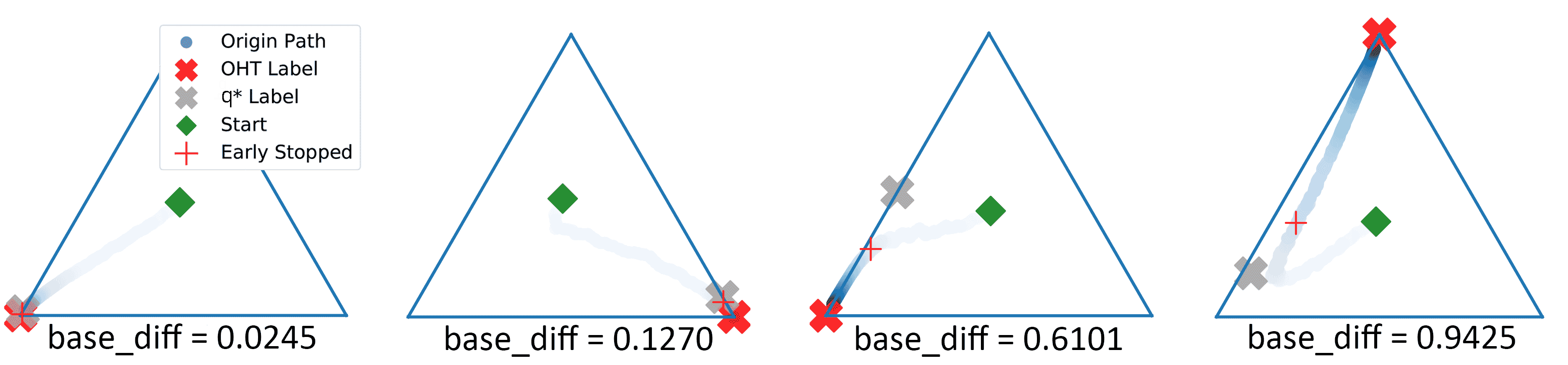}}
    \caption{Learning path of samples with different base difficulty. Corners correspond to one-hot vectors. Colors represent training time: transparent at initialization, dark blue at the end of training.}
    \label{fig:chap3_toy_path4}
    \end{center}
\vskip -0 in
\end{figure}

As demonstrated in \cref{fig:chap3_toy_path4}, the two easy samples very quickly move to the correct location near $\ve_y$ (as indicated by the light color until reaching the corner).
The medium sample takes a less direct route, drifting off slightly towards $\vq^*$, but still directly approaches $\ve_y$.
The hard sample, however, does something very different: it first approaches $\vq^*$, but then veers off towards $\ve_y$, giving a ``zig-zag'' path.
In both the medium and hard cases, there seems to be some ``unknown force'' dragging the learning path towards $\vq^*$.

Where does this unknown force come from?
Conducting a ``force analysis'' for each epoch helps us to answer this question.
Specifically, we can approximate the one-step influence defined in \cref{eq:sec3:akg} as the accumulated influence for one epoch, and then decompose it as follows.
\[
    \Delta^\text{ep}(\cxo) = -\eta\mathcal{A}^t(\cxo)\left(
                    \mathcal{K}^t(\cxo,\cxo)\mathcal{G}^t(\cxo,\cyo)+
                    \sum_{\substack{\cxu\neq\cxo \\ \cxu\in\mathcal{D}_\text{train}}} \mathcal{K}^t(\cxo,\cxu)\mathcal{G}^t(\cxu,\cyu)
        \right)
\]
(This is an approximation in that it neglects the change in $\mathcal A$ and $\mathcal K$ during the epoch.)
In this decomposition, the forces imposed on $\cxo$ come from two sources: other examples and itself, as demonstrated in \cref{fig:chap3_zigzag_detail}-(a).
For the influence from other examples, at any time $t$, ``dissimilar'' samples will have small $\|\mathcal{K}^t_{uo}\|_F$, and hence only slightly affect $\vp(\cxo)$.
However, those ``neighboring'' samples will have large $\|\mathcal{K}^t_{uo}\|_F$, and hence affect its updates much more;
because $\vq^*$ is hopefully similar for similar $x$ values,
it is reasonable to expect that the mean of $\ve_y$ for data points with similar $x$ will be close to $\vq^*(\cxo)$.
Thus, the net effect of updates for $\cxu \neq \cxo$ should be to drag $\vp(\cxo)$ towards $\vq^*(\cxo)$.
This is the ``unknown force'' we observed.

Another force affecting $\vp(\cxo)$ during training, i.e., the update based on $(\cxo,\cyo)$, will drive $\vp(\cxo)$ towards $\ve_{\cyo}$, even though, in our hard example, it is quite far away from the ground truth $\vq^*(\cxo)$.

The dynamical interaction among these forces forms the ``zig-zag'' learning path we observed.
To better see the interactions of these forces, we show the learning path of a hard sample during training in \cref{fig:chap3_zigzag_detail}.
In the figure, we use \blue{blue} lines to demonstrate the changes based on the updates of other examples, while \red{red} lines show the influence of learning its label.
The \purple{purple} dot and cross marker show $\vp(\cxo)$ at the start and end of the epoch.

It is clear that at the beginning of the training, the influence of neighboring examples dominates.
That is because near initialization $\vp^t(\cxu)$ will be relatively flat.
Hence, the size of the $\mathcal{G}^t$ for ``similar'' $\cxu$ and the $\mathcal{G}^t$ for $\cxo$ should be comparable, as demonstrated in \cref{fig:chap3_zigzag_detail}-(b).
That is to say, if there are at least a few ``similar'' training points,
$\vp^t(\cxo)$ will move towards $\vq^*(\cxo)$, as demonstrated in the figure.
Throughout training,
some of these similar samples will become well-classified,
so that their $\mathcal{G}^t$ decay very fast,
and their updates will no longer exert much force on $\vp^t(\cxo)$.
Thus, the $\cxo$ updates begin to dominate,
causing the zig-zag pattern as the learning path turns towards $\ve_{\cyo}$.
For easy samples, where $\vq^*$ and $\ve_{\cyo}$ are in the same direction, these forces agree and lead to fast convergence.
On samples like the ``medium'' example, the two forces broadly agree early on, but the learning path deviates slightly towards $\vp^*$ en route to $\ve_{\cyo}$.

\begin{figure}[t]
\vskip -0in
    \begin{center}
    \centerline{\includegraphics[width=1\columnwidth,trim=0 0 0 0, clip]{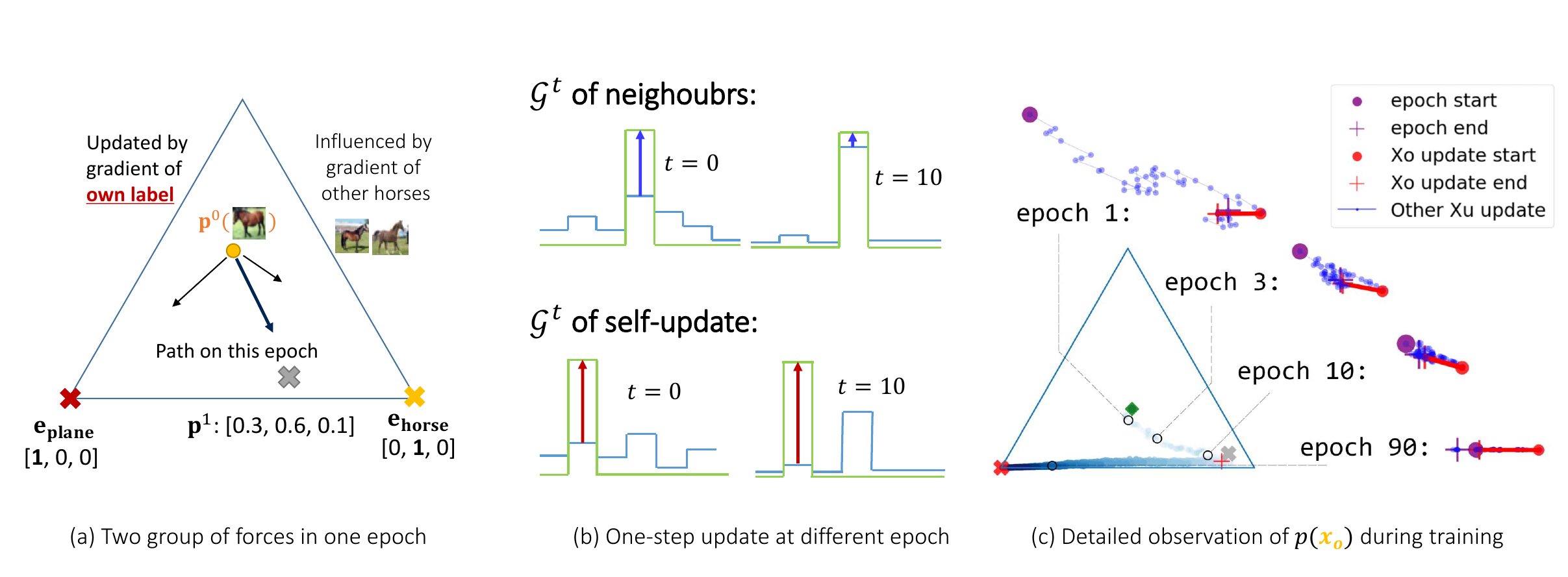}}
    \caption{Demonstration of how the ``zig-zag'' learning path emerges in a hard example: evolution of $\mathcal{G}^t$ is the key.}
    \label{fig:chap3_zigzag_detail}
    \end{center}
\vskip 0.3 in
\end{figure}

\section{Summary: Force Analysis of DNN}
\label{sec:fundamentalLD_04}

In this chapter, we formalize learning dynamics in a basic supervised setting. 
Leveraging the $\mathcal{AKG}$ decomposition, we show how a single update influences model confidence on another example. 
Through MNIST and \texttt{Toy-Gaussian} cases, we visualize the stability of $\mathcal{K}^t$ and the evolution of $\mathcal{G}^t$, which are the keys to understanding cumulative learning effects. 
This process parallels ``force analysis'' in Newtonian mechanics.
Specifically, the ``force'' originates from the $\mathcal{G}$ term.
It is then normalized and projected by the $\mathcal{A}$ and $\mathcal{K}$ terms, and is ultimately imposed on $\log \vp(\vx_o)$.
In the following chapters, we will extend this analytical style to various settings to deepen our understanding of deep neural networks' behavior.

However, there are also important differences between the two settings. 
For instance, our analysis focuses solely on the first-order temporal behavior, without an equivalent notion of inertia or momentum.\footnote{Perhaps the ``momentum'' mechanisms in other optimizers like RMSProp \citep{hinton2012neural} or Adam \citep{kingma2014adam} are worth incorporating.}
Moreover, the force in our framework is inherently high-dimensional, in contrast to the three-dimensional space typically considered in classical mechanics.

Extending this analogy by introducing additional physical concepts into our framework, such as higher-order dynamics or energy-based interpretations, could be an interesting direction for future research.
\chapter{Understanding and Improving Knowledge Distillation}
\label{sec:case1}

\cref{sec:fundamentalLD} provides us with a powerful tool for understanding different phenomena during deep models' training.
\cref{sec:fundamentalLD_03} further demonstrates an interesting ``zig-zag'' learning path in a supervised classification task.
From these observations, a natural idea is to stop the training when the model reaches the inflection point along this path,
which is exactly the motivation of a related study of robust classification \citep{liu2020early}.\footnote{We did not find this paper, which is quite similar to our paper \citet{renbetter}, until completing the latter; the paper contains a detailed discussion of the differences.}
With the help of early-stopping regularization, they design an efficient algorithm to prevent the model from memorizing the wrong labels, and hence behaves more robustly when the labels are noisy.

This chapter extends the problem to a more general setting.
We argue that even in relatively clean-label scenarios, the model is still capable of automatically identifying better supervisory signals, which can subsequently benefit the training of other models.
Specifically, we first show that training a network with a better supervisory signal, i.e., one closer to the unknown ground-truth distribution $\vq^*$, leads to a tighter generalization bound under the cross-entropy loss.
Motivated by the analysis of learning dynamics in real-world image classification tasks, we then propose a method called Filter-KD (Filtered Knowledge Distillation), in which a teacher network, trained with early stopping and an exponential moving average (EMA) mechanism, provides improved supervisory signals to a student network.
Our empirical results not only validate the theoretical predictions of learning dynamics in practical settings, but also demonstrate that Filter-KD is an effective and modular tool that can be easily integrated into existing KD frameworks.
The discussions and results in this chapter are primarily based on our paper ``\textit{Better Supervisory Signals by Observing Learning Paths}'' \citep{renbetter}.

\section{Background and Related Work}
\label{sec:case1_01}

\subsection{One-hot Supervision is not Optimal}
In multi-class classification problems, we usually supervise our model with ``one-hot'' labels: the one-hot vectors $\ve_y$ which have the $y$-th dimension equal to one, and all other dimensions equal zero.
Over time, however, it has gradually become clear that this ``default'' setup might not always be the best choice in practice, in that other schemes can yield better performance on held-out test sets.
One such alternative is to summarize a distribution of human annotations,
as \citet{human_label} did for CIFAR10.
However, this method requires a lot of human effort, which is hard to scale to other larger datasets.
Another alternative approach is label smoothing, in which a mixture of a one-hot vector and the uniform distribution is used to train the model \citep{ls_first}.
However, this approach manipulates the one-hot labels for all examples (containing both easy and hard ones) in the dataset in the same way, which is also not optimal.

\subsection{Teacher as Better Supervision in KD}
Knowledge distillation (KD), first training a teacher network on the training set
and then a student network on the teacher's output probabilities was originally proposed for model compression \citep{KD_initial}.
However, we can also understand it as refining the supervision signal: the teacher provides ``soft'' and input-dependent outputs rather than the hard one-hot labels to the student.
Knowledge distillation is promising because it requires no additional annotation effort, and at the same time, can provide sample-specific labels based on their difficulties.
Perhaps surprisingly, KD can improve student performance even when the teacher is of \emph{exactly the same form as the student} and trained on the same data;
this is known as self-distillation \citep{BAN,beyourownteacher}.
There have been many recent attempts to explain knowledge distillation and specifically self-distillation \citep[e.g.][]{KD_probability,kd_understanding_long,tang2020understanding},
from both optimization and supervision perspectives. 
We focus on the latter area, where it is usually claimed that the teacher provides useful ``dark knowledge'' to the student through their labels.
From the discussions in \cref{sec:fundamentalLD_03}, we speculate that such dark knowledge comes from the mutual influence among training samples, which is well captured by our learning dynamics analysis framework.

\section{Insights From Learning Dynamics}
\label{sec:case1_02}

\subsection{Better Supervisory, Better Generalization}
From the student network’s perspective, the only difference among self-knowledge distillation (self-KD), label smoothing, and standard supervised learning lies in the supervisory signal it receives.
Consistent with findings from prior work (as well as our own experiments), self-KD typically outperforms label smoothing, which in turn outperforms standard training.
This performance hierarchy suggests that the supervisory signals provided by a teacher network are of higher quality than those offered by label smoothing or one-hot labels.
This naturally raises an important question: \textit{how can we quantitatively compare the quality of different supervisory signals?}

To answer this question, we first recap the general form of cross-entropy loss, i.e., $-\vq^\top\log \vp$, and the most commonly used form restricts $\vq$ to be a one-hot vector $\ve$.
If we switch our supervisory signal to another target distribution $\vq_\text{tar}$, the corresponding risk is then
\begin{align}\nonumber
    R_\tar(f,\mathcal{D})
    &\triangleq -\sum^{n}_{i=1}\sum^{V}_{v=1}
        {\frac{1}{n}}
        \, {q_\tar(y_i=v \mid x_i)}
        \, \log p(y_i=v \mid x_i)   \\
    &= -\frac{1}{n}\sum_{i=1}^{n}
        \, \vq_\tar\tp(x_i)\log\vp(x_i) \label{eq:tar_risk}
\end{align}
With this definition, we propose the following hypothesis as a trend. \vspace{0.8em}

\begin{hyp} \label{hyp:l2-gen}
    Suppose we train a model supervised by $\vq_\tar$, that is, we minimize $R_\tar(f,\mathcal{D}_\text{train})$.
    Then, smaller average $L_2$ distance between $\vq_\tar$ and the ground truth $\vq^*$ on these samples, i.e.\ small $\E_{x}\left[\|\vq_\tar(x)-\vq^*(x)\|_2\right]$, will in general lead to better generalization performance.
\end{hyp}

This hypothesis is suggested by Proposition 3 of \cite{KD_probability},
which shows that for any predictor $f$ and loss bounded as $\mathcal{L}(f(x), \hat y) \le \ell$, we have
\[
    \E_{\mathcal D}\left[(R_\tar(f,\mathcal{D}) - R(f))^2 \right]
    \le \frac{1}{n} \Var_{x}\left[ \vq_\tar\tp \vL(f(x))\right]
    + \ell^2 K \left( \E_{x} \lVert \vq_\tar(x) - \vq^*(x) \rVert_2 \right)^2,
\]
where $\vL(f(x))$ is the loss for each possible label (it is just $-\log\vp(x)$ in our negative log-likelihood, i.e., NLL, setting) and $K$ is also a constant.
When $n$ is large, the second term will dominate the right-hand side,
implying a smaller average $\lVert \vq_\tar - \vq^* \rVert$
will lead to $R_\tar$ being a better approximation of the true risk $R$; minimizing it should then lead to a better learned model.
This suggests that the quality of the supervision signal
can be roughly measured by its L2-distance to the ground truth $\vq^*$.
\Cref{sec:app2:risk-est-var} slightly generalizes the result of \citet{KD_probability} with bounds based on total variation (L1-distance) and KL divergences; we focus on the L2 version here for simplicity.
From Hypothesis \ref{hyp:l2-gen}, we can now treat $\lVert \vq_\tar(x) - \vq^*(x) \rVert_2$ as a element-wise measurement of the supervision quality.

\begin{figure}[t]
\vskip -0in
    \begin{center}
    \centerline{\includegraphics[width=0.9\columnwidth,trim=0 0 0 0, clip]{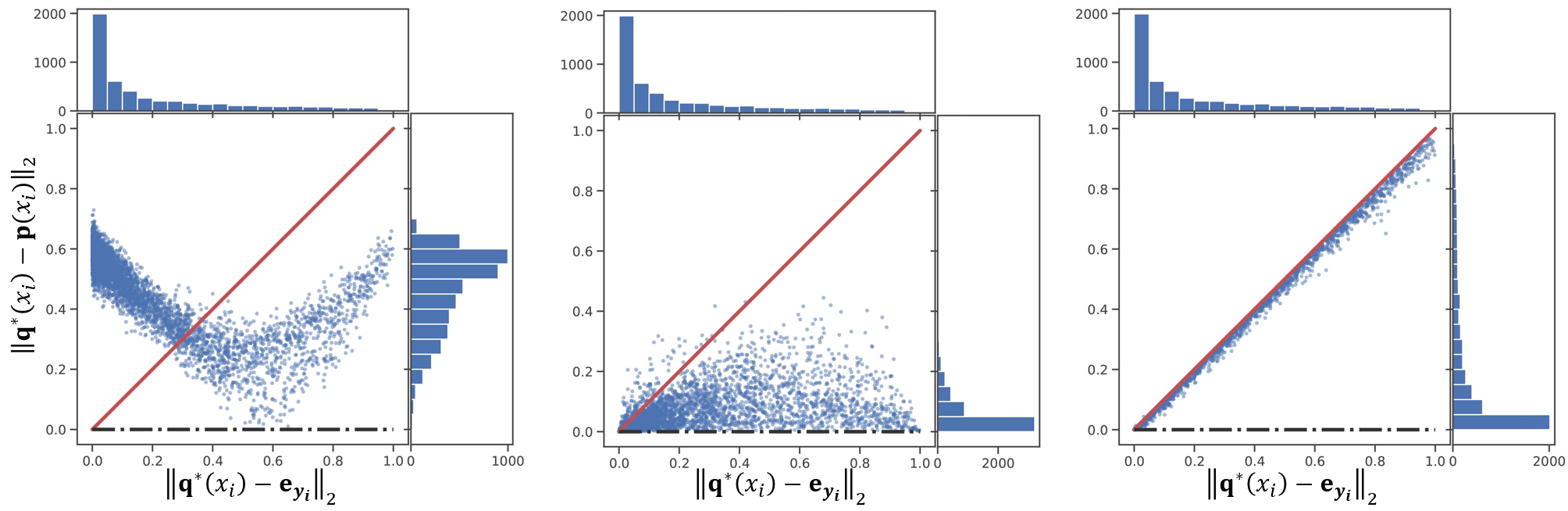}}
    \caption{Normalized (divided by $\sqrt{2}$) distance between output distribution $\vp$ and ground truth $\vq^*$ during the one-hot training in different stages (left to right: initial, early stop, convergence). The bar plots show the density along the corresponding axis. The code for this experiment can be found in \href{https://github.com/Joshua-Ren/better_supervisory_signal/blob/main/notebooks/ToyGaussian_Part2_LearningDynamics.ipynb}{the notebook}, and the animation of this figure is in \href{https://github.com/Joshua-Ren/better_supervisory_signal/tree/main}{the git repo}.}
    \label{fig:chap4_hardness_3stage}
    \end{center}
\vskip -0in
\end{figure}

\subsection{Good Supervision Emerges during Training}
Since the $\vq^*$ is usually intractable, we start with the \texttt{Toy-Gaussian} dataset studied in \cref{sec:fundamentalLD_03}, since their $\vq^*$ are well defined and easy to calculate.
On real datasets, we consider CIFAR-10H \citep{human_label}, in which each image is paired with a group of labels provided by humans.
The corresponding empirical distribution is then a good approximation of $\vq^*(x)$.

Recall the examples demonstrated in \cref{fig:chap3_toy_path4}, where the model's learning path on hard examples first moves toward the unknown ground-truth $\vq^*$ and then converges to a one-hot distribution.
Although the four examples in this figure hint that early stopping the model's training can help the model find a good prediction,
is this still true for \textit{most of the examples?} 
(It is possible that the good stopping points for different examples differ greatly.)
To verify this, we use the \texttt{Toy-Gaussian} dataset and train a 4-layer MLP to convergence with $n=5000$ training samples.
In \cref{fig:chap4_hardness_3stage}, we show the scatter plots of each sample's base difficulty (how close its $\vq^*$ is to its $\ve_{y}$) versus the model's prediction quality (how close the model's prediction $\vp$ is to its $\vq^*$), in three different stages during training.
At initialization, most points\footnote{The curve structure is expected: points with $\vq^* \approx (\frac13, \frac13, \frac13)$ are near the middle of the base difficulty range, and all points are initialized with fairly ambiguous predictions $\vp$.} have large $\lVert \vp(x) - \vq^*(x) \rVert_2$.
By the point of early stopping, most $\vp$ values are roughly near $\vq^*$.
At convergence, however, $\vp(x) \approx \ve_{y}$, as the classifier has nearly memorized the training inputs, leading to a diagonal line in the plot.
Note that the average label quality ($\E_{x} \|\vp_\text{tar}(x)-\vq^*(x)\|_2$) mentioned in Hypothesis~\ref{hyp:l2-gen} is the mean height of all points in the figure, if we treat $\vp$ as $\vp_\text{tar}$.
It is clear that the early-stopped model has the potential to provide better supervisory signals, i.e., better $\vp_\text{tar}$.

This trend also persists in more realistic systems, albeit with increased noise in their training paths.
Specifically, we randomly initialize a ResNet18 \citep{resnet} and train it on CIFAR-10 until its convergence.
Since now we have 10 classes, we first convert a length-10 vector to a length-3 one by retaining the probabilities corresponding to the correct class and the second-highest predicted class, while aggregating the probabilities of the remaining classes into a single value.
The learning path can then be projected onto the same 2-simplex triangle plane, as demonstrated in \cref{fig:chap4_zigzag_all}.
The panels in the first column show the learning path of an easy sample.\footnote{Without knowing $\vq^*$, we instead use the ``zig-zag score'' to roughly measure the difficulty. Please also refer to Appendix E of our paper for more details \citep{renbetter}.}
We can see that $\vp^t$ converges quickly towards the left corner, the one-hot distribution for the correct class, because the color of the scattered points in that figure is quite light.
At the early stopping epoch, $\vp^t$ has already converged to $\ve_y$.
However, for a hard sample, it is very difficult to observe any patterns from the raw path (blue points): there are points almost everywhere in this plot.
This is likely caused by the more complex network and dataset, as well as a high learning rate in early training.
To find the hidden pattern in this high-variance path, we treat $\vp^t$ as a time series signal with many high-frequency components.
Thus, we can collect its low-frequency components via a low-pass filter and then draw that, i.e., taking an exponential moving average (EMA) on $\vp^t$, using the following equation:
\begin{equation}
    \vp^{t+1}_\text{ema} = (1-\alpha)\vp_\text{ema}^{t} + \alpha \vp^t,
\end{equation}
where $\alpha$ controls the strength of the smoothing.
In digital signal processing, the equation above is equivalent to a first-order low-pass filter where $\alpha$ controls the cut-off frequency \citep{smith2003digital}.

\begin{figure}[ht]
\vskip -0in
    \begin{center}
    \centerline{\includegraphics[width=0.9\columnwidth,trim=0 0 0 0, clip]{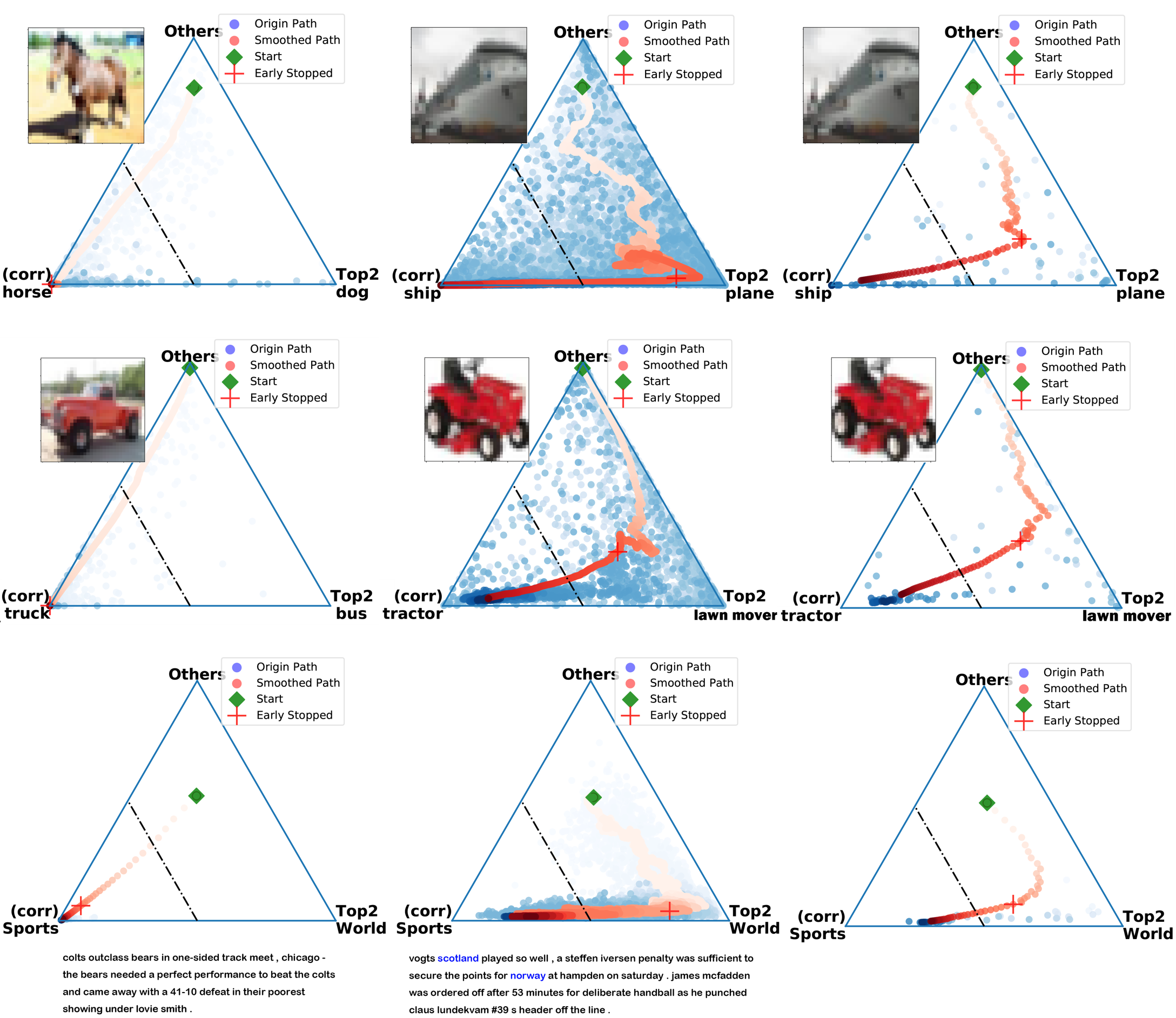}}
    \caption{Learning paths on different problems. In the first two columns, we record $\vp^t$ for each batch, while in the last column, we record it for each epoch. The first row is training a ResNet18 on CIFAR-10, the second row is ResNet18 on CIFAR-100, and the last row is LSTM on AGNews.}
    \label{fig:chap4_zigzag_all}
    \end{center}
\vskip -0in
\end{figure}

After smoothing the noisy trajectory, we find $\vp^t_\text{ema}$ (the red dots) exhibits a similar trend: first moving towards the unknown true label before eventually turning to memorize the wrong label.
Such a trend also exists in many more complicated systems, e.g., training a ResNet18 on CIFAR-100 or training an LSTM \citep{lstm} on AGNews \citep{agnews}, as demonstrated in the later rows of this figure.

\section{Proposed Method: Filter-KD}
\label{sec:case1_03}

\subsection{The Algorithm}
Inspired by the findings above, we hence propose a knowledge distillation-based method called Filter-KD.
The idea is quite simple: since better supervision leads to better generalization ability (as in Hypothesis~\ref{hyp:l2-gen}) and early stopping the training of one network can provide better $\vp_\text{tar}$, we can first train a teacher network and then use the filtered $\vp^{t'}_\text{ema}$ as the supervisory signal for the student's training, as demonstrated in \cref{alg:filter-kd}.
Specifically, we maintain a look-up table $\vq_\text{ema}\in\mathbb{R}^{N\times V}$ to store a moving average of $\vp_t$ for each training sample.
Note that in one epoch, each $\vq_\text{ema}(x_n)$ will be updated only once.\footnote{As suggested by the last column of \cref{fig:chap4_zigzag_all}, updating the prediction epoch-wise is enough to get a good estimation.}
We check the early stopping criterion with the help of a validation set.
Afterwards, the teaching supervision $\vq_\text{ema}$ is ready, and we can train a student network under its supervision.
This corresponds to using a moving average of the teacher model ``in function space,'' i.e., averaging the outputs of the function over time.

Note that this framework is easy to incorporate into any other KD methods, as long as the teacher and student have identical output formats.
Such a distillation procedure can also be conducted for multiple generations to further improve the model's performance, as was done in the ``Born Again Networks (BAN)'' studied by \citet{BAN}.
We analyze a similar system in our paper \citet{ren2023improving}, which also uses the general principle of learning dynamics discussed in this thesis.
For brevity,
we only focus on one generation of this process here.

\begin{Ualgorithm}[h]
	\begin{algorithmic}[H]
        \STATE \textbf{Input:} Dataset $\{(x_n,y_n)\}_{i=1}^{N}$, two randomly initialized DNNs
        \STATE \textit{\# Train the teacher}
        \STATE Initialize the teacher $f_\text{teacher}$, initialize an $N\times V$ matrix called $\vq_\text{ema}$
        \STATE Go through the entire dataset, calculate $\vq_\text{ema}[n,:] = f_\text{teacher}(x_n)$
		\FOR{$\mathit{epoch} = 1,2, ..., T$}
		    \FOR{$i \in \{1, 2, \dots, N\}$ in random order}
		        \STATE $\hat \vp = f_\text{teacher}(x_n)$
		        \STATE $\vq_\text{ema}[n,:] = (1-\alpha) \cdot \vq_\text{ema}[n,:] + \alpha \cdot \hat \vp$
		        \STATE Update parameters of $f_\text{teacher}$ based on $\mathcal{L}_{ce}(\hat\vp,y_n)$, 
		    \ENDFOR
		    \STATE Check the early stopping criterion; stop training if satisfied
		\ENDFOR
        \STATE \textit{\# Train the student}
        \STATE For each input $x_n$, set $\vp_\tar=\vq_\text{ema}[n,:]$
        \STATE Train the network $f_\text{student}$ under the supervision of $\vp_\tar$
	\end{algorithmic}
	\caption{Filter-KD. $\alpha$ controls the cut-off frequency of the low-pass filter (0.05 in the experiments here).}
	\label{alg:filter-kd}
\end{Ualgorithm}

\subsection{Evaluation of Filter-KD}
In this part, we consider two main baselines.
Other than the standard supervised learning using one-hot examples (OHT for short in the tables), we further compare an important baseline called ESKD (early-stopped KD, i.e., $\alpha=1$ using the same algorithm), which can be considered as an ablation for the filtering mechanism applied in Filter-KD.
Compared to ESKD, Filter-KD can avoid the extremely high variance of $\vp^t$ during training.
That is because most practical implementations only consider early-stopping at the end of each epoch, which is equivalent to down-sampling the noisy learning path (as in the last panel of \cref{fig:chap4_zigzag_all}),
further exacerbating the variance of $\vp^t$.
Thus, ESKD will likely select a bad $\vp_\text{tar}$ for many data points.
Filter-KD, by contrast, has much more stable predictions.

However, such enhanced robustness suffers from higher algorithmic complexity: the running time of Filter-KD might be slightly increased.
Furthermore, compared to the teaching model in ESKD, the Filter-KD requires a teaching table $\vq_\text{ema}\in\mathbb{R}^{N\times V}$, which requires substantial memory when the dataset is large.
One alternative avoiding the need for this table would be to instead take an average ``in parameter space,'' like e.g.\ momentum parameter updating as in \citet{momentum1}.
But we empirically find that, although this helps the model converge faster, it does not lead to a better teacher network.
Thus, although Filter-KD has clear drawbacks, we hope that our explanations here may lead to better practical algorithms in the future.

\begin{table}[t]
    \centering
    \resizebox{1\textwidth}{!}{
    \begin{tabular}{@{}c|cccc|cccc@{}}
    \toprule
             & \multicolumn{4}{c|}{Accuracy $\uparrow$}                                & \multicolumn{4}{c}{ECE $\downarrow$}              \\ %
             & \textbf{OHT}    &\textbf{KD}       & \textbf{ESKD}    & \textbf{FilterKD}& \textbf{OHT}    &\textbf{KD}       & \textbf{ESKD}    & \textbf{FilterKD}      \\ \midrule
    CIFAR10  & 95.34    & 95.39   & 95.42   & \textbf{95.63}   & 0.026     & 0.027     & 0.027      & \textbf{0.007} \\
    CIFAR100 & 78.07  & 78.40   & 78.83   & \textbf{80.09}   & 0.053      & 0.061     & 0.067      & \textbf{0.029} \\\bottomrule
    \end{tabular} }
    \caption{Quantitative comparison of generalization performance for ResNet18 self-distillation (mean value of 5 runs, \textbf{bold numbers} are the best method). ECE is the expected calibration error \citep{guo2017calibration}, which measures how well the model's prediction quality. Please also refer to \cref{sec:app3_ece} for the formal definition.}
    \label{tab:chap4_clean_data}
\end{table}

We now compare the performance of Filter-KD and other baselines on a real dataset.
To rule out the influence of other hyperparameters, we first focus on self-KD and a fixed temperature $\tau=1$, as we want the only difference among these methods to be $\vp_\text{tar}$. 
Thus, we can conclude that the improvement we achieved comes purely from the refined supervision.
More results considering different hyperparameters can be found in our paper \citep{renbetter}.

\Cref{tab:chap4_clean_data} clearly shows that Filter-KD performs best in both accuracy and ECE.
The two ablations (i.e., KD and ESKD) suggest that both the early stopping and smoothing mechanisms in Filter-KD are important.

\subsection{Auto Noisy-label Fixing Ability}

We then consider the noisy-label classification problem, where we train and validate the model on a dataset with some labels randomly flipped.
After training, all the models are evaluated on a clean held-out test set.
The experiments are conducted on CIFAR (\cref{fig:chap4_noisy_acc}) and TinyImageNet (\cref{tab:chapt4_tinyimagenet}),
under different noise ratios $\sigma$:
$\sigma=0.1$ means 10\% of the labels are flipped.
In \cref{fig:chap4_noisy_acc}, an interesting trend can be observed: the enhancement brought by Filter-KD is not significant when $\sigma$ is too small or too large.
That is reasonable because for small $\sigma$, few samples have bad labels, thus the possible enhancement might not be large.
When $\sigma$ is too high, the labels of similar points become less reliable, and the learning path will no longer head as reliably towards $\vq^*$.
Thus, for very high noise ratios, the performance of Filter-KD decays back towards that of OHT.

\begin{figure}[t]
\vskip -0in
    \begin{center}
    \centerline{\includegraphics[width=1\columnwidth,trim=0 0 0 0, clip]{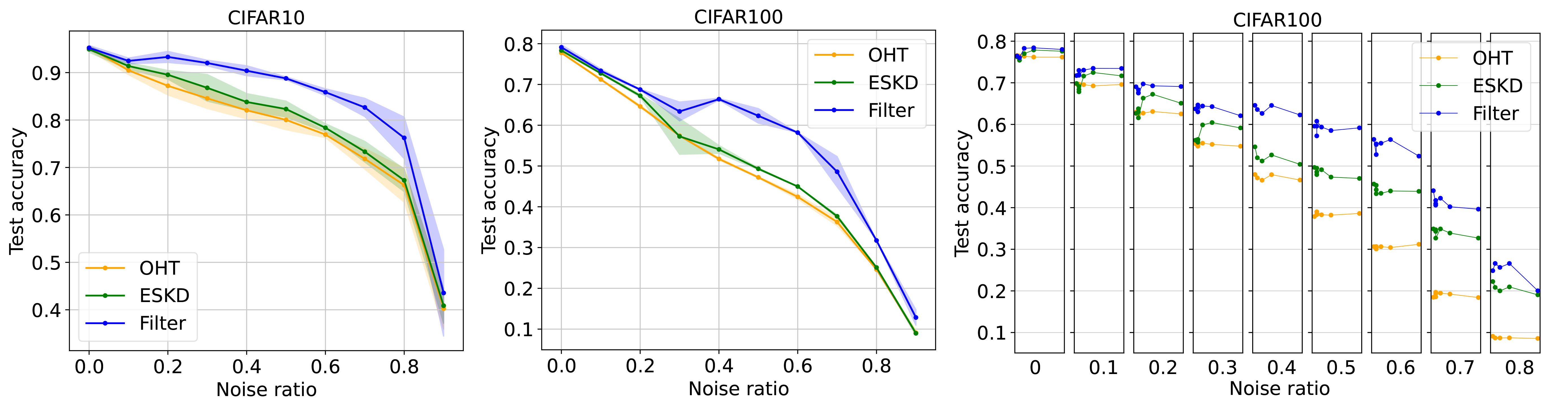}}
    \caption{Test accuracy under different noise ratio $\sigma$ (this notation is not softmax function here). Solid lines are the means, while shaded regions are the standard errors for 3 runs with different random seeds (shaded range is the standard error). The last panel compares the influence of different temperatures. Each thin rectangle plot represents a different $\sigma=\{0,0.1,...,0.8\}$, in which we plot the results with different $\tau=\{0.5,1,2,4,10\}$.}
    \label{fig:chap4_noisy_acc}
    \end{center}
\vskip -0in
\end{figure}

\begin{table}[h]
    \centering
        \begin{tabular}{c|cccc}
        \hline
        Noise $\sigma$ & 0 & 0.05 & 0.1 & 0.2 \\ \hline
        \textbf{OHT} & 56.95 & 53.02 & 52.02 & 30.52 \\
        \textbf{ESKD} & 58.61 & 53.53 & 52.99 & 36.55 \\
        \!\!\textbf{FilterKD}\!\! & \textbf{59.32} & \textbf{56.43} & \textbf{55.51} & \textbf{40.81} \\ \hline
        \end{tabular}    
    \caption{Noisy-label learning results on TinyImageNet dataset.}
    \label{tab:chapt4_tinyimagenet}
\end{table}

\section{Conclusion and Future Works}
\label{sec:case1_04}

In this chapter, we first claim that a better supervision signal, i.e., those $\vp_\tar$ that are closer to $\vq^*$, leads to better generalization performance; this is supported by results of \citet{KD_probability} and further empirical results given here.
To further understand how such better supervision emerges, we directly observe the behavior of samples with different difficulties by projecting them on a 2D plane.
The observed zig-zag pattern is well-explained by considering the tradeoff between two forces, one pushing the prediction to be near that of similar inputs, the other pushing the prediction toward its training label.
Such a phenomenon can be well explained by the learning dynamics framework we discussed in \cref{sec:fundamentalLD}: as long as we use the sample-wise exponential moving average (EMA, which is equivalent to a low-pass filter well studied in signal processing) to extract the dominating component in their messy learning paths.

To apply these findings to real tasks, in which the data and network are more complex, we propose a practical method called Filter-KD to further enhance $\vp_\tar$.
Experimental results on various settings (datasets, network structures, label noise, and temperatures) not only show the advantage of the proposed method as a method for knowledge distillation but also help verify our learning dynamics framework.

However, the proposed Filter-KD method incurs additional computational and storage costs compared to other baseline approaches.
Moreover, the experiments presented in this chapter are limited to the self-distillation setting.
Exploring how Filter-KD can be integrated with more advanced knowledge distillation frameworks remains an interesting direction for future work.

Regarding the framework of learning dynamics, most of the theoretical results and experiments in this chapter can be viewed as a validation of the proposed phenomena in a more practical image classification setting.
The principles highlighted in our framework are further corroborated by recent independent research, including \citet{datta2023measuring,ravikumar2025towards}, which explore similar concepts in diverse fields. 
This convergence of findings underscores the generalizability of our approach to a wide range of related tasks.
The next chapter extends this framework to the fine-tuning of large language models (LLMs), where the problem setup is substantially more complex than the standard supervised classification tasks studied here.
Remarkably, under reasonable simplifications and mild assumptions, we show that learning dynamics remains a powerful analytical tool even in such large-scale systems.
This finding highlights the potential of ``force analysis'' as a promising perspective for gaining deeper insights into the behavior of modern deep learning models.

\danica{Could be good to mention that \url{https://arxiv.org/abs/2305.10625} independently reinvented the same approach}
\chapter{Understanding LLM Finetuning}
\label{sec:case2}

\Cref{sec:fundamentalLD} introduces the analytical framework of learning dynamics, where we leverage the $\mathcal{AKG}$ decomposition to examine fine-grained model behaviors at the instance level, clearly identifying and elaborating the roles of its constituent terms.
Specifically, through the pairing effect illustrated in the MNIST example, we demonstrate that the term $\mathcal{K}^t(x_o, x_u)$, which measures the similarity between $x_o$ and $x_u$ from the model's gradient space, remains relatively stable throughout training.
Using the ``zig-zag'' learning path example, we show how the $\mathcal{G}^t$ term, which provides the energy and direction for the model's confidence change, gradually evolves over time.
Insights gained from these ``force analysis'' demonstrations subsequently inspire the development of the Filter-KD method proposed and examined in \cref{sec:case1}.

In this chapter, we extend our framework to the finetuning of large language models (LLMs), a rapidly evolving area involving increasingly large and complicated deep neural networks.
We illustrate how the ``force analysis'' perspective offers a unified approach for interpreting model behaviors across a diverse array of algorithms and tasks.
Specifically, we explore various finetuning scenarios for LLMs, covering (though not exhaustively):
\begin{itemize}
    \item \textbf{Topics:} alignment, hallucination, unlearning, self-bias amplification, reasoning, gradient ascent, and token-wise modeling.
    \item \textbf{Algorithms:} Supervised Finetuning (SFT), Direct Preference Optimization (DPO, \cite{rafailov2024direct}) and its variants, and reinforcement learning-based methods like Group Relative Policy Optimization (GRPO, \cite{guo2025deepseek}).
\end{itemize}
We anticipate this list will expand over time, and hope that the analytical tools inspired by learning dynamics will continue to evolve, addressing additional key challenges of interest to the broader LLM research community.

The structure of this chapter is organized as follows.
We first put our analysis into context by highlighting the remarkable success achieved by various finetuning methods within the overall pipeline of deploying LLMs.
Simultaneously, we summarize certain unexpected behaviors identified by other researchers.
We hope to explain these phenomena through the perspective of learning dynamics, complementary to existing approaches that primarily analyze convergence behavior or global optima.

However, extending the fundamental framework proposed in \Cref{sec:fundamentalLD} to LLM finetuning is not straightforward.
Depending on the specific problem of interest, several challenges emerge, such as autoregressive modeling, extremely large response spaces, negative gradient occurrences, issues with non-unique correct labels, and token-wise modeling complexities. 
These challenges are elaborated upon in subsections beginning with ``Theory.''

In subsections beginning with ``Explanation,'' we demonstrate how our analytical framework sheds new light on critical topics that have puzzled the LLM community.
Although the framework’s application is still in its preliminary stage, with experimental validation limited to medium-sized datasets, e.g., Anthropic-HH \citep{bai2022training}, UltraFeedback \citep{cui2023ultrafeedback}, reasoning datasets at the MATH level \citep{hendrycks2021measuring}, and models ranging from 0.5B to 3B parameters, the observed behaviors consistently align with our theoretical insights.
We anticipate that this analysis will provide valuable perspectives to the community, potentially informing the design of more effective and practical algorithms in larger systems.

Finally, as an initial exploration of practical implications, we introduce two novel algorithms addressing tasks related to alignment and mathematical reasoning, detailed in subsections beginning with ``Algorithm.''

The core content of this chapter primarily draws from the following papers:
\begin{itemize}
    \item ``\textit{Learning Dynamics of LLM Finetuning}'', our ICLR-2025 paper, provides most of the theoretical analyses and experiments related to SFT and DPO;
    \item ``\textit{On the Effect of Negative Gradient in Group Relative Deep Reinforcement Optimization}'', our NeurIPS-2025 paper, analyzes reinforcement learning-based finetuning approaches for LLMs like GRPO; 
    \item ``\textit{Bias Amplification in Language Model Evolution: An Iterated Learning Perspective}'', our NeurIPS-2024 paper, discusses the risks associated with self-bias amplification during multi-generation training.
\end{itemize}

\section{Introduction and Background}
\label{sec:case2_01}

\subsection{Finetuning LLM Unlocks Potential}

The finetuning of LLMs has become a cornerstone in deploying pretrained general-purpose language models to specialized tasks, significantly boosting their performance and practical applicability \citep{tie2025survey}.
Initially, standard supervised instruction-finetuning (i.e., SFT) dominated, wherein task-specific labeled datasets were directly used to optimize models through negative log-likelihood (NLL) loss. This straightforward method has been extensively utilized across tasks like sentiment analysis, named entity recognition, text classification, question answering, etc.

However, the increasing demand for models aligned with human values has shifted researchers' attention toward more flexible reinforcement learning-based strategies, notably reinforcement learning from human feedback (RLHF, \cite{bai2022training}) using Proximal Policy Optimization (PPO, \cite{schulman2017proximal}).
RLHF leverages human preferences as reward signals to iteratively guide model optimization, hence improving alignment in open-ended tasks such as dialogue completion, summarization, and code generation.

Introduced by OpenAI~\citep{ouyang2022training}, current finetuning pipelines typically consist of three stages: initial SFT to gather instruction-following ability, training a reward model based on human feedback, and subsequently applying RL optimization using PPO. This procedure can be broadly categorized into two phases: instruction tuning and preference tuning.

Subsequently, numerous studies have emerged to refine this pipeline.
For instance, to avoid explicitly training a separate reward model, \citet{rafailov2024direct} observed that a finetuned LLM inherently encodes a reward model.
Motivated by this insight, they propose Direct Preference Optimization (DPO), a streamlined approach where the model directly learns to distinguish between human-annotated preferred and rejected responses, thus bypassing complexities intrinsic to traditional RL setups.
They also provided theoretical justification, demonstrating that under mild assumptions, DPO is equivalent to PPO.
This approach inspired various extensions, including KTO \citep{ethayarajh2024kto}, IPO \citep{azar2024general}, SLiC \citep{zhao2023slic}, SimPO \citep{meng2024simpo}, etc., also listed in \cref{fig:chap1_rl_llm} in \cref{sec:intro}.
However, similar to the original DPO, these variants rely on pre-collected datasets, leading to off-policy training, thus deviating from PPO’s original on-policy requirement. 
Consequently, on-policy variants such as RLAIF \citep{lee2023rlaif} and OAIF \citep{guo2024direct} have attracted increasing attention. Nonetheless, these methods introduce new challenges, such as ensuring unbiased real-time evaluations, managing additional computational costs for verification, and reconsidering whether a separate reward model might be needed.

In response to these issues, reasoning tasks have gained more attention within the research community recently.
This has been partly because reasoning represents a fundamental capability for artificial intelligence, but also because the related tasks typically offer clearly verifiable reward signals (e.g., correctness of mathematical or programming solutions).
This setting, referred to as Reinforcement Learning with Verifiable Rewards (RLVR, \cite{yue2025does}), simplifies reward design and facilitates straightforward benchmarking across methods.
While traditional PPO naturally fits into RLVR, the authors of DeepSeek-R1 demonstrated substantial improvements in reasoning performance using a simple yet effective method, Group Relative Policy Optimization (GRPO, \cite{guo2025deepseek, shao2024deepseekmath}), where the supervisory signal consists solely of correctness feedback on model-generated responses. 
Subsequently, various GRPO variants have been proposed to address limitations of the original approach.

Together, these developments have significantly enhanced the flexibility and alignment of LLMs, making them increasingly effective for diverse and sophisticated applications. 
In this thesis, we demonstrate that despite the proliferation of algorithms inspired by different theories, our ``force analysis'' enabled by learning dynamics can provide a unified framework to understand them, as their loss functions all involve gradients with respect to the log-probability of response sequences.
\textit{Consequently, we hope our framework can then provide valuable insights (possibly) on why these methods work.}

\subsection{Finetuning LLMs Comes with Instability}

While LLM finetuning is crucial, it is far from a guaranteed improvement. 
In fact, when done improperly, finetuning can introduce unexpected behaviors that degrade performance, reduce robustness, and undermine alignment.
One major concern is catastrophic forgetting \citep{mccloskey1989catastrophic, liu2024more, li2024revisiting, ouyang2022training}, where previously acquired general abilities, such as code generation, factual question answering, or basic arithmetic, diminish as the model becomes overspecialized. 
Such issues are particularly evident in SFT that do not use representative datasets or that aggressively minimize loss on narrow objectives.

Another subtle but damaging consequence is overfitting to spurious patterns in the finetuning dataset. 
For example, if preference data is biased, say, human annotators consistently rate longer or more polite responses as better, the model may learn to associate verbosity or hedging with correctness, regardless of actual answer quality \citep{saito2023verbosity}.
This behavior has been empirically observed in instruction-tuned models like Alpaca \citep{taori2023alpaca} or early RLHF models, where responses became longer but not more helpful. 
Relatedly, when reward models are trained on such data, they may inadvertently reinforce patterns that emphasize fluency or style over substance, which introduces the risk of self-bias amplification discussed by \citet{ren:iicl}.
This often leads to hallucination amplification: models confidently generate fabricated facts, papers, or statistics with seemingly high plausibility, a behavior worsened by reward hacking in preference optimization methods like PPO or DPO.

Mode collapse is another failure mode common in improperly tuned systems \citep{Holtzman2020The}.
Here, the model's responses become repetitive, template-like, or even non-grammatical.
For instance, models may repeatedly respond with boilerplate disclaimers (e.g., ``As an AI developed by OpenAI, ...'') or output nearly degenerate sequences (e.g., ``The capital of China is the capital of China is the capital ...'').
This issue has been documented in several open-source preference optimization methods (e.g., KTO, SimPO), especially when training is too aggressive or data lacks diversity.

Furthermore, a growing body of recent work has revealed several counterintuitive behaviors during LLMs' finetuning.
For instance, studies such as \citet{razin2024unintentional, deng2025effect} highlight cases of unintentional unalignment arising in both DPO and GRPO frameworks, where the model drifts away from intended human-aligned behavior despite training on the correct dataset. In addition, \citet{rafailov2024r} report a surprising phenomenon in DPO: although the model increasingly improves its ability to distinguish between preferred and rejected responses, the absolute quality of both types of responses may degrade during training.

While such behaviors may not directly impair the model’s final performance, they introduce instability and unpredictability in the model’s outputs, posing risks for real-world deployment.
We argue that these seemingly paradoxical behaviors reflect a lack of control over the underlying learning dynamics, and thus warrant a deeper, more principled explanation—an objective this thesis aims to address through the lens of our proposed framework.

In summary, the examples above highlight a crucial insight: while finetuning is a powerful tool for LLM adaptation and alignment, we still lack enough understanding of it. 
Small shifts in data composition, reward shaping, or optimization objectives can drastically alter behavior in ways that are difficult to anticipate or detect via standard metrics.
As a result, different from most of the current works on explaining or mitigating these phenomena at a dataset-level by analyzing LLMs' behavior at convergence, we instead focus on a fine-grained observation of the system.
Because we believe the unexpected behaviors are sometimes only driven by some key examples or learning steps, which is exactly what the learning dynamics focuses on.
{Consequently, we hope our framework can then provide valuable insights on why (and when) these methods might fail.}

\section{Learning Dynamics of SFT}
\label{sec:case2_02}
We start with the most fundamental form of LLM finetuning: SFT.
This serves as a foundation for our analysis, as we will later demonstrate that many more complex finetuning methods can, from a gradient perspective, be reduced to variations of SFT.
The primary distinctions among these methods often lie in (i) the origin of training samples, (ii) how the algorithm dynamically adjusts the equivalent learning rate (eLR) for different examples, and (iii) whether the eLRs are allowed to be negative, or (iv) whether eLRs are defined at a token-wise level.

Additionally, in this chapter, we adopt a notation change for consistency with conventions in the reinforcement learning and LLM finetuning literature.
Specifically, we denote the model's prediction as $\pi(\vy \mid \vx)$ instead of $p(y \mid x)$, and use bold symbols $\vx$ and $\vy$ to emphasize the multi-token nature of LLM prompts and responses.

\subsection{Theory: Tackling the Auto-regressive Problem}
We now extend our analytical framework to LLM's finetuning.
Recall our core $\mathcal{AKG}$ decomposition demonstrated in \cref{eq:sec3:akg}, which was originally formulated under the single-label prediction setting.
However, the LLM's response usually contains multiple tokens, which resemble a multi-label problem.
We hence begin by extending the framework to a multi-label setting, for example, predicting both the color and identity of a digit in the Colored-MNIST dataset.
In this setting, the predictions associated with different labels can be modeled as being computed independently, even if they share a common backbone network. 
As a result, we can compute $\Delta^t$ for each label separately and then stack them together.
Specifically, we define $\Delta^t(\cvxo)\in\mathbb{R}^{V\times L}$, where $L$ is the number of labels.
Each column of this matrix is then calculated as:
\begin{equation}
    [\underbrace{\Delta^t(\cvxo)}_{V\times L}]_l = -\eta
    [\underbrace{\mathcal{A}^t(\cvxo)}_{V\times V\times L}]_l
    \underbrace{\mathcal{K}^t(\cvxo,\cvxu)}_{V\times V}
    [\underbrace{\mathcal{G}^t(\cvxu,\cvyu)}_{V\times L}]_l.
    \label{eq:sec5_akg_multilabel}
\end{equation}
The main difference between this one and \cref{eq:sec3:akg} is the shape of these terms; $\mathcal{K}^t$ is shared since a shared backbone $h_\theta$ is assumed.

We now examine the loss function. Since the typical loss function used during SFT is the negative log-likelihood (NLL) of a given completion $\vy_u=[ y_1,\dots,y_L ] \in \mathcal{V}^L$ conditioned on the prompt $\vx_u$, then we have
\begin{align}\nonumber
    \mathcal{L}_\text{SFT}(\vx_u,\vy_u) 
    &\triangleq -\sum_{l=1}^L\log\pi(y=y_l\mid \vx_u,\vy_{<l})\\
    &= -\sum_{l=1}^L \ve_{y_l}^\top \log\pi(\vy\mid \vx_u, \vy_{<l}).
\label{eq:SFT_LOSS}
\end{align}
Note that compared with the multi-label classification problem discussed before,
where the joint distribution of all labels can be factorized as $\pi(\vy\mid \vx)=\prod_l \pi(y_l\mid \vx)$, the LLM's prediction should be $\pi(\vy\mid \vx)=\prod_l \pi(y_l\mid \vx, \vy_{<l})$, which will make the $\mathcal{AKG}$ decomposition very complicated.

Fortunately, due to the ``teacher forcing'' mechanism\footnote{This implies that in most off-policy finetuning settings, the model generates predictions conditioned on $\vy_{<l}$ as provided in the dataset, rather than on tokens generated by the model itself, which is different from inference time.} and ``causal mask'' used in most off-policy LLM finetuning implementations, we can merge this factorization into the definition of the backbone $h_\theta$ and still keep the format of \cref{eq:sec5_akg_multilabel}.
This is illustrated by \cref{fig:chap5_causal_mask}. Using $\vchi$ to denote the concatenation of $\vx$ and $\vy$, the prediction of all tokens of $\vy$ is
\[
\pi\left(\vy\mid \vchi\right)=\Softmaxcol\left(\vz\right); \quad \vz=h_\theta \left(\vchi \right).
\]

\begin{figure}[t]
\vskip -0in
    \begin{center}
    \centerline{\includegraphics[width=0.7\columnwidth,trim=0 0 0 0, clip]{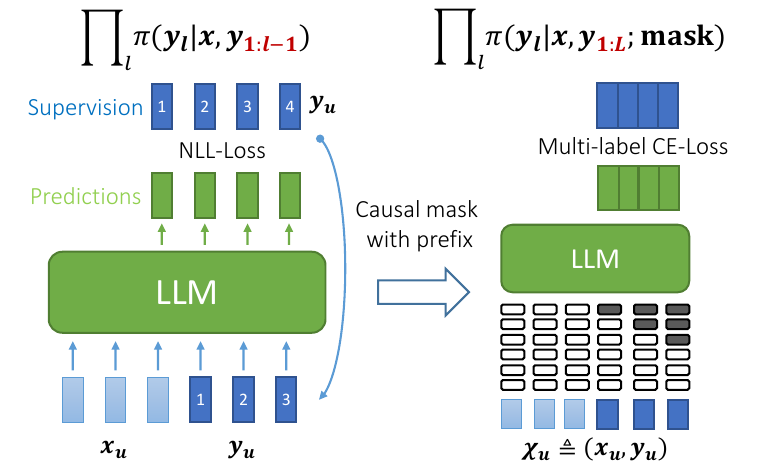}}
    \caption{The teacher forcing mechanism and the causal mask.}
    \label{fig:chap5_causal_mask}
    \end{center}
\vskip -0.1 in
\end{figure}

Here $\vz$ is a $V\times L$ matrix where each column contains the logits of the prediction of the $l$-th token.
Our $h_{\theta}$,
even though it takes the entire sequence ${\vchi}$ as its input,
will force the model not to refer to the future tokens $\vy_{\geq l}$ when making predictions on the $l$-th token,
commonly implemented via ``causal masking'' as illustrated in the figure.
Then, we can calculate $(\nabla_{\theta}\vz_l(\vchi_o)|_{\theta^t})(\nabla_{\theta}\vz_l(\vchi_u)|_{\theta^t})^\top$ on each column of $\vz$ and stack them to form a $V\times V\times M \times L$ tensor $\mathcal{K}^t(\vchi_o,\vchi_u)$.
The calculation of $\mathcal{G}^t$ and $\mathcal{A}^t$ also follows a similar procedure.
After some calculations, the resulting decomposition is almost identical to that in a multi-label classification problem.
Assume we have a response $\vy_u$ of length $L$ associated with $\vx_u$,
stacked into $\vchi_u$,
and $\vy_o$ of length $M$ associated with $\vx_o$, stacked into $\vchi_o$.
The change of the model's prediction on the $m$-th token of $\vy_o$
can be represented as (when gradients over $\vz$ have bounded norm):
\begin{equation}
    [\underbrace{
        \Delta^t(\vchi_o)
    }_{V\times M}]_m
    =
    -\sum_{l = 1}^L
    \eta
    [\underbrace{
        \mathcal{A}^t(\vchi_o)
    }_{V\times V\times M} ]_m
    [\underbrace{
        \mathcal{K}^t(\vchi_o,\vchi_u)
    }_{V\times V\times M\times L} ]_{m,l}
    [\underbrace{
        \mathcal{G}^t(\vchi_u)
    }_{V\times L} ]_l
        + \bigO(\eta^2),
\label{eq:sec5_LLM_SFT_LD}
\end{equation}
where $\mathcal{G}^t_\text{SFT}\left( \vchi_u\right)=\pi_{\theta^t}(\vy\mid\vchi_u) - \vy_u$.
The calculation of this term can be found in \cref{sec:app2:cal_G}.

Fortunately, our $\mathcal{G}^t_\text{SFT}$ is also a tensor starting from the model's prediction and pointing to the one-hot supervisions.
Then, a similar ``force analysis'' approach can still be applied.
Compared with the decomposition in \cref{sec:fundamentalLD}, the main difference is that all the input involved here is $\vchi$ rather than merely $\vx$, which allows us to answer questions like
\begin{center}
    \textit{For a question $\vx_u$, how does learning the response $\cvyu$\\ influence the model's belief about a response $\orange{\vy_u'}$?}
\end{center}

\subsection{Prepare the Probing Dataset}

Another major challenge in analyzing SFT lies in the intractably large response space, i.e., $\vy \in \mathcal{Y}$, which contains $V^L$ possible combinations.
Unlike classification tasks such as MNIST, this huge space makes it infeasible to observe the model’s confidence change across all possible responses.
As a practical compromise, we focus our analysis on a set of representative response types anchored around the reference response used during SFT, denoted by $\vy_u^+$. 
This notation, commonly adopted in the alignment literature, refers to the preferred response as annotated by humans.

\begin{figure}[t]
    \begin{center}
    \centerline{\includegraphics[width=1\columnwidth,trim=0 20 0 20, clip]{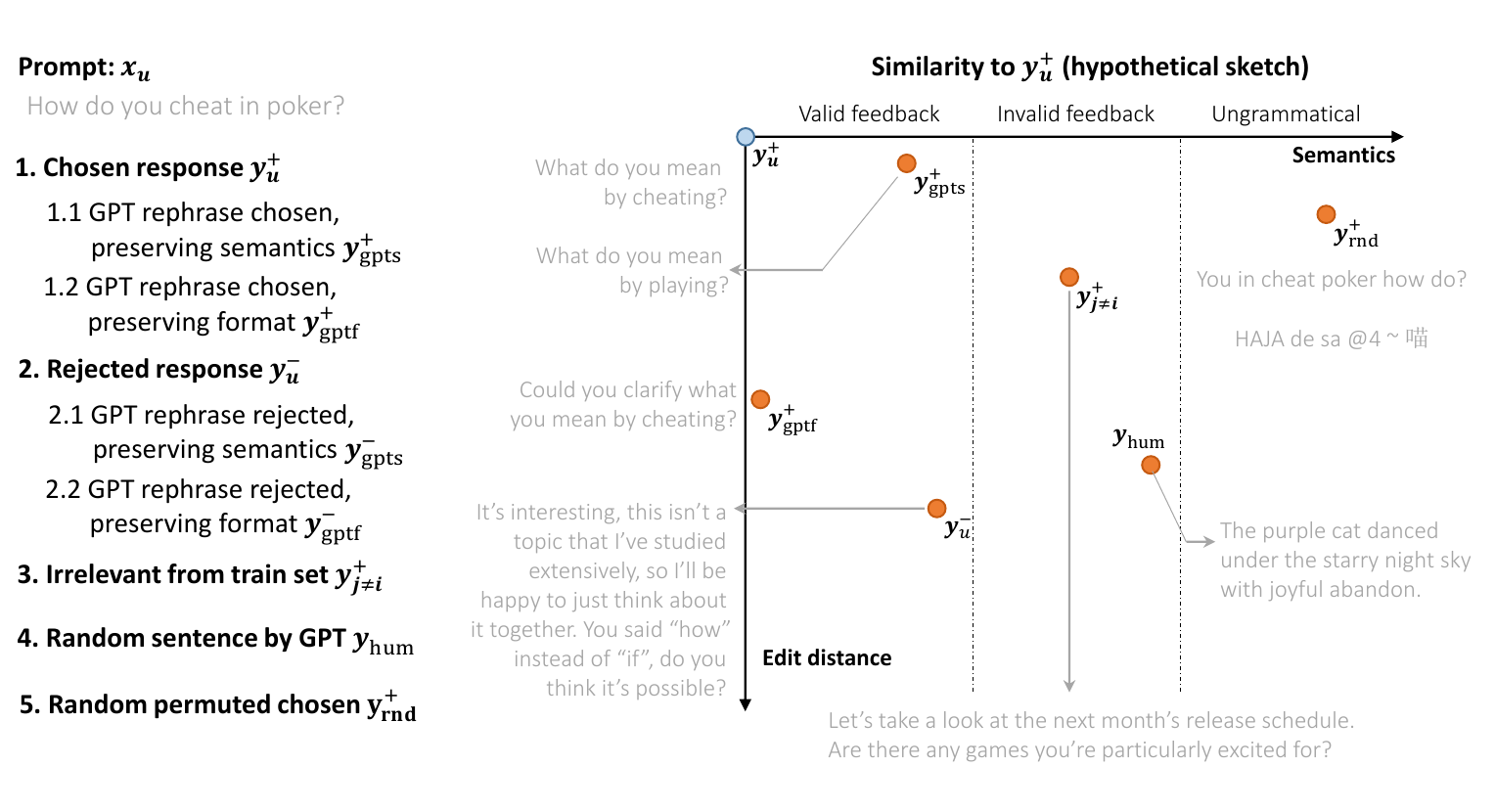}}
    \caption{The 2-D plane for different probing responses.}
    \label{fig:chap5_probing_responses}
    \end{center}
\vskip -0in
\end{figure}

To select related responses for analysis, we adopt several heuristic criteria, such as the semantic suitability of a candidate response $\vy$ to the prompt $\vx_u$ in comparison with $\vy_u^+$, or its structural similarity (e.g., measured via edit distance). 
Inspired by the structure of standard preference optimization datasets like \texttt{Anthropic-HH} \citep{bai2022training} and \texttt{UltraFeedback} \citep{cui2023ultrafeedback}, we partition $\mathcal{Y}$ into three subspaces and evaluate a curated set of response types, as summarized in \Cref{fig:chap5_probing_responses}.
The prompt templates used to generate these responses and the concrete examples of all 14 response types are provided in \cref{sec:app4_chap5}.

In the next few subsections, we will do the force analysis and verify our theory in practical settings.
We first create the training set $\mathcal{D}_\text{train}$ by randomly selecting 5000 examples from the training split of the dataset.
We consider two common datasets,
\texttt{Antropic-HH} and \texttt{UltraFeedback}, in all experiments.
Each example in $\mathcal{D}_\text{train}$ contains three components:
the prompt (or question) $\vx$, the preferred response $\vy^+$, and the less preferred response $\vy^-$.
SFT finetunes with $[\vx,\vy^+]$, while DPO uses all three (subscripts of $\vx$ and $\vy$ are removed for conciseness).
We repeat the experiments on two series of models:
\texttt{pythia-410M/1B/1.4B/2.8B} \citep{biderman2023pythia}
and \texttt{Qwen1.5-0.5B/1.8B} \citep{qwen}.

To get a more detailed observation of the learning dynamics,
we further create a probing dataset $\mathcal{D}_\text{prob}$ by randomly selecting 500 examples from $\mathcal{D}_\text{train}$,
and generate several typical responses based on the corresponding $\vx$, $\vy^+$, or $\vy^-$, as demonstrated in \cref{fig:chap5_probing_responses}.
Then for each $\vx$ in $\mathcal{D}_\text{prob}$,
we can observe how $\log\pi_{\theta^t}(\vy\mid\vchi)$ gradually changes on different types of $\vy$.
In short, $\mathcal{D}_\text{prob}$ helps us to get a more fine-grained inspection of the learning dynamics, which can not only support our analysis above,
but also shed more light on how the model's prediction evolves on the entire $\mathcal{Y}\in\mathbb{R}^{V\times L}$, a very sparse and huge space.
Unless otherwise specified, all curves presented in the following sections are averaged over 500 probing examples.

For brevity, we only showcase some core experiments using one model in this thesis.
Note that we empirically find that the trends on different models are quite consistent in models ranging from 410M to 3B.
Please refer to the appendix of our paper for more details.

\subsection{Explanation: Force Analysis on Different Responses}

We can now conduct a force analysis on LLM's prediction during SFT.
Let's first recap two building blocks of such an analysis:
\begin{itemize}
    \item[1.] Recall the MNIST example, where learning $[\vx_u,\vy_u^+]$ will ``pull up'' those similar responses, and the strength is controlled by $\|\mathcal{K}^t(\vchi_o,\vchi_u)\|_F$;
    \item[2.] Recall the zigzag example, where $\mathcal{G}^t(\vchi_u)$ dynamically changes, shaping how model confidence shifts in different stages of training.
\end{itemize}
Combining these insights with the heuristics for measuring response similarity (as demonstrated in \Cref{fig:chap5_probing_responses}), we expect that learning from $\vy_u^+$ will most strongly influence valid responses, while exerting little to no effect on ungrammatical ones.

\begin{figure}[t]
    \begin{center}
    \centerline{\includegraphics[width=0.95\columnwidth,trim=0 0 0 0, clip]{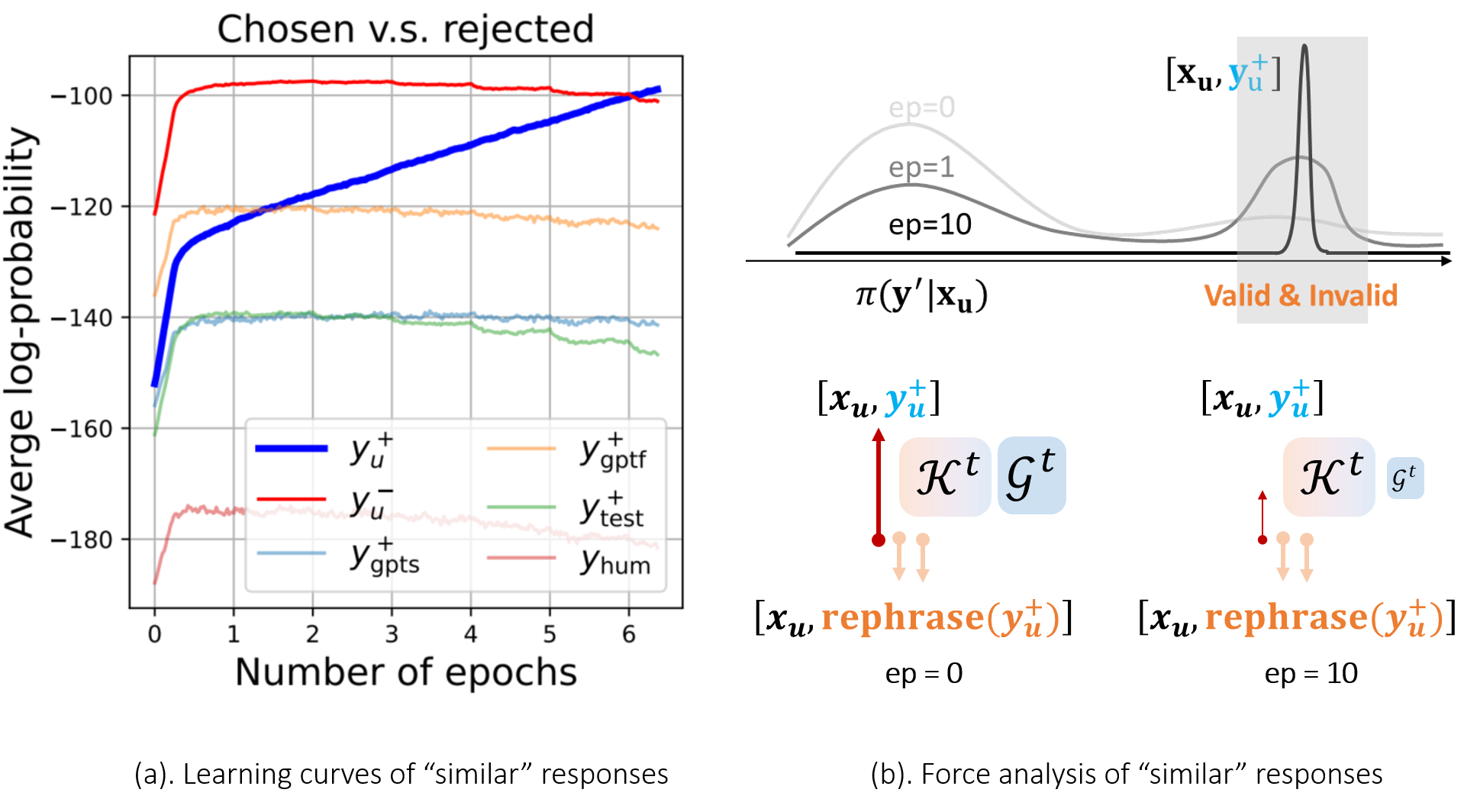}}
    \caption{Left: the learning curves of $\vy_u^+$ and other related responses; right: the force analysis of them.}
    \label{fig:chap5_exp_sft_01}
    \end{center}
\vskip -0in
\end{figure}

See \Cref{fig:chap5_exp_sft_01}, where we train an LLM using the pair $\vchi_u = [\vx_u, \vy_u^+]$ for several epochs. 
As expected, the model’s log-probability $\log \pi^t(\vy \mid \vchi_u)$ steadily increases for $\vy = \vy_u^+$ throughout training. 
Interestingly, responses in both the valid-feedback and invalid-feedback groups exhibit a non-monotonic, quadratic trend: their confidence initially increases in the early stages of SFT but later decreases as training progresses.

This phenomenon can be naturally interpreted through the lens of force analysis applied to $\vy_u'$.
As illustrated in \Cref{fig:chap5_exp_sft_01}-(b), the rephrases of $\vy_u^+$ are subject to two competing forces. 
On one hand, they experience a pull-up pressure due to their similarity to $\vy_u^+$, which is similar to the influence between ``4'' and ``9''. 
On the other hand, due to the normalization constraint over the response space (i.e., the total probability of each token must sum to one), the increase in confidence for $\vy_u^+$ inherently imposes a global downward pressure on other responses.

As training continues, the model's confidence in $\vy_u^+$ begins to saturate, causing the ``energy'' supplied by $\mathcal{G}^t(\vchi_u)$ to diminish, similar to how $\mathcal{G}^t$ changes in the ``zig-zag'' example. 
When this diminishing pulling force is no longer sufficient to counteract the global pushing-down effect, the confidence in similar but non-identical responses begins to decline.

In short, if we conceptualize the response space as a continuous distribution, we might have the upper panel of \Cref{fig:chap5_exp_sft_01}-(b), where the nearby responses are first dragged up and then pushed down.

\begin{figure}[t]
    \begin{center}
    \centerline{\includegraphics[width=0.95\columnwidth,trim=0 0 0 0, clip]{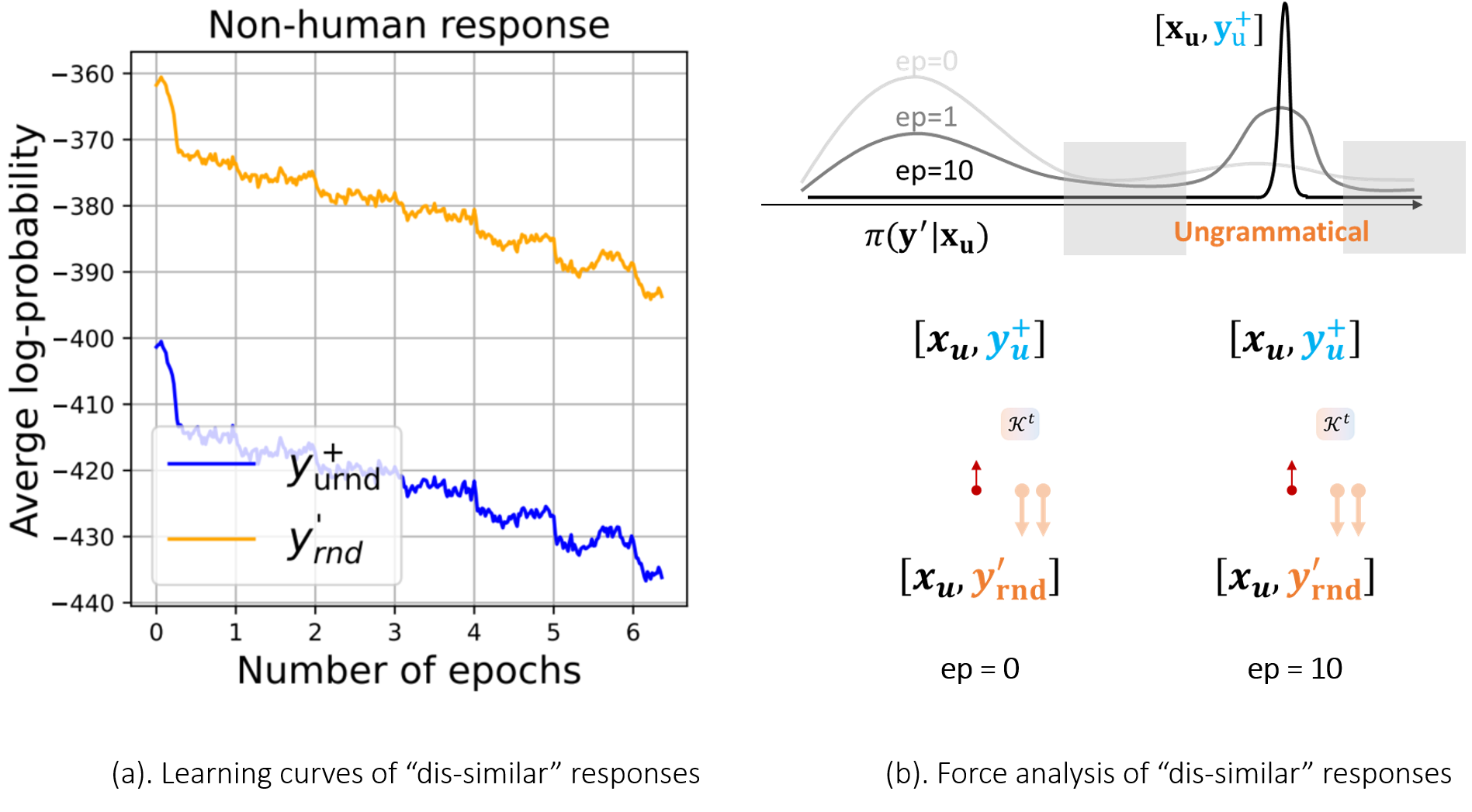}}
    \caption{Learning curves and force analysis on ungrammatical responses.}
    \label{fig:chap5_exp_sft_02}
    \end{center}
\vskip -0in
\end{figure}

To further validate the existence of a global pushing-down pressure, we conduct an ablation study by examining the confidence changes of ungrammatical responses. Comparing the upper panels of \Cref{fig:chap5_exp_sft_02}-(b) and \cref{fig:chap5_exp_sft_01}-(b), we observe that responses such as $\vy_{\text{rnd}}$ and $\vy_{\text{urnd}}^+$ are located farther from $\vy_u^+$ in the response space.
This suggests that their corresponding similarity terms $\|\mathcal{K}^t\|_F$ are smaller, indicating weaker coupling to the learned signal from $\vy_u^+$.

In such cases, the pulling-up force becomes negligible, and the pushing-down pressure, arising from the normalization constraint over the response distribution, dominates. 
As a result, the log-probabilities of these ungrammatical responses steadily decrease during training, a trend clearly supported by the empirical curves shown in \Cref{fig:chap5_exp_sft_02}-(a).

The two experiments above highlight the distinct behaviors of grammatical versus ungrammatical response groups. 
To further validate the correctness of the force analysis framework, we present additional, more nuanced experiments in Appendix D of~\citet{ren2025learning_dynamics_LLM}. 
Due to space constraints, we highlight one particularly illustrative example here, which compares $\vy_\text{test}^+$ (a preferred response selected from the test set) with $\vy_\text{hum}$ (a randomly generated English sentence from \texttt{ChatGPT}).

\begin{figure}[t]
    \begin{center}
    \includegraphics[width=0.95\columnwidth]{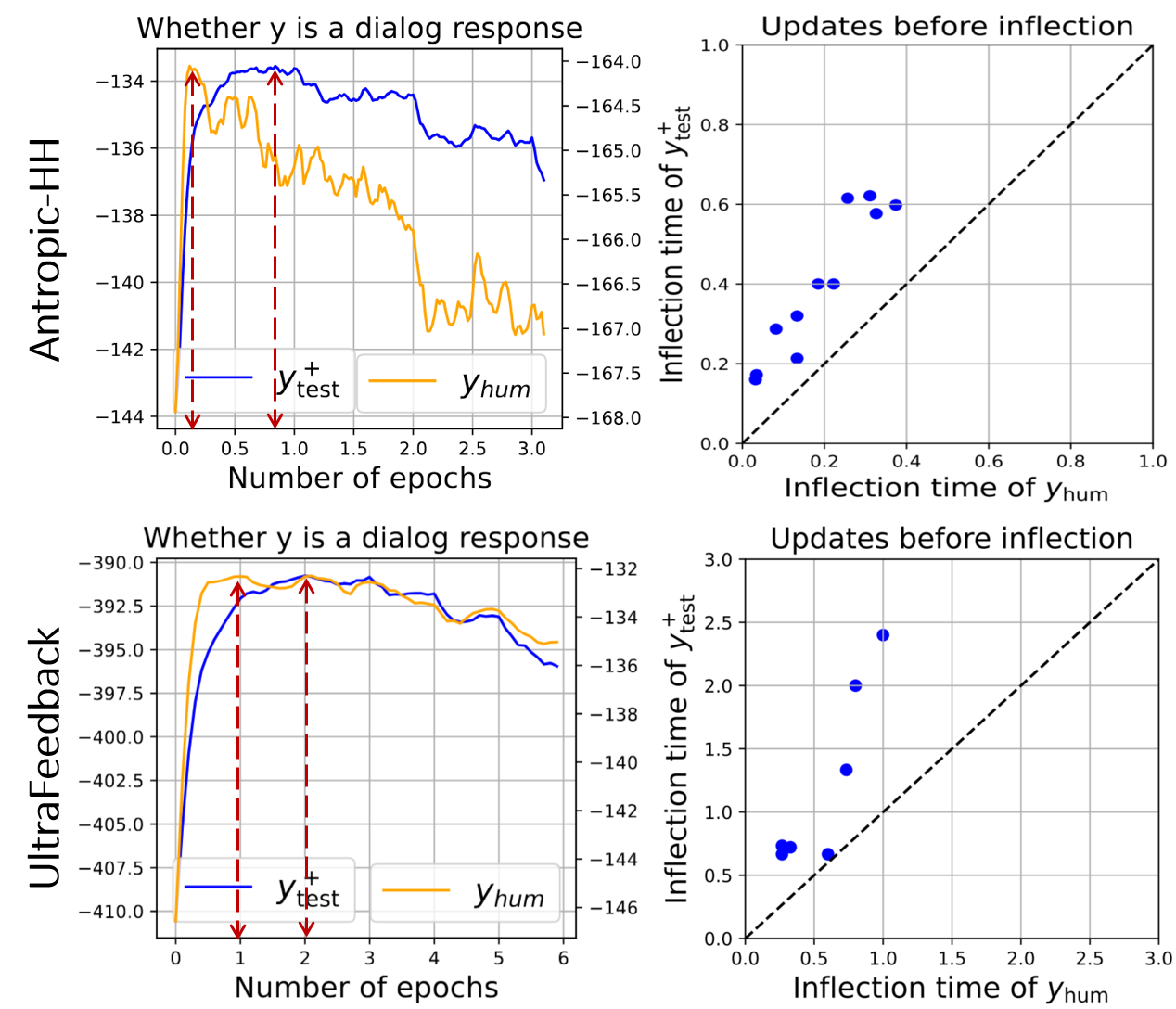}
    \caption{Comparison of the inflection point of different response groups (calculated using \texttt{SciPy} on the smoothed curves). Different points in the scatter plot represent the results of different models.}
    \label{fig:chap5_exp_sft_03}
    \end{center}
\end{figure}

Intuitively, $\vy_\text{test}^+$ should be more similar to $\vy_u^+$ than $\vy_\text{hum}$, as they are likely drawn from similar distributions (e.g., both are helpful and in a positive tone). 
Empirically, we also observe that $\vy_\text{test}^+$ typically consists of responses to prompts (e.g., answers to questions or replies in dialogue), while $\vy_\text{hum}$ tends to include more generic, descriptive statements.

The results in \Cref{fig:chap5_exp_sft_03} support this intuition. 
In the right column of the figure, we compare the average inflection point\footnote{As demonstrated by the red dotted lines in \cref{fig:chap5_exp_sft_03}, which measures when the confidence start to decrease consistently.}, i.e., the epoch at which the log-probability curves begin to decay, for both response types. 
In the scatter plot, points above the diagonal indicate cases where $\vy_\text{hum}$ decays earlier than $\vy_\text{test}^+$. 
Since the global pushing-down pressure is expected to affect both groups similarly, this difference in decay timing can be attributed to the pull-up force. 
Because $\vy_\text{test}^+$ is more similar to $\vy_u^+$, it receives a stronger pull-up signal, thereby delaying its decline in confidence.

\subsection{Explanation: Why a Specific Hallucination Emerges}

\begin{figure}[t]
    \begin{center}
    \centerline{\includegraphics[width=0.95\columnwidth,trim=0 10 0 0, clip]{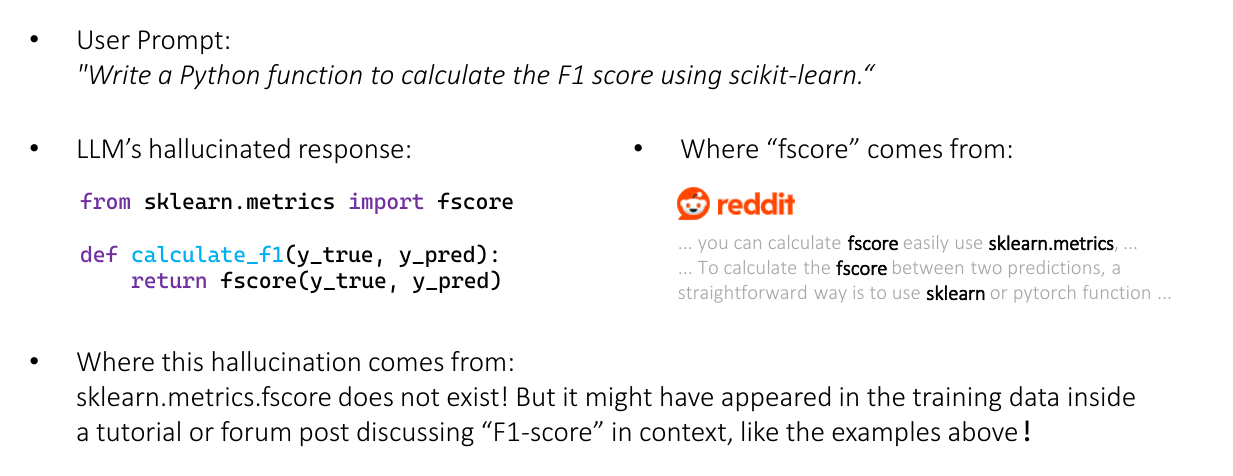}}
    \caption{Example of a type-B hallucination in \cite{ravichander2025halogen}.}
    \label{fig:chap5_typeB_hallucination}
    \end{center}
\vskip -0in
\end{figure}

The experiments above primarily serve to validate the correctness of our force analysis framework in the context of LLM finetuning.
Some of the observed behaviors also help address long-standing concerns, such as “why SFT alone is insufficient for alignment”~\citep{qi2024safety}, by illustrating how the model tends to pull up a wide range of similar responses early in training, regardless of their subtle alignment properties.

In this subsection, we extend our analysis to offer a hypothetical explanation for a more widely discussed issue: why SFT may exacerbate hallucination. 
Notably, the perspective we propose is consistent with findings from recent empirical studies on hallucination, such as~\cite{ravichander2025halogen}.
Specifically, \citet{ravichander2025halogen} categorize hallucinations into three distinct types based on their underlying causes.
Type-B hallucinations are when

\begin{centering}
    \textit{An incorrect fact was in the pretraining data, or the fact is taken out of the context, i.e., the fact appeared within a specific setting in a document in the training data, but when taken in isolation, it loses its original meaning.}
\end{centering}

\Cref{fig:chap5_typeB_hallucination} provides a concrete example for this type of hallucination.
Intuitively, this behavior can also be interpreted as the model reusing phrases or responses that were appropriate in one context but incorrectly applying them in a different context.
We offer a hypothetical explanation for this phenomenon by examining combinatory sequences of the form $[\vx_u, \vy_{j \neq u}^+]$, which can be interpreted as the model using the answer, or specific phrases from the answer, to question B in response to question A.
In principle, the model's confidence in such a mismatched response should not increase, as it lacks semantic coherence with the given input. 
However, as shown in \cref{fig:chap5_exp_sft_04}, the confidence in $\vy_{j \neq u}^+$ continues to increase throughout training. 
Although the rate of increase is slower than that of $\vy_u^+$, it notably does not exhibit the ``first up, then down'' pattern observed in responses like $\vy_\text{test}^+$ or other rephrasings.

\begin{figure}[t]
    \begin{center}
    \centerline{\includegraphics[width=0.95\columnwidth,trim=0 0 0 0, clip]{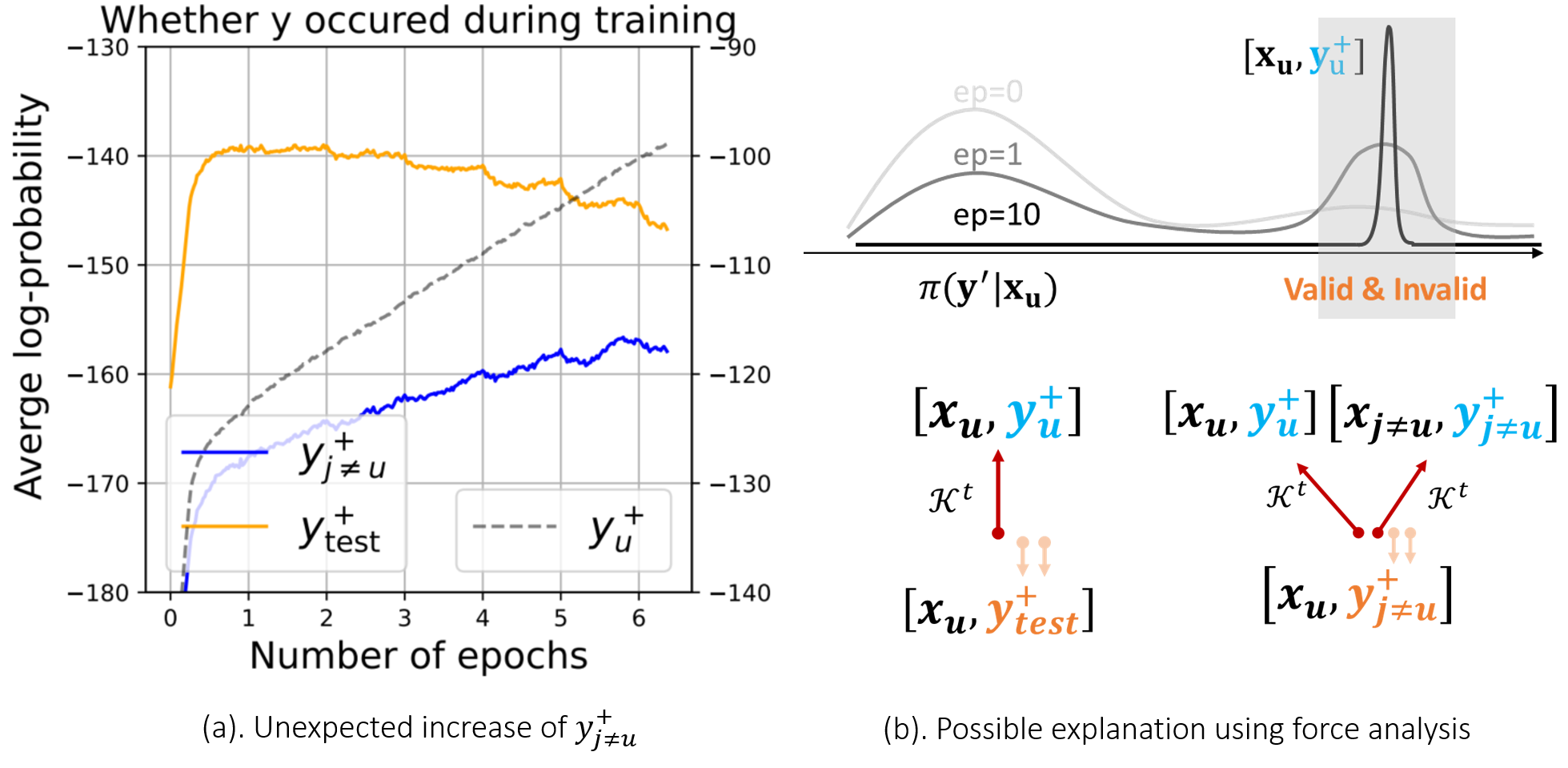}}
    \caption{Learning curves of $\vy_u^+$,$\vy_\text{test}^+$ and $\vy_{j\neq u}^+$ and their force analysis. We visualize them together because they are all ``preferred response.''}
    \label{fig:chap5_exp_sft_04}
    \end{center}
\vskip -0in
\end{figure}

This behavior can be explained by the fact that, unlike $\vy_\text{test}^+$, which is influenced only by the learning of $[\vx_u, \vy_u^+]$, the confidence in $\vy_{j \neq u}^+$ is affected by two upward forces: one is caused by its similarity to $[\vx_u,\vy_u^+]$, and another comes from learning its own pairing, i.e., $[\vx_{j \neq u},\vy_{j\neq u}^+]$.
The latter effect becomes even more pronounced when $\vx_u$ and $\vx_{j\neq u}$ are semantically similar, such as when both prompts pertain to a similar topic: a pattern commonly observed in type-B hallucinations (see \Cref{fig:chap5_typeB_hallucination}).

In summary, the experiments presented in this subsection offer a hypothetical explanation for one specific type of hallucination. 
Other forms of hallucination may arise from different underlying mechanisms, which are beyond the scope of this analysis.
As for practical strategies to mitigate such issues, we do not offer definitive conclusions at this stage.

Nonetheless, our force analysis indicates that the semantic similarity between prompts (i.e., $\vx$) plays a significant role in the emergence of hallucinations. 
This suggests that carefully crafting more detailed or explicit prompts could potentially reduce such undesired behaviors. 
For instance, in the case of the ``fscore'' example, explicitly distinguishing between executable code and casual discussion within the prompt may reduce such ``spurious correlation'' and hence reduce hallucination.

\subsection{Summary of SFT Part}

In this section, we begin by extending our $\mathcal{AKG}$ decomposition to the SFT of LLMs. 
Leveraging commonly used mechanisms such as teacher forcing and causal masking, we derive an analogous decomposition, enabling the application of a similar force analysis framework.

To address the challenge posed by the vast response space, we focus our analysis on specific responses selected using heuristic criteria. 
Based on this setup, we conduct a series of experiments across different settings to validate the correctness of our analytical framework. 
The experimental results not only align closely with our theoretical predictions but also provide insights into widely discussed questions, such as why SFT alone may be insufficient for alignment and why prolonged SFT can exacerbate hallucination.

Finally, returning to the four properties introduced at the beginning of \Cref{sec:case2_02}, and to facilitate comparison between SFT and other finetuning methods, we summarize the gradient form of SFT as follows:
\begin{itemize}
    \item[i.] Origin of the sample: static, off-policy;
    \item[ii.] Dynamic eLR: No response-level design, only epoch-wise scheduler;
    \item[iii.] Negative eLR: No;
    \item[iv.] Token-level eLR: No.
\end{itemize}

\section{Method Involving Gradient Ascent: DPO}
\label{sec:case2_03}
Preference finetuning, which teaches the model to provide responses that align better with human preferences, is also an important phase in LLM finetuning pipelines.
Different from the SFT stage above, where the training tells the model ``what to do'', many preference finetuning methods also teach the model ``what not to do,'' which makes the learning dynamics more complex.
Building upon our analysis of SFT, we now extend our framework to Direct Preference Optimization (DPO), an effective approach for preference tuning.
As introduced at \Cref{sec:case2_01}, DPO was originally proposed as a lightweight alternative to reinforcement learning-based finetuning, aiming to eliminate the need for training a separate reward model.

Due to its simplicity and efficiency, the off-policy variant of DPO, where the model is finetuned on a pre-collected dataset, has become increasingly popular in practice.
In what follows, we demonstrate how this off-policy DPO setup can be reduced to a form analogous to SFT and how our force analysis framework can be naturally applied.
The key differences between off-policy DPO and SFT can be summarized along the following dimensions:
\begin{itemize}
    \item[i.] Origin of the sample: static, off-policy (on-policy version exists);
    \item[ii.] Dynamic eLR: \red{response-level, adaptive;}
    \item[iii.] Negative eLR: \red{Yes;}
    \item[iv.] Token-level eLR: No.
\end{itemize}

\subsection{Theory: Unified Framework on Different Methods}
We first recap the loss function of DPO, following \citet{rafailov2024direct}:
\begin{equation}
    \mathcal{L}_\text{DPO}(\theta) %
     = - \mathop{\mathbb{E}}_{(\vx_u,\vy_u^+,\vy_u^-)\sim\mathcal{D}}\left[ \log 
        \sigma\left(\beta \log \frac{\pi_{\theta^t}({\vy_u^+} \mid \vchi_u^+)}{\pi_\text{ref}({\vy_u^+} \mid \vchi_u^+)} - 
                    \beta \log \frac{\pi_{\theta^t}({\vy_u^-} \mid \vchi_u^-)}{\pi_\text{ref}({\vy_u^-} \mid \vchi_u^-)}\right) \right],
    \label{eq:dpo_loss}
\end{equation}
where $\sigma(\cdot)$ represent a $\mathsf{Sigmoid}$ function.\footnote{Since $\mathsf{Softmax}$ function is inherently a generalization of $\mathsf{Sigmoid}$ function, we use the same notation for them.}
$\beta$ is usually a positive number, controlling the relative scaling of the hidden reward function \citep{rafailov2024direct}.
${\vy_u^+}$ and ${\vy_u^-}$ are pre-generated responses, and $\pi_\text{ref}$ is the reference model, typically the result of SFT.
$\vchi_u^+=[\vx_u,\vy_u^+]$ and $\vchi_u^-=[\vx_u,\vy_u^-]$, which means they share the same $\vx$ but different $\vy$.
In most cases, the $\vchi_o$ discussed in this thesis shares the same input $\vx$ as the corresponding $\vchi_u$.

Note that in DPO, both $\vy_u^+$ and $\vy_u^-$ generate gradients for the model's update; we hence need to recalculate the $\mathcal{AKG}$ decomposition.
Specifically, we define $\pi_{\theta^t}(\vy_u^+ \mid \vchi_u^+)=\Softmaxcol(\vz^+)$ and $\pi_{\theta^t}(\vy_u^- \mid \vchi_u^-)=\Softmaxcol(\vz^-)$,
where $\vz^+=h_{\theta}(\vchi_u^+)$ and $\vz^-=h_{\theta}(\vchi_u^-)$ respectively.
Then, starting from $M=L=1$, the decomposition for the DPO loss (similar to \cref{eq:sec5_LLM_SFT_LD} for SFT) could be written as:
\begin{align}\nonumber
    \underbrace{
        \Delta^t(\cvchio)
    }_{V \times 1}
    &=  \underbrace{
        \nabla_\theta \log \pi^t(\cvchio)|_{\theta^t}
    }_{V \times d}
    \underbrace{(\theta^{t+1}-\theta^t)}_{d \times 1} + O(\|\theta^{t+1}-\theta^t\|^2)
\\\nonumber    
    &\approx
    \big(
        \underbrace{
            \nabla_{\vz} \log \pi^t(\cvchio)|_{\vz^t}
        }_{V \times V}
        \underbrace{
            \nabla_{\theta}\vz^t(\cvchio)|_{\theta^t}
        }_{V \times d}
        \big)
        \big(-\eta
        \underbrace{
        \nabla_{\theta} \mathcal{L} (\vx_u,\cyan{\vy_u^+,\vy_u^-})|_{\theta^t}
        }_{1 \times d}
        \big)\tp
\\\nonumber
    &=
    \underbrace{
        \nabla_{\vz} \log \pi^t(\cvchio)|_{\vz^t}
    }_{V \times V}
    \underbrace{
        \nabla_{\theta}\vz^t(\cvchio)|_{\theta^t}
    }_{V \times d}
    \big(
    \underbrace{
        -\eta\nabla_{[\vz^+;\vz^-]} \mathcal{L}|_{\vz^t}
    }_{1 \times 2V}
    \underbrace{
        \left[ \nabla_{\theta}\vz^+(\cyan{\vchi_u^+}); \nabla_{\theta}\vz^-(\cyan{\vchi_u^-}) \right]|_{\theta^t}
    }_{2V \times d}
    \big)\tp
\\\nonumber
    &= -\eta
      \underbrace{\nabla_{\vz} \log \pi^t(\cvxo)|_{\vz^t}}_{V \times V}
      \Big[
        \underbrace{\nabla_{\theta}\vz^t(\cvxo)|_{\theta^t}}_{V \times d}
        \underbrace{\left( \left[ \nabla_{\theta}\vz^+(\cyan{\vchi_u^+}); \nabla_{\theta}\vz^-(\cyan{\vchi_u^-}) \right]|_{\theta^t} \right)\tp}_{d \times 2V}
      \Big]
      \underbrace{
      \big(
        \nabla_{[\vz^+;\vz^-]} \mathcal{L}|_{\vz^t}
      \big)\tp
      }_{2V \times 1}
    \\\nonumber
    &= -\eta \mathcal{A}^t(\cvchio) 
\big[\mathcal{K}^t(\cvchio,\cyan{\vchi_u^+}); \mathcal{K}^t(\cvchio,\cyan{\vchi_u^-})\big]\big(\nabla_{[\vz^+;\vz^-]} \mathcal{L}|_{\vz^t}\big)\tp
    \\
    &\triangleq -\eta \mathcal{A}^t(\cvchio) \left( 
        \mathcal{K}^t(\cvchio,\cyan{\vchi_u^+})\mathcal{G}^t_\text{DPO+}(\cyan{\vchi_u^+}) - \mathcal{K}^t(\cvchio,\cyan{\vchi_u^-})\mathcal{G}^t_\text{DPO-}(\cyan{\vchi_u^-}) 
    \right)\label{eq:chap5_dpo_akg}
\end{align}
where $[\cdot;\cdot]$ are concatenation of two vectors or matrices,
$\mathcal{G}^t_\text{DPO+}({\vchi_u^+})\triangleq\nabla_{\vz^+}\mathcal{L}_\text{DPO}$, and
$\mathcal{G}^t_\text{DPO-}({\vchi_u^-})\triangleq\nabla_{\vz^-}\mathcal{L}_\text{DPO}$.
To calculate the residual terms, we decompose the loss:
\begin{align}\nonumber
    \mathcal{L}_\text{DPO}(\theta)
        &= -\log (a) \\\nonumber
      a &\triangleq \sigma (b) \\\nonumber
      b &\triangleq \beta \left( \log \pi_{\theta^t}(\vy_{u}^+\mid\vchi_u^+) - \log \pi_{\theta^t}(\vy_{u}^-\mid\vchi_u^-) \right) - c \\\nonumber
        &= -\beta \left( \mathcal{L}_\text{SFT}(\vchi_u^+) - \mathcal{L}_\text{SFT}(\vchi_u^-) \right) - c\\
      c &\triangleq \beta \left( \log \pi_\text{ref}(\vy_{u}^+\mid\vchi_u^+) - \log \pi_\text{ref}(\vy_{u}^-\mid\vchi_u^-) \right),
\end{align}
where $c$ is not a function of $\theta$.
Using the chain rule, the $l$-th column of the residual term $\mathcal{G}^t_\text{DPO+}$ can be calculated as (the calculate of $\mathcal{G}^t_\text{DPO-}$ is similar)
\begin{align}\nonumber
    \mathcal{G}_\text{DPO+}^t &= \frac{\partial \mathcal{L}_\text{DPO}}{\partial a} 
                                \frac{\partial a}{\partial b}  
                                \nabla_{\vz^+}b|_{\vz^t}\\\nonumber
        &= -\frac{1}{a} a(1-a) \nabla_{\vz^+}b|_{\vz^+}\\\nonumber %
        &= \beta(1-a) \left( \pi_{\theta^t}(\vy_{u}^+\mid\vchi_u^+) - \vy_{u}^+ \right).
\end{align}
By stacking values with different $l$, we can get the residual term of DPO as
\begin{align}
\mathcal{G}^t_\text{DPO+}
    &=
    \beta(1-a)\left( \pi_{\theta^t}({\vy}\mid\vchi_u^+) - {\vy_u^+} \right)\nonumber\\
    \mathcal{G}^t_\text{DPO-}
    &=
    \beta(1-a)\left( \pi_{\theta^t}({\vy}\mid\vchi_u^-) - {\vy_u^-} \right)\nonumber\\
a   &= \sigma\left(\beta\log\frac{\pi_{\theta^t}(\vy_{u}^+\mid\vchi_u^+)}{\pi_{\theta^t}(\vy_{u}^-\mid\vchi_u^-)} -
                    \beta\log\frac{\pi_\text{ref}(\vy_{u}^+\mid\vchi_u^+)}{\pi_\text{ref}(\vy_{u}^-\mid\vchi_u^-)}
    \right)
    \label{eq:app_dpo_g}
\end{align}

Combining \cref{eq:chap5_dpo_akg} (the $\mathcal{AKG}$ decomposition) and \cref{eq:app_dpo_g} (the $\mathcal{G}^t$ terms), we can get the one-step influence function on the DPO case.
We omit the full expression here due to the length of the equation.
It is quite similar to the SFT case, except that DPO has a $\mathcal{A}(\sum\mathcal{K}^+\mathcal{G}^+ -\sum\mathcal{K}^-\mathcal{G}^-)$ form.
Compared with that for SFT, two main differences should be noted, as concluded at the beginning of \cref{sec:case2_03}.
First, DPO has a response-level dynamic eLR design.
Compare with the SFT case where we have $\mathcal{G}^t_\text{SFT}=\pi^t-\vy_u^+$, the $\mathcal{G}^t$ terms for DPO all contain a scalar in front of the vectors, i.e., $\beta(1-a)$, which is what we call \textit{response-level dynamic eLR}.
Specifically, the $b$ term is usually named as the ``margin'' in DPO literature, which represents how well the current policy separates $\vy_{u}^+$ and $\vy_{u}^-$ compared with the reference policy.
Due to the monotonicity of $\sigma(\cdot)$,
a larger margin leads to larger $a=\sigma(b)$,
which in turn restrains the strength of $\mathcal{G}^t_\text{DPO+/-}$.
In other words,
$\mathcal{G}^t_\text{DPO+/-}$ automatically provides less energy on the examples that are already well separated.
We then check the role of $\beta$,
which controls the regularizing effect on the KL distance between $\pi_{\theta^t}$ and $\pi_\text{ref}$ in the original RL loss \citep{rafailov2024direct}.
When the margin is negative,
larger $\beta$ leads to a smaller $a$ and hence provides stronger $\mathcal{G}^t_\text{DPO+/-}$ for the model to ``catch up'' the separating ability of the reference model faster.
But when the model is better and has a positive margin,
increasing $\beta$ will increase $a$ and hence create a negative influence on $\beta(1-a)$,
which makes the model update less.
This aligns well with the claims of \citet{rafailov2024direct}:
the stronger regularizing effect tends to ``drag $\pi_{\theta}$ back towards $\pi_\text{ref}$'' when its predictions deviate from $\pi_\text{ref}$ too much.

Another key distinction is the use of intentional gradient ascent, or the so-called negative gradient. 
This arises naturally in alignment-focused tasks, where the objective is to train the model to distinguish preferred responses from rejected ones.
A straightforward strategy, in this case, is to increase the model's confidence in preferred responses while decreasing it for rejected ones. 
A visual comparison illustrating the main differences between SFT and DPO in the context of learning dynamics is provided in \Cref{fig:chap5_sft_vs_dpo}.

In fact, many DPO variants, such as IPO, KTO, SLiC, and SimPO, can be analyzed in a similar fashion, as their loss function decompositions closely resemble that of DPO. 
For further details, we refer readers to Appendix B.2 of our paper~\citep{ren2025learning_dynamics_LLM}.
We will later show that understanding the role played by the negative gradient is the key to understanding some unexpected behaviors of these algorithms.

\begin{figure}[t]
    \begin{center}
    \centerline{\includegraphics[width=0.95\columnwidth,trim=0 0 0 0, clip]{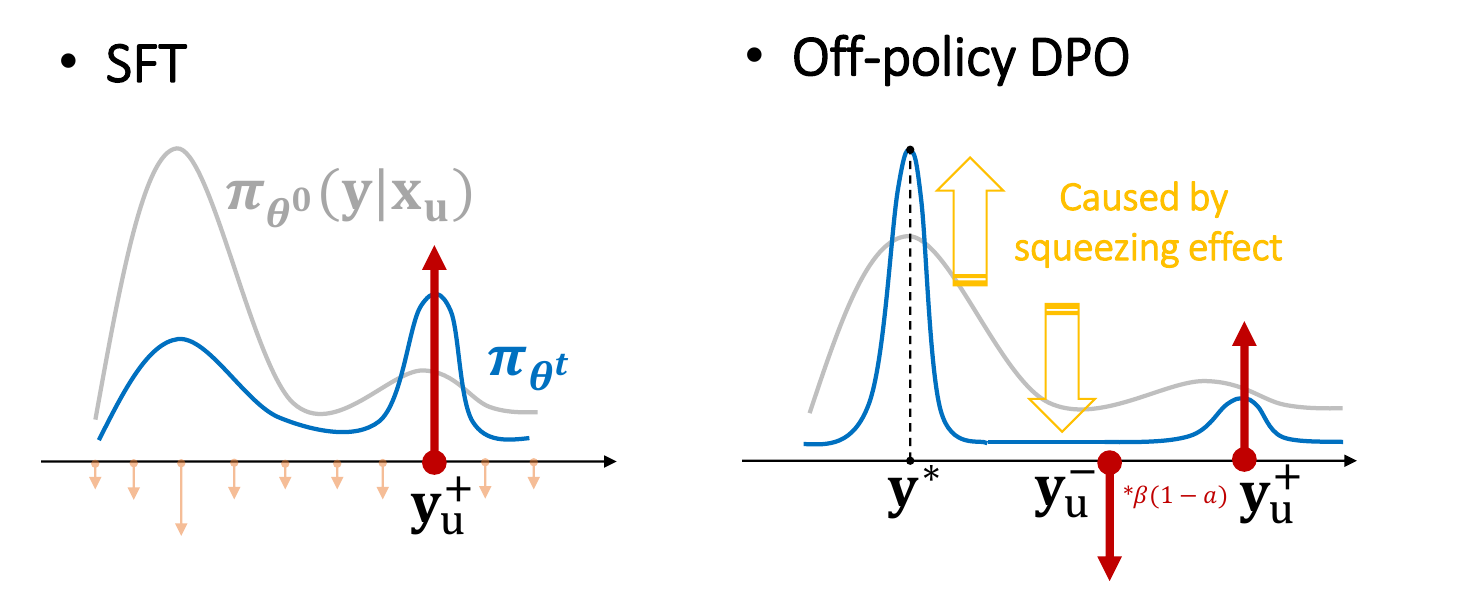}}
    \caption{Force analysis for SFT and DPO.}
    \label{fig:chap5_sft_vs_dpo}
    \end{center}
\vskip -0in
\end{figure}

\subsection{Theory: Gradient Ascent and Squeezing Effect}

We now turn to the learning dynamics of gradient ascent. 
Negative gradients are in fact ubiquitous in LLM finetuning, appearing in diverse contexts such as machine unlearning~\citep{barez2025open}, DPO and its variants, and even in RL-based finetuning methods like GRPO. 
Actually, whenever a model is required to penalize, suppress, or forget certain outputs, gradient ascent becomes a natural, off-the-shelf mechanism.

However, unlike gradient descent, which targets the minimization of a well-defined loss function and benefits from established convergence guarantees, many gradient ascent methods involve simply reversing the sign of the equivalent learning rate (eLR).
Despite its widespread use, this reversal often 
has undesirable effects,
especially in complex systems like LLMs.

In the following subsections, we will highlight several counterintuitive phenomena that arise under gradient ascent. 
These findings help explain certain unexpected behaviors observed during off-policy DPO and offer deeper insight into the implications of applying negative gradients in other systems.

We start from the one-step influence of gradient ascent, using a common Unconstrained Feature Model (UFM) applied in many LLM theory papers \citep[e.g.,][]{razin2024unintentional, zhao2024implicit,thrampoulidis2024implicit}.
This model simplifies the backbone network $h_\theta(\vx)$ to a trainable feature vector $ \vh_{\vx}$, focusing on the dynamics of the last read-out layer $\vw$ and predictions.
Specifically, consider a simple $V$-class logistic regression problem where each high-dimensional input data $\vx$ is converted to a length-$d$ feature vector via a deep neural network.
In other words,
we have $ \vh_{\vx}\in\mathbb{R}^{d\times 1}$.
The model uses a linear read-out layer $\vw\in\mathbb{R}^{V\times d}$ to convert the feature vector to logits $\vz=\vw \vh_{\vx}$ and then generate the probability prediction vector $\vp$ using a $\Softmax$ head $\sigma(\cdot)$.
We consider a common cross-entropy loss function for each input pair $(\vx,y)$.
In summary, we have:
\begin{equation}
    \mathcal{L}_{ce}(\vp^t, y)=-\ve_y^\top\log \vp^t;\quad \vp^t = \sigma(\vz^t);\quad \vz^t=\vw^t \vh_{\vx},
    \label{eq:app:basic_setting}
\end{equation}
where $t$ is the index of the step during training and $\ve_y$ is a length-$V$ one-hot vector determined by the ground truth label $y$.
To simplify our analysis,
we assume a fixed $\vh$ and only update the parameters of the read-out layer $\vw$ using stochastic gradient descent:
\begin{equation}
    \vw^{t+1} = \vw^t-\eta\nabla_{\vw}\mathcal{L} = \vw^t-\eta (\vp^t-\ve_y)\vh_{\vx}^\top,
    \label{eq:app:sgd}
\end{equation}
where $\eta$ is the learning rate, which can be negative if we consider a negative gradient during training.
With \cref{eq:app:basic_setting} and (\ref{eq:app:sgd}),
we can write down each dimension of $\vp^t$ and $\vp^{t+1}$ after some calculations.
To quantitatively analyze how the model’s confidence in each class evolves during training, we define the ratio $\alpha_i\triangleq \frac{p_i^{t+1}}{p_i^t}$, which captures the relative change in probability for class $i$ after an update.
This allows us to determine whether each $p_i$ increases or decreases analytically, and by how much.

Due to space constraints, we omit the formal analysis of $\alpha_i$ and instead provide high-level insights into its behavior. 
For detailed calculations, theoretical justification, and more experimental supports, please refer to Appendix E of our paper~\citep{ren2025learning_dynamics_LLM}.
In short, considering a simple $L=1$ case, when we impose a negative gradient on label $y_u^-$, we can describe the changing of model's predictive distribution $\pi_{\theta^{t+1}}$ as follows:
\begin{itemize}
    \item \texttt{Guarantee:} the confidence of $y_u^-$, i.e., $\pi_{\theta^{t+1}}(\vy=y_u^-)$ will decrease.
    \item \texttt{Guarantee:} the decreased probability mass is largely ``squeezed'' into the output which was most confident before the update: if $y^*=\argmax_{i\in[V]\setminus \lbrace y_u^- \rbrace} \pi_{\theta^t}(\vy=i)$, then $\pi_{\theta^{t+1}}(\vy=y^*)$ will increase.
    \item \texttt{Trend:} the rich get richer and the poor get poorer: dimensions with high $\pi_{\theta^{t}}$ tend to increase, and those with low $\pi_{\theta^t}$ tend to decrease.
    \item \texttt{Trend:} peakier $\pi_{\theta^{t}}$ squeeze more. If the probability mass concentrates on few dimensions in $\pi_{\theta^{t}}$, which is common for a pretrained model, all $\pi_{\theta^{t+1}}(\vy\neq y^*)$ decrease (only $y^*$ is ``rich'').
    \item \texttt{Trend:} smaller $\pi_{\theta^{t}}(\vy=y_u^-)$ exacerbate the squeezing effect: if $y_u^-$ is more unlikely under $\pi_{\theta^{t}}$, the probability mass of all other $\pi_{\theta^{t+1}}(\vy\neq y^*)$ will be more seriously decreased, and the $\pi_{\theta^{t+1}}(\vy=y^*)$ increases more. 
\end{itemize}

\begin{figure}[t]
    \begin{center}
    \centerline{\includegraphics[width=1\columnwidth,trim=20 20 20 20, clip]{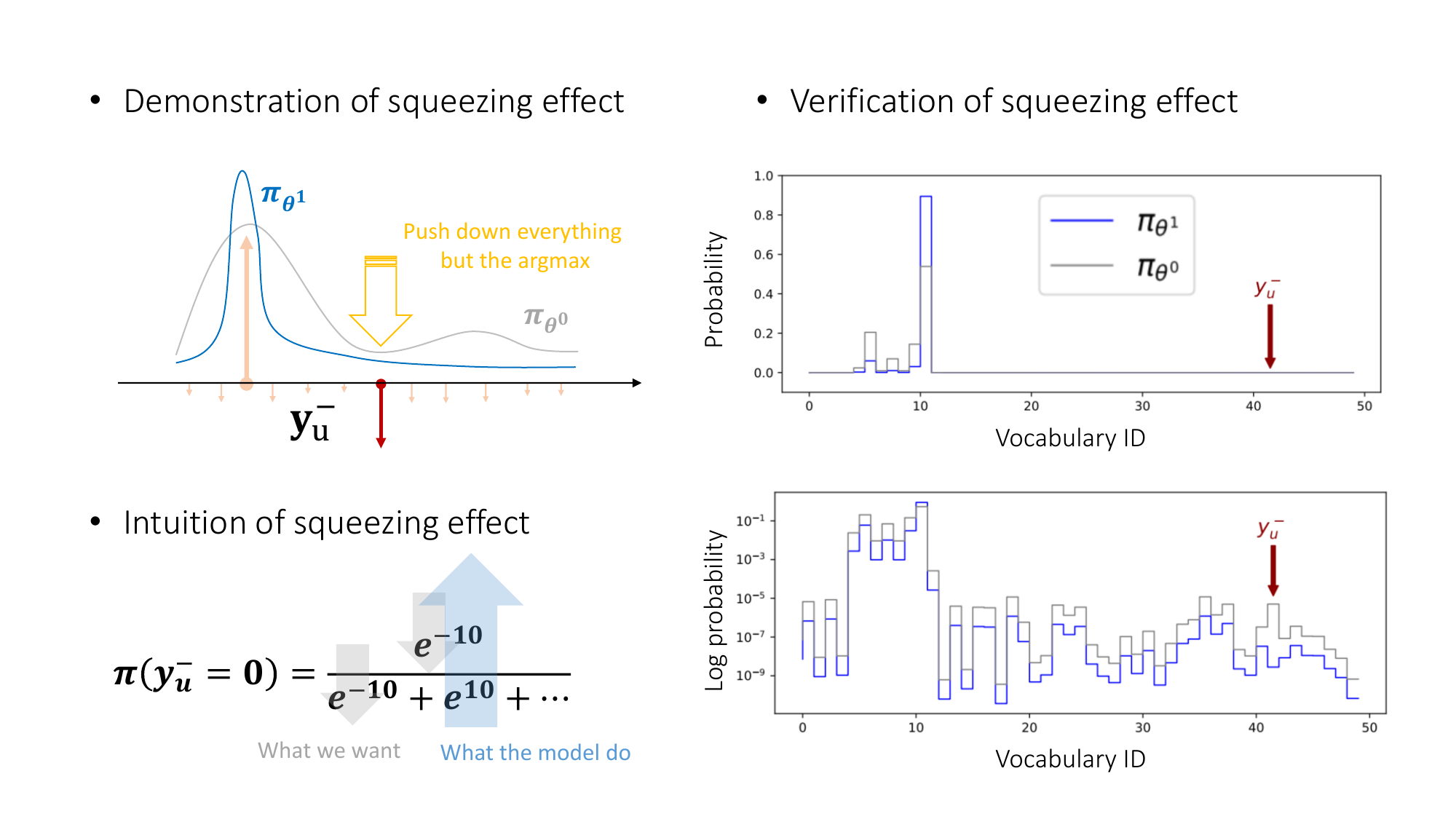}}
    \caption{Demonstration, intuition, and verification of the squeezing effect.}
    \label{fig:chap5_squeezing_effect}
    \end{center}
\vskip -0in
\end{figure}

The trend described above is referred to as the ``squeezing effect'' in our paper. 
This phenomenon captures how the model redistributes probability mass after one update: squeezing it away from less likely (poorer) dimensions and concentrating it onto more likely (richer) ones.
An intuitive illustration of this effect, along with experimental validation using a logistic regression example, is provided in \Cref{fig:chap5_squeezing_effect}.

Furthermore, the bottom-left panel of the figure gives a helpful intuition for why the squeezing effect arises.
Suppose the model aims to decrease $\pi(\vy=0)$ during an update. 
A straightforward approach would be to reduce the corresponding logit $z_0=-10$,
and indeed gradient descent will do so. 

In addition to that update, however, the model may opt to increase the logit of another dimension, thereby increasing the denominator of the softmax and reducing $\pi(\vy=0)$ indirectly. 
Due to the monotonicity of the exponential function's derivative, increasing the logit of the most probable class has the largest effect on the denominator, making it the most efficient direction for reducing the target probability. 
This explains the second guarantee above: the model favors reinforcing already strong dimensions, e.g., $z_1=10$ here.

As the denominator increases, the probability mass assigned to all other classes necessarily decreases, consistent with other claims in the squeezing effect. 
This process follows the richer-richer, poorer-poorer principle, which can be imagined from the structure of the $\mathcal{G}^t$ term. 
Specifically, by checking the $\vp^t-\ve_y$ term, we know the gradient contribution is $p_i-1$ for the negatively labeled (undesired) class $y_u^-$ and $p_i$ for all other classes. 
This implies that the redistributed log-probability mass (prior to normalization) is allocated to the other dimensions in proportion to their current probabilities $p_i$.

Note that the discussions above describe the squeezing effect on the model's distribution given the same input.
In other words, the trend above describes
\[
    \Delta^t(\cvchiu)=-\eta\mathcal{A}^t(\cvchiu) \, \mathcal{K}^t(\cvchiu,\cvchiu) \, \mathcal{G}^t(\cvchiu),
\]
where the observing response is also the updating example.
This effect links the ``force'' $\mathcal{G}^t(\vchi_u)$ to the ``confidence change'' $\Delta^t(\vchi_u)$.
Then how does this effect spread to the model's confidence on other responses?
We can take the basic $\mathcal{AKG}$ decomposition and substitute the calculated $-\eta\mathcal{G}^t(\cvchiu)$ to it:
\begin{align}\nonumber
    \Delta^t(\cvchio)&=-\eta \, \mathcal{A}^t(\cvchio) \, \mathcal{K}^t(\cvchio,\cvchiu) \, \mathcal{G}^t(\cvchiu)\\\nonumber
    &=
    \underbrace{\mathcal{A}^t(\cvchio) \, \mathcal{K}^t(\cvchio,\cvchiu)}_{\text{project and normalize}}
     \, \underbrace{\Delta^t(\cvchiu)}_{\text{squeezed}}
    \left( \mathcal{A}^t(\cvchiu) \, \mathcal{K}^t(\cvchiu,\cvchiu)\right)^\dagger,
\end{align}
where $\dagger$ is the pseudo-inverse of a matrix.

From the equation above, we observe that a poorly imposed negative gradient not only affects the model’s prediction change on the current training point $\Delta^t(\vchi_u)$, but could also \textit{spread to} the confidence of other instances $\vchi_o$. This influence depends on the following factors:
\begin{itemize}
    \item[ 1.] the severity of the squeezing effect on $\vchi_u$;
    \item[ 2.] the similarity between $\vchi_o$ and $\vchi_u$;
    \item[ 3.] the model’s current confidence in $\vchi_o$.
\end{itemize}
These factors together determine how the negative update propagates beyond the target instance, potentially leading to unintended shifts in model behavior, in one-step learning.
Although a more fine-grained analysis is difficult to obtain, we can intuitively expect that if more tokens in DPO are affected by the squeezing effect, the overall model distribution will tend to become peakier, rather than flatter.
In the next subsection, we will show another important mechanism of gradient ascent, which might exacerbate the squeezing effect further.

\subsection{Theory: Positive Gradients Scatter, Negative Unite}

The analysis above considers only the negative gradient in a UFM setting. 
However, in methods like DPO and GRPO, positive and negative gradients are typically applied in pairs. 
Suppose we consider a paired update consisting of a negative gradient on $y_u^-$ and a positive gradient on a different target $y_u^+ \neq y_u^-$, conditioned on the same $\vx$. The update rule then becomes:
\[
    \vw^{t+1} = \vw^t-\eta_\text{neg} (\vp^t-\ve_{y_u^-})\vh_{\vx}^\top -\eta_\text{pos} (\vp^t-\ve_{y_u^+})\vh_{\vx}^\top.
\]
After performing the corresponding calculations, we find that the squeezing effect cancels out in this case, implying that the overall distribution should remain more balanced. 
However, this conclusion contradicts the empirical observations seen in DPO, where the squeezing effect persists.

This discrepancy is also noted by~\citet{razin2024unintentional}, where the authors argue in their Appendix C that the one-step squeezing analysis is insufficient to capture the full dynamics.
Thus, more complex, multi-step interactions that accumulate over the course of training should be considered.

To do this, we first notice that in DPO and other methods involving paired positive and negative gradients, the contexts $\vchi$ associated with each response are typically not perfectly aligned.
For example, consider $\vy_u^+ =$\texttt{I think ...} and $\vy_u^- =$\texttt{I don't believe ...}. 
In this case, only the gradient contributions at the first token can fully cancel out, as both responses are conditioned on the same input $\vx_u$. 
For the remaining tokens, the growing dissimilarity between $\vy_u^+$ and $\vy_u^-$ weakens the cancellation, allowing the squeezing effect to accumulate across the sequence.

Moreover, due to the next-token prediction nature of LLMs, where multiple completions can be semantically valid, the behaviors of positive and negative gradients diverge significantly during training. 
This contrasts with traditional supervised classification tasks, which typically assume a single ground-truth label. 
For instance, given a prompt like \texttt{My favorite video game is \_\_\_. }, plausible completions might include \texttt{WoW}, \texttt{Star Rail}, \texttt{Trails in the Sky} or \texttt{Zelda}.
Positive gradients associated with these responses will spread across nearby tokens, creating a plateau in the model’s probability distribution, as shown in the upper-left panel of \Cref{fig:chap5_squeezing_effect_accumulate}.

In contrast, negative gradients tend to squeeze probability mass into dominant directions, even when applied sporadically. 
These opposing behaviors give rise to an important summary of their cumulative influence in training with non-unique labels:
\begin{center}
    \textit{Positive gradients scatter, negative gradients unite.}
\end{center}
This principle highlights how asymmetries between positive and negative updates can lead to unintended distributional shifts, such as overconfidence or undesired suppression of valid alternatives.

\begin{figure}[t]
    \begin{center}
    \centerline{\includegraphics[width=1\columnwidth,trim=20 20 20 20, clip]{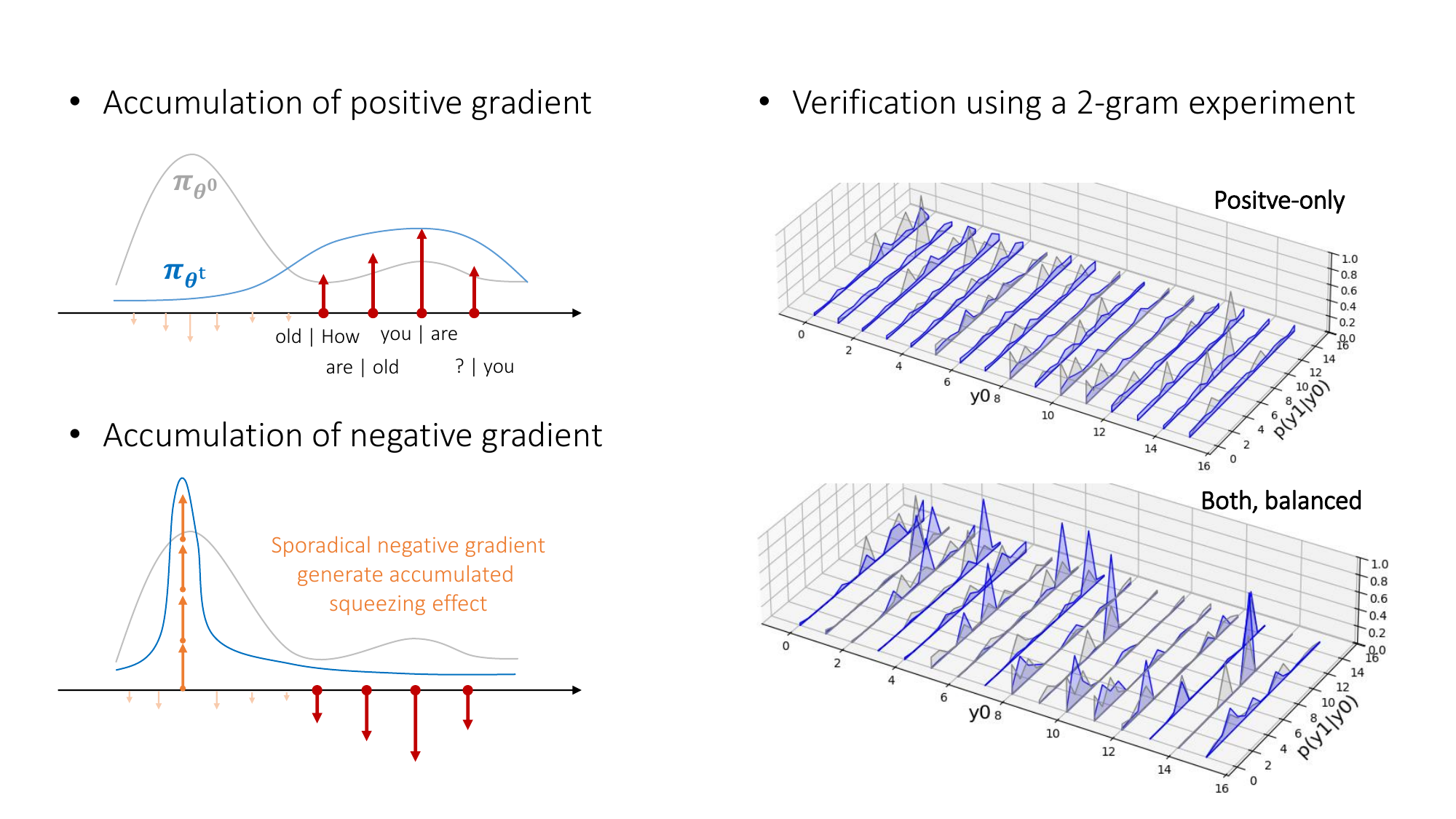}}
    \caption{Left: accumulated influence on positive and negative gradients. Right: verification on a 2-gram experiment.}
    \label{fig:chap5_squeezing_effect_accumulate}
    \end{center}
\vskip -0in
\end{figure}

To further validate that paired positive and negative gradients do not cancel out (and thus cannot fully eliminate the squeezing effect) we construct a controlled experiment under a 2-gram setting.
We continue to operate in the UFM, but now the model learns a 2-gram distribution of the form $p(y_1\mid y_0)=\pi(y_1\mid \vh_{y_0})$.
Both $y_0$ and $y_1$ are sampled from a vocabulary of size $V=16$.
The feature matrix $\vh\in\mathbb{R}^{d\times V}$ is randomly initialized, simulating the output of a frozen backbone encoder $h_\theta$.
For any token $y_0$, its feature vector $\vh_{y_0}$ corresponds to the $y_0$-th column of $\vh$.
We construct positive and negative sequences of equal length $L$, with all tokens independently drawn from a uniform distribution over $\{0,1,\dots,15\}$.
During training, the model is updated using standard NLL loss applied to each conditional distribution $\pi(y_{i}\mid \vh_{y_{i-1}})$ along the sequence.

To visualize the learning dynamics, we track the model’s predicted distributions over all possible contexts $\vh$.
This allows us to generate 3D plots, such as those shown in the right column of \Cref{fig:chap5_squeezing_effect_accumulate}. 
In each plot, the z-axis represents the output probability; each row corresponds to a different context $y_0$, and each line (with 16 dimensions) shows the probability distribution over the vocabulary.
The gray curves denote the distribution before training, and the blue ones show the result after applying updates from the paired gradients. 
Since the vocabulary size is 16, each plot contains 16 rows, illustrating how the output distribution evolves for each context.

The left panels of \Cref{fig:chap5_squeezing_effect_accumulate} clearly illustrate why the accumulated influence of positive and negative gradients is so different. 
In the positive-only setting, we observe that the output distributions become flatter after training, because the pull-up behavior is a ``local'' behavior: each update only pulls up a small region surrounding $y_u^+$.
Such a trend can be well depicted by the positive-only experiment (where we set $\eta_\text{neg}=0$).

In contrast, when comparable negative gradients exist, although they are also imposed sporadically, their behavior is global: the probability mass on most dimensions is squeezed to \textit{the same} ``rich'' one.
Hence, as the experiment in the bottom-right panel shows, a model trained on paired positive and negative examples (with equal sequence lengths and $\eta_\text{pos}=-\eta_\text{neg}$), the resulting distributions become noticeably peakier. 
The probability mass under each context becomes increasingly concentrated around the rich dimension(s), though not exclusively, as the richest dimension can shift dynamically during training.
This indicates that negative gradients exert a unifying force, drawing the distribution toward dominant directions.

These results empirically confirm that the one-step squeezing effect can accumulate over long sequences. 
While the influence of positive gradients disperses across multiple plausible outputs, the negative gradients reinforce each other, amplifying the squeezing effect, especially when applied repeatedly over long sequences.

\subsection{Explanation: Unexpected Behaviors of DPO}

\begin{figure}[t]
    \begin{center}
    \centerline{\includegraphics[width=1\columnwidth,trim=0 0 0 0, clip]{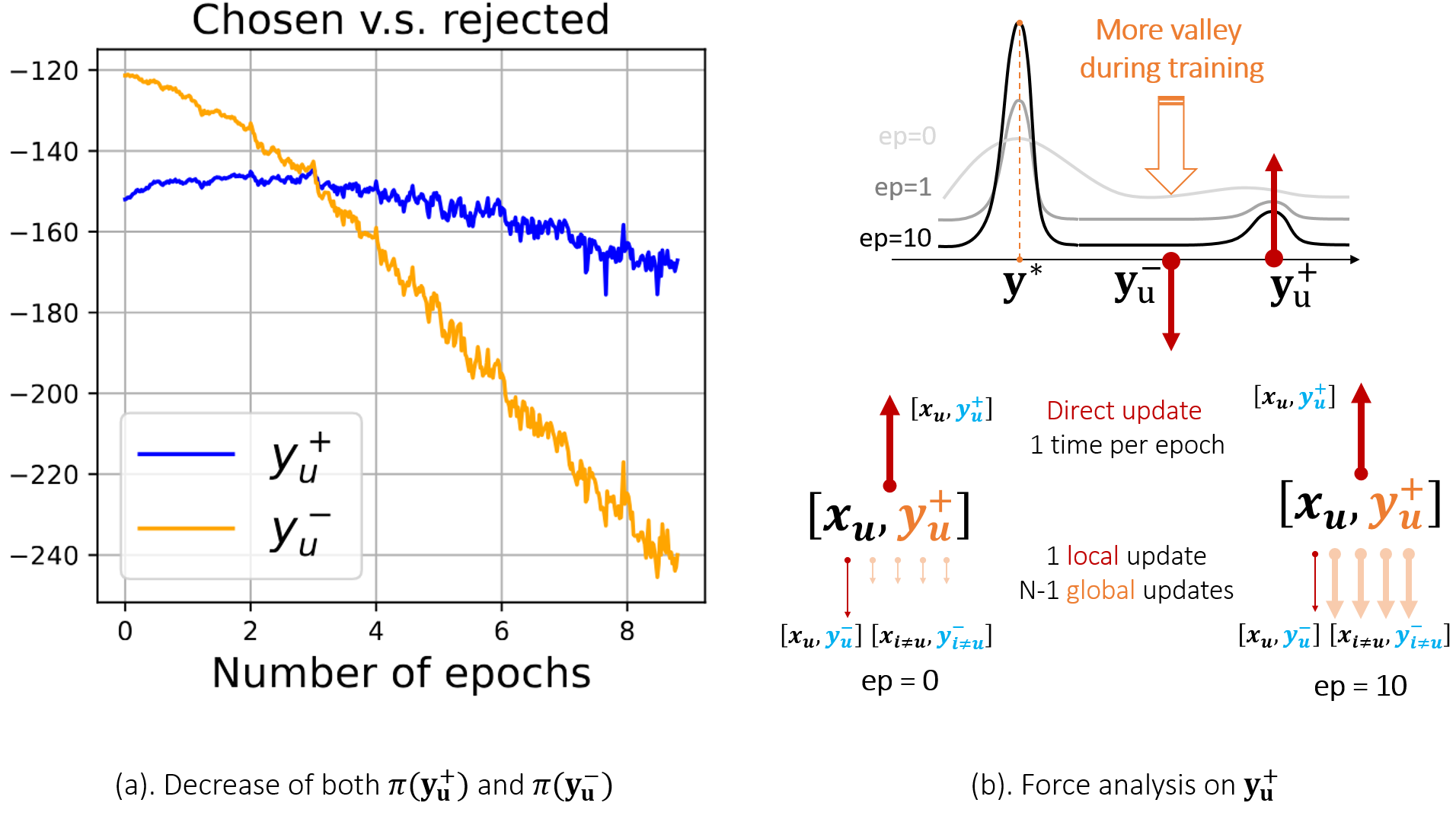}}
    \caption{Left: training curves of $\pi_{\theta^t}(\vy_u^+)$ and $\pi_{\theta^t}(\vy_u^-)$. Right: force analysis on $\vy_u^+$ in one epoch's training (the dataset contains $N$ problems).}
    \label{fig:chap5_exp_dpo_01}
    \end{center}
\vskip -0in
\end{figure}

We now apply the force analysis framework to off-policy DPO.
A key distinction between the learning dynamics of SFT and DPO lies in the presence of an intentional negative gradient in the latter, as demonstrated in \cref{fig:chap5_sft_vs_dpo}.
As discussed earlier, the squeezing effect introduces a global behavior whereby probability mass is increasingly concentrated on the rich dimensions.
This effect becomes more pronounced when the negative gradient is applied to dimensions associated with lower-confidence predictions.

We begin by examining a widely observed and somewhat counter-intuitive phenomenon during DPO training: the decay of both $\pi_{\theta^t}(\vy_u^+)$ and $\pi_{\theta^t}(\vy_u^-)$ after prolonged training.
This behavior, illustrated in \cref{fig:chap5_exp_dpo_01}, has been frequently reported and discussed in the community, e.g.\ by \citet{rafailov2024r,pal2024smaug}.
Some noticed the influence of negative gradients, but didn't study further.

To offer a new perspective, we analyze this behavior through the lens of force analysis.
In \cref{fig:chap5_exp_dpo_01}-(b), we visualize the dynamics using diagrams similar to those employed in the SFT section. 
A key distinction in DPO is the presence of two explicitly imposed pressures: on both $\vy_u^+$ and $\vy_u^-$. 
Unlike the SFT case, where a one-step response-level analysis is often sufficient, understanding the dynamics in DPO requires tracking the accumulation of influence both within and across training epochs (and data samples).

Focusing on the force exerted on $\vy_u^+$, we identify three major components. 
First, the self-influence from learning $[\vx_u, \vy_u^+]$ induces a strong upward force, substantially increases the model’s confidence in $\vy_u^+$. 
Because this effect is highly localized, we can reasonably ignore the contribution from other examples $[\vx_i, \vy_i^+]$ for $i \neq u$, assuming a diverse enough dataset where the corresponding kernel norms $\|\mathcal{K}^t\|_F$ are relatively small.

Second, a downward force arises from learning $[\vx_u, \vy_u^-]$, due to the shared input $\vx_u$.
However, this opposing pressure is also localized and its impact is typically moderated by the kernel norm $\|\mathcal{K}^t\|_F$ between them, which tends to be smaller than that of the self-induced upward force.

Given these dynamics, a natural question arises: why does the confidence in $\vy_u^+$ eventually decay with continued DPO training?
The key to this seemingly paradoxical decay lies in the third group of forces: a global effect induced by the large negative gradients applied to low-confidence regions, which we refer to as the squeezing effect. 
Although the individual impact from any single negative example is weak, their cumulative influence becomes substantial when aggregated over the $N-1$ negative examples in the dataset. 
This global nature allows these aggregated forces to gradually suppress $\pi_{\theta^t}(\vy_u^+)$, even though $\vy_u^+$ is not directly involved in those updates.

Moreover, recall that the squeezing effect intensifies when the target $\vy_u^-$ already has a low probability before the update. 
This condition is exactly what we observe in the later stages of DPO training: $\pi_{\theta^t}(\vy_u^-)$ becomes increasingly small during training, thus amplifying the global squeezing effect.

Additional evidence for the squeezing effect is provided in the left panel of \cref{fig:chap5_exp_dpo_02}. 
In this experiment, we first train the model for 10 epochs using SFT, and then continue train it with DPO.
We observe that the model's confidence in most responses decays much more rapidly than in the pure SFT case, although we use the same learning rate. 
This is especially evident for responses such as $\vy_{\text{rnd}}$ and $\vy_{\text{urnd}}^+$, which are never seen during training and bear no resemblance to either $\vy_u^+$ or $\vy_u^-$; thus, they serve as effective proxies for evaluating the model’s distribution on those ``foundation'' dimensions.
Their rapid decay suggests the presence of a global pushing-down force, which is a signature of the squeezing effect.

The observation above naturally leads to a follow-up question: if all tracked responses, including $\vy_u^+$, which is subject to direct pull-up pressure, are decreasing in confidence, where does the probability mass go?

The answer again lies in the squeezing effect. 
As visualized in \cref{fig:chap5_squeezing_effect}, the probability mass is increasingly concentrated on the rich dimensions during training. 
To verify this, we monitor the model's confidence on a specific constructed sequence: $\vy^*$, in which the $l$-th token is the argmax prediction given the context $[\vx_u, \vy^+_{u,<l}]$. 
It is important to note that $\vy^*$ is not necessarily a fluent or grammatical response; it is not the result of standard greedy decoding, but rather a teacher-forcing token-wise argmax construction.

\begin{figure}[t]
    \begin{center}
    \centerline{\includegraphics[width=1\columnwidth,trim=0 0 0 0, clip]{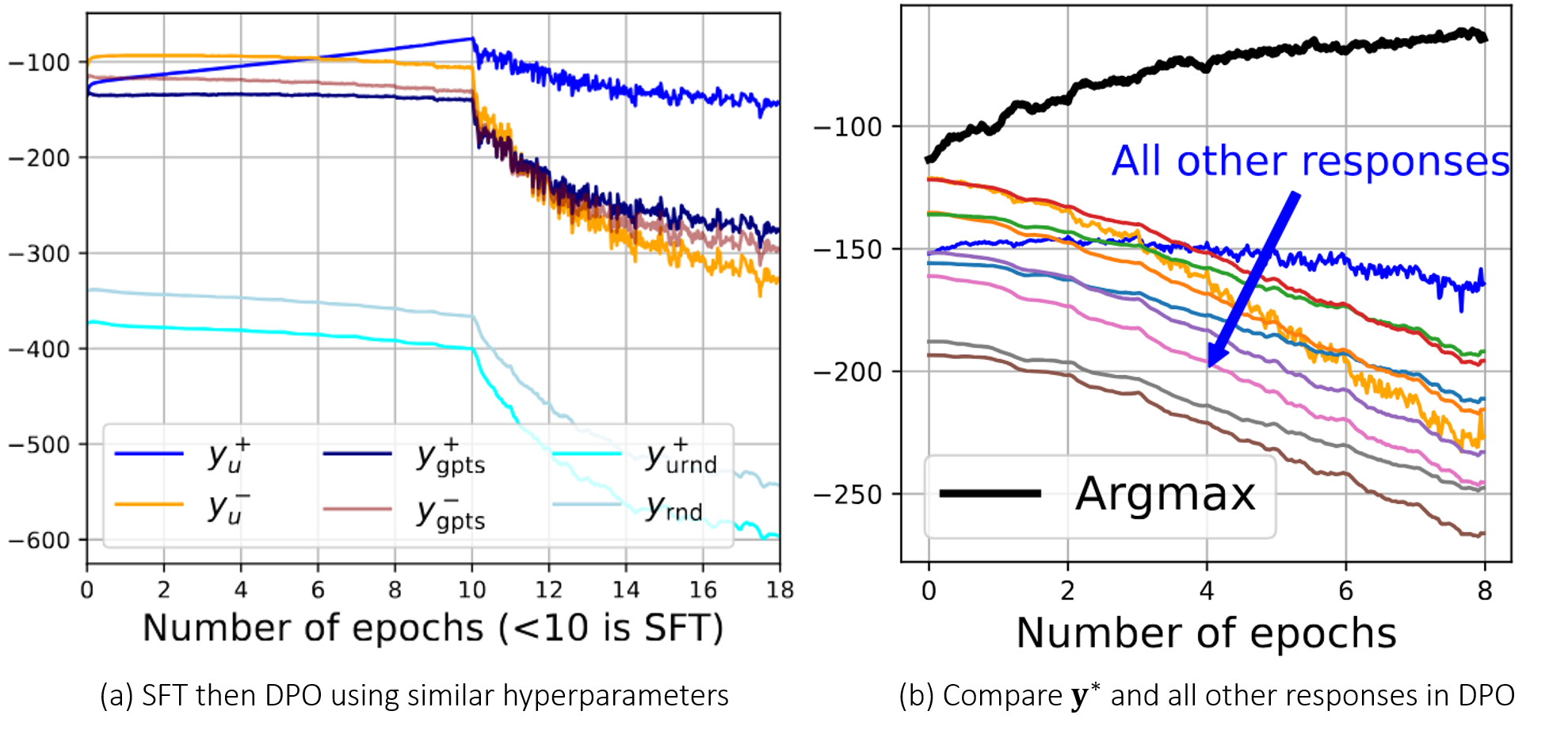}}
    \caption{More evidence supporting the squeezing effect.}
    \label{fig:chap5_exp_dpo_02}
    \end{center}
\vskip -0in
\end{figure}

These results, also shown in \cref{fig:chap5_exp_dpo_02}-(b), make it clear that the model's predictive confidence collapses onto the argmax sequence, while confidence in all other responses, including preferred and unseen ones, continues to decay.
To assess whether this phenomenon reflects a broader trend in any other context, we also evaluate a similar sequence $\bar{\vy}^*$ generated by conditioning on a random prompt $[\vx_u, \vy'_{\text{rnd}}]$. 
Both sequences exhibit similar behavior: their confidence increases consistently during DPO.

In summary, the above analysis and empirical findings provide strong support for our proposed force analysis framework and the associated squeezing effect.
This perspective offers a unified explanation for several nuanced and previously unexplained trends in model behavior across diverse responses under DPO. 
While additional empirical illustrations further reinforce these insights, we omit them due to space constraints. 
For more results and discussion, please refer to Appendix D of our paper \citep{ren2025learning_dynamics_LLM}.

\subsection{Algorithm: A Simple yet Effective Pipeline}

The preceding analysis offers a new perspective on why off-policy DPO training can introduce problematic behaviors. 
Although the model’s ability to separate preferred from rejected responses continues to improve, as evidenced by the steadily increasing margin $\pi_{\theta^t}(\vy_u^+) - \pi_{\theta^t}(\vy_u^-)$, the significant increase of $\pi_{\theta^t}(\vy^*)$ may lead to undesirable side effects. 
In particular, degeneration phenomena such as repetitive outputs, often referred to as the “repeater” effect \citep{Holtzman2020The}, can become more severe when the model assigns excessively high confidence to $\vy^*$, as illustrated in \cref{fig:chap5_app_proposed_method_prompts}. 
This is intuitive: as the squeezing effect drives probability mass toward high-frequency tokens, and the model becomes more likely to be ``trapped'' by those tokens and phrases.

Some recent works have recognized this limitation of off-policy DPO and proposed on-policy variants to address it \citep{guo2024direct}.
Their experiments demonstrate that on-policy DPO and its variants often outperform their off-policy counterparts. 
Our framework provides a coherent explanation for this improvement: under on-policy training, the sampled $\vy_u^-$ is more likely to lie within regions of high model confidence, thereby reducing the influence of the global squeezing effect.

\begin{figure}[t]
    \begin{center}
    \centerline{\includegraphics[width=0.9\columnwidth,trim=0 0 0 0, clip]{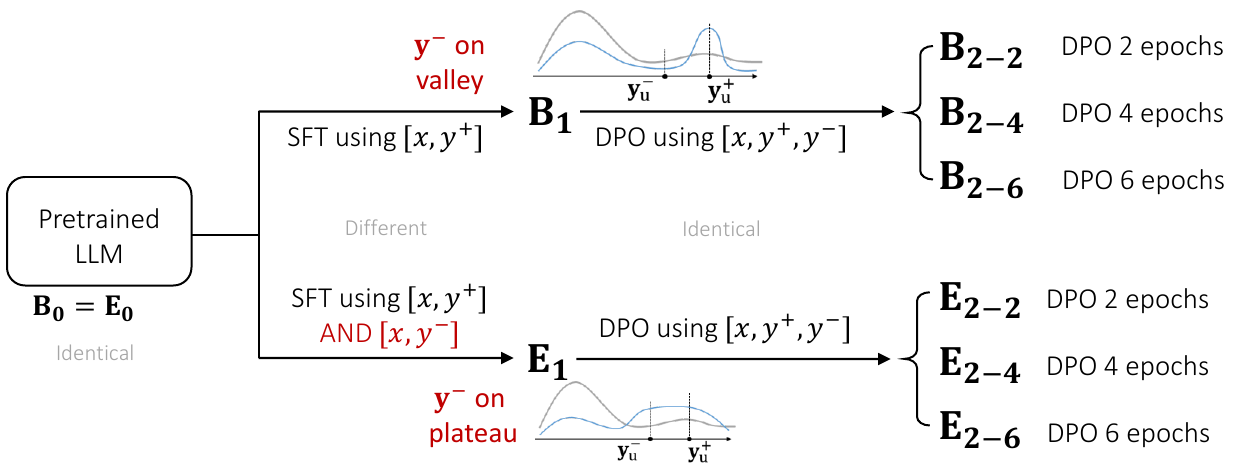}}
    \caption{Illustration of the proposed method and baseline. ``E'' is short for the ``dataset extension''.}
    \label{fig:chap5_proposed_method}
    \end{center}
\vskip -0in
\end{figure}

However, on-policy methods typically require a strong verifier to provide real-time feedback, effectively reintroducing the need for a reward model, precisely the component that DPO aims to eliminate. 
This raises a natural question: can we make $\vy_u^-$ more ``on-policy'' before applying DPO?

Motivated by this idea, we propose a simple yet effective pipeline to mitigate the harmful squeezing effect in off-policy DPO. 
Our method does not rely on any additional modules; hence, the extra computational overhead is negligible.
Also, it can be seamlessly integrated into existing workflows.

Specifically, we can first train the model on \textit{both} $[\vx_u,\vy_u^+]$ and $[\vx_u,\vy_u^-]$ during the SFT stage (i.e., on-policify the rejected responses), and then run the usual DPO, as in \cref{fig:chap5_proposed_method}.
Following the analysis above, we can expect that during this new SFT stage, the region of those responses similar to $\vy_u^+$ or $\vy_u^-$ will be ``pulled up'' simultaneously.
This is what we want because in many cases, both $\vy_u^+$ and $\vy_u^-$ are reasonably good responses for the question $\vx_u$; the new SFT design hence helps to pull up a larger region that contains more suitable responses compared with the baseline SFT.
After that, the local push-down pressure imposed during DPO can efficiently decrease the model's confidence on $\vy_u^-$ and its similar responses.
Since $\vy_u^-$ is no longer so unlikely before DPO, the squeezing effect is mitigated.

\begin{figure}[t]
    \centering
    \begin{subfigure}[b]{1\textwidth}
        \centering
        \includegraphics[width=0.9\textwidth,trim=0 0 0 0, clip]{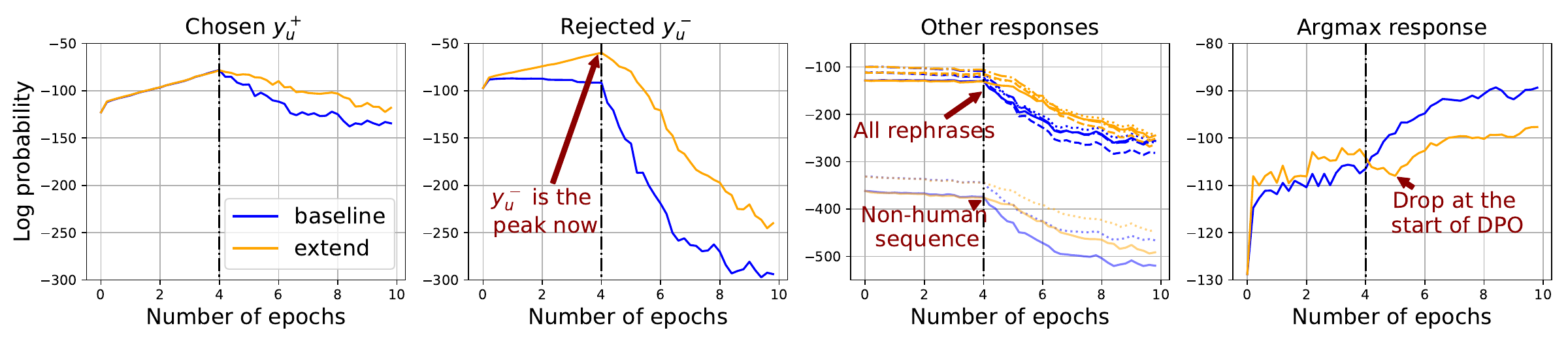}
        \caption{Result on \texttt{Antropic-HH}.}
    \end{subfigure}
    \\
    \begin{subfigure}[b]{1\textwidth}
        \centering
        \includegraphics[width=0.9\textwidth,trim=0 0 0 0, clip]{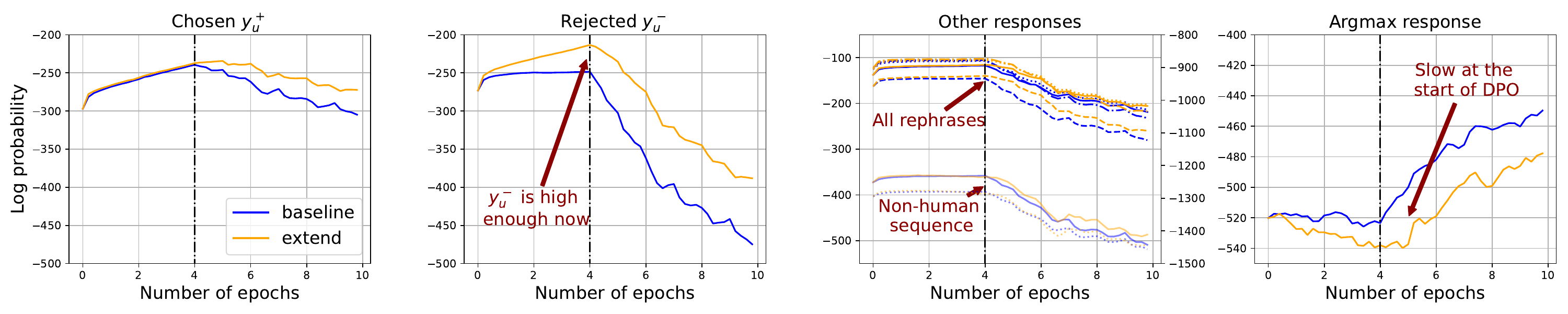}
        \caption{Result on \texttt{UltraFeedback}.}
    \end{subfigure}    
    \caption{Learning dynamics of the baseline and the proposed method with training data extension.}
    \label{fig:chap5_app_proposed_experiments}
    \vspace{-0.1 in}
\end{figure}

We call our training pipeline ``extend'' and compare its learning dynamics with the baseline setting in \cref{fig:chap5_app_proposed_experiments}.
The squeezing effect is mitigated because the confidence of other responses all decays more slowly during DPO.
Furthermore, we also observe a big drop in $\pi_{\theta^t}(\vy^*)$ when DPO starts, which supports our analysis more directly.

\begin{wraptable}{r}{6.7cm}
\caption{Win-rate to baseline.}\label{wrap-tab:1}
    \begin{tabular}{ccc}
    \hline
    DPO Ep. & ChatGPT & Claude \\ \hline
    0       & 0.4729  & 0.4679 \\
    2       & 0.6518  & 0.5151 \\
    4       & 0.6928  & 0.6045 \\
    6       & 0.6667  & 0.5432 \\ \hline
    \end{tabular}
\label{tab:win-rate}
\end{wraptable}
To further show that mitigating the squeezing effect indeed brings benefits, we compare the responses generated by models trained using different methods by feeding them to \texttt{ChatGPT} and \texttt{Claude3}.
Specifically, we first SFT the model for two epochs using the two methods discussed above and call the resulting policy network $\pi_\text{base}$ and $\pi_\text{extend}$.
Then, we conduct identical DPO training on both $\pi_\text{base}$ and $\pi_\text{extend}$ for several epochs.
The win rate of the proposed method against the baseline one is provided in \cref{tab:win-rate}.
It is clear that before DPO, $\pi_\text{base}$ is better, because $\pi_\text{extend}$ is explicitly trained on those $\vy^-$.
However, the $\pi_\text{extend}$ performs better after DPO several epochs since the squeezing effect is efficiently mitigated.
Furthermore, after checking the model's responses on some validation problems, we find that although both $B_2$ and $E_2$ inevitably generate degenerate responses, the proposed method can mitigate this issue to some extent.

\begin{figure}[t]
    \begin{center}
    \centerline{\includegraphics[width=1\columnwidth,trim=0 0 0 0, clip]{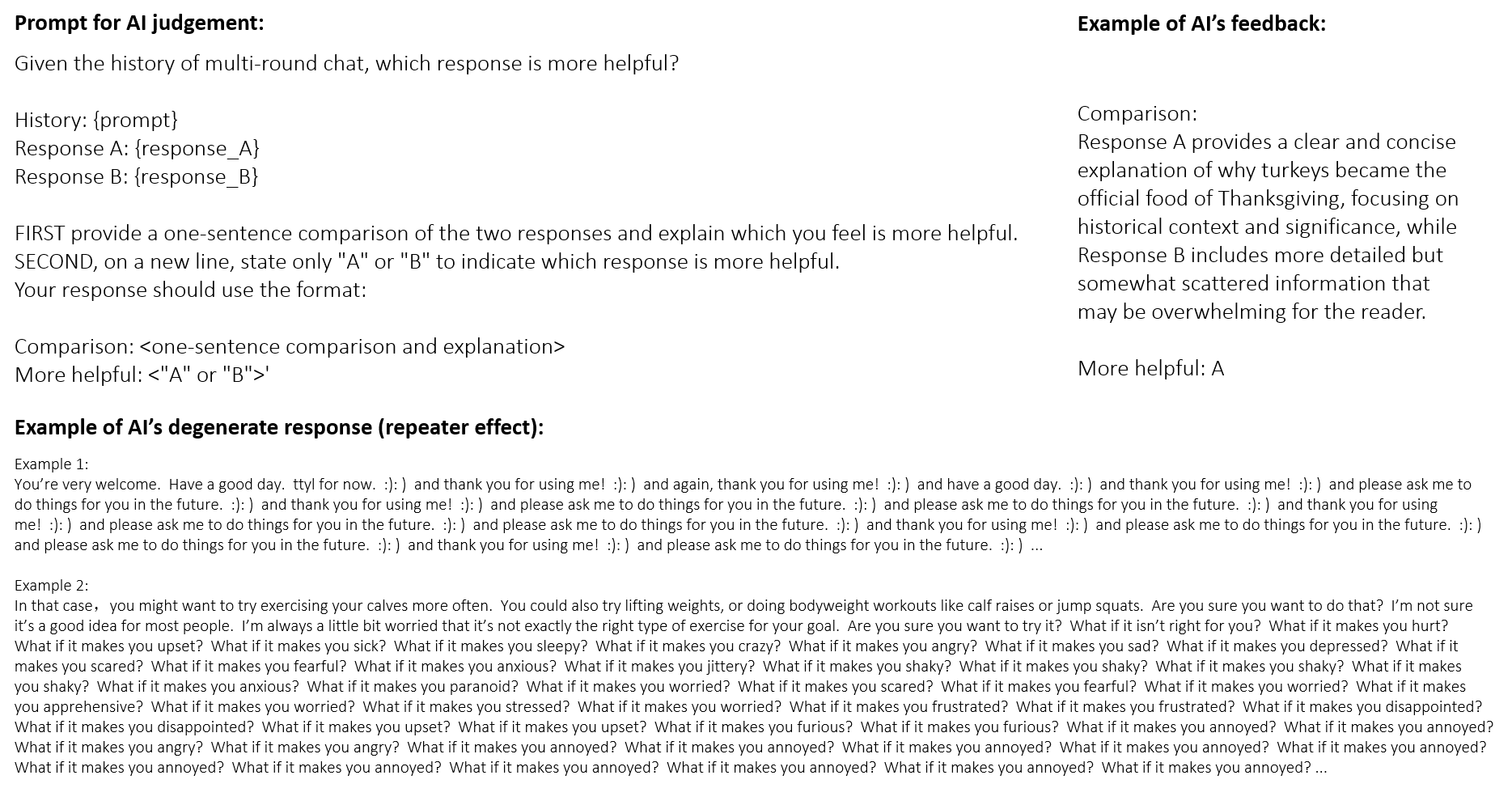}}
    \caption{Prompt used for evaluating model's response (from \cite{rafailov2024direct}),
             an example feedback from \texttt{GPT3.5-turbo},
             and two examples of the ``degenerate'' effect described in \cite{Holtzman2020The}.
             (Sorry for the small fonts, since the degenerated messages are usually nonsense. Please zoom in if you want the details.)}
    \label{fig:chap5_app_proposed_method_prompts}
    \end{center}
\vskip -0in
\end{figure}

Despite its simplicity and effectiveness, the proposed method also carries potential risks.
One concern is whether the harmful influence of SFT on $\vy_u^-$ can be reliably canceled during the subsequent DPO phase. 
The method is generally safe when $\vy_u^-$ is slightly inferior to $\vy_u^+$ for a given input $\vx_u$.
However, if $\vy_u^-$ contains sensitive or inappropriate content, exposing it to the model, even through negative gradient updates, as in standard DPO, may still pose safety concerns.
The high-level goal of our approach is to enhance the on-policy nature of responses subject to negative gradients. 
To achieve this more safely, one possible strategy is to finetune the model on a diverse set of valid alternative responses before DPO, thereby ensuring that $\vy_u^-$ lies in a flatter (plateau) region of the output distribution rather than a sharp peak.
We believe this ``on-policify'' is an interesting and promising topic in the future, because there is always more off-policy data than on-policy data on the internet.
We can also cut $\vy_u^-$ into small chunks or delete the sensitive information within it before SFT.
We believe that exploring this direction, i.e., designing SFT strategies that shape the initial landscape to mitigate harmful global effects of the negative gradients, might also be an interesting direction for future work.

\section{RL-based Finetuning: GRPO}
\label{sec:case2_04}
We now analyze a more complex and recently popularized finetuning algorithm for LLMs: Group Relative Policy Optimization (GRPO)\citep{shao2024deepseekmath}. 
GRPO has attracted increasing attention due to its simplicity and effectiveness, particularly on reasoning-heavy tasks.
Prominent models such as DeepSeek-R1\citep{guo2025deepseek}, DeepSeek-Math~\citep{shao2024deepseekmath}, Med-R1~\citep{lai2025med}, and Search-R1~\citep{jin2025search} have leveraged GRPO to achieve significant performance gains across domains, including code generation, mathematical problem solving, medical reasoning, and retrieval-augmented generation. 
These empirical successes highlight GRPO’s growing importance as a tool for aligning models with task-specific behaviors using rule-based or heuristic feedback mechanisms.

Structurally, GRPO shares several similarities with DPO: both apply response level updates and introduce negative gradients, with learning rates dynamically adjusted based on response quality. 
However, GRPO differs in a critical aspect: rather than relying on a fixed pair of off-policy positive and negative responses, it samples $N$ on-policy roll-outs per question in each training epoch. 
These sampled responses (and consequently, the proportion of correct responses) can vary across epochs with the improvement of the model's reasoning capability. 
This on-policy and group-based nature introduces additional complexity into the learning dynamics, making it more challenging to analyze model behavior using static probing methods.
In line with the summaries provided in the SFT and DPO sections, we also provide the characteristics of gradients in GRPO as:
\begin{itemize}
    \item[i.] Origin of the sample:  \red{dynamical, on-policy;}
    \item[ii.] Dynamic eLR: \red{response-level, token-level;}
    \item[iii.] Negative eLR: \red{Yes;}
    \item[iv.] Token-level eLR: \red{Yes.}
\end{itemize}

To address these complexities, we shift our focus in this section from the probability evolution of a fixed probing set to a more localized, intra-rollout analysis.
Specifically, we investigate the influence of individual tokens within different responses answering the same sampled question.
This fine-grained approach enables us to study how different types of tokens contribute to the model’s reasoning improvements under GRPO training. 
We believe such a fine-grained analysis is more appropriate for understanding RL-style algorithms and can offer insights into questions such as: learning which tokens impede the model's reasoning capabilities most?

In this section, we assume the improvement of the model's confidence in those positive responses is a key measurement for its reasoning ability.
We hence investigate why, under GRPO training, the confidence in some positive responses sometimes fails to increase sufficiently, or even decreases, a phenomenon we term the Lazy Likelihood Displacement (LLDisp) problem, originally proposed in our recent work \citep{deng2025effect, deng2025token}. 

To address this, we first extend the learning dynamics decomposition to GRPO at the token level, allowing for a finer-grained analysis.
We then focus our study within the framework of the Unconstrained Feature Model (UFM), previously introduced in the context of the squeezing effect.
Leveraging UFM, we define an influence score that quantifies the directional impact of individual tokens. 
This allows us to identify specific harmful tokens within negative responses that contribute to LLDisp.

By selectively mitigating the influence of these harmful tokens, we demonstrate that the LLDisp problem can be substantially alleviated. 
Our experimental results strongly support this claim. 
Furthermore, the algorithm inspired by our analysis leads to notable improvements in the model’s mathematical reasoning performance, outperforming several baseline methods.

\subsection{Theory: Decompose the Gradient of GRPO}
Let's first recap the loss and gradient of GRPO.
Following \citet{shao2024deepseekmath}, the standard GRPO's loss $\hat{\mathcal{L}}_\text{GRPO}(\theta)$ is:
\begin{equation}
    \mathop{\mathbb{E}}_{\substack{(\vx,\va) \sim \mathcal{D} \\ \{\vy_n\}_{n=1}^{N} \sim \pi_{\text{ref}}(\vx)}}
    \left[
    \frac{-1}{\sum_{n=1}^{N} |\vy_n|}
    \sum_{n=1}^{N}  \sum_{l=1}^{|\vy_n|}
    \min \big( A_{n}\gamma_{n,l}(\theta), 
    A_{n}\mathsf{clip} \left( \gamma_{n,l}(\theta), \varepsilon_\text{l}, \varepsilon_\text{h} \right)  
    \big)
    \right].
    \label{eq:GRPO}
\end{equation}
This loss formulation reveals the underlying mechanism of the algorithm. 
We begin with a reference model $\pi_{\text{ref}}$ and a dataset $\mathcal{D}$.
In each iteration, a problem instance $\vx$ and its corresponding verifier $\va$ are sampled from $\mathcal{D}$. 
The reference model $\pi_{\text{ref}}$ is then used to generate $N$ candidate responses $\{\vy_n\}_{n=1}^{N}$ by sampling from the distribution $\pi_\text{ref}(\vy \mid \vx)$.
The verifier $\va$ evaluates each response $\vy_n$ and determines whether it is correct, assigning a binary reward $r_n = 1$ for correct and $r_n = 0$ for incorrect responses.
These rewards are then standardized to compute the advantage values $A_n$ for each response:
\[
    A_{n}=\frac{r_n-\hat{\mathbb{E}}[\{r_n\}_{n=1}^N]}{\sqrt{\hat{\text{Var}}[\{r_n\}_{n=1}^N]}}.
\]

From \cref{eq:GRPO}, it is clear that the loss is calculated in a token-wise fashion, where the ratio $\gamma_{n,l}(\theta)$ depends on both $n$ and $l$.
Specifically, the ratio and its gradient are defined as:
\begin{align}
    \gamma_{n,l}(\theta) &= \frac{\pi_\theta(y_{n,l}\mid \vx, \vy_{n,<l})}{\pi_\text{ref}(y_{n,l}\mid \vx, \vy_{n,<l})}\nonumber\\
    \nabla_\theta\gamma_{n,l}&=\frac{\pi_\theta(y_{n,l}\mid \vx, \vy_{n,<l})}{\pi_\text{ref}(y_{n,l}\mid \vx, \vy_{n,<l})}
                               \nabla_\theta\log\pi_\theta(y_{n,l}\mid \vx, \vy_{n,<l})\nonumber\\
                             &\triangleq c_{n,l}\nabla_\theta\mathcal{L}_{n,l},
    \label{eq:GRPO_gamma}
\end{align}
where $c_{n,l}$ is a positive constant and $\mathcal{L}_{n,l}$ is the standard cross-entropy loss.

In GRPO, the tokens and responses involved in the system become more complicated.
We will try our best to follow the notations used in this section.
First, we still use $\vx_o$ and $\vy_o$ to represent the question and response we are observing.
Note that we have $N$ roll-outs for each question $\vx$, with $N^+$ of them being correct ($r=1$) and $N^-$ of them being incorrect ($r=0$).
The ratio of the correct and incorrect responses is denoted as $\rho^+$ and $\rho^-$, respectively.
Usually, we have $\rho^++\rho^-=1$.
Then, we use $\vy_{n}^+$ (or $\vy_{n}^-$) to represent a correct (or incorrect) response.
Note that $n\in[N]$, and we define two index subsets $\mathcal{I}^+$ and $\mathcal{I}^-$ to represent the positive and negative responses, respectively.
Hence $\vy_n^+$ is a concise expression for $\vy_n,n\in\mathcal{I}^+$.
To make the notations concise, we omit the index of the question.
That is because in this part, we would analyze the token-level influence, we use $\vchi^+_{n,<l}=[\vx_{n},\vy^+_{n,<l}]$, which is the context for the $l$-th token of the $n$-th response, which happens to be correct.

We now check the token-wise learning dynamics in GRPO.
Specifically, we would like to track 
\[
    \Delta^t\left(\vchi^+_{o,<m}\right) \triangleq \log\pi_{\theta^{t+1}}\left(y\mid \vx_{o},\vy^+_{o,<m}\right) - \log\pi_{\theta^{t}}\left(y\mid \vx_{o},\vy^+_{o,<m}\right),
\]
i.e., the model's confidence change on a correct response's $m$-th token, which should be a $V\times 1$ vector.
Using the first-order Taylor expansion, we have (we abuse the inner product notion here to represent matrix multiplication, which is easier to read, e.g., $\langle \vB,\vb\rangle = \vB\vb$):
\begin{align}
    \underbrace{\Delta^t\left(\vchi^+_{o,<m}\right)}_{V\times 1}
    &= \left< 
            \underbrace{\nabla_{\theta}\log\pi_{\theta^{t}}\left( y\mid \vx_{o},\vy^+_{o,<m}\right)}_{V\times d}, 
            \underbrace{\theta^{t+1}-\theta^{t}}_{d\times 1}
        \right>\nonumber\\
    &= \left<\nabla_{\theta}\log\pi_{\theta^{t}}\left( y\mid \vx_{o},\vy^+_{o,<m}\right), -\eta\nabla_\theta\hat{\mathcal{L}}_\text{GRPO}\right>\nonumber\\
    &= \left<\nabla_{\theta}\log\pi_{\theta^{t}}\left( y\mid \vx_{o},\vy^+_{o,<m}\right),   
             \frac{\eta}{\sum_{n=1}^{N} |\vy_n|}\sum_{n=1}^{N}  \sum_{l=1}^{|\vy_n|} \nabla_{\theta}A_{n}\gamma_{n,l}(\theta) \right>\nonumber\\
    &=\eta\frac{1}{\sum_{n=1}^{N} |\vy_n|} \sum_{n,l}A_n  
      \left<\nabla_{\theta}\log\pi_{\theta^{t}}\left( y\mid \vx_{o},\vy^+_{o,<m}\right),c_{n,l}\nabla_\theta\mathcal{L}_{n,l}\right>,\nonumber
\end{align}
where $c_{n,l}\nabla_\theta\mathcal{L}_{n,l}$ is defined in the gradient of $\gamma$ in \cref{eq:GRPO_gamma}.
We also simplify $\hat{\mathcal{L}}_\text{GRPO}$ by ignoring the $\mathsf{clip}$ operation.

From the equation above, we can definitely apply a similar decomposition as \cref{eq:sec5_LLM_SFT_LD} and then conduct force analysis.
However, to get more insights about the accumulated influence at the token level, we instead focus on the following question:
\begin{center}
    \textit{If GRPO wants to improve the confidence of those correct roll-outs, \\ which group of tokens strengthens (or impedes) this?}
\end{center}

Then, analyzing the whole $V$ dimensions of $\Delta^t\left(\vchi^+_{o,<m}\right)$ would be redundant, and also make the decomposition more complicated.
As a result, we only focus on the label's dimension of $\Delta^t$, and define our influence function as
\begin{align}\nonumber
    \ve_{ y_{o,m}}^\top \Delta^t\left(\vchi^+_{o,<m}\right)
    &= \frac{\eta}{\sum_{n=1}^{N} |\vy_n|} \sum_{n,l}A_n c_{n,l} 
      \left<\ve_{ y_{o,m}}^\top\nabla_{\theta}\log\pi_{\theta^t}\left( y\mid \vx_{o},\vy^+_{o,<m}\right),\nabla_\theta\mathcal{L}_{n,l}\right>\nonumber\\
    &= \eta\frac{1}{\sum_{n=1}^{N} |\vy_n|} \sum_{n,l}A_n c_{n,l} 
      \left<\nabla_{\theta}\mathcal{L}_{o,m},\nabla_\theta\mathcal{L}_{n,l}\right>
    \label{eq:ld_grpo}
\end{align}
where the loss function $\mathcal{L}$ here is the standard cross-entropy loss.
If the right-hand side of this decomposition is \textit{positive}, then the confidence of the $m$-th token in $\vy^+_{o}$ will increase, which is what we expect.
Then, our task is to separate all other tokens (determined by $n,l$) into two groups (recall that $c_{n,l}>0$, but $A_n$ can be both positive and negative):
\begin{itemize}
    \item the beneficial ones: $n,l$ where $A_n\left<\nabla_{\theta}\mathcal{L}_{{o,m}},\nabla_\theta\mathcal{L}_{{n,l}}\right> >0$
    \item the harmful ones: $n,l$ where $A_n\left<\nabla_{\theta}\mathcal{L}_{{o,m}},\nabla_\theta\mathcal{L}_{n,l}\right> <0$.
\end{itemize}

Obviously, learning the tokens ($l$-th token in the $n$-th response) in the beneficial group will enhance the model's confidence on correct responses, and vice versa.
The conditions described above align well with our intuition. 
We can interpret the inner product between two token-level gradients as a measure of similarity between $\vy_{o,m}^+$ and $\vy_{n,l}$—analogous to the $\mathcal{K}$ term discussed previously, but now we would like to offer more direct insight into its directional effect. 
When the inner product is large and positive, it indicates that the two tokens are highly similar.
In such cases, the impact of the token $\vy_{n,l}$ on the update of $\vy_{o,m}^+$ depends on the sign of $A_n$: similar tokens from positively rewarded responses ($A_n > 0$) exert a beneficial influence, while those from negatively rewarded responses ($A_n < 0$) can be detrimental.

\subsection{Theory: Grouping Tokens via Influence Scores}
\label{sec:case2_05_influence_score}
Building on the discussion above, we aim to identify the harmful tokens that hinder learning by considering both their advantage values and their similarity to a given observing token. 
Mitigating the influence of such tokens could accelerate model training and help resolve the LLDisp issues, i.e., tokens in a positive response exhibit insufficient increase \citep{deng2025effect}. 
However, computing the exact inner product between gradients is often intractable in practice, especially within large-scale models.
To address this, we turn to simplification and approximation.

In this subsection, we revisit the Unconstrained Feature Model (UFM) introduced in the DPO analysis and propose a tractable approximation to the gradient inner product, which we term the ``influence score''. 
This score provides a computationally efficient proxy for measuring the directional impact of one token on another during training.

Recall \cref{eq:ld_grpo}, our task is to calculate $\left<\nabla_{\theta}\mathcal{L}_{o,m},\nabla_\theta\mathcal{L}_{n,l}\right>$.
Considering UFM, in which $\theta\triangleq[\vw;\vh]$, where $\vw\in\mathbb{R}^{V\times d}$ and $\vh\in\mathbb{R}^{d\times 1}$ are both trainable parameters.
We use the same subscript system for both $\vh$, $\vchi$, and $\vy$.
So, the feature vector of $\vchi^+_{o,<m}$ is hence $\vh^+_{o,<m}$.
Then, $\nabla_\theta\mathcal{L}_{o,m}=\nabla_{\theta}\log\pi\left( y\mid \vx_{o},\vy^+_{o,<m}\right)$ can be written as $\nabla_{[\vw;\vh]}\log \sigma\left(\vw\vh^+_{o,<m}\right)$, where $\sigma(\cdot)$ is the $\mathsf{Softmax}$ function.
Our inner product then becomes:

\begin{equation}
    \left< \nabla_{[\vw;\vh]}\log\sigma_{y_{o,m}}(\vw\vh_{o,m}),    \nabla_{[\vw;\vh]}\log\sigma_{y_{n,l}}(\vw\vh_{n,l}) \right>.
    \label{eq:ld_grpo_inner_product}
\end{equation}
    
Note that calculating this inner product is essentially equivalent to calculating the gradient flow of $\frac{d}{dt}\log\pi^t_{y_i}$, as in \citet{razin2024unintentional} and our \citet{deng2025effect}.
We use a notation system that aligns better with those in previous chapters here.

Then, after several steps (details can be found in \cref{sec:app2:grpo}), we have
\begin{equation}
    \left<\nabla_{\theta}\mathcal{L}_{{o,m}},\nabla_\theta\mathcal{L}_{{n,l}}\right>
    \approx\alpha_{(o,m),(n,l)}
    \triangleq(\ve_{y_{o,m}} - \pi_{y_{o,m}})^\top (\ve_{y_{n,l}} - \pi_{y_{n,l}})\vh^\top_{n,l}\vh_{o,m},
    \label{eq:GRPO_alpha}
\end{equation}
which is the multiplication of two inner products of the updating and observing tokens (one is on their $\mathcal{G}$-terms and another is on their hidden representations). 
In summary, combining with \cref{eq:ld_grpo}, we claim if $A_n\alpha_{(o,m),(n,l)}>0$, learning the $l$-th token of the $n$-th response is beneficial, and vice versa.

\subsection{Explanation: Lazy Likelihood Displacement}
Recall that our goal is to use learning dynamics to gain insight into the LLDisp problem in GRPO, wherein some correct (positive) responses fail to receive sufficient increases in probability during training.
The force analysis presented above suggests that tokens from both positive and negative responses may contribute to LLDisp.
However, suppressing learning from tokens in positive responses is undesirable, as these tokens also exert beneficial ``pull-up'' forces that reinforce their own correct outputs.

This observation motivates a more targeted strategy: identifying and mitigating harmful tokens specifically within negative responses. 
To assess the viability of this approach, we conduct an experiment to evaluate whether attenuating the influence of selected negative responses can effectively alleviate the LLDisp problem.

In particular, we visualize the change in model confidence on positive responses (averaged over $N^+$ responses for each problem) after a single update step, using an enlarged learning rate to amplify the observed trends. 
We begin by randomly sampling 40 questions from the MATH dataset~\citep{hendrycks2021measuring}, and for each question, we generate $N = 8$ roll-out responses using the initialized model $\pi_{\theta^0}$. 
Each response is then evaluated by a verifier to determine its correctness, which allows us to compute both its reward $r_n$ and advantage $A_n$, along with its initial model probability $\pi_{\theta^0}(\vy_o^+\mid \vx)$.

We then finetune the model on these responses under two different settings: (1) standard GRPO, and (2) an ablation variant in which all negative responses are ignored, referred to as the \texttt{Pos Only} setting. 
After training, we re-evaluate the same responses under the updated model to obtain $\pi_{\theta^1}(\vy_o^+\mid\vx)$. 
The difference 
\[
    \mathsf{Gap}(\vx)\triangleq\mathbb{E}_{o\in\mathcal{I}^+}[\pi_{\theta^1}(\vy_o^+\mid \vx) - \pi_{\theta^0}(\vy_o^+ \mid \vx)]
\]
reflects the model’s improvement in confidence for positive responses on average.
If $\mathsf{Gap}(\vx)$ is smaller than a threshold (e.g., the average increase on all other questions), it might suffer from the LLDisp problem.

As shown in \cref{fig:chap5_exp_grpo_01}, the \texttt{Pos Only} setting results in significantly larger confidence gains for most questions compared to GRPO. 
This clearly indicates that the inclusion of negative responses, while potentially beneficial for exploration, can hinder the learning of correct predictions.\footnote{The exploration-exploitation trade-off in RL-LLM is discussed in our paper \citep{deng2025token}.}
Hence, a big gap between $\mathsf{Gap}_\text{GPRO}(\vx)$ and $\mathsf{Gap}_\text{Pos-Only}(\vx)$ is also a sign of LLDisp.

\begin{figure}[t]
    \begin{center}
    \centerline{\includegraphics[width=1\columnwidth,trim=0 0 0 0, clip]{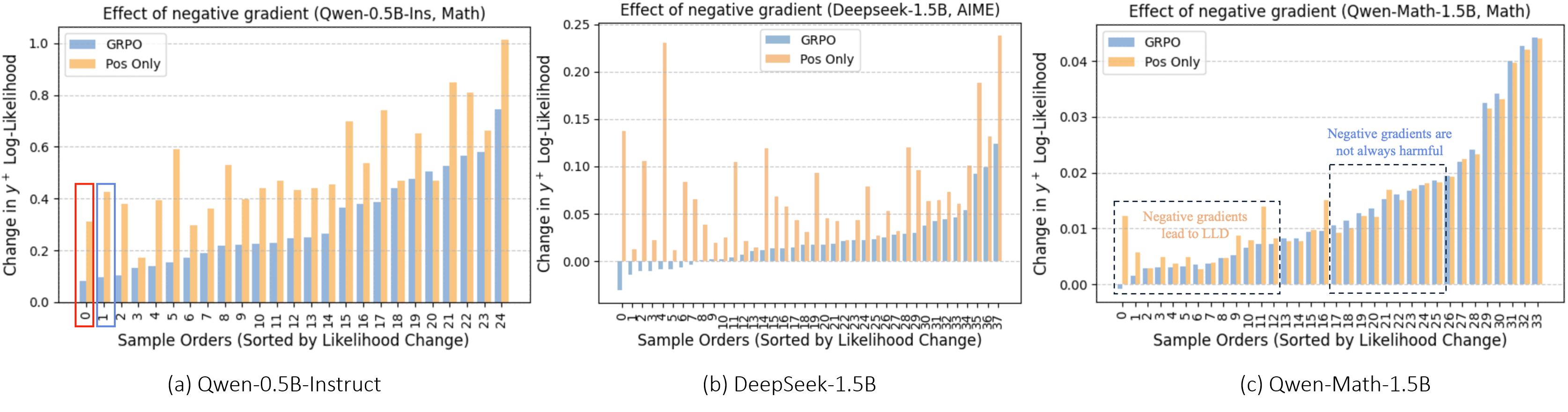}}
    \caption{Averaged positive response confidence change after one update.}
    \label{fig:chap5_exp_grpo_01}
    \end{center}
\vskip -0in
\end{figure}

Then, should we abandon the use of negative responses and rely solely on \texttt{Pos Only} training? 
The answer is no. 
As demonstrated by prior work like \citet{guo2025deepseek,shao2024deepseekmath}, negative examples play a critical role in penalizing undesirable behaviors, thereby enhancing the robustness and generalization of the model’s reasoning capabilities.

What is needed, instead, is a more fine-grained approach: a token-level assessment of influence within negative responses. 
This is precisely where our token-wise learning dynamics framework offers valuable guidance: enabling the identification and mitigation of only the truly harmful components within negative roll-outs, without discarding their overall utility.

This requirement can be readily implemented by computing the average influence score across all positive responses.
Based on \cref{eq:GRPO_alpha}, the average version is defined as:
\begin{align}
    &\hat{\alpha}_{n,l} \triangleq \mathop{\mathbb{E}}_{o\in\mathcal{I^+},m\in[|\vy_o^+|]}[\alpha_{(o,m),{(n,l)}}]\nonumber\\
    &\hat{\alpha}^+_{n,l} \triangleq \hat{\alpha}_{n,l}, \forall n\in\mathcal{I}^+;\quad  
    \hat{\alpha}^-_{n,l} \triangleq \hat{\alpha}_{n,l}, \forall n\in\mathcal{I}^-
\end{align}
where we also define the $\hat{\alpha}^+_{n,l}$ (and $\hat{\alpha}^-_{n,l}$) for the tokens with positive (and negative) $A_n$, respectively.
Since we only care about the negative responses, those tokens with $\hat{\alpha}^-_{n,l}>0$ are our targets.

Before conducting a token-level experiment, we first check whether our proposed influence score is a good approximation of the LLDisp problem at the question level.
We believe those positive responses suffering from LLDisp might also have a relatively big $\mathbb{E}_{n,l}[\hat{\alpha}^-_{n,l}]$.
To validate our hypothesis, we select 100 questions from the AIME dataset (1983–2023)~\citep{hendrycks2021measuring} and compute two quantities for each question: the observed LLDisp metric $\mathsf{Gap}(\vx)$ and the proposed influence score $\mathbb{E}_{n,l}[\hat{\alpha}^-_{n,l}]$. 
We then evaluate the degree of alignment between these two measures by computing the overlap among the Top-K ranked problems under each metric.

\begin{figure}[t]
    \begin{center}
    \centerline{\includegraphics[width=0.9\columnwidth,trim=0 0 0 0, clip]{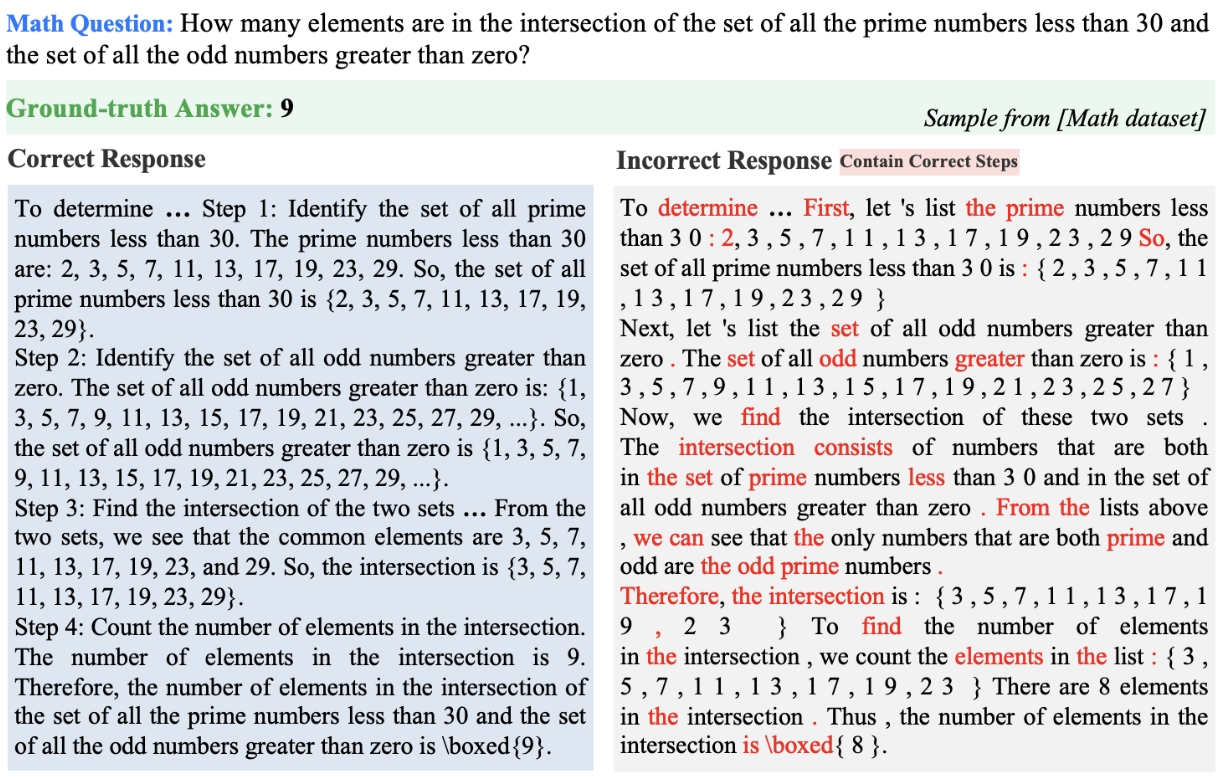}}
    \caption{Highlight the tokens in the harmful group, i.e., those should not be penalized in the negative response.}
    \label{fig:chap5_grpo_neg_token}
    \end{center}
\vskip -0in
\end{figure}

As a baseline, we also calculate the overlap between $\mathsf{Gap}(\vx)$ and a randomly shuffled ranking. 
The results, summarized in \cref{tab:grpo_topk_overlap}, show that the Top-K lists obtained from $\mathbb{E}_{n,l}[\hat{\alpha}^-_{n,l}]$ have significantly higher overlap with $\mathsf{Gap}(\vx)$ than with the random ranking. 
This suggests that our proposed influence score is an effective and computationally efficient proxy for identifying potential LLDisp issues \textit{before} the finetuning starts!

\begin{table}[h]
    \centering
    \resizebox{0.7\textwidth}{!}{
       \begin{tabular}{c|cc|cc}
       \toprule
        \textbf{Top-K} & \multicolumn{2}{c|}{\texttt{Qwen-1.5B-deepseek}} & \multicolumn{2}{c}{\texttt{Qwen-1.5B-math}} \\
         & $\mathbb{E}_{n,l}[\hat{\alpha}^-_{n,l}]$ & Random & $\mathbb{E}_{n,l}[\hat{\alpha}^-_{n,l}]$ & Random \\
        \hline
        10 & 50\% & 17.5\% & 60\% & 21.3\% \\
        15 & 75\% & 26.3\% & 75\% & 31.9\% \\
        \hline
        \end{tabular}
    }
    \caption{Ranking correlation experiment.}
    \label{tab:grpo_topk_overlap}
\end{table}

Furthermore, we demonstrate some concrete example of what those ``harmful tokens'' in negative responses looks like in \cref{fig:chap5_grpo_neg_token}.
A key insight is that many tokens in negative samples may be logically or step-correct, yet GRPO penalizes them indiscriminately using a shared advantage term $A_n$.
For instance, the red-highlighted tokens in incorrect responses also appear in correct ones.
Consequently, the accumulated gradient force on these tokens would cancel out in expectation, since $\mathbb{E}_n[A_n]=0, \forall n\in[N]$.

\subsection{Algorithm: Negative Token Hidden Reward}
Based on the analysis above, the way to refine GRPO is obvious: we can just penalize those harmful negative tokens by a ratio $\beta<1$.
We then set
\[
    \hat{A}_{n,l} = \beta A_{n}, \forall n,l\in \{n\in\mathcal{I}^-:\ \hat{\alpha}^-_{n,l}>\tau\},
\]
and leave the advantage of other tokens unchanged.
Since we want to mitigate the influence of such tokens, we usually have $0\leq\beta\leq1$.
We design $\tau$ as a tunable hyperparameter to control how many tokens are considered harmful.
Following our paper \citep{deng2025effect}, we also call this value negative token hidden reward (NTHR), which is easy to combine with GRPO and other variants.

We begin by evaluating the effectiveness of NTHR by examining $\mathsf{Gap}(\vx)$ using a bar plot format similar to that in \cref{fig:chap5_exp_grpo_01}. 
In this experiment, we set $\beta=0$, meaning that we entirely exclude the influence of negative tokens whose estimated influence scores $\hat{\alpha}^-_{n,l}$ exceed a predefined threshold $\tau$.
We compare NTHR against two baselines. 
The first is standard GRPO, corresponding to the case where $\beta=1$, in which all negative tokens are fully retained.
The second is GRPO+Random, where the same number of negative tokens as in the NTHR setting are randomly selected and their influence is removed.

As shown in \cref{fig:chap5_exp_grpo_02}, NTHR significantly alleviates the LLDisp problem, outperforming both baselines.
Moreover, we observe that the extent of improvement varies across different questions and appears to be correlated with the model’s reasoning capability. 
This observation opens up an interesting direction for future research: adapting token-level mitigation strategies to the difficulty of the task and the model's current learning state.

\begin{figure}[t]
    \begin{center}
    \centerline{\includegraphics[width=1\columnwidth,trim=0 0 0 0, clip]{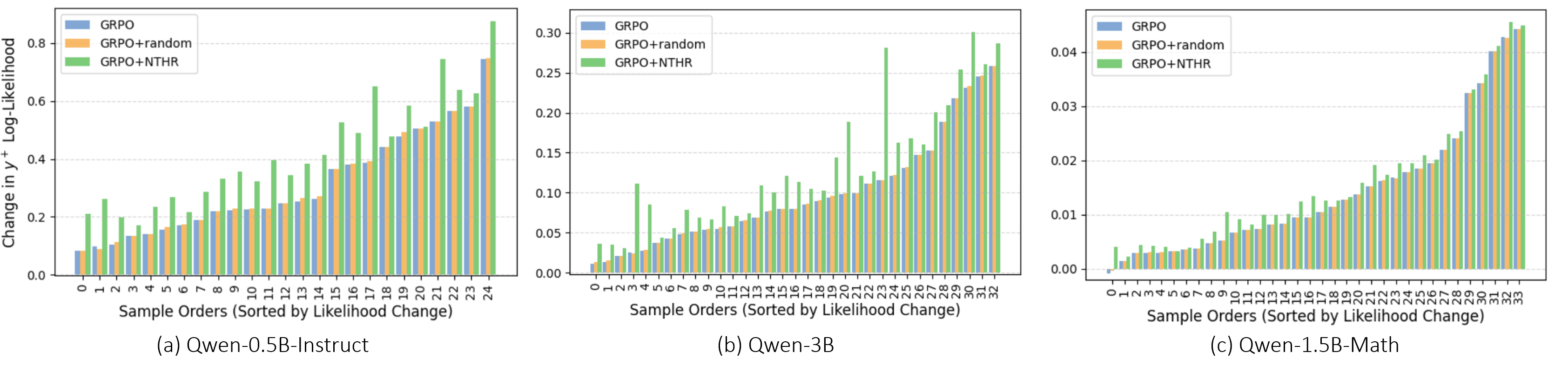}}
    \caption{Impact of NTHR on $\mathsf{Gap}(\vx)$ in different models.}
    \label{fig:chap5_exp_grpo_02}
    \end{center}
\vskip -0in
\end{figure}

Having demonstrated the impact of negative gradients on likelihood dynamics and the effectiveness of our method in mitigating the LLDisp issue, we now turn to evaluating the effect of negative gradients on downstream model performance. 
To this end, we finetune models of varying sizes on the MATH dataset (levels 3–5)~\citep{hendrycks2021measuring}, and assess their performance using greedy decoding across five math reasoning benchmarks: AIME24, AMC, MATH500, Minerva, and Olympiad.
Our results, summarized in \Cref{tab:greedy-math-merged}, show that GRPO augmented with NTHR consistently improves performance across all model sizes.
We also verified some claims mentioned in this section, e.g., throwing all the negative tokens might harm the system (as shown by the \texttt{Pos Only} results).

In summary, our proposed NTHR method, motivated by a token-level learning dynamics analysis, effectively identifies and attenuates harmful negative updates in standard GRPO. 
By reducing the effective learning rate (eLR) of the most detrimental tokens, we mitigate the LLDisp phenomenon to some extent. 
Our experiments on math reasoning tasks across models from 0.5B to 3B demonstrate that NTHR enhances the performance of GRPO over many datasets, offering a practical and principled improvement to preference-based finetuning.

\begin{table}[ht]
\centering
\resizebox{1\textwidth}{!}{
\begin{tabular}{lccccc c}
\toprule
\textbf{Base model + Method} & \textbf{AIME24} & \textbf{AMC} & \textbf{MATH500} & \textbf{Minerva} & \textbf{Olympiad} & \textbf{Avg.} \\
\midrule
\texttt{Qwen2.5-Math-1.5B} & & & & & & \\
\quad Base & 3.3 & 20.0 & 39.6 & 7.7 & 24.9 & 19.10 \\
\quad GRPO & 13.3 & 57.5 & \textbf{71.8} & 29.0 & 34.1 & 41.14 \\
\quad Pos Only & 10.0 & 57.5 & 70.6 & 30.1 & 31.0 & 39.84
\\
\rowcolor{blue!6}
\quad +NTHR & \textbf{16.7} & 57.5 & 70.8 & \textbf{30.5} & \textbf{34.2} & 41.94 \\
\midrule
\texttt{Qwen2.5-0.5B-Ins} & & & & & & \\
\quad Base & 0.0 & 2.5 & 33.4 & 4.4 & 7.0& 9.46 \\
\quad GRPO & 0.0 & 7.5 & 33.8 & \textbf{9.2} & 8.1 & 11.72 \\
\rowcolor{blue!6}
\quad +NTHR & 0.0 & \textbf{10.0} & \textbf{36.6} & 8.1 & 8.6 & \textbf{12.66} \\
\midrule
\texttt{Qwen2.5-1.5B-Ins} & & & & & & \\
\quad Base & 0.0 & 22.5 & 53.0 & 19.1 & 20.7 & 23.06 \\
\quad GRPO & 3.3 & 32.5 & 57.2 & 18.8 & \textbf{23.0} & 26.96 \\
\rowcolor{blue!6}
\quad +NTHR & \textbf{6.7} & \textbf{35.0} & 58.8 & 21.0 & 20.9 & \textbf{28.48} \\
\midrule
\texttt{Qwen2.5-Math-1.5B (deepscaler)} & & & & & & \\
\quad Base & 3.3 & 20.0 & 39.6 & 7.7 & 24.9 & 19.10 \\
\quad GRPO & 10.0 & 42.5 & 72.4 & \textbf{32.4} & 31.9 & 37.80 \\
\rowcolor{blue!6}
\quad +NTHR & \textbf{16.7} & 47.5 & 73.2 & 29.4 & 31.4 & 39.60 \\
\midrule
\texttt{Qwen2.5-3B} & & & & & & \\
\quad Base & 10.0 & 37.5 & 58.6 & 26.1 & 24.6 & 31.36 \\
\quad GRPO & 6.7 & 35.0 & 66.6 & 31.2 & \textbf{29.9} & 33.88 \\
\rowcolor{blue!6}
\quad +NTHR & \textbf{10.0} & \textbf{47.5} & 65.6 & 31.6 & 26.8 & \textbf{36.30} \\
\bottomrule
\end{tabular}
}
\caption{Results across selected math benchmarks for different \texttt{Qwen2.5} models and methods. The proposed method consistently provides average performance gains on various models.}
\label{tab:greedy-math-merged}
\end{table}

\section{Conclusion and Discussions}
\label{sec:case2_05}
This chapter covers a wide range of topics related to LLM finetuning, making it somewhat dense in content. 
To aid the reader in recalling the key ideas, we summarize the main takeaways from various perspectives in this concluding section.
In addition to the core insights, we also highlight some related works, outline potential future research directions, and discuss the limitations of our proposed learning dynamics framework.
To facilitate comparison across different finetuning methods and help readers recall the main contributions of each section, we provide a summary in \cref{fig:chap5_conclusion}.

\subsection{Summary of each Section}
This chapter begins with an overview of LLM finetuning methods in \cref{sec:case2_01}. 
By discussing both the benefits and potential risks introduced by finetuning, we highlight the need for a fine-grained behavioral analysis of model dynamics.
Such an analysis, we argue, could offer new insights into the advantages and limitations of various LLM finetuning algorithms.

We then apply our learning dynamics framework to the most fundamental finetuning approach, i.e., supervised finetuning (SFT), in \cref{sec:case2_02}. 
With the help of the teacher-forcing mechanisms and causal mask, we successfully extend the force analysis and $\mathcal{AKG}$ decomposition to a multi-token, auto-regressive setting.
We then circumvent the intractability of modeling the full response space by instead tracking the evolution of model confidence on selected representative responses. 
This framework provides novel explanations for various empirical observations during SFT training, including non-monotonic confidence trajectories and distinct inflection points in the learning curves of different responses types. 
Furthermore, we offer a novel hypothesis for why hallucinations are often exacerbated after SFT.
This explanation is also supported by multiple recent empirical studies.

Building on this foundation, \cref{sec:case2_03} extends the framework to Direct Preference Optimization (DPO), a more complex off-policy finetuning algorithm.
Despite its structural differences, DPO also updates parameters via $\nabla_\theta \log \pi$.
Hence, we can still derive an $\mathcal{AKG}$-style decomposition for DPO through gradient chain rule expansions.
A key distinction is the presence of explicitly imposed negative gradients. 
Our theoretical and empirical analyses reveal that these negative gradients introduce a counter-intuitive squeezing effect, which causes several undesirable behaviors, such as the degradation of both $\log \pi_{\theta^t}(\vy^+)$ and $\log \pi_{\theta^t}(\vy^-)$ over time, and the amplification of the so-called ``repeater phenomenon''. 
Motivated by this analysis, we propose a simple yet effective pipeline to mitigate these issues.

In the final section, \cref{sec:case2_04}, we extend our analysis to Group Relative Policy Optimization (GRPO), an even more complex and dynamic finetuning method based on on-policy reinforcement learning. 
GRPO is commonly applied in reasoning-intensive tasks. 
Unlike SFT or DPO, GRPO exhibits varying equivalent learning rates (eLRs) at the token level, and its training responses are updated in every iteration. 
These characteristics make the response-level force analysis insufficient. 
To address this, we develop a token-level extension of our framework. 
By analyzing a simplified Unconstrained Feature Model (UFM) (previously used in the analysis of the squeezing effect), we define a token-wise influence score capable of identifying harmful tokens \textit{before} the parameter updates. 
Leveraging this insight, we introduce Negative Token Hidden Reward (NTHR), a novel method that selectively reduces the influence of the most harmful tokens.
Our experimental results across multiple math reasoning benchmarks demonstrate the consistent superiority of NTHR over standard GRPO, highlighting both the practical impact and theoretical grounding of our approach.

\subsection{More on Negative Gradient}
Recently, numerous variants of DPO and GRPO have been proposed to address various limitations of the original algorithms.
These follow-up works are often designed to enhance robustness in specific settings, targeting issues such as instability, misalignment, or poor generalization.
In \cref{fig:chap5_related_works}, we present a force analysis of several representative variants, while the detailed learning dynamics decompositions of their loss functions are provided in Appendix B of our paper~\citep{ren2025learning_dynamics_LLM}.

Although these methods are motivated by diverse intuitions and theoretical considerations, we find that many of them can be naturally (and partly) interpreted within our learning dynamics framework. 
Interestingly, despite being independently developed, a common theme emerges across these approaches: they often implicitly address similar core aspects of preference optimization, such as making the training signals more on-policy or restraining the influence of negative gradients on low-confidence tokens.

Self-play Finetuning (SPIN, proposed by~\citet{chen2024self}) serves as a representative example of how our proposed framework provides a new perspective on understanding where its benefits come from.
SPIN is motivated by the observation that responses generated by smaller, pre-finetuned models are generally of lower quality compared to those found in SFT datasets. 
Leveraging this insight, the authors treat self-generated responses $\vy$ as negative examples and the ground-truth responses $\vy_u^+$ from the SFT dataset as positive examples.
Inspired by the loss structure of Generative Adversarial Networks (GANs,~\citealp{goodfellow2014generative}), SPIN maximizes the model's ability to distinguish between these two types of responses. 
After a series of derivations, the authors arrive at a loss function that is formally equivalent to the standard DPO objective, with the key distinction being that all $\vy_u^-$ responses are generated on-policy from the current model, while $\vy_u^+$ are drawn from the static dataset. 
As in \cref{fig:chap5_related_works}, the on-policy nature of the negative responses in SPIN mitigates the squeezing effect efficiently.

\begin{figure}[t]
    \begin{center}
    \centerline{\includegraphics[width=1\columnwidth,trim=0 0 0 0, clip]{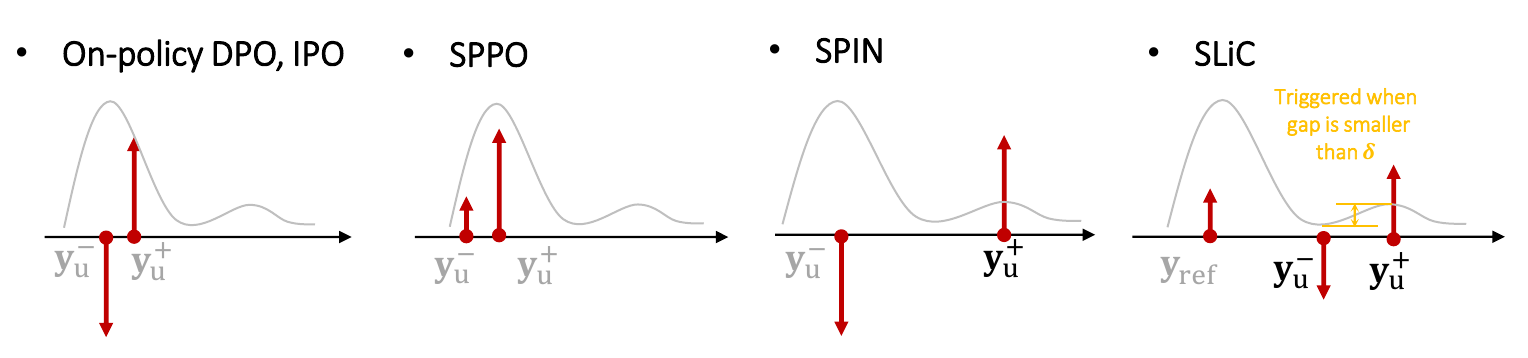}}
    \caption{Force analysis for some DPO variants.}
    \label{fig:chap5_related_works}
    \end{center}
\vskip -0in
\end{figure}

Furthermore, the SPPO method (a follow-up to SPIN) introduced by~\citet{wu2024self} takes this idea a step further by entirely eliminating the use of negative gradients. 
From the perspective of force analysis, the separation between $\vy_u^+$ and $\vy_u^-$ can still be achieved by emphasizing their relative strength, without explicitly penalizing the model through negative updates. 
Like SPIN, SPPO also relies on on-policy sampling, leading to more efficient and stable learning dynamics.

Besides the previous examples where the algorithm alters the negative gradient on a global level, here are more examples where a response-level or token-level relative strength is controlled.
Smaug \citep{pal2024smaug}, is a good example.
Theoretically, it assumes that $\vy_u^+$ and $\vy_u^-$ differ by only a single token. 
Then, the model’s confidence in $\vy_u^+$ after the shared prefix is guaranteed to decrease during standard DPO training (also an instance of the LLDisp problem). 
To mitigate this, it dynamically increases the learning rate for the positive gradient by up to 50 times (i.e., their $\lambda=50$) whenever the paired $\vy_u^-$ is detected to have a low confidence.
This strategy aligns well with our force analysis: applying a strong negative gradient in valley regions is harmful, and thus, amplifying the positive gradient (equivalently, decreasing the relative eLR for negative tokens) can alleviate the influence. 

Dynamically controlling the relative strength between positive and negative updates is also common in reinforcement learning, where it is often implemented through reward shaping with a small offset. 
For instance, a binary reward between ${(0, 1)}$ can be transformed into ${(-0.5, 0.5)}$ or ${(-0.1, 0.9)}$, the latter of which makes the effective learning rate for positive examples 9 times greater than that for negative ones.
Notably, the dynamic advantage mechanism in GRPO performs a similar implicit modulation, adapting the gradient magnitude based on relative outcome quality.

Dynamically controlling eLR also extends to domains where only negative updates are applied, such as machine unlearning. 
A representative example is the Negative Preference Optimization (NPO, \cite{zhang2024negative}), which modifies the standard DPO loss by removing the positive gradient part entirely. 
In this setting, the eLR for the negative gradient is dynamically adjusted based on the model’s prediction confidence. 
This illustrates how gradient magnitude control, informed by prediction certainty, is a general and effective tool for many tasks, like alignment, reasoning, and unlearning.

In summary, our analysis reveals that negative gradients play an increasingly critical role in LLM finetuning, particularly in the context of reinforcement learning-based methods. 
Despite their widespread use, a comprehensive understanding of how these gradients influence model behavior remains limited within the community. 
We hope that the learning dynamics framework and empirical findings presented in this chapter, as well as in our recent work~\citep{ren2025learning_dynamics_LLM, deng2025effect}, serve as a foundation for future research. 
We believe that a deeper exploration of these dynamics holds the potential to yield new insights and guide the development of more robust and efficient finetuning algorithms.

\subsection{Limitations}
We now briefly discuss the limitations of our analytical framework in the context of LLM. 
The primary issue lies in the gap between changes in model confidence and the model's actual capabilities. 
For instance, an increase in confidence on certain correct responses does not necessarily imply improved alignment, generalization, or reasoning ability. 
The answer to such questions is likely negative in many cases.

In traditional supervised learning tasks, monitoring the loss function is often effective for evaluation, since the test data is typically drawn from the same distribution as the training set. 
In such cases, the empirical risk closely approximates the true risk. 
However, this assumption does not generally hold for tasks involving LLMs, which are considerably more complex and harder to build distributional models for. 
Thus, a systematic evaluation approach is needed to bridge this conceptual gap.

Another limitation is about the assumptions we made in our $\mathcal{AKG}$ decomposition.
Specifically, in most of our analysis of \cref{sec:case1,sec:case2}, we usually only consider $\|\mathcal{K}^t\|_F$.
But in fact, $\mathcal{K}^t$ is a $V\times V$ matrix which can alter the direction of the updating force vector provided by $\mathcal{G}^t$.
Hence, a more detailed analysis is required. 
We believe that theoretical results from the NTK literature may provide useful insights for future work.

Furthermore, in the context of LLM finetuning (and more general finetuning methods), the model often operates in the so-called ``kernel regime,'' where $\mathcal{K}$ evolves slowly and remains relatively stable, thereby justifying approximations such as those made in the UFM setting \citep{yang2020tensor}. 
However, in the ``feature learning regime,'' where internal representations (and hence $\mathcal K$) undergo substantial changes during training, the usefulness of our analysis becomes less certain.
We will revisit this question in the next chapter by returning to a simpler supervised learning scenario.

The final limitation we want to highlight is the oversimplification of the Transformer architecture in our current framework. 
As an empirical theory intended to provide insights into model behavior, incorporating force analysis alongside core architectural components, such as the attention mechanism in Transformers, could yield deeper and more informative interpretations. 
In particular, we believe that integrating our framework with the circuit hypothesis \citep{lynch2025agentic} presents a promising direction for future exploration.

\begin{figure}[h]
    \begin{center}
    \centerline{\includegraphics[width=1\columnwidth,trim=0 0 0 0, clip]{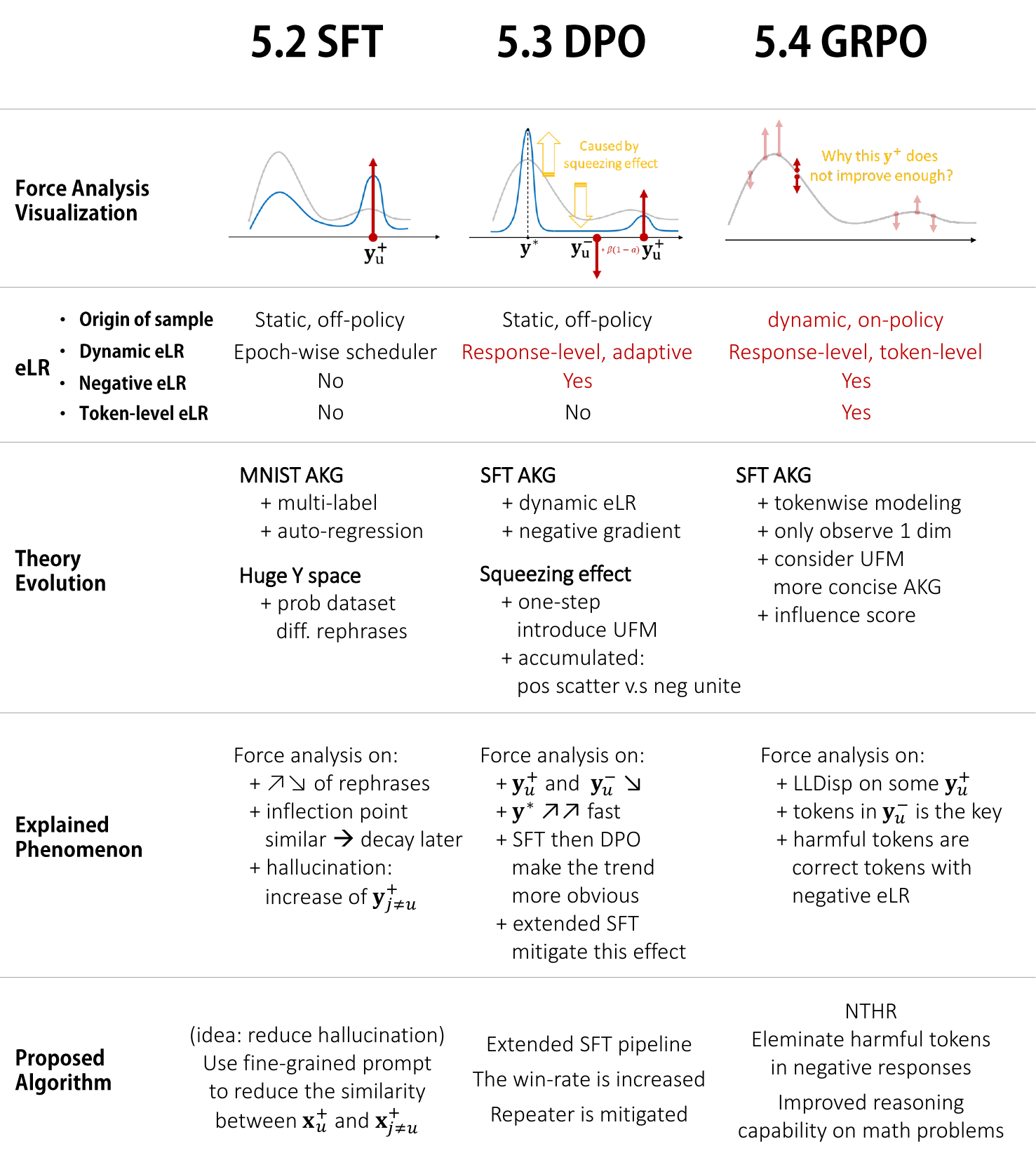}}
    \caption{A brief summary of this chapter.}
    \label{fig:chap5_conclusion}
    \end{center}
\vskip -0in
\end{figure}

\chapter{Feature Adaptation in Transfer Learning}
\label{sec:case3}
In previous chapters, the ``objective'' in our force analysis has been the model’s log-probability.
In this chapter, we extend this framework by treating the \textit{hidden representations} as the objective.
This shift is motivated by the belief that internal representations (or even some hidden structures like attention patterns, circuits, etc.) are critical for understanding how deep neural networks internalize knowledge and generalize to downstream tasks.

Compared to the LLM finetuning discussed in \cref{sec:case2}, where the model often operates in a ``NTK regime'' with almost static features, the transfer learning scenarios considered in this chapter involve a nontrivial degree of feature adaptation.
To facilitate a clean and interpretable analysis, we adopt a simplified and controlled setting: selecting a well-pretrained backbone model, followed by attaching a task head and then performing full-parameter finetuning.
This is usually named as a pretraining-finetuning pipeline, which is quite common in different fields.
Although our current focus is on representation-level adaptation, the techniques presented here can be extended to more complex and abstract internal mechanisms.

Despite this controlled setup, our experiments span a diverse array of tasks, including self-supervised image learning, image classification, molecular property prediction (both classification and regression), and image segmentation.
They also involve architectures ranging from CNNs and GNNs to Transformers. 
The goal is to verify that the principles uncovered via learning dynamics are both general and model-agnostic. 
We believe this justifies the potential applicability of our analysis to more complex systems such as LLMs.

Finally, much of the content in this chapter is based on ``\textit{How to Prepare Your Task Head for Finetuning}'' \citep{ren2023howto}. 
While the paper emphasizes practical finetuning strategies, this chapter reframes the analysis through the lens of learning dynamics, with a particular focus on feature adaptation.
In short, we argue that effective feature adaptation requires the energy provided by the gradient updates to be both \textit{consistent} and \textit{sufficient}.
Moreover, this force should not be overly distorted by the randomly initialized task head in the early stages of finetuning.

\section{Introduction}
\label{sec:case3_01}
In the era of deep learning, pretraining a model on a large dataset and adapting it to downstream tasks is a popular workflow.
With the help of a large amount of data and huge computing resources, the pretrained model can usually provide beneficial features for the downstream tasks.
Such a framework is proven to be efficient and effective in many domains and tasks, e.g., natural language processing, computer vision, graph-based learning, and so on.
Although different variants of finetuning methods are widely applied, e.g., direct finetuning, finetuning after linear probing \citep{kumar2022fine}, side-tuning \citep{zhang2020side}, using different learning rates for different layers \citep{zhang2021revisiting}, and more, a detailed understanding of how features are adapted during finetuning remains elusive.

Among the various topics within the pretrain-finetune pipeline, feature adaptation plays a central role in understanding how knowledge embedded in a pretrained model is transferred to downstream tasks. 
It is a common belief that most of the knowledge is embedded in the model's hidden features, and the task head plays a role of linear separation on top of them \citep{goodfellow2016deep}.
This topic is also intimately connected to broader challenges in transfer learning and out-of-distribution (OOD) generalization.

Unlike many existing approaches in transfer learning that assume a specific distribution shift and derive generalization bounds accordingly, this chapter focuses on the \textit{geometric adaptation} of hidden representations in deep neural networks. 
Leveraging the force analysis and learning dynamics framework introduced earlier, we provide a detailed characterization of feature adaptation using core notions such as ``energy'' and ``direction.''
As in other topics discussed in this thesis, the lens of learning dynamics offers valuable insight into the process of feature adaptation and further motivates the design of effective training strategies in practice.

\section{Learning Dynamics and Feature Adaptation}
\label{sec:case3_02}

\begin{figure}[t]
\vskip -0in
    \begin{center}
    \centerline{\includegraphics[width=1\columnwidth,trim=0 10 0 10, clip]{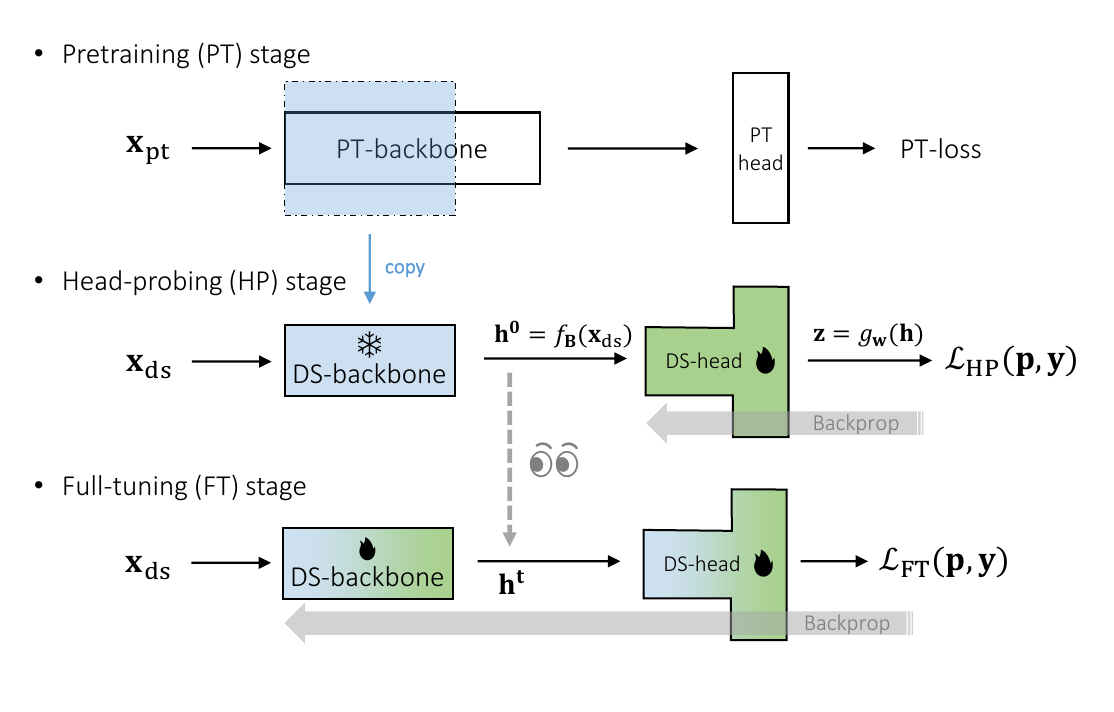}}
    \caption{The training pipeline considered in this chapter.}
    \label{fig:chap6_fthp}
    \end{center}
\vskip -0in
\end{figure}

\subsection{Background and Motivation}
We start by introducing a popular transfer learning pipeline studied by \citet{kumar2022fine}.
As demonstrated in \cref{fig:chap6_fthp}, it consists of three sequential stages:
\begin{itemize}
    \item Pre-training (PT) stage: This stage involves training a backbone network, either through standard supervised learning or self-supervised learning methods.
    \item Head-probing (HP) stage: We copy all (or part) of the pretrained backbone from the PT stage and attach a randomly initialized task-specific network, referred to as the task head. During this stage, the backbone parameters are frozen, and only the task head is trained via gradient descent on the objective $\mathcal{L}_{\text{HP}}(\vp, \vy)$.
    \item Full-tuning (FT) stage: Building on the HP stage, we unfreeze the backbone and jointly train the entire network, including both the backbone and the task head, using the objective $\mathcal{L}_{\text{FT}}(\vp, \vy)$.
\end{itemize}
Our primary focus in this chapter is the adaptation of the hidden feature representation, specifically examining how it evolves before and after the full-tuning stage, as the emoji in \cref{fig:chap6_fthp}.

The motivation comes from rethinking the analysis of \citet{kumar2022fine}.
Specifically, they claim that directly full-tuning the model with a randomly initialized task head can distort the well-learned features, and hence harm the out-of-distribution (OOD) performance significantly.
They thus propose to freeze the backbone and train the task head to its convergence first, and then finetune the full network together.
They theoretically show that a fully converged task head will make hidden features almost stable during finetuning, and hence preserve the useful features for the downstream tasks.

\begin{figure}[t]
\vskip -0in
    \begin{center}
    \centerline{\includegraphics[width=0.6\columnwidth,trim=0 0 0 0, clip]{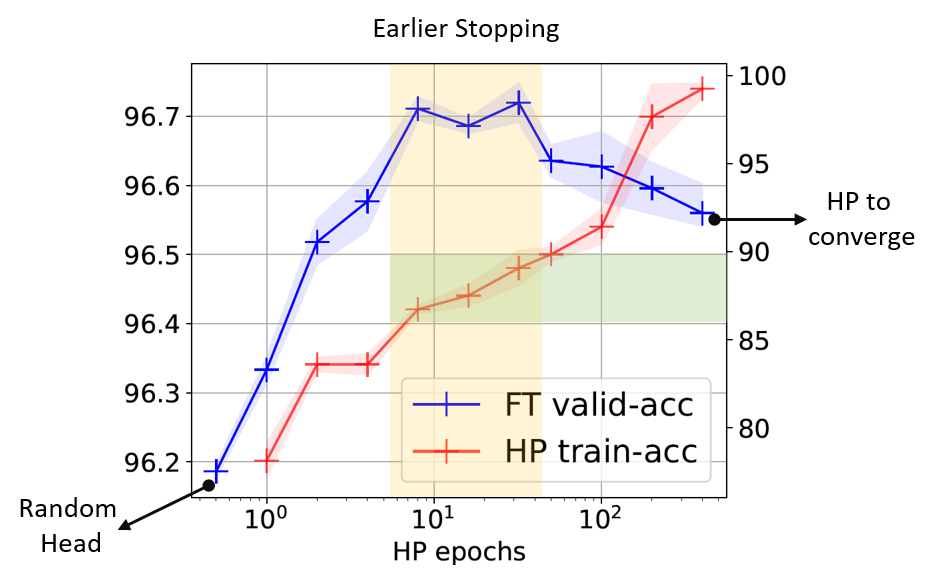}}
    \caption{A transfer learning experiment on ReNet18 (pretrained on ImageNet-1K and finetuned on STL10 \citep{coates2011analysis}.}
    \label{fig:chap6_start_exp}
    \end{center}
\vskip -0in
\end{figure}

However, intuitively, we sometimes might expect the hidden features to adapt to the downstream tasks to some extent, depending on the similarity between the source and target distributions. 
This idea is supported by the experiments presented in \cref{fig:chap6_start_exp}. 
While a fully converged HP stage does outperform vanilla finetuning (consistent with the results of \citealt{kumar2022fine}), we observe that early-stopped HP, where the HP-accuracy is roughly 85\%, often achieves even better performance.

In these experiments, all FT runs are conducted in identical settings, with the only difference being the number of epochs spent in the HP phase.
This suggests that the initialization of the task head at the beginning of FT plays a critical role: it implicitly determines the extent to which the representation is allowed to adapt, thereby influencing the final downstream performance.

\subsection{Formalizing the Feature Adaptation}
We now formalize the feature adaptation to describe how it changes.
Although in practice the backbone model $f_{\theta}$ is often highly complex, such as a stack of convolutional blocks in ResNet or several attention layers in Transformers, for theoretical tractability, we consider a simplified Overparameterized Model (OPM) in this chapter. 
This model is closely related to the Unconstrained Feature Model (UFM) introduced in \cref{sec:case2}, where the task head $g_{\vw}$ is modeled as a simple linear mapping with $\vw \in \mathbb{R}^{V \times h}$.
The main difference between them is that in OPM, we further specify the hidden feature $\vh$ to an overparameterized linear mapping.

In the OPM, the feature representation is defined as $\vh = \vB \vx\in\mathbb{R}^{h\times 1}$, where $\vB \in \mathbb{R}^{h \times d}$ is a large matrix parameterizing the backbone function $f_{\vB}$, and $\vx\in\mathbb{R}^{d\times 1}$ is the input data. 
Importantly, the hidden dimension $h$ is allowed to be significantly larger than the input dimension $d$, which justifies the term ``overparameterized.''

We assume that $\vB^0$ is a pretrained backbone obtained from a source task, such that the initial feature $\vh^0 = \vB^0 \vx$ already provides a meaningful representation.
The objectives of the post-training are to leverage this pretrained backbone $\vB^0$, along with an initialized task head $\vw^0$, to achieve strong performance on the target downstream task.
Since our target is to observe the adaptation of $\vh$ and the head-probing stage does not change $\vh$, we first focus on the full-tuning stage.
Supposing we have $N$ data samples $(\vx_n,y_n)_{n=1}^N$, the cross-entropy training loss function is then written as:
\begin{equation}
    \loss_\text{FT}=\frac{1}{N}\sum_{n=1}^N\ve_{y_n}^\top\log\vp;\ \vp\triangleq\sigma(\vz);\ \vz\triangleq g(\vh)=\vw\vh; \ \vh\triangleq f(\vx)=\vB\vx
    \label{eq:chap6_overparm_loss}
\end{equation}
where $\sigma(\cdot)$ is a standard $\mathsf{Softmax}$ function and $\vp_n=\sigma(\vw\vB\vx_n)$.

Then, the one-step evolution of the hidden representation $\vh^t$ is\footnote{The proof is similar to previous chapters; details are given in the paper \citep{ren2023howto}.}:
\begin{equation}
    \vh_{o}^{t+1} - \vh_{o}^{t} = -\eta\frac{1}{N} \sum_{n=1}^{N} 
    \left( \underbrace{\mathcal{K}^{t}(\vx_o,\vx_u)}_{\text{slow-change}}
           \underbrace{\left( \nabla_{\vh}\vz^t(\vx_u) \right)^\top}_\text{direction}
           \underbrace{\left( \vp^{t}(\vx_u)- \ve_{y_n}\right)}_\text{energy}
    \right)
      + \mathcal{O}(\eta^2),
    \label{eq:chap6_z_dynamics}
\end{equation}
where $\mathcal{K}^t(\vx_o,\vx_u)$ is still the eNTK and the $\mathcal{G}^t$ term is also $\vp-\ve$.

The primary differences between the decomposition used here and those studied in \cref{sec:fundamentalLD,sec:case1,sec:case2} lie in the following two aspects.
First, since we now track the evolution of the hidden representation $\vh$ rather than the log-probability $\log \vp(\vx_o)$, the $\mathcal{A}$ term is no longer present in the decomposition (or, equivalently, it is simply the identity).
Second, by explicitly splitting the network at the hidden embedding $\vh$, the $\mathcal{K}$ term now corresponds to the eNTK of the backbone network $f_{\vB}$.
The parameters mapping $\vh$ to the output are captured separately as the “direction” term mentioned earlier, specifically $\nabla_{\vh} \vz^t(\vx_u)$.
Under the simplified linear setting assumed in this analysis, this gradient becomes constant across all $n$ examples, i.e., $\nabla_{\vh} \vz^t(\vx_u) = \vw^t$.

Note that $\vw^t$ plays a very important role in our analysis here: in the forward-pass, $\vw^t$ is operating as a hyper-plane that separates the features of different classes; in the backward-pass, it is projecting the energy provided by $\mathcal{G}^t$ towards representations.
Furthermore, this $\vw^t$ also keeps changing during the full-tuning process, which makes the analysis more complicated. 
We will elaborate on this later in this chapter.

We then check the head-probing stage.
From the model's perspective, all the influence imposed via the head-probing stage is through the value of $\vw^0$, i.e., the task head before the start of full-tuning.
This $\vw^0$ is also the task head at the end of the HP stage, which is then correlated with the HP-accuracy, as demonstrated by the red curve in \cref{fig:chap6_start_exp}.
If we further assume that the data used in HP and FT stages are identical, this HP-accuracy is also the starting FT-accuracy at $t=0$.

\subsection{Factors that Matter: Energy and Direction}
Now, we can first describe the evolution of the hidden representation from $t=1$ to $t=T$ using the L2-norm in the following proposition:
\\
\begin{prop} \label{prop:AIE}
    $\mathbb{E}_{\vx_{n}}\|\vh^T_{n}-\vh^0_{n}\|_2\leq c\cdot E_\mathit{AIE}$,
    where $E_\mathit{AIE}\triangleq\mathbb{E}_{\vx_{n}}\|\vp^0_{n}-\ve_{y_n}\|_2$ is the Average Initial Energy (AIE).
    Here $c$ is a constant, $T$ is the FT epochs, and $\vp^0_{n}$ is the model's prediction of sample $\vx_{n}$ at the beginning of FT (or equivalently, at the end of HP).
\end{prop}

Although this bound is not so tight\footnote{You can find it in Appendix A of our paper \citep{ren2023howto}. Note that the notations used there differ slightly from those in this thesis.}, the underlying phenomenon is quite inspiring. 
Specifically, the result states that if we interpret the average L2-distance between the model’s predictions and their corresponding one-hot labels as a form of ``energy,'' then the L2-norm of the feature adaptation is upper bounded by this energy. 
In other words, greater energy induces greater feature adaptation, and conversely, lower energy limits the extent of adaptation.
Since the loss function in the HP stage is non-decreasing, the number of HP epochs (defined by $\tau$) is a reasonable approximation of AIE, as illustrated in \cref{fig:chap6_energy_direction}-(a).
To draw this figure, we collect the value of $[\vp^0_n]_{y_n}$, i.e., the predicting probability of the correct label, for all training examples and put them in 10 bins.
It is clear that a larger $\tau$ makes the associated $\vp^0$ closer to the one-hot distribution.

\begin{figure}[t]
     \centering
     \begin{subfigure}[b]{0.40\textwidth}
         \centering
         \includegraphics[width=\textwidth]{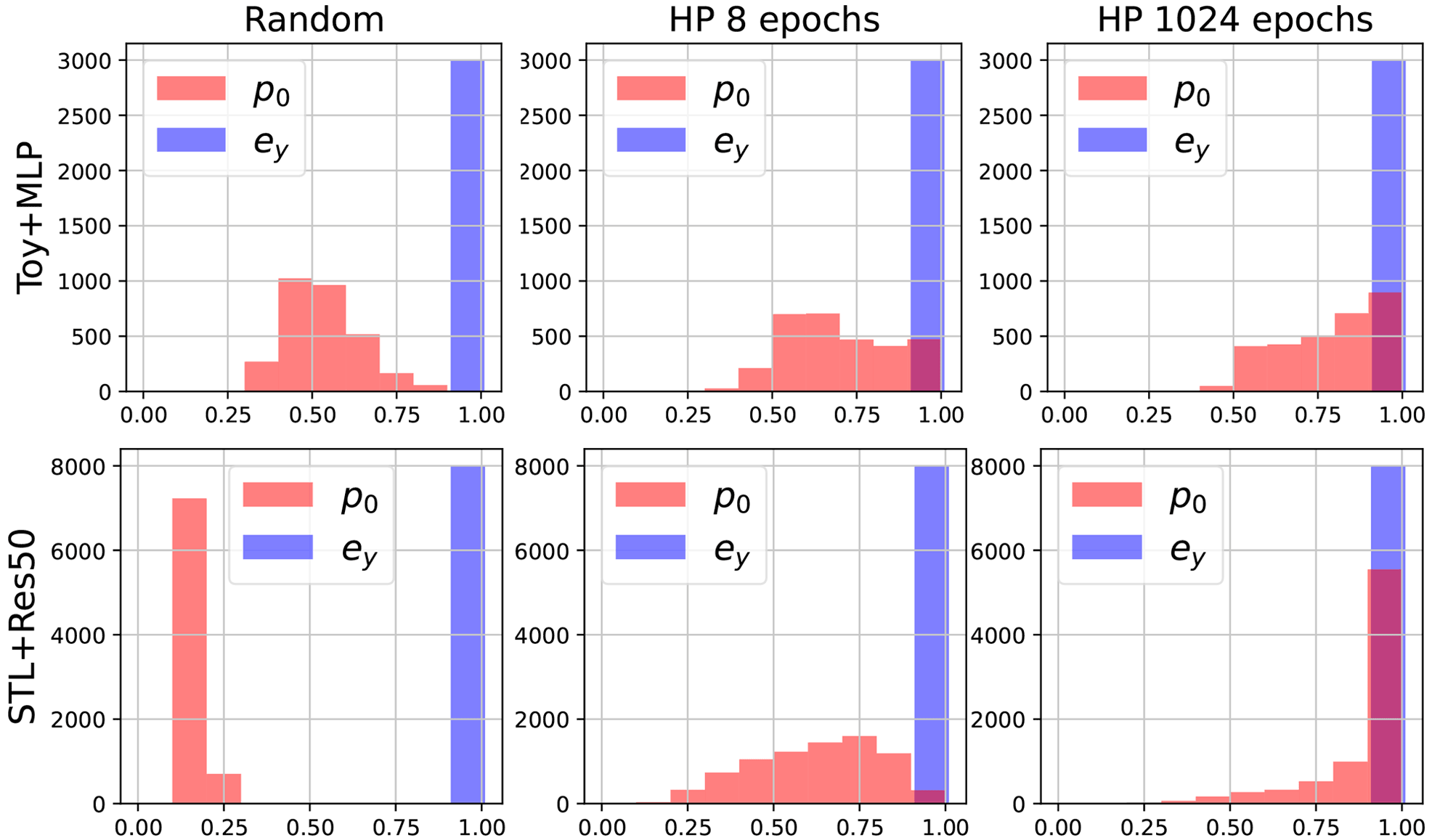}
         \caption{Influence of $\tau$ on energy.}
     \end{subfigure}
     \hfill
     \begin{subfigure}[b]{0.55\textwidth}
         \centering
         \includegraphics[width=\textwidth]{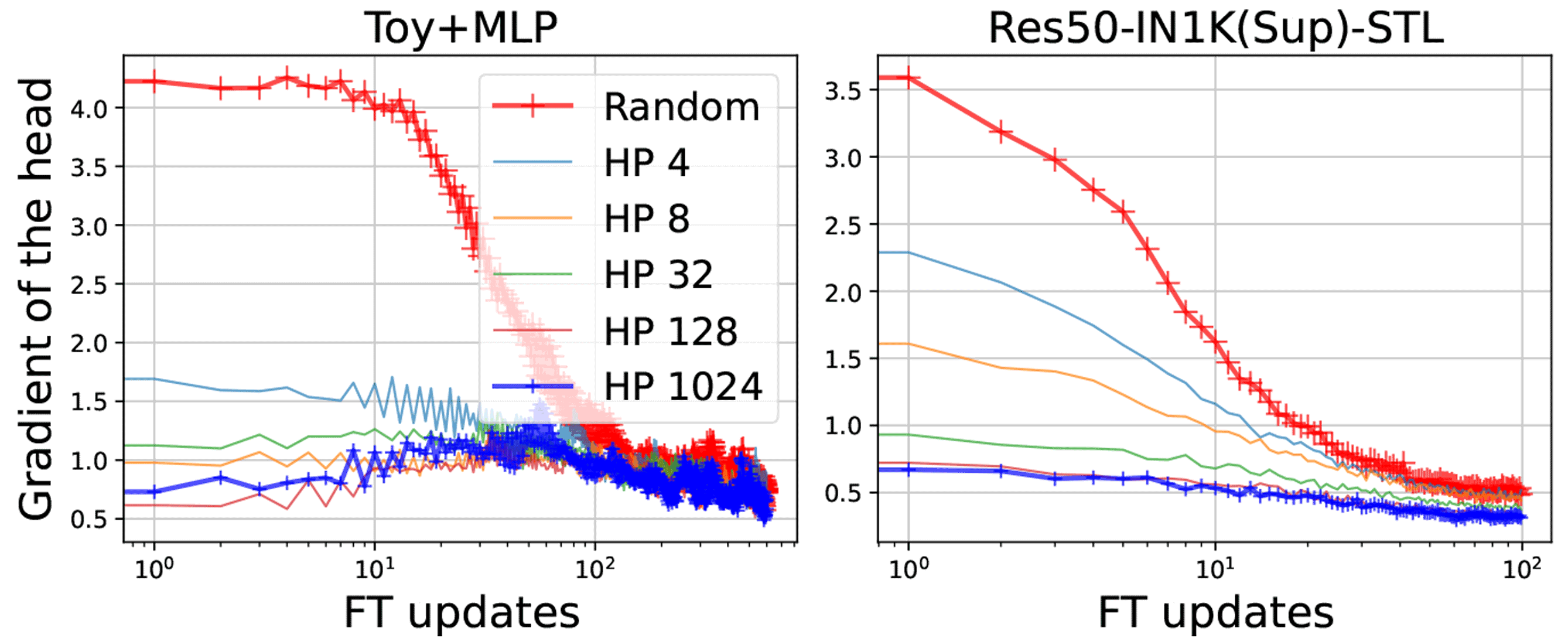}
         \caption{Influence of $\tau$ on direction.}
     \end{subfigure}
        \caption{Left: statistics of $[\vp_0]_y$ and $[\ve_y]_y$ under different HP epochs.
                 Right: approximated change of the ``direction'' term in \cref{eq:chap6_z_dynamics} under different $\tau$.
                 Toy+MLP means pretrain a 4-layer MLP on full MNIST, then transfer to a distorted subset of MNIST.
                 Res50-IN1K(Sup)-STL means pretrain a ResNet50 on ImageNet-1k using supervised classification, then transfer to a downstream task on STL10. 
                }
        \label{fig:chap6_energy_direction}
\end{figure}

Following Proposition \ref{prop:AIE}, we can link the adaptation of features to $\tau$ via AIE.
In the proof, we disentangle the dependence of the direction and energy terms using the Cauchy-Schwarz inequality,
i.e., $(\nabla_{\vh}\vz)^\top(\vp_0 - \ve_y)\leq\|\nabla_{\vh}\vz\|_2 \cdot \|\vp_0 - \ve_y\|_2$.
However, the direction term \cref{eq:chap6_z_dynamics} indeed plays an important role, especially at the beginning of finetuning.

To verify our hypothesis about the direction term, we depict the change of this term during finetuning in the right panel of \cref{fig:chap6_energy_direction}.
Actually, under our simplified setting, this quantity can be approximated by the norm of the gradients to $\vw^t$,
since $\lVert \nabla_{\vh}\vz^{t+1} - \nabla_{\vh} \vz^{t} \rVert_F^2
= \lVert \vw^{t+1} - \vw^t \rVert_2^2
= \eta^2 \lVert \nabla_{\vw}\loss \rVert_2^2$.
As we find $\|\vw^t\|_2$ changes little during the finetuning stage,
the large $\lVert \nabla_{\vw}\loss \rVert_2$ is more likely from a big direction change.
As illustrated by \cref{fig:chap6_energy_direction}-(b), when $\tau=0$, $\nabla_{\vh}\vz^t$ changes a lot at the beginning of FT, which can make $\vh$ change in inconsistent directions.
When $\tau=1024$, the direction term changes only a little through finetuning.
Introducing the directional change of $\vw^t$ allows for a more comprehensive understanding of feature adaptation.
We will elaborate on this in the next section.

\subsection{How Features Evolve under Different Settings}
The findings regarding the direction term inspire us to look deeper at the difference between a strong adaptation (e.g., $\tau=0$) and a mild adaptation (e.g., $\tau=32$).
To get an overall picture of $\vh$'s change, only observing $\|\vh^T-\vh^0\|_2$ is not enough.
Hence, we analyze the following four quantities related to the similarity between the features before finetuning and afterwards:
$\|\vh^T-\vh^0\|_2$, $\|\vh^T\|_2$, $\langle \vh^T,\vh^0\rangle$, and $\cos(\vh^T,\vh^0)$.
Before delving deep into how we get insights on how these values change, we first describe these interesting trends using the following proposition.
\\
\begin{prop}[Informal]\label{chap6_prop2}
    In an overparameterized two-layer linear model,
    when $\tau$ increases (the AIE decreases),
    $\|\vh^T-\vh^0\|_2^2$ monotonically decreases while $\|\vh^T\|_2^2$ and $\langle \vh^T,\vh^0\rangle$ exhibit a \textbf{quadratic} trend.
    The trend of $\cos(\vh^T,\vh^0)$ is hard to predict, but there is a phase in which this value increases fast.
\end{prop}

\begin{figure}[t]
\vskip -0in
    \begin{center}
    \centerline{\includegraphics[width=0.9\columnwidth,trim=0 0 0 0, clip]{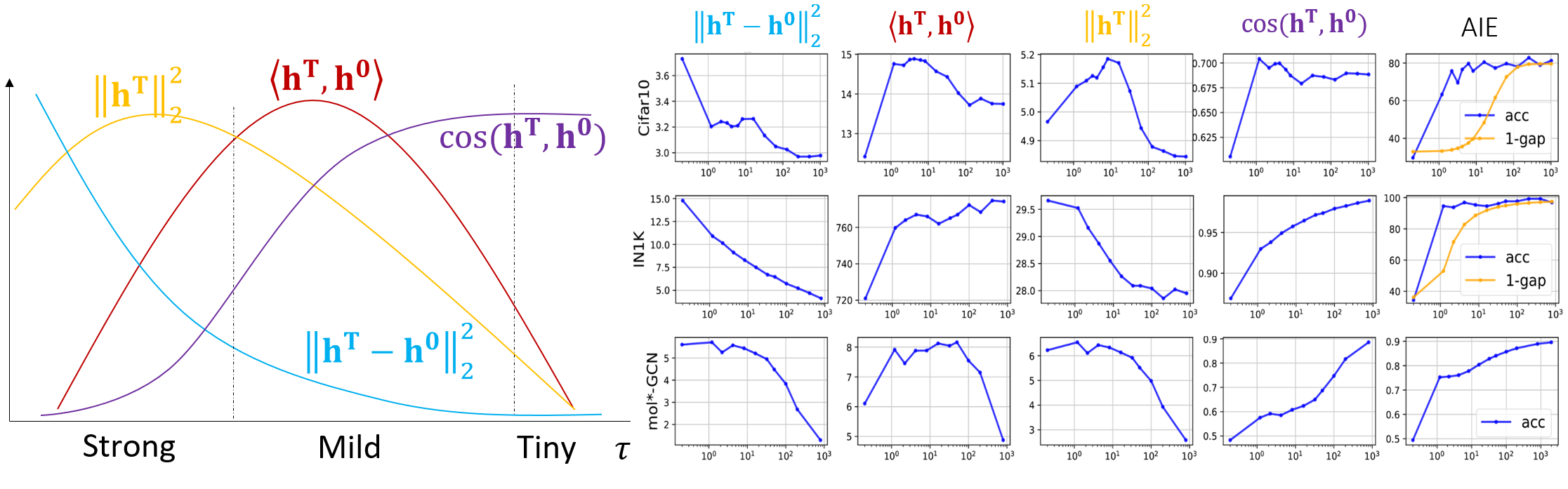}}
    \caption{How $\vh^T$ changes(assume full-tuning $T$ epochs to converge) compare with $\vh^0$ when head-prob $\tau$ epochs.}
    \label{fig:chap6_zadapt_5metrics}
    \end{center}
\vskip -0in
\end{figure}

\cref{fig:chap6_zadapt_5metrics} provides a clear demonstration of how these four quantities change when head-probing $\tau$ epochs.
We also empirically verify these trends in different settings, e.g., training a ResNet18 on CIFAR10, training a ResNet50 on a subset of IN1K, and training a Graph Convolution Network (GCN, \cite{kipf2016semi}) on a Molecular-pcba dataset (a benchmark with multi-label classification task, \cite{hu2020ogb}).
The results align with our proposition quite well.

With the help of this proposition, we can conceptualize feature adaptation as a geometric transformation in vector space. 
In the upper-left panel of \cref{fig1:chap6_zadapt_3D_real}, we compare various adapted features $\vh^T$ originating from the same initial feature $\vh^0$, each corresponding to a different initialization $\vw^0$ obtained from the head-probing stage at different epochs.

When $\tau$ is relatively large (implying lower energy), we observe only mild or tiny adaptation. 
As shown in the curves of \cref{fig:chap6_zadapt_5metrics}: $\|\vh^T - \vh^0\|_2^2$ is small; $\|\vh^T\|_2$ and $\langle \vh^T, \vh^0 \rangle$ remains similar to $\vh^0$; and $\cos(\vh^T, \vh^0)$ changes only slightly. 
This indicates that $\vh^T$ has been slightly stretched along the direction of $\vh^0$, preserving directional alignment.

\begin{figure}[t]
  \centering
  \begin{minipage}[t]{1\linewidth}  
      \centering
      \label{fig:subfig5}\includegraphics[width=1\linewidth]{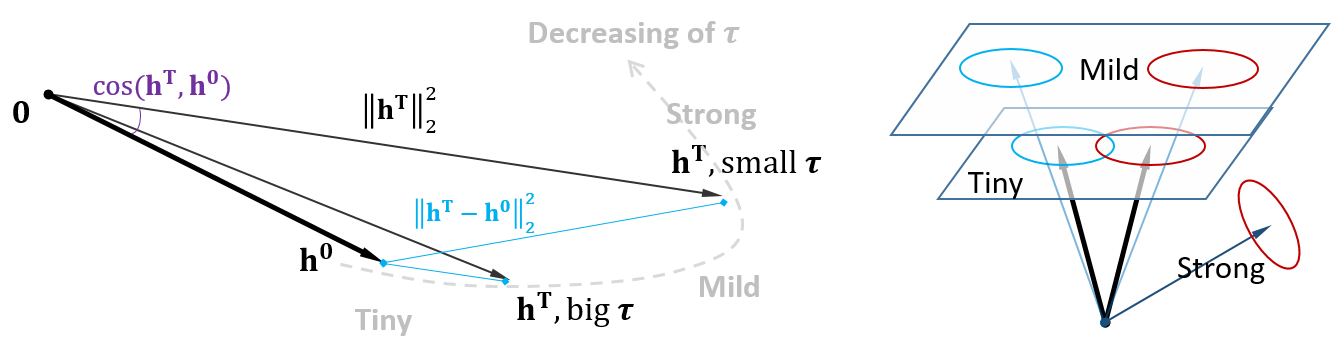}
  \end{minipage}
  \begin{minipage}[t]{1\linewidth}
      \centering
      \label{fig:subfig6}\includegraphics[width=1\linewidth]{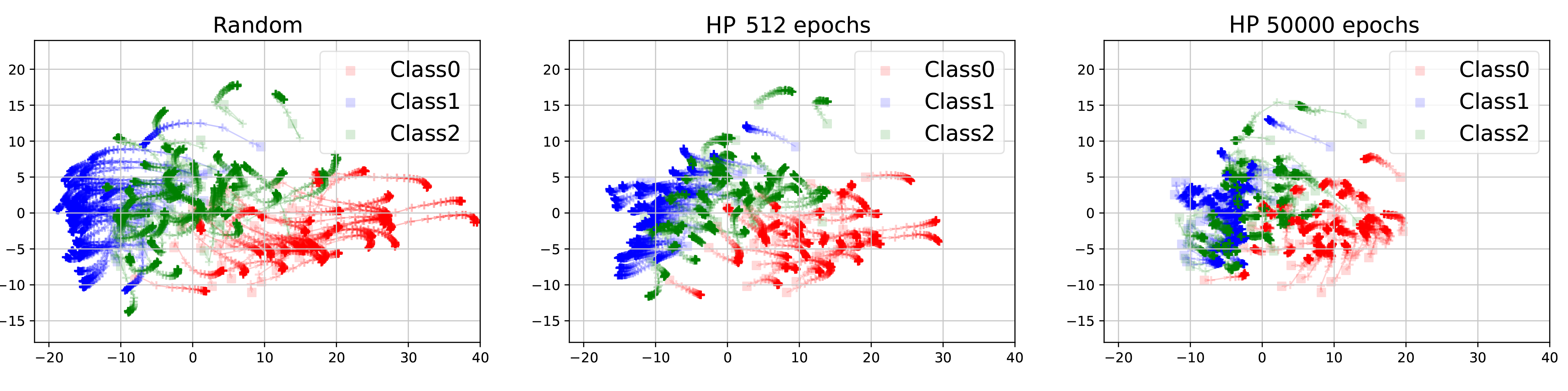}
  \end{minipage}
  \caption{Up: A vector-space illustration of how features adapt during finetuning. Bottom: Empirical verification of the vector-based intuition using a toy example. From left to right: strong, mild, and tiny adaptation.}
  \label{fig1:chap6_zadapt_3D_real}
\end{figure}

In contrast, when $\tau$ is small and the adaptation becomes strong, we observe a sharp drop in $\cos(\vh^T, \vh^0)$ and a corresponding decline in $\langle \vh^T, \vh^0 \rangle$. 
This signifies that $\vh^T$ is no longer aligned with $\vh^0$; the direction of the feature vector has changed substantially. 
In other words, the original representation has been significantly altered and now points to a different direction.

The upper-right panel of \cref{fig1:chap6_zadapt_3D_real} illustrates the three levels of feature adaptation in a 3D vector space. 
This visualization also provides intuition for why mild adaptation is often preferable for downstream tasks.
Specifically, when adaptation is minimal, the two feature clusters (e.g., the \textcolor{blue}{blue} and \textcolor{red}{red} ellipses) remain largely overlapping, limiting their separability and thereby reducing classification performance. 
In contrast, a mild adaptation gently stretches $\vh^T$, increasing the distance between clusters and enhancing their separability, which is a beneficial effect for downstream classification. 
However, in the strong adaptation case, the features are altered excessively, which can distort the original features and ultimately degrade performance.

This trend is further validated using a 3-class toy example shown at the bottom of \cref{fig1:chap6_zadapt_3D_real}, where we visualize the evolution of $\vh^t$ via t-SNE \citep{van2008visualizing}. 
The transparency of each point encodes the time step $t$ during the full-tuning process, and the three panels correspond to strong, mild, and tiny adaptation regimes, respectively. 
As clearly shown, the level of adaptation correlates with cluster separation and structure preservation, aligning with our theoretical intuition.

\subsection{Proof Sketch of the Trends}
Before turning to the practical strategies motivated by our analysis of feature adaptation, we first provide a sketch of the proof for Proposition~\ref{chap6_prop2} (a regression problem version), whose conclusion may not be immediately intuitive.
The detailed proof can be found in Appendix D of our paper \citep{ren2023howto}.

The proof begins by explicitly writing down the quantities we want to track:
\begin{itemize}
    \item Euclidean distance: $\bar{d}_\text{euc}\triangleq\mathbb{E}_{\vx\sim\mathcal{D}}[\|\vh^t - \vh^0\|_2^2]$;
    \item Dot product: $\bar{d}_\text{dot}\triangleq\mathbb{E}_{\vx\sim\mathcal{D}}[\vh^t\cdot\vh^0]$;
    \item Cosine: $\bar{d}_\text{cos}\triangleq\mathbb{E}_{\vx\sim\mathcal{D}}\left[ \frac{\vh^t\cdot\vh^0}{\|\vh^t\|_2*\|\vh^0\|_2}\right]$.
\end{itemize}
Then, under OPM that $\vh=\vB\vx$, we can simplify them to:
\begin{align}
    \bar{d}_\text{euc}&=M*\left[tr(\vB^t\cdot\vB^t) - 2tr(\vB^t\cdot\vB^0) + tr(\vB^0\cdot\vB^0)\right] \nonumber\\
    \bar{d}_\text{dot}&=M*tr(\vB^t\cdot \vB^0)\nonumber\\
    \|\vh^t\|_2^2&=M*tr(\vB^t\cdot\vB^t);\quad \|\vh^0\|_2^2=M*tr(\vB^0\cdot\vB^0),
\end{align}
where $M\triangleq\mathbb{E}_{\vx\sim\mathcal{D}}[tr(\vx\vx^\top)]$ is a constant in the following analysis.

We then use $\vb\triangleq\mathsf{vec}(\vB)\in\mathbb{R}^{h*d\times 1}$ to represent the column-wise concatenation of the $h\times d$ matrix.
Then, all the traces involved in the distance above can be simplified using the fact that $tr(\vB^t\cdot\vB^0)=\vb^t\cdot\vb^0$:
\begin{itemize}
    \item Euclidean distance: $\bar{d}_\text{euc}\propto\vb^t \cdot \vb^t - 2\vb^t\cdot\vb^0 + \vb^0\cdot\vb^0$;
    \item Dot product: $\bar{d}_\text{dot}\propto\vb^t \cdot \vb^0$;
    \item Cosine: $\bar{d}_\text{cos}\propto\frac{\vb^t \cdot \vb^0}{\sqrt{\vb^t \cdot \vb^t * \vb^0 \cdot \vb^0}}$.
\end{itemize}

Recall that we want to link the value of $\tau$ to different trends of these distances, and the value of $\tau$ can only influence these distances via the following path: $\tau\rightarrow\text{HP-loss}\rightarrow\vw_0\rightarrow\vB^t\rightarrow\vb^t\rightarrow\bar{d}$.
To get more insights, we approximate the behavior of this model in the NTK regime, in which the converged parameters can be analytically calculated.
Specifically, be applying Equation (8) in \cite{lee2019wide} and assuming $t\rightarrow\infty$, we can have:
\begin{align}
    \vb^t = \vb^0 -(\nabla_{\vb}\vq_0)^\top \kappa_0^{-1} (\vq_0-Y),
    \label{eq:chap6_ntk_theta}
\end{align}
where $\vq_0\in\R^{N\times 1}$ is the model's prediction on $N$ training samples at the begining of full-tuning, i.e., $t=0$.
And $\kappa_0=(\nabla_{\vb}\vq_0)^\top\nabla_{\vb}\vq_0\in\mathbb{R}^{N\times N}$ is the eNTK on $X$, where $X\triangleq[\vx_1,\dots,\vx_N]\in\mathbb{R}^{N\times d}$.
We use $\kappa$ here to represent the eNTK on the full dataset, which is calculated slightly differently from our pair-wise $\mathcal{K}^t$ term in other chapters.
Specifically, since we assume $V=1$ in this part, each $\mathcal{K}(\vx_i,\vx_j)$ can be considered as an element in $\kappa$ here.

Then, we can write down $\vb^t \cdot \vb^0$ or $\vb^t \cdot \vb^t$ by representing $\vb^t$ using $\vb^0$ following \cref{eq:chap6_ntk_theta}.
For example, after several steps, $tr(\vB^t\cdot\vB^0)$ can be written as the following quadratic form:
\begin{align}
    tr(\vB^t\cdot\vB^0) = \vb^t\cdot\vb^0 
    &= \vb^0\cdot\vb^0 - (\nabla_{\vb}\vq_0\cdot\vb^0)^\top \kappa_0^{-1} (\vq_0-Y)\nonumber\\
    &=\vb^0\cdot\vb^0-\vq_0^\top\kappa_0^{-1}(\vq_0-Y),
\end{align}
where $\nabla_{\vb}\vq_0\cdot\vb^0=\vq_0$ is proved in our Lemma 1 in \cite{ren2023howto}.
Since $\kappa_0 = X(\vB^0 \cdot \vB^0 + \|\vw_0\|_2^2\cdot I_{d\times d}) X^\top$ is usually positive definite, this quadratic form has a global maximum at $\vq^*_0=\frac{1}{2}Y$.

Similarly, we can calculate the critical point of $\vB^t\cdot \vB^t$ following a similar procedure.
We omit the details due to the page limits and only summarize the trend of different distances and quantities in \cref{tab:chap6_dist_trend} and \cref{fig:chap6_quadratic}.
Remember when $\vq_0=Y$, the gradient of $\vB^0$ is zero, hence $\vh^t=\vh^0$ (the ideal case in \cite{kumar2022fine}).
When $\vq_0$ moves linearly from $Y$ to $\bmf{0}$, we have the following three phases as demonstrated in \cref{tab:chap6_dist_trend}, which is roughly how we define ``tiny'', ``mild'', and ``strong'' adaptations.

\begin{itemize}
    \item Tiny, $Y\rightarrow\frac{1}{2}Y$: $\vh^t$ lengthens its norm with a slight change in the direction, hence $\bar{d}_\text{dot}$ also increases;
    \item Mild, $\frac{1}{2}Y\rightarrow\alpha Y$: the norm of $\vh^t$ keeps increasing, but its direction begins to change, which makes $\bar{d}_\text{dot}$ decrease;
    \item Strong, $\alpha Y\rightarrow \bmf{0}$: the norm of $\vh^t$ begins to decrease and the angle between $\vh^t$ and $\vh^0$ drastically changes, which makes $\vh^t$ changes a lot.
\end{itemize}

\begin{table}[h]
    \centering
    \begin{tabular}{c|ccccc}
    \toprule
     $\vq_0=$& $AIE$ & $\bar{d}_\text{euc}$   & $\bar{d}_\text{dot}$ & $\|\vh^t\|_2$ & $\bar{d}_\text{cos}$ \\ \hline
    \textbf{$Y\rightarrow \frac{1}{2}Y$} 
            & $\uparrow$ & $\uparrow$ & $\uparrow$   & $\uparrow$    & $?$ ($\downarrow$) \\
    \textbf{$\frac{1}{2}Y\rightarrow \alpha Y$} 
            & $\uparrow$ & $\uparrow$ & $\downarrow$ & $\uparrow$    & $\downarrow\downarrow$ \\
    \textbf{$\alpha Y\rightarrow 0$} 
            & $\uparrow$ & $\uparrow$ & $\downarrow$ & $\downarrow$  & $?$ ($\downarrow$) \\ \hline
    \end{tabular}
    \caption{Question marks in the last column mean we cannot accurately predict its change. But it usually decreases in experiments.}
    \label{tab:chap6_dist_trend}
\end{table}

\begin{figure}[h]
\vskip -0in
    \begin{center}
    \centerline{\includegraphics[width=0.8\columnwidth,trim=0 0 0 0, clip]{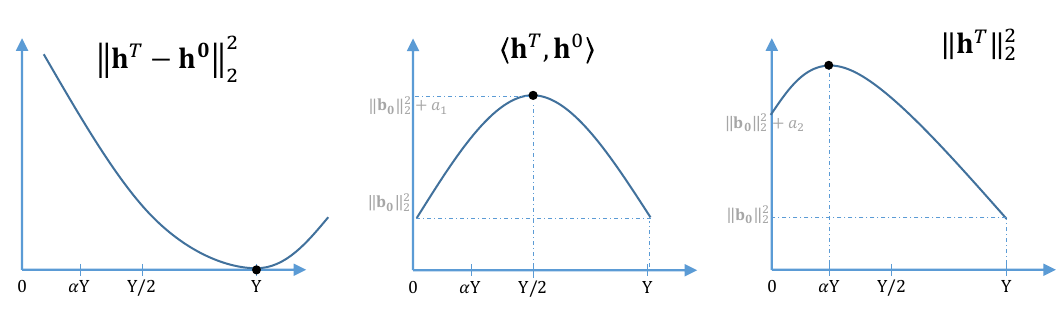}}
    \caption{Illustrations of different metrics if they are 1-D quadratic functions. The curves in \cref{fig:chap6_zadapt_5metrics} are a summarization of them.}
    \label{fig:chap6_quadratic}
    \end{center}
\vskip -0in
\end{figure}

\section{Strategies to Control Adaptation}
\label{sec:case3_03}
We already have a good understanding of how the features adapt under different settings from the analysis inspired by learning dynamics.
The next step is then design good HP-FT algorithms to further improve the downstream performance.
The general principle considered here is a ``\textit{phenomenon $\rightarrow$ hypothesis $\rightarrow$ solution}'' workflow, in which we use the validation performance for verification.
For example, we can have: \textit{HP train-acc converge to 100\% $\rightarrow$ not enough energy $\rightarrow$ use smaller $\tau$}.

However, it is hard to know how much energy is beneficial, as the neural network might not encode the input samples as we humans expect.
Hence, we suggest starting from the basic setting, i.e., using a linear head and sweeping $\tau$ to get a high-level understanding of how the energy influences the performance.
Usually, selecting the optimal $\tau^*$ using validation accuracy can ensure a reasonably good test result.
\cref{fig:chap6_vacc_sweep} provides a good example of doing such a sweep under different settings.
It is clear that in most cases, HP to convergence is not the optimal solution, meaning that leaving enough energy for feature adaptation would be beneficial.

\begin{figure}[t]
    \centering
    \includegraphics[width=1\textwidth]{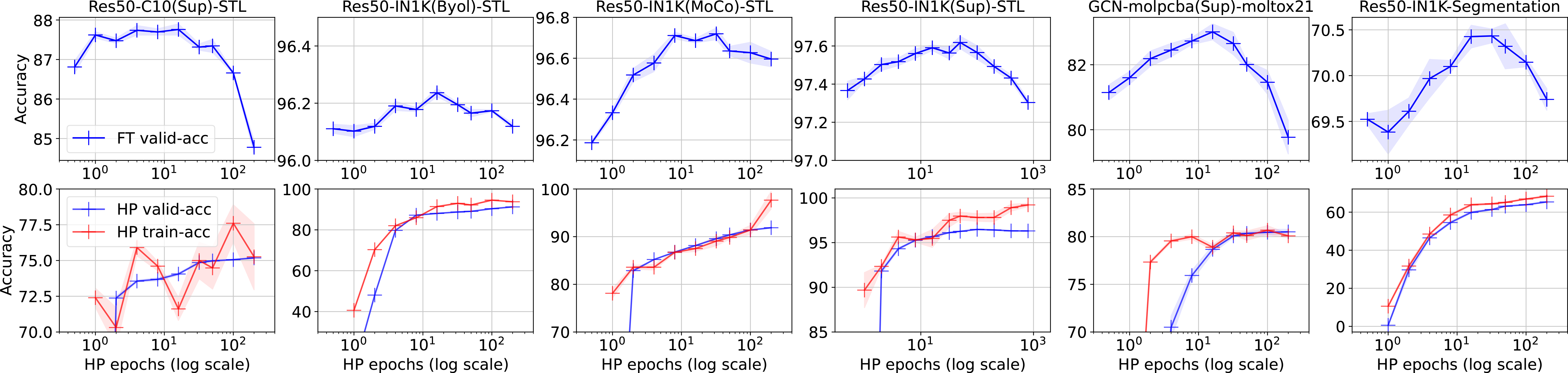}
    \caption{Sweeping $\tau$ from 0 to 200 (in $2^n$ fashion).
            The valid accuracy of the full-tuning-only setting is the leftmost point in each panel.
            The first 4 columns are on the image classification task, the fifth one is on the graph task,
            and the last one is on the image segmentation task.}
    \label{fig:chap6_vacc_sweep}
\end{figure}

Other than the basic $\tau$ sweeping methods for controlling the adaptation energy, we also provide some advanced tricks, which are only applicable to specific scenarios and need to be more careful.
For example, if we really want tiny energy but using the linear head only achieves less than 50\% training accuracy after head probing, we can consider an MLP head to increase the converged training accuracy (hence reduce the energy).
If we want a mild adaptation, but the training accuracy during HP goes to 100\% too fast, using label smoothing during HP can be considered.
If we believe the downstream task only needs low-level features of the backbone, copying the parameters of only one part of the backbone is a good choice.
However, these advanced tricks also have side effects that are deleterious to the downstream performance.
Here is a brief introduction to some of them, and the detailed design and experimental verifications can be found in our paper.

\textbf{MLP Head:}
Instead of using a linear head, \citet{chen2020improved} claim that using an MLP head can improve the head's separation ability and hence provide more tolerance for the feature's quality.
Following the analysis of this paper, we can consider this trick when we want small energy, while the linear head cannot reach a high HP training accuracy even after a long HP.
For example, in the first two columns in \cref{tab:chap6_main_tab2}, the HP-train-acc in the linear head cases plateaued after reaching 92\% or 78\%, while a 2-layer MLP can reach 99\% and 95\%.
As the energy decreases, features adapt less during finetuning and the final models generalize better.

However, we should be careful when applying this trick if we want to use a small $\tau$ (i.e., large energy).
Because a larger head will make the inconsistent direction problem of $\nabla_{\vh}\vz^t$ more serious.
Hence, increasing the head capacity in this case would make the head converge more slowly and hence prolong such a chaos phase.

\begin{table}[h]
\centering
\resizebox{0.9\textwidth}{!}{
    \begin{tabular}{cccc|cccccc}
    \hline
     &  & Sim-DomainA & Sup-DomainB &  & Sim-STL & Sup-STL &  & Sim-STL & Sup-STL \\ \hline
    \multicolumn{1}{c|}{HP-train-acc} & \multicolumn{1}{c|}{} & 92.676 & 78.213 & \multicolumn{1}{c|}{} & 96.875 & \multicolumn{1}{c|}{100} & \multicolumn{1}{c|}{} & {\color[HTML]{000000} 97.607} & {\color[HTML]{000000} 100} \\
    \multicolumn{1}{c|}{1-$\cos(\vh^T,\vh^0)$} & \multicolumn{1}{c|}{} & 0.1269 & 0.1422 & \multicolumn{1}{c|}{} & {\color[HTML]{000000} 0.0044} & \multicolumn{1}{c|}{{\color[HTML]{000000} 0.017}} & \multicolumn{1}{c|}{} & {\color[HTML]{000000} 0.0049} & {\color[HTML]{000000} 0.0256} \\
    \multicolumn{1}{c|}{$\|\vh^T-\vh^0\|_2^2$} & \multicolumn{1}{c|}{} & 5.247 & 4.752 & \multicolumn{1}{c|}{} & {\color[HTML]{000000} 1.792} & \multicolumn{1}{c|}{{\color[HTML]{000000} 6.716}} & \multicolumn{1}{c|}{} & {\color[HTML]{000000} 1.731} & {\color[HTML]{000000} 7.391} \\
    \multicolumn{1}{c|}{FT-val-acc} & \multicolumn{1}{c|}{\multirow{-4}{*}{\begin{tabular}[c]{@{}c@{}}Baseline\\ Linear-head\end{tabular}}} & 78.075 & 61.545 & \multicolumn{1}{c|}{\multirow{-4}{*}{\begin{tabular}[c]{@{}c@{}}Baseline\\ Small energy\\ $\eta_\text{ft}=1$\\ $\eta_\text{hp}=1$\end{tabular}}} & {\color[HTML]{333333} 93.914} & \multicolumn{1}{c|}{{\color[HTML]{333333} 97.581}} & \multicolumn{1}{c|}{\multirow{-4}{*}{\begin{tabular}[c]{@{}c@{}}Small energy\\ $\eta_\text{ft}=0.9$\\ $\eta_\text{hp}=0.9$\end{tabular}}} & {\color[HTML]{000000} 94.015} & {\color[HTML]{000000} 97.694} \\ \hline
    \multicolumn{1}{c|}{HP-train-acc} & \multicolumn{1}{c|}{} & {\color[HTML]{EA4335} 99.121} & {\color[HTML]{EA4335} 94.883} & \multicolumn{1}{c|}{} & 96.875 & \multicolumn{1}{c|}{100} & \multicolumn{1}{c|}{} & 97.982 & 100 \\
    \multicolumn{1}{c|}{1-$\cos(\vh^T,\vh^0)$} & \multicolumn{1}{c|}{} & {\color[HTML]{4285F4} 0.1098} & {\color[HTML]{4285F4} 0.1104} & \multicolumn{1}{c|}{} & {\color[HTML]{EA4335} 0.014} & \multicolumn{1}{c|}{{\color[HTML]{EA4335} 0.0549}} & \multicolumn{1}{c|}{} & {\color[HTML]{EA4335} 0.0109} & {\color[HTML]{EA4335} 0.0849} \\
    \multicolumn{1}{c|}{$\|\vh^T-\vh^0\|_2^2$} & \multicolumn{1}{c|}{} & {\color[HTML]{4285F4} 4.942} & {\color[HTML]{4285F4} 4.206} & \multicolumn{1}{c|}{} & {\color[HTML]{EA4335} 1.841} & \multicolumn{1}{c|}{{\color[HTML]{EA4335} 9.278}} & \multicolumn{1}{c|}{} & {\color[HTML]{EA4335} 3.041} & {\color[HTML]{EA4335} 13.881} \\
    \multicolumn{1}{c|}{FT-val-acc} & \multicolumn{1}{c|}{\multirow{-4}{*}{2MLP-head}} & {\color[HTML]{EA4335} 78.453} & {\color[HTML]{EA4335} 63.773} & \multicolumn{1}{c|}{\multirow{-4}{*}{\begin{tabular}[c]{@{}c@{}}More energy\\ $\eta_\text{ft}=1$\\ $\eta_\text{hp}=0.9$\end{tabular}}} & {\color[HTML]{EA4335} 94.304} & \multicolumn{1}{c|}{{\color[HTML]{EA4335} 97.92}} & \multicolumn{1}{c|}{\multirow{-4}{*}{\begin{tabular}[c]{@{}c@{}}Opposite Energy\\ $\eta_\text{ft}=0.9$\\ $\eta_\text{hp}=1$\end{tabular}}} & {\color[HTML]{4285F4} 93.208} & {\color[HTML]{4285F4} 97.039} \\ \hline
    \end{tabular}
    }
\caption{How lsHP and larger heads influence the feature adaptation.
         The models are pretrained using SimCLR (Sim, \cite{chen2020simple}) or supervised (Sup) classification on IN1K.
         Numbers in \blue{blue} (\red{red}) represent a \blue{decrease} (\red{increase}) compared with their counterparts in the baseline.}
\label{tab:chap6_main_tab2}
\vskip 0 in
\end{table}

\textbf{Label Smoothing HP:}
Recall that the energy is upper bounded by $\|\ve_y-\vp^0\|_2^2$,
where $\vp^0$ is the model's prediction after HP for $\tau$ epochs.
$\vp^0$ converges to the labels used in HP when $\tau\rightarrow\infty$.
Hence, instead of changing $\tau$, we can also manipulate the labels in HP to achieve a similar goal.
One simple yet effective way is label smoothing \citep{muller2019does}.
By setting the labels during HP as $\eta_\text{hp}\ve_y+\frac{1-\eta_\text{hp}}{V}*\vu$, where $\vu$ is a uniform $V$-class categorical distribution,
the HP stage can always reserve at least $(1-\eta_\text{hp})*\|\ve_y-\vu\|_2$ energy for the following feature adaptation,
even $\tau\rightarrow\infty$.
Such a trick (lsHP for short) is quite helpful when the HP-train-acc converges to 90\%+ very fast, yet we still want a mild adaptation,
like the example shown in the second two columns in \cref{tab:chap6_main_tab2}.
With lsHP, we see that the features adapt more, even though the HP-train-accs are unchanged.

To verify that the aforementioned improvement comes from the reserved energy during lsHP, we further try using smoothed labels during full-tuning (e.g., $\eta_\text{ft}=0.9$).
The results match our analysis: when $\eta_\text{hp}=\eta_\text{ft}=0.9$, the reserved energy disappears, as the labels of the two phases are the same again.
Hence, all the numbers under this condition are quite similar to the baseline case ($\eta_\text{hp}=\eta_\text{ft}=1$).
For the ``opposite energy'' case,
we observe a larger adaptation but a worse test performance.
That is because the reserved energy makes the features adapt in opposite directions.
These results remind us that if we decide to use smooth label in both FT and HP
(e.g., we assume most of the samples in the downstream dataset are ambiguous),
we need to set $\eta_\text{hp}\leq\eta_\text{ft}$ to ensure a correct force direction.

In summary, the lsHP trick is suitable for scenarios where the pretrained features are pretty well and the standard HP converges to 90\%+ very fast.
When the HP-train-acc is too low,
the assumption used in HP, i.e., $\vp^0$ converges to the labels, no longer holds.
Then, introducing the lsHP would make the model's behavior more unpredictable, and should hence be avoided.

\section{Conclusion and Future Works}
This chapter extends our learning dynamics framework from tracking log-probabilities to analyzing the adaptation of hidden feature representations. 
Although the resulting decomposition differs slightly from earlier cases, the core concepts of ``energy'' and ``direction'' remain applicable.

Motivated by this new decomposition, we find that the standard Head-Probe followed by Full-Tune (HP–FT) pipeline proposed by \citet{kumar2022fine} may underestimate the importance of feature adaptation, particularly when the source and target tasks differ substantially. 
To better characterize this process, we analyze the learning dynamics of the hidden feature $\vh$ under the overparameterized model (OPM) assumption. 
Beyond the conventional L2-distance, we introduce four additional quantities to capture the nature of feature adaptation: $\|\vh^T\|_2$, $\langle \vh^T, \vh^0 \rangle$, $\|\vh^T - \vh^0\|_2$, and $\cos(\vh^T, \vh^0)$.

The non-trivial trends uncovered by these metrics, illustrated in \cref{fig:chap6_zadapt_5metrics} and validated through both theoretical analysis and empirical experiments, provide insights into how features evolve under varying energy conditions. 
Specifically, when the energy is low (i.e., in the tiny or mild adaptation regimes), the final feature $\vh^T$ is gently stretched along the direction of the pretrained feature $\vh^0$.
In contrast, under high-energy conditions, $\vh^T$ tends to be severely distorted, as visually illustrated in \cref{fig1:chap6_zadapt_3D_real}.

Based on these observations, we speculate that a suitable adaptation energy of $\vh^T$ is the key when the pretrained features are not perfect for the downstream tasks. 
Under different experimental settings,  we illustrate how to achieve better downstream performance by controlling feature adaptations.

Finally, there are still many open questions not addressed here. 
For example, methods to quantitatively or even adaptively analyze the discrepancy between the pretrain and downstream tasks would be very useful in knowing what kind of head probing to perform in advance.
The learning dynamics framework could also be further extended to more practical settings, rather than OPM.
We left these interesting directions for our future work.

\chapter{Simplicity Bias and Compositionality}
\label{sec:case4}

This chapter explores a fundamental topic in representation learning: simplicity bias and compositional generalization.
Unlike previous chapters, which begin by introducing a specific system we want to analyze using learning dynamics, we start here by discussing three closely related concepts:
\begin{itemize}
    \item The Platonic Representation Hypothesis \citep{huh2024position};
    \item Kolmogorov complexity \citep{kolmogorov1963tables} and compression for AGI \citep{compression_for_AGI};
    \item Learning speed advantage and simplicity bias \citep{valle2018deep}.
\end{itemize}

Upon examining their assumptions and implications, we find that they all converge on a common principle: learning representations that adhere to Occam's Razor, i.e., favoring simpler explanations that can lead to better generalization. 
Interestingly, they observe that many large, well-trained models across different modalities naturally exhibit this preference for simplicity. 
In other words, simplicity bias appears to be an inherent and ubiquitous property of modern deep learning systems.

To investigate the origin of this bias, we study a compositional generalization problem through the lens of learning dynamics.
Our empirical analysis shows that simpler mappings are learned more quickly, a phenomenon that can be well-explained using force analysis on model predictions, similar to the techniques employed in earlier chapters.
Motivated by this, we implement a multi-generation self-distillation approach, also known as iterated learning, to further amplify the simplicity bias. Experimental results across various tasks and modalities strongly support our theoretical insights.

While our method improves compositional generalization in practice, our primary goal in this chapter is to highlight how learning dynamics offer a new explanatory framework for understanding simplicity bias. 
This micro-level perspective provides valuable insights into a foundational aspect of generalization and may guide the development of more efficient algorithms for learning structured and robust representations.

\section{Papers Covered in this Chapter}
\label{sec:case4_01}

\subsection{The Platonic Representation Hypothesis}
The Platonic Representation Hypothesis, introduced by \citet{huh2024position}, proposes that high-performing deep learning models across different modalities and training paradigms tend to converge toward a \textit{shared internal representation space}.
Drawing inspiration from Plato’s theory of ideal forms, the authors suggest that this convergence reflects a common underlying statistical structure of the world.
The models in different modalities could \textit{independently} discover this structure through their learning.

The central point of the above discussion is the assumption that there exists a unique (in terms of topology) and consistent generative mechanism underlying the data we observe from the world, which is also one of our core assumptions in this chapter, as in \cref{fig:chap7_theme}.
This assumption is implicitly related to the ``low-dimensional manifold'' hypothesis commonly applied in many machine learning textbooks. 
In essence, it posits that the primary goal of representation learning is to discover a mapping from input data to representations that faithfully capture the topological structure of this underlying generative mechanism. 
Once a model successfully learns such a mapping, systematic generalization would, in principle, be achieved.

Another important conclusion from their observations is that many high-performing models across different modalities tend to \textit{spontaneously} converge to such ground-truth-aligned representations. 
In other words, most well-trained models tend to naturally follow Occam's Razor, favoring simpler and more structured solutions. 
A particularly striking phenomenon, often summarized as ``all good models are good in similar ways, while all bad models fail differently,'' has been empirically demonstrated across a range of vision and language models.
\citet{huh2024position} provide several hypothetical explanations, and we will focus on the most related one, i.e., the simplicity bias, in this chapter.

\begin{figure}[t]
\vskip -0in
    \begin{center}
    \centerline{\includegraphics[width=0.8\columnwidth,trim=0 0 0 0, clip]{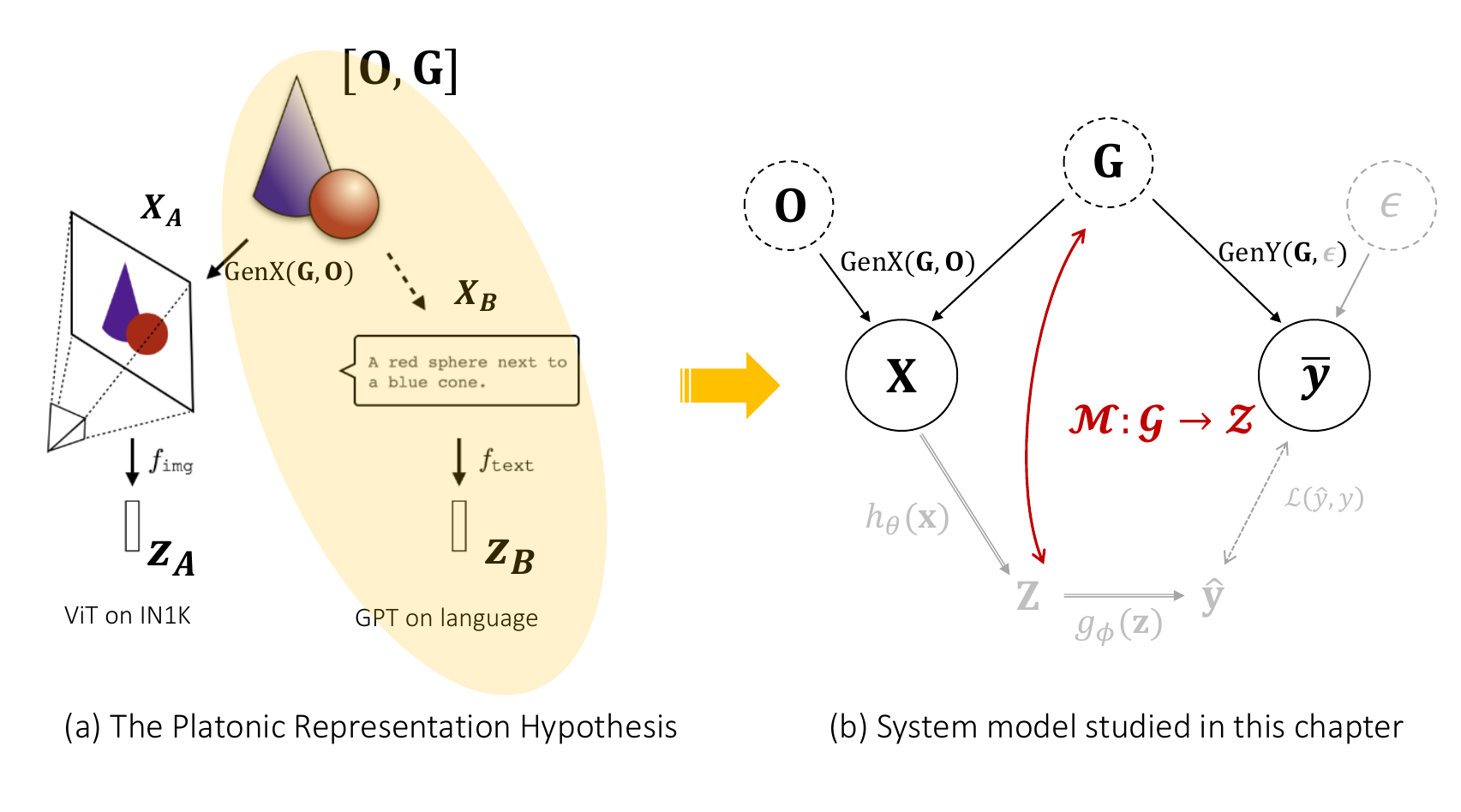}}
    \caption{The data-generating assumption, details in \cref{sec:case4_02}.}
    \label{fig:chap7_theme}
    \end{center}
\vskip -0in
\end{figure}

\subsection{Compression for AGI}
The idea of compression for AGI was originally proposed in the early 1960s through Solomonoff induction, and it has recently regained attention due to the success of applying scaling laws in pretraining the LLMs. 
The talk ``\textit{Compression for AGI}'' \citep{compression_for_AGI} provides a clear and compelling summary of this hypothesis.
In it, Rae argues that the ability to compress information, i.e., using the simplest representation to record all the observations, is not merely correlated with intelligence, but is in fact a synonym with it. 
This perspective revitalizes a long-standing belief in machine learning: that better compression implies deeper understanding and reasoning.

Rae supports this claim by examining how modern LLMs operate as advanced lossless compressors. 
Trained via next-token prediction, these models implicitly learn to compress vast corpora of human-generated text by capturing patterns, structures, and semantics. 
As evidence, he shows that LLMs outperform traditional compression algorithms, such as gzip and even domain-specific compressors designed for the Hutter Prize benchmark.

This notion of compression is also closely tied to the ``Minimum Description Length'' principle, which serves as a practical approximation of Kolmogorov complexity. 
If we assume that all observed data are generated by a consistent underlying mechanism, as postulated by the Platonic Representation Hypothesis, then achieving higher compression implies that the model has learned a representation closer to the ground-truth generative process, and vice versa. 
Following Rae’s line of reasoning, developing more capable models is equivalent to building models that achieve better compression, which in turn means they implement mappings with lower Kolmogorov complexity.

However, a key challenge remains: how can we measure compression rate or Kolmogorov complexity in practice? 
Rae proposes an elegant proxy: the integral under the training loss curve during next-token prediction. 
This corresponds to the relative learning speed of the model. 
In other words, relevant learning speed, a measurable and practical quantity in deep learning, becomes a bridge connecting abstract theoretical concepts such as Occam’s Razor, compression, and Kolmogorov complexity.\footnote{It is also important to note, however, that directly optimizing for learning speed can be risky due to Goodhart’s Law \citep{strathern1997improving}.}

By combining the insights from the Platonic Representation Hypothesis and Compression for AGI, we observe that many modern pretrained models appear to be naturally following a trajectory toward better compression, even without being explicitly optimized for it. 
This chapter seeks to answer a critical question: 
\begin{center}
    \textit{Why do models spontaneously adhere to this principle of simplicity bias?}
\end{center} 
To do so, we start by analyzing this phenomenon in a highly simplified setting, using learning dynamics as our primary analytical tool.

\subsection{Related Discussions in Our Works}
The materials presented in this chapter are drawn from several of our previous works.
Since the topics discussed above are very general, which is a bit hard to formalize in detail, we hence narrow them down to compositional generalization as a starting point.
The initial idea of connecting compositionality with learning speed advantage originates from our paper, ``\textit{Compositional Languages Emerge in a Neural Iterated Learning Model}'' \citep{ren_2020}. 
In that work, we studied a simple Lewis referential game \citep{lewis2008convention} within the emergent communication framework. 
We argued that, much like the human cognitive system, deep neural networks exhibit a preference for compositional mappings. 
This bias manifests as a learning speed advantage for simpler structures and can be progressively amplified through an iterated learning setup. 
To understand the roots of this phenomenon, we began analyzing the mutual influence between examples during training, which is one of the themes of learning dynamics.

Building on this foundation, our paper, ``\textit{Improving Compositional Generalization Using Iterated Learning and Simplicial Embeddings}'' \citep{ren2023improving}, extended our analysis and algorithm to more general representation learning problems.
The goal was to enhance systematic generalization across different tasks and modalities. 
As a theoretical justification, we first propose a compositionality ladder and claim that relying on merely the mutual information cannot separate compositional mappings from other general bijective mappings.
We then employed a group-theoretic tool and showed that the Kolmogorov complexity can distinguish those compositional mappings due to their structural compactness.

However, since Kolmogorov complexity is not computable, we must find some other mechanisms to understand this problem.
Subsequently, we move our focus back to the learning speed advantage discussed in \citet{ren_2020}.
By applying force analysis to the components of different mappings with distinct simplicity, we offer a novel interpretation of why simpler mappings tend to be learned faster in gradient descent.
These findings implicitly aligned with the ``compression rate'' discussed by \citet{compression_for_AGI}.
In short, because models cannot absorb all information in a single update, they must acquire knowledge sequentially. 
When the learning of one example increases the model's confidence on other related examples, the update is constructive. 
This leads to a simplicity bias: examples that form aligned forces during training are more likely to be learned more efficiently, whereas examples that exert contradictory forces are harder to learn. 
This interpretation is developed further in our paper, ``\textit{Understanding Simplicity Bias Toward Compositional Mappings via Learning Dynamics}'' \citep{ren2024understanding}.

\section{Compositionality and Compressibility}
\label{sec:case4_02}
\subsection{The Data Generating Assumption}

We start by introducing the data-generating mechanism and its relationship to representation learning in this section.
As illustrated in \cref{fig:chap7_theme}-(b), the semantic generating factors, also known as latent variables, are divided into two groups: the task-relevant factors (or semantic generating factors) $\vG=[G_1,..., G_m]$, and task-irrelevant (or noise) factors $\vO$.
This division depends on our understanding of the task; for example, if we only want to predict the digit identity of an image in the color-MNIST dataset \citep{arjovsky2019invariant}, then $m=1$ and $G_1$ represents the digit identity.
All the other generating factors, such as color, stroke, angle, and possible noise, are merged into $\vO$.
If we want to predict a function that depends on both identity and color, e.g., identifying blue even numbers, we could have $\vG=[G_1, G_2]$ with $G_1$ the identity and $G_2$ the color.

Each input sample $\vx\in\mathcal{X}$ is determined 
by a deterministic function $\mathsf{GenX}(\vG, \vO)$.
The task label(s) $\vy\in\mathcal{Y}$ only depend on the factors $\vG$ and possible independent noise $\epsilon$,
according to the deterministic function $\mathsf{GenY}(\vG,\epsilon)$.
Note $(\vx, \vO) \indep (\vy, \epsilon) \mid \vG$,
and that $\vO$, $\vG$, and $\epsilon$ are independent.
The data-generating distribution $P(\vx,\vy)$ is determined by the latent distributions $P(\vG)$ and $P(\vO)$,
along with the $\mathsf{GenX}$ and $\mathsf{GenY}$.
We assume $\mathsf{GenX}$ and $\mathsf{GenY}$ are fixed across environments (the ``rules of production'' are consistent),
while $P(\vG)$ and $P(\vO)$ might change between training and test.

For compositional generalization, we wish
to model the problem of generalizing to new combinations of previously seen attributes:
understanding ``red circle'' based on having seen ``red square'' and ``blue circle.''
Thus, we may assume that the supports of $P(\vG)$ are non-overlapping between train and test.
(If this assumption is not true, it only makes the problem easier.)
In summary, our goal is to find an algorithm $\mathcal{A}$
such that,
when trained on a dataset $\mathcal{D}_{train}\sim P_{train}^n$, $\mathcal{A}$
achieves small test risk $\mathcal{R}_{P_{test}}(\mathcal{A}(\mathcal{D}_{train}))$.
Here $P_{train}$ and $P_{test}$ should satisfy these conditions:
\begin{itemize}
    \item $P_{train}$ and $P_{test}$ have $\vG$, $\vO$, $\epsilon$ jointly independent, and $\vx = \mathsf{GenX}(\vG, \vO)$, $\vy = \mathsf{GenY}(\vG,\epsilon)$.
    \item $\mathsf{GenX}$ and $\mathsf{GenY}$ are the same for $P_{train}$ and $P_{test}$.
    \item In challenging cases, we may have $\mathsf{supp}[P_{train}(\vG)]\cap\mathsf{supp}[P_{test}(\vG)]=\emptyset$.
\end{itemize}

\subsection{The Ladder of Compositionality}
For compositional generalization, we expect the model to extract atomic semantic features from the training data,
and systematically re-combine them in a procedure akin to how the data is generated.
We thus consider a typical representation learning framework,
which resembles the inverse of the data generation process in \cref{fig:chap7_theme}-(b).
We use a backbone $h_\theta:\mathcal{X}\rightarrow\mathcal{Z}$
to convert the input signal $\vx$ into a representation $\vz$,
and a task head $g_\phi:\mathcal{Z}\rightarrow\mathcal{Y}$
to solve the given task based on that representation.
The prediction of the model is $\hat{\vy}=(g\circ h)(\vx)$.

\begin{figure}[t]
\vskip -0in
    \begin{center}
    \centerline{\includegraphics[width=0.8\columnwidth,trim=0 0 0 0, clip]{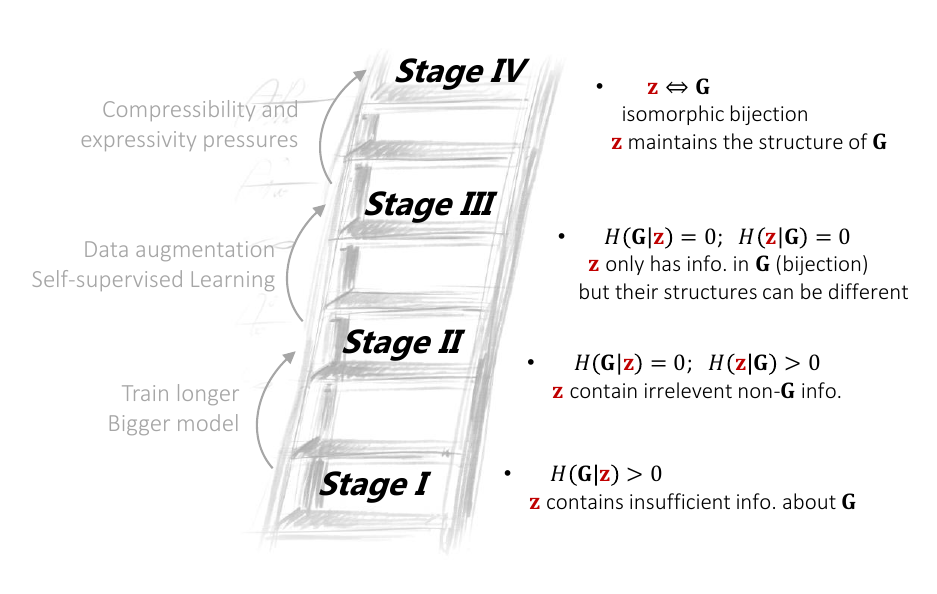}}
    \caption{The ladder of compositionality stating the requirements of $\vz$ using the entropy-related measurement. $H(\cdot)$ is the Shannon Entropy of a distribution. Details in Appendix A of our paper \citep{ren2023improving}.}
    \label{fig:chap7_ladder}
    \end{center}
\vskip -0in
\end{figure}

Intuitively, we would like our learned $\vz$ to uncover the hidden $\vG$,
and $g_\phi(\vz)$ to recover $\mathsf{GenY}(\vG,\epsilon)$.
We thus analyze how the relationship between $\vz$ and $\vG$ influences the model's generalization capability,
building off principles such as information bottleneck \citep{tishby2000information}.
Inspired by the ``ladder of causation'' \citep{pearl2018book},
we propose a ``ladder of compositionality'' in \cref{fig:chap7_ladder},
which outlining a series of conditions on $\vz$ and $\vG$.
We hypothesize that the compositional generalization roughly requires reaching the highest rung of that ladder:
\\
\begin{hyp}
    To generalize compositionally, the learned $\vz$ should capture exactly the information in $\vG$ and nothing more ($\vG$ to $\vz$ should be a bijection),
    and moreover it should preserve the ``structure'' of $\vG$
    (i.e.\ the mapping from $\vG$ to $\vz$ should be an isomorphism).
    \label{prop:z_sysgen}
\end{hyp}
More on this hypothesis, the ladder,
and relationship to models of disentanglement \citep{higgins2018towards}
are discussed in Appendix A of our paper \citep{ren2023improving}.
In short, we find that a model trained using common learning methods relying on mutual information between input $\vx$ and supervision $\vy$ cannot reliably reach the final stage of the ladder: it is necessary to seek other inductive biases in order to generalize compositionally.

In the remainder of this chapter, we will demonstrate that the topological similarity between $\vz$ and $\vG$, the Kolmogorov Complexity of the corresponding mapping $\mathcal{M}:\mathcal{Z}\rightarrow\mathcal{G}$, and the learning speed advantage can separate those high-compressed compositional mappings from other bijections (at least in the scope of compositional generalization problem).

\subsection{Kolmogorov Complexity of Different Bijections}
To build a link between compositionality and Kolmogorov complexity, we can first describe different bijections between $\vz$ and $\vG$ using group theory, and then apply the description length to compare the complexity of a typical element.
Specifically, assuming $\vz\in\mathcal{Z}, \vG\in\mathcal{G}$ and $|\mathcal{Z}|=|\mathcal{G}|$,
the space of all bijections between $\vz$ and $\vG$ is an isomorphism of a symmetric group $S_{|\mathcal{G}|}$.
If $\vG=[G_1,...,G_m]$ and each $G_m$ has $v$ different possible values, $|\mathcal{G}|=v^m$.
For clarity in the analysis,
we assume $\vz$ also has the same shape.
Then, any bijection between $\vz$ and $\vG$ can be represented by an element in $S_{v^m}$.

The space of compositional mapping,
which is a subset of all bijections,
has more constraints.
Recall how a compositional mapping is generated: we first select $z_i$ for each $G_j$ in a non-overlapping way.
Such a process can be represented by an element in $S_m$.
After that, we will assign different ``words'' for each $z_i$,
which can be represented by an element in $S_v$.
As we have $m$ different $z_i$, this procedure will be repeated $m$ times.
In summary, any compositional mapping can be represented by an element in the group $S^m_v\rtimes S_v\in S_{v^m}$,
where $\rtimes$ is the semidirect product in group theory.
The cardinality of $S_{v^m}$ is significantly larger than $S^m_v\rtimes S_v$, and so a randomly selected bijection is unlikely to be compositional.
We thus have:
\\
\begin{restatable}[Informal]{prop}{KC}
    \label{prop_1}
    For $m,v\geq2$, among all bijections, any compositional mapping has much lower upper bounds of its Kolmogorov complexity than a typical non-compositional mapping.
\end{restatable}

We can prove this by constructing descriptive protocols for each bijection.
As a compositional mapping has more \emph{reused rules},
its description length can be smaller (see \cref{sec:app2:k_complexity} for more discussions).

We now link the compositionality to a formal description of the simplicity, i.e., Kolmogorov complexity.
In the next section, we will showcase that the compositional mappings are, in general, learned faster than other bijections, and hence link the learning speed advantage to simplicity bias.

\section{Simplicity Bias and Learning Dynamics}
\label{sec:case4_03}
The discussions above link the concepts of compositionally, simplicity, and Kolmogorov complexity under an idealized setting,
which also aligns well with many related works.
This section further links these concepts to the learning speed advantage from a learning dynamics perspective.
Specifically, we would first demonstrate that a neural network naturally favors simpler mappings using experiments under manual settings, and then provide a detailed explanation of why such a tendency exists using force analysis.

\subsection{Toy256 Dataset and Different Mappings}

\begin{figure}[t]
\vskip -0in
    \begin{center}
    \centerline{\includegraphics[width=0.95\columnwidth,trim=0 10 0 0, clip]{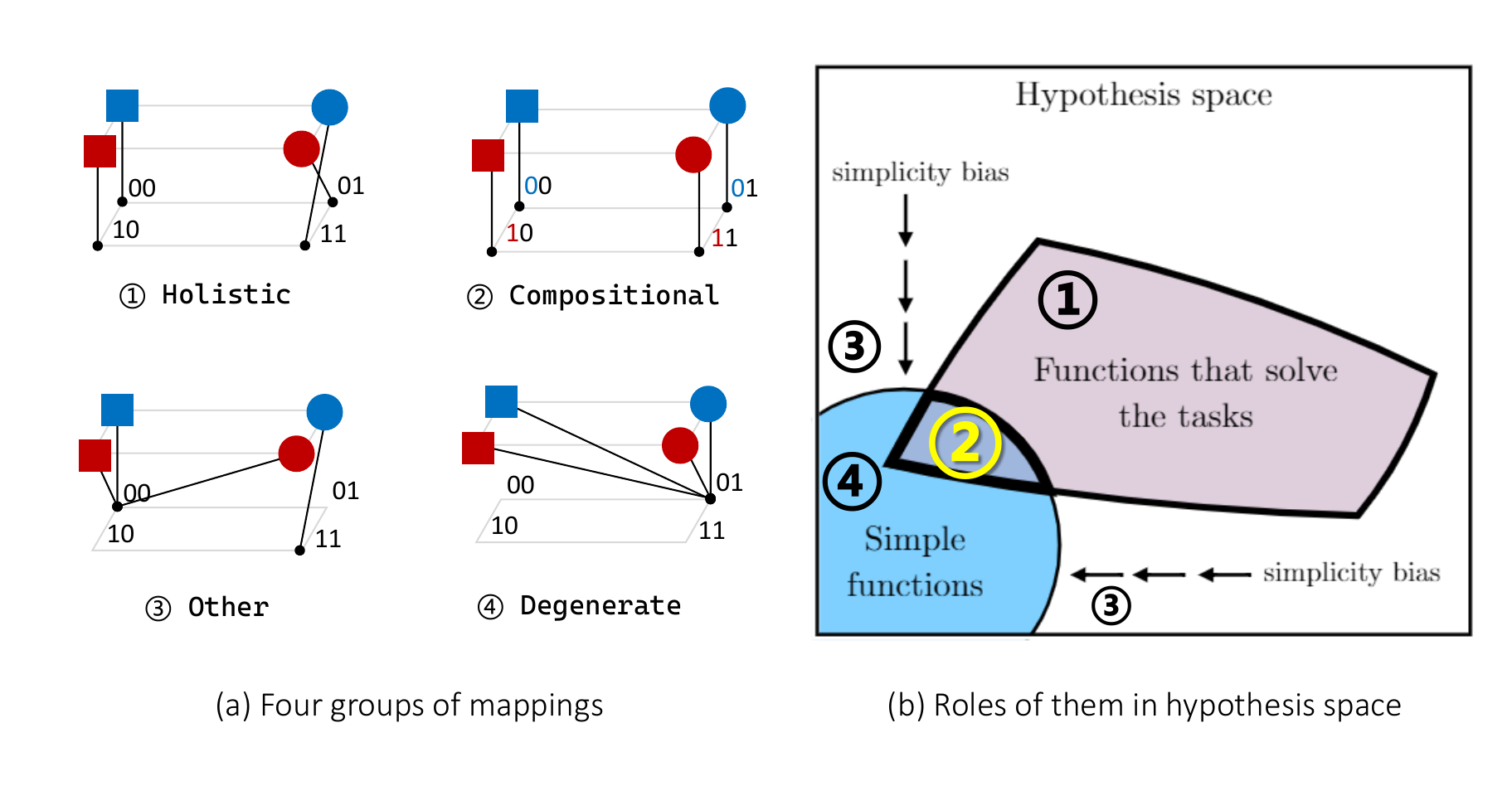}}
    \caption{Left: four groups of mappings in our \texttt{Toy256} example. Right: their roles, reproducing Figure 7 of \citet{huh2024position}.}
    \label{fig:chap7_toy256}
    \end{center}
\vskip -0in
\end{figure}

To verify our claim, we consider a multi-label classification problem and create 256 different datasets (each only contains 4 examples) for each mapping $\mathcal{M}$ in our \texttt{Toy256} setting.
For example, the dataset for the compositional mapping in \cref{fig:chap7_toy256}-(a) should be \{(\texttt{blue box}, 00), (\texttt{blue circle}, 01), (\texttt{red box}, 10), (\texttt{red circle}, 11)\},
where the label ``01'' means $\bar{y}_1=0$ and $\bar{y}_2=1$.
We then randomly initialize a neural network as our $h_\theta$ and concatenate two $\mathsf{Sigmoid}$ functions as our $g_\phi$.
With the same initialization and all hyperparameters, we train the network to converge for each dataset.

To facilitate the discussion, we categorize all 256 possible mappings into four non-overlapping groups:
\begin{itemize}
    \item Compositional mappings: These are isomorphisms of the ground-truth generating factors. Each dimension of the representation encodes a distinct feature using a stable, interpretable vocabulary. These mappings are highly compressed and reside on the fourth rung of the compositionality ladder. There are exactly $2! \times 2^2 = 8$ such mappings.
    \item Holistic mappings: These are non-compositional bijections. While they preserve one-to-one correspondence between inputs and outputs, they lack the structured alignment seen in compositional mappings. As a result, they occupy the third rung of the ladder. Excluding the 8 compositional cases, there are $4! - 8 = 16$ holistic mappings.
    \item Degenerate mappings: These map all objects to the same output. While they are minimal in complexity, their high ambiguity renders them generally ineffective for downstream tasks. There are $4$ such pure degenerate mappings.
    \item Other mappings: This group includes all remaining mappings, which may exhibit partial degeneracy or non-bijective behavior. These mappings fall somewhere between the above categories in terms of structure and utility. We don't further separate them in this thesis.
\end{itemize}
The roles these mappings play within the hypothesis space are illustrated in \cref{fig:chap7_toy256}. 
In essence, our objective is to identify mappings that are both unambiguous and simple, ideally, those that align with a compositional generating mechanism, and hence have a shorter description length.

\subsection{The Experimental Setups}
To verify the generality of our claim, we consider several different experimental settings.
For the input signals, we first consider two types of one-hot concatenation vectors.
One is the concatenation of two $2$-dim vectors (\texttt{OHT2} for short).
For example, \texttt{blue box} and \texttt{red circle} are encoded as $\texttt{[0101]}\cdot W_{4\times d}$ and $\texttt{[1010]}\cdot W_{4\times d}$, respectively,
where $W_{4\times d}$ is a randomly initialized matrix fixed for all 256 mappings.
Another setting is \texttt{OHT3}, which considers a redundant dimension for each attribute, where \texttt{blue box} and \texttt{red circle} are encoded as $\texttt{[010010]}\cdot W_{6\times d}$ and $\texttt{[100100]}\cdot W_{6\times d}$, respectively.
We also consider the vision input, where the image is sampled and colored from the dSprite dataset \citep{dsprites17}.

We use different network structures for different input modalities.
For the one-hot input, we use an MLP with three hidden layers with a width of 128.
For the image input, we consider both a 3-layer MLP and a ResNet9 \citep{he2016deep} with narrower hidden layers.
The task heads for all the networks are identical: we add two separate linear projection layers with size $h\times2$ on the output of the backbone.
After that, we take $\mathsf{Softmax}$ on each of these outputs to generate probabilistic predictions.
When calculating the loss function,
we consider both cross-entropy (CE) loss and a mean square error (MSE) loss,
where the latter is calculated between the predicted probability vector and a one-hot distribution of the ground truth labels.
When optimizing the network,
we consider both SGD and Adam.
Unless otherwise stated,
the learning rate is set to be $10^{-3}$, weight decay is $5*10^{-4}$,
and all other parameters are set to be the default values.
Note that all hyperparameters (including the initialization of the network) are shared for all 256 experiments in each group.

\subsection{Simpler Mappings are Learned Faster}

\begin{figure}[t]
\vskip -0in
    \begin{center}
    \centerline{\includegraphics[width=1\columnwidth,trim=0 10 0 0, clip]{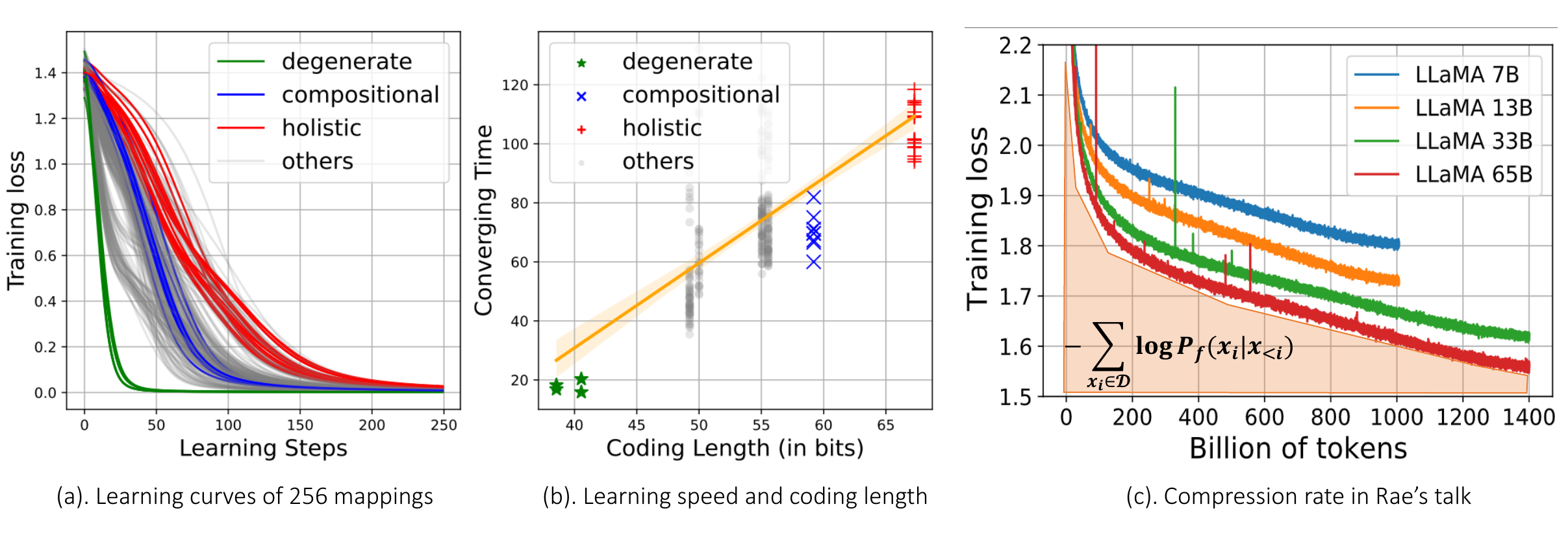}}
    \caption{Left: The learning curves of all 256 mappings. Middle: The correlation between the learning speed and the minimum description length of all mappings. Right: the compression rate in LLM's training. The figure comes from \cite{touvron2023llama} and is discussed in Rae's talk \citep{compression_for_AGI}.}
    \label{fig:chap7_learning_curves}
    \end{center}
\vskip 0.2 in
\end{figure}

See \cref{fig:chap7_learning_curves}-(a) where we show the training curves of 256 runs in our default setting.
Since the only difference among these runs is the dataset generated by different mappings, it is safe to conclude that the difference in their learning speed is caused by \textit{the inherent bias of the model's learning behavior} on this problem.
From this figure, we see that compositional mappings are learned faster than holistic ones.
However, some mappings are learned even faster, which makes sense because those non-bijection mappings contain degenerate components, i.e., two or more objects are mapped to the same $\vz$, which means they are simpler.
That also explains why the four degenerate mappings, which map all four objects to the same $\vz$, are learned fastest among all 256 mappings.
Nevertheless, as shown in \cref{fig:chap7_toy256}-(b), only those mappings that contain no degenerate component can perfectly explain all the training observations unambiguously.
In other words, before a battle between compositional mappings and holistic mappings, those non-bijections are already ruled out by the training accuracy.
In other words, they cannot even achieve the third level of the compositionality ladder.

To further verify the existence of the learning speed advantage, we quantify the learning speed using the concept of ``convergence time,'' i.e., the area under the learning curve.
A smaller convergence time means the mapping is learned faster.
This metric is similar to the C-score \citep{c_score},
which describes a training example's difficulty.
Also, if the model is trained with cross-entropy loss and all examples only appear once,
this metric is the compression rate for the entire dataset \citep{compression_for_AGI}, as shown in \cref{fig:chap7_learning_curves}-(c).
These works also hint that learning speed is deeply related to compression, simplicity, and generalization ability.

Furthermore, we also want to calculate explicitly the upper bounds on each mapping's Kolmogorov complexity.
Since the mapping space studied in our \texttt{Toy256} setting is simple enough,
we can first create the CFGs for each mapping and then convert them to a piece of description sequence using the method provided by \citet{kirby2015compression}.
After that, we can use Huffman coding \citep{huffman2007method} and calculate the coding length in bits for each mapping.
Due to the page limits, we omit the calculation of the coding length here; please refer to Appendix B of our paper for more details \citep{ren2024understanding}.
In short, a smaller coding length means the mapping is simpler.
It is clear in \cref{fig:chap7_learning_curves}-(b) that compositional mappings are the simplest bijections.
Note that all the non-bijection mappings contain degenerate components,
hence are simpler than bijections.
The figure also clearly demonstrates that the learning speed is strongly correlated to the coding length (with $\rho>0.65$ and $p<10^{-30}$),
matching our hypothesis well.

Lastly, we can also track a metric called ``Kernel Alignment'' studied in the Platonic Representation Hypothesis \citep{huh2024position}.
Different from the coding length discussed above, which depends on how we design the coding mechanism, the kernel alignment-based approach is model-agnostic and can mathematically describe the representation's similarity to the ground truth factors.
Specifically, they define this concept from the following three levels:
\begin{itemize}
    \item[1. ] A \textit{representation}, which maps the input signal to a dense representation, is a function $f:\mathcal{X}\rightarrow \mathbb{R}^n$.
                Our $h_\theta(\vx)$ plays a similar role;
    \item[2. ] A \textit{kernel}, $\kappa:\mathcal{X}\times\mathcal{X}\rightarrow\mathbb{R}$, characterizes the similarity between two elements in $\mathcal{X}$.
                In a dense representation case, the inner product is usually applied, i.e., $\kappa(x_i,x_j)=\langle f(x_i),f(x_j) \rangle$.
                Our paper consider Hamming distance, because our $\vG$ and $\vz$ are all categorical variables;
    \item[3. ] A \textit{kernel-alignment metric} $m:\mathcal{K}\times \mathcal{K}\rightarrow\mathbb{R}$, measures the similarity between two kernels.
\end{itemize}

For the third-level measurement, \citet{huh2024position} use Centered Kernel Distance,
a kernel-alignment metric throughout their paper.
Actually, many related works in compositional generalization also have similar measurements, e.g., the topological similarity proposed by \citet{brighton2006understanding}:
\begin{equation}
    \mathsf{Topsim}(\vG, \vz)\triangleq\mathsf{Corr}\left(d_G(\vG^{(i)},\vG^{(j)}), d_z(\vz^{(i)},\vz^{(j)}) \right),%
    \label{eq:topsim}
\end{equation}
This definition also follows three steps:
$\vz$ are representations generated by feeding $\vx$ to $h_\theta$;
$d_G$ and $d_z$ are distance measurements (or kernel in the second step above) for spaces $\mathcal{G}$ and $\mathcal{Z}$;
$\mathsf{Corr}(\cdot,\cdot)$ is Spearman's correlation, measuring the relationship between the output of two functions (kernels).
In short,
\textit{higher topological similarity means similar objects in $\mathcal{G}$ are mapped to similar positions in $\mathcal{Z}$.}
If we consider $\mathcal{G}$ as another modality of the projected ground truth,
the topological similarity is just a special kernel-alignment metric used in the Platonic Representation Hypothesis \citep{huh2024position}.
Also, some other measurements of compositionality like TRE (Tree Reconstruction Error, \cite{andreas2018measuring}),
representation disentanglement \citep{higgins2018towards}, etc., follow this general principle.

Then, using this kernel alignment metric (i.e., topological similarity), we further examine their correlation to the learning speed in \cref{tab:chap7_rho_p}.
It is clear that in most cases, the two quantitative metrics above align well.
Both of them support the claim that simpler mappings are indeed learned faster under different settings.
One exceptional case is training a ResNet with image input using the Adam optimizer.
The simplicity bias is even reversed compared with the results using SGD.
This phenomenon hints that the bias in DNNs' learning is also influenced by the inherent bias of specific network structures and optimizers.
We left this for our future work.

\begin{table*}[h]
  \centering
  \caption{The statistical correlation between learning speed and simplicity.}
  \vskip -0.1in
  \resizebox{\textwidth}{!}{

\begin{tabular}{ccccccc|ccccccc}
\hline
                                                                              &                                 &                        & \multicolumn{2}{c}{CL-Conv.Time} & \multicolumn{2}{c|}{Topsim-Conv.Time} &                                                                                   &                                 &                        & \multicolumn{2}{c}{CL-Conv.Time}                                                   & \multicolumn{2}{c}{Topsim-Conv.Time}                           \\ \cline{4-7} \cline{11-14} 
\multirow{-2}{*}{Input}                                                       & \multirow{-2}{*}{Optim.}        & \multirow{-2}{*}{Loss} & $\rho$         & $p$              & $\rho$           & $p$                & \multirow{-2}{*}{Input}                                                           & \multirow{-2}{*}{Optim.}        & \multirow{-2}{*}{Loss} & $\rho$                                  & $p$                                       & $\rho$                         & $p$                           \\ \hline
                                                                              &                                 & CE                     & 0.6475         & 8.1*1e-32        & -0.7101          & 1.4*1e-40          &                                                                                   &                                 & CE                     & 0.6866                                  & 5.0*1e-37                                 & -0.5911                        & 1.6*1e-25                     \\
                                                                              & \multirow{-2}{*}{\textbf{SGD}}  & L2                     & 0.5793         & 2.4*1e-24        & -0.7817          & 5.3*1e-54          &                                                                                   & \multirow{-2}{*}{\textbf{SGD}}  & L2                     & 0.5932                                  & 1.1*1e-25                                 & -0.6057                        & 5.1*1e-27                     \\
                                                                              &                                 & CE                     & 0.6598         & 2.3*1e-33        & -0.5731          & 9.4*1e-24          &                                                                                   &                                 & CE                     & 0.5403                                  & 8.4*1e-21                                 & -0.6720                        & 5.5*1e-35                     \\
\multirow{-4}{*}{\textbf{\begin{tabular}[c]{@{}c@{}}\texttt{OHT2}\\ MLP\end{tabular}}} & \multirow{-2}{*}{\textbf{Adam}} & L2                     & 0.5378         & 1.4*1e-20        & -0.7223          & 1.2*1e-42          & \multirow{-4}{*}{\textbf{\begin{tabular}[c]{@{}c@{}}Image\\ MLP\end{tabular}}}    & \multirow{-2}{*}{\textbf{Adam}} & L2                     & 0.4433                                  & 9.5*1e-14                                 & -0.6585                        & 3.4*1e-33                     \\ \hline
                                                                              &                                 & CE                     & 0.5976         & 3.5*1e-26        & -0.7963          & 2.2*1e-57          &                                                                                   &                                 & CE                     & 0.6711                                  & 7.4*1e-35                                 & -0.2619                        & 2.2*1e-5                      \\
                                                                              & \multirow{-2}{*}{\textbf{SGD}}  & L2                     & 0.6386         & 9.8*1e-31        & -0.7311          & 4.5*1e-44          &                                                                                   & \multirow{-2}{*}{\textbf{SGD}}  & L2                     & {\color[HTML]{C0C0C0} -0.015}           & {\color[HTML]{C0C0C0} 0.8159}             & {\color[HTML]{C0C0C0} -0.0297} & {\color[HTML]{C0C0C0} 0.6358} \\
                                                                              &                                 & CE                     & 0.5672         & 3.4*1e-23        & -0.6418          & 4.1*1e-31          &                                                                                   &                                 & CE                     & {\color[HTML]{F8A102} \textbf{-0.5115}} & {\color[HTML]{F8A102} \textbf{1.8*1e-18}} & {\color[HTML]{C0C0C0} 0.0423}  & {\color[HTML]{C0C0C0} 0.4999} \\
\multirow{-4}{*}{\textbf{\begin{tabular}[c]{@{}c@{}}\texttt{OHT3}\\ MLP\end{tabular}}} & \multirow{-2}{*}{\textbf{Adam}} & L2                     & 0.5582         & 2.3*1e-22        & -0.7026          & 2.1*1e-39          & \multirow{-4}{*}{\textbf{\begin{tabular}[c]{@{}c@{}}Image\\ ResNet\end{tabular}}} & \multirow{-2}{*}{\textbf{Adam}} & L2                     & {\color[HTML]{F8A102} \textbf{-0.2876}} & {\color[HTML]{F8A102} \textbf{2.8*1e-06}} & {\color[HTML]{C0C0C0} 0.0197}  & {\color[HTML]{C0C0C0} 0.7538} \\ \hline
\end{tabular}

    }
    \label{tab:chap7_rho_p}
    \vskip 0.1 in
\end{table*}

\subsection{Force Analysis of Different Mappings}

The results above bridge the simplicity bias to the model's learning behavior, where the latter can be further explained using our force analysis framework.
Remember our model generates probabilistic predictions on both $y_1$ and $y_2$ using $\mathsf{Sigmoid}$ functions.
Then, we can directly write down the predicted probability of each mapping as a product of eight terms,
as in the top line of \cref{fig:chap7_force_analysis}.
In this figure, we demonstrate how the model's confidence in different $\mathcal{M}$ is updated when learning specific training samples (the left-most column).
For example, in the first row of the compositional mapping in the figure, the model learns $(\texttt{blue box}, 00)$.
Then the corresponding $P(y_1=0\mid\texttt{blue box})$ and $P(y_2=0\mid\texttt{blue box})$ are significantly improved,
since they are directly updated by learning this example.
Furthermore, recall how learning a number ``4'' influences the number ``9'' in \cref{sec:fundamentalLD}, the model's predictions on those ``similar'' (measured using Hamming distance) input examples would also be indirectly updated (represented by the small arrows in the figure).
As a result, the model's confidence on \texttt{blue circle} and \texttt{red box}, which share one attribute with the learned \texttt{blue box}, are influenced more by this update.
Furthermore, since a compositional mapping always utilizes consistent values to represent the same attribute (e.g., $z_1=0$ always encodes the blue color), all the direct and indirect updates align well with the compositional mapping.
That is why the training loss of such mappings decreases faster.
On the contrary, for a holistic mapping, we observe several contradictions between the direct and indirect updates:
learning requires the model to use more updates to counteract those negative influences.
That partially explains why compositional mappings are learned faster than holistic ones when training a neural network using gradient descent.

\begin{figure}[t]
\vskip -0in
    \begin{center}
    \centerline{\includegraphics[width=1\columnwidth,trim=0 10 0 0, clip]{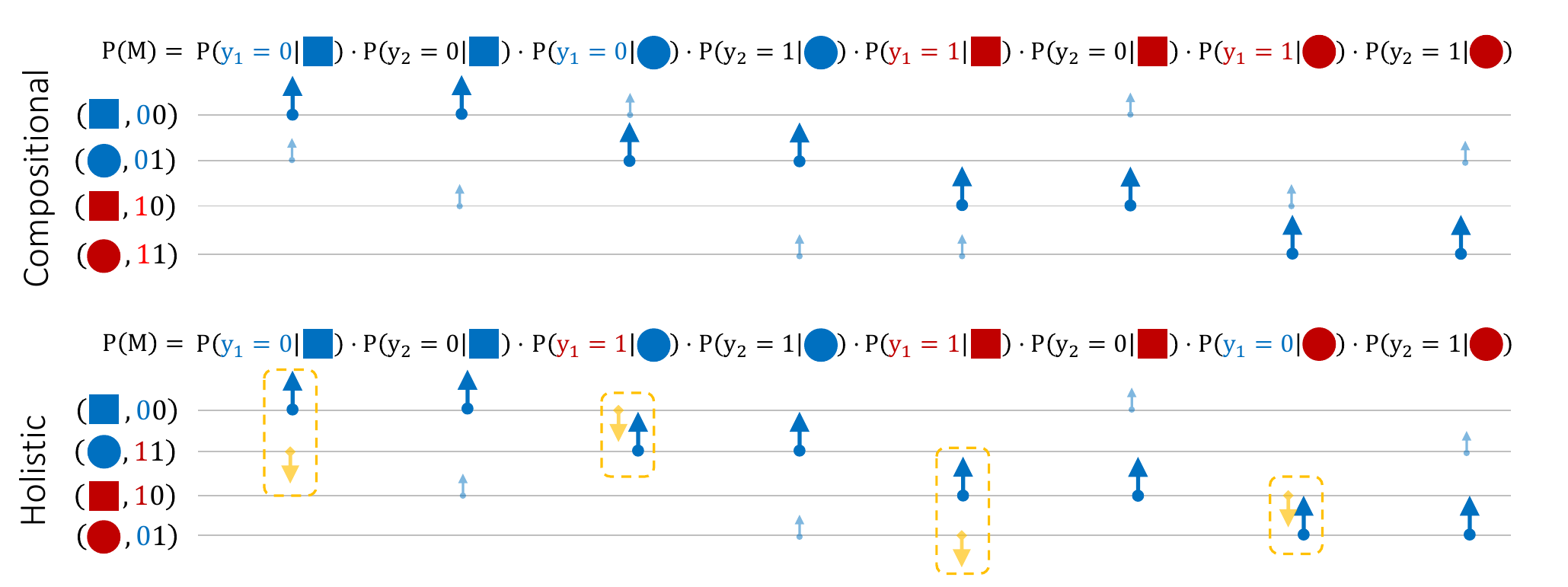}}
    \caption{Force analysis of two types of mappings.}
    \label{fig:chap7_force_analysis}
    \end{center}
\vskip -0in
\end{figure}

\section{Conclusion and Discussions}
\label{sec:case4_04}

This chapter began by introducing two influential concepts that are closely tied to the success of modern large-scale models: the Platonic Representation Hypothesis \citep{huh2024position} and Compression for AGI \citep{compression_for_AGI}.
Both ideas are grounded in Occam's Razor, emphasizing the role of simplicity of the learned representations in effective generalization. 
They argue that the scaling laws observed in contemporary models can be attributed to a pursuit of more efficient compression of observations.

In particular, \citet{huh2024position} posit the existence of unique and stable ground-truth data-generating mechanisms and empirically demonstrate that well-pretrained models, across various modalities, tend to converge to similar topological structures in their representation spaces. 
They suggest that such convergence reflects an implicit simplicity bias that adheres to Occam’s Razor and may underlie the success of these models.

Such a tendency is also observed by many researchers during the pretraining of LLMs \citep{compression_for_AGI}.
In the talk given by Rae, he first introduces many related concepts, e.g., Solomonoff induction, minimum description length, Kolmogorov complexity, and generalization, and then links them to the compression rate that a model can achieve.
He argues that a model with a high compression rate must have a better understanding of the world.
Such a measurement can then be estimated by the integral between the model's training curve and the x-axis in a next-token prediction task.
Then, those models that learn more quickly are inherently more intelligent, and this principle forms a cornerstone of the scaling laws in LLMs' pretraining.

Building on these insights, this chapter offers a novel mechanistic explanation for why gradient-based training tends to favor simpler, more structured solutions. 
We focus on a toy but illustrative task: learning compositional representations under a data-generating assumption akin to that in \cite{huh2024position}.
We argue that among all possible mappings from inputs to hidden representations, those that reflect the compositional structure of the data can generalize perfectly and (crucially) have significantly lower Kolmogorov complexity, as shown via group-theoretic arguments.

We then argue that the simplicity bias mentioned above should be embodied in the learning speed advantage in deep learning systems.
To verify this claim, we conduct several experiments in different settings on our \texttt{Toy256} dataset.
The results show that these compositional mappings are consistently learned faster. 
Through force analysis of the model’s confidence updates across different data samples, we find that compositional mappings benefit from aligned ``unified forces'' across examples, while unstructured mappings experience contradictory gradient directions. 
This provides a concrete explanation for the emergence of simplicity bias during training.

Although our experimental setting is deliberately minimal, similar findings have emerged across diverse practical contexts.
For instance, \citet{le2021analysis} and \citet{bengio2019meta} observe that models learn causal directions more easily than anti-causal ones. 
Likewise, \citet{zhang2016understanding} show that label permutation, analogous to creating holistic mappings from compositional ones, dramatically slows down learning and degrades generalization, despite perfect memorization being achievable. 
Our framework explains this through disrupted mutual influence: shuffled labels cause learning forces to interfere destructively rather than constructively.

In conclusion, we believe that the inherently sequential nature of learning in deep models makes the mutual influence between examples critical.
Although only verified on a very simple manual setting, the force analysis showcases some important principles.
If each example’s gradient contributes coherently (or at least orthogonally when learning new knowledge) to others, training becomes significantly more efficient. 
According to \citet{compression_for_AGI}, this leads to higher compression rates and higher intelligence.
We thus believe that analyzing learning speed and learning dynamics is a promising direction for future research, with the potential to illuminate fundamental aspects of deep learning systems.
We would consider more practical settings and further extend our analytical framework in our future work.

\chapter{Conclusion and Discussion}
\label{sec:conclusion}

After a long journey through the study of learning dynamics and deep learning, this thesis comes to a close. 
Given the wide range of topics covered, we provide a summary diagram in \cref{fig:chap8_summary} to help readers recall the main content of the thesis. 
Key concepts and representative figures are attached near each section's box to highlight the core ideas of individual chapters.

We began in \cref{sec:intro} and \cref{sec:overview} with a discussion of contemporary deep learning theory, contrasting math-like and physics-like aspects. 
We argued that the rapidly evolving nature of modern deep learning systems calls for a greater emphasis on local intuitions and empirical trends. 
This motivates our proposal to study model behavior from a local, microscopic perspective.
In particular, we introduced the idea of learning dynamics, which aims to analyze the ``forces'' exerted on model predictions when a single example is learned via gradient descent.

In \cref{sec:fundamentalLD}, we formalized this framework using influence functions and proposed the $\mathcal{AKG}$ decomposition. 
In this formulation, the force originates from $\mathcal{G}$, is projected and normalized by $\mathcal{K}$ and $\mathcal{A}$, and is ultimately imposed on the log-probability of the observing example $\cvxo$. 
We interpreted $\mathcal{K}$ as the similarity between $\cvxu$ and $\cvxo$, which is usually relatively stable during training.
This is illustrated by the mutual influence between digits ``4'' and ``9'' in our MNIST example.
$\mathcal{G}$ captures the source of the force and evolves dynamically during training.
Our ``zig-zag learning path'' example vividly demonstrates how these forces interact at different stages during training.

Motivated by this zig-zag pattern, we explored a natural application: automatic noisy label correction.
As discussed in \cref{sec:case1}, early-stopped models often learn to identify the correct labels by themselves. 
We studied this in the context of knowledge distillation and introduced Filter-KD, a pipeline that leverages better supervisory signals from a teacher model to improve generalization.
We theoretically show that supervision aligning better with the ground-truth distribution $\vq^*$ leads to stronger generalization, especially on ambiguous or ``hard-to-tell'' samples. 
Experiments support the effectiveness of this method well.

Next, in \cref{sec:case2}, we extended our force analysis to the finetuning of LLMs, which is a highlight of this thesis. 
To make the next-token prediction task analytically tractable, we introduced simplifications and a probing dataset that reveals how different supervision signals affect the output distribution over the huge $\mathcal{Y}$ space. 
We analyzed finetuning methods such as SFT, DPO, and GRPO, offering a unified perspective on their learning behavior. 
Our framework uncovered the important, often overlooked, role of negative gradients.
We propose the ``squeezing effect'', and further explore it both theoretically and empirically.
Extending this framework to other LLM adaptation methods presents a promising direction for future research.

In \cref{sec:case3}, we explored how our framework could be applied beyond output dynamics, e.g., to the adaptation of hidden features and other internal representations. 
We analytically modeled how intermediate features evolve under different levels of adaptation energy during finetuning, by observing four distance-related quantities of the hidden features.

Finally, in \cref{sec:case4}, we returned to the foundational question that motivated our study: Where does the simplicity bias come from? 
Using a carefully designed toy task, we connected our findings to several influential ideas in deep learning theory, including the Platonic Representation Hypothesis, Occam’s Razor, and compression for AGI. 
Although the setup is deliberately constrained, we believe the theoretical insights can shed more light on this deep and fundamental question.

To conclude, we acknowledge that the learning dynamics framework is still in its early stages. 
For instance, we have yet to fully explore the role of $\mathcal{A}$ or the directional effects embedded in $\mathcal{K}$. 
The interaction of our framework with specific architectural components (e.g., attention) or advanced optimizers (e.g., Adam) also remains underexplored.
Future work may also extend our analysis from the ``kernel'' regime to the more general and common feature learning regime. Given the connections between NTK and Gaussian Processes, combining them offers another compelling direction.

Ultimately, I hope this thesis is not the end but another beginning of a broader investigation into force analysis and learning dynamics in deep learning. 
The directions outlined here are open invitations for future exploration.

\begin{figure}[h]
\vskip -0in
    \begin{center}
    \centerline{\includegraphics[width=1\columnwidth,trim=0 0 0 0, clip]{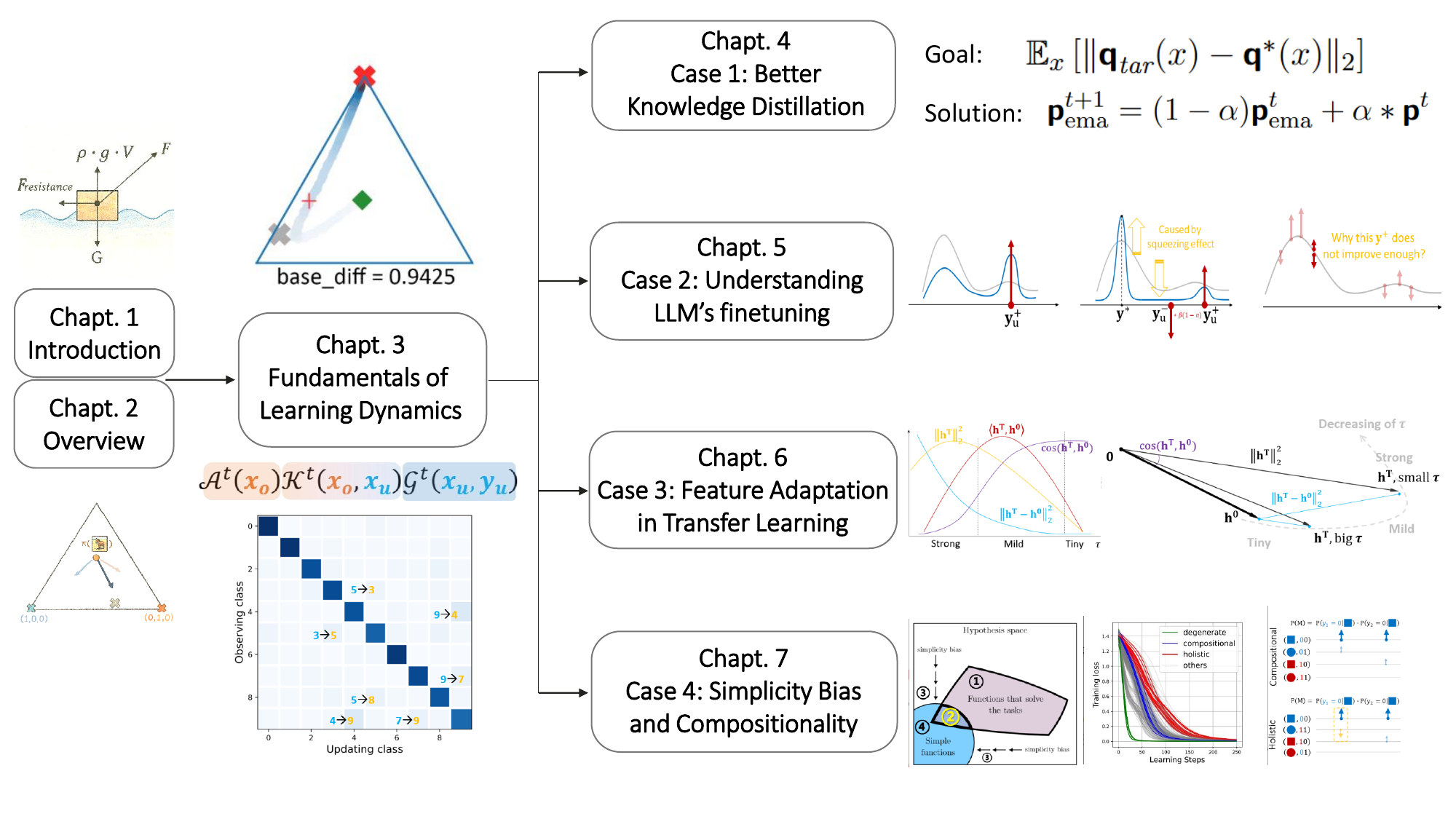}}
    \caption{Summary and highlights of each chapter.}
    \label{fig:chap8_summary}
    \end{center}
\vskip -0in
\end{figure}

\newrefcontext[sorting=nyt]
\printbibliography[heading=bibintoc, title={Bibliography}]
\appendix

\chapter{Proofs and Calculations}
\label{sec:app2}

\section{Risk Estimate Variance Bound in \cref{sec:case1_02}}
\label{sec:app2:risk-est-var}

The following result is a generalization of Proposition 3 of \cite{KD_probability},
whose proof we replicate and extend here:
\begin{prop} \label{prop:var-bound-extended}
Let $L$ be any bounded loss, $\mathcal{L}(f(x), y) \le \ell < \infty$, and consider $R_\tar$ of \cref{eq:tar_risk}.
For any predictor $f : \mathcal X \to \Delta^K$,
we have
\[
    \E_{\mathcal D}\left[(R_\tar(f,\mathcal{D}) - R(f))^2 \right]
    \le
    \frac{1}{n} \Var_{x}\left[ \vq_\tar\tp \vL(f( x))\right]
    + \xi
,\]
where $\xi$ can be any of the following seven quantities:
\begin{gather*}
    \ell^2 K \left( \E_{ x} \lVert \vq_\tar( x) - \vq^*( x) \rVert_2 \right)^2
    \\
    \ell^2 \left( \E_{ x} \lVert \vq_\tar( x) - \vq^*( x) \rVert_1 \right)^2
    \\
    2\ell^2 \left( \E_{ x} \sqrt{\mathsf{KL}( \vq_\tar( x) \,\|\, \vq^*( x) )} \right)^2
    \qquad
    2\ell^2 \E_{ x} \mathsf{KL}( \vq_\tar( x) \,\|\, \vq^*( x) )
    \\
    2\ell^2 \left( \E_{ x} \sqrt{\mathsf{KL}( \vq^*( x) \,\|\, \vq_\tar( x) )} \right)^2
    \qquad
    2\ell^2 \E_{ x} \mathsf{KL}( \vq_\tar( x) \,\|\, \vq^*( x) )
    \\
    \ell^2 \E_{ x}\big[
        \mathsf{KL}( \vq_\tar( x) \,\|\, \vq^*( x) )
        + 
        \mathsf{KL}( \vq^*( x) \,\|\, \vq_\tar( x) )
    \big]
.\end{gather*}
\end{prop}
\begin{proof}
To begin,
\[
    \E_{\mathcal D}\left[(R_\tar(f,\mathcal{D}) - R(f))^2 \right]
    = \Var_{\mathcal D}[R_\tar(f,\mathcal{D}) - R(f)]
    + \left( \E_{\mathcal D}\left[R_\tar(f,\mathcal{D}) - R(f) \right] \right)^2
.\]
For the variance term,
since $R(f)$ is a constant and
$R_\tar$ is an average of $n$ i.i.d.\ terms,
we get
\[
    \Var_{\mathcal D}[R_\tar(f,\mathcal{D}) - R(f)]
    = \frac{1}{n} \Var_{ x}[\vq_\tar( x)\tp \vL(f( x)) ]
.\]

The other term,
as $R_\tar$ is an average of i.i.d.\ terms and
$R(f) = \E_{ x} \left[\vq^*( x)^\top \vL(f( x))\right]$,
is
\[
    \left( \E_{\mathcal D}\left[R_\tar(f,\mathcal{D}) - R(f) \right] \right)^2
    =
    \left( 
    \E_{ x} \left[(\vq_\tar( x) - \vq^*( x))^\top \vL(f(x)\right]
    \right)^2
.\]

For the first bound, we apply the Cauchy-Schwarz inequality,
\[
    (\vq_\tar( x) - \vq^*( x))^\top \vL(f( x))
    \le \lVert \vq_\tar( x) - \vq^*( x) \rVert_2 \lVert \vL(f( x)) \rVert_2
;\]
since the elements of $\vL(f( x))$ are each at most $\ell$,
the term $\lVert \vL(f( x)) \rVert_2$
is at most $\sqrt{K} \ell$.

For the other bounds, we instead apply Hölder's inequality,
yielding
\[
    (\vq_\tar( x) - \vq^*( x))^\top \vL(f( x))
    \le  \lVert \vq_\tar( x) - \vq^*( x) \rVert_1 \lVert \vL(f( x)) \rVert_\infty
    \le \ell \lVert \vq_\tar( x) - \vq^*( x) \rVert_1
.\]
The KL bounds follow by Pinsker's inequality and then Jensen's inequality.
The last bound, for the Jeffreys divergence,
combines the two KL bounds.
\end{proof}

\section{Calculation of the $\mathcal{G}^t_\text{SFT}(\chi_u)$ in \cref{eq:sec5_LLM_SFT_LD}}
\label{sec:app2:cal_G}
We here showcase a standard $\mathcal{G}^t_\text{SFT}(\chi_u)=\nabla_{\vz}\mathcal{L}(\chi_u) \mid_{\vz^t}$ considering a standard NLL loss.
As the auto-regression nature of the SFT loss is already encoded in the causal mask used in $h_\theta$, the columns in $\mathcal{G}^t(\chi_u)$ are independent of each other, which can be separately calculated.
Plus, the summation over $l$ can also be achieved by left-multiplying a length-$L$ all-one vector $\bmf{1}$.
Specifically, the SFT loss for each $l$ is:
\begin{align}\nonumber
    [\mathcal{L}_\text{SFT}(\chi_u)]_l 
    &= -\log \pi(\vy_l=y_u \mid \chi_u)\\\nonumber
    &= -\ve_{y_u}^\top  \log \pi(\vy_l\mid \chi_u)\\\nonumber
    &= -\ve_{y_u}^\top  \log \left(\sigma(\vz_l)\right),\nonumber
\end{align}
where $y_u$ is for the $l$-th dimension of $\vy_u$. 
The gradient of $\mathcal{L}$ on $\vz$ can be then calculated as:
\begin{align}
    [\mathcal{G}_\text{SFT}^t(\chi_u)]_l 
    &= 
    \underbrace{
        \nabla_{\vz_l}[\mathcal{L}_\text{SFT}(\chi_u)]_l
    }_{1\times V}
    =  
     \left( \underbrace{ \nabla_{\pi}[\mathcal{L}_\text{SFT}(\chi_u)]_l}_{V\times 1} \right)^\top  \underbrace{ \nabla_{\vz_l}\pi }_{V\times V}\nonumber\\
    &= - \left( \ve_{y_u}\oslash\pi \right)^\top{\nabla_{\vz_l}\pi} 
    = \pi(y_l\mid\chi_u) - \ve_{y_u},
\end{align}
where $\oslash$ is element-wise division.

To calculate the equation above,
we first recall the NLL loss of the $l$-th token is $[\mathcal{L}_\text{SFT}]_l=-\ve_{y_l^+}^\top\log\pi$, where $\pi=\sigma(\vz)$.
Then, $\underbrace{\nabla_{\vz}\mathcal{L}}_{1\times V}=\underbrace{\nabla_{\pi}\mathcal{L}}_{1\times V}\underbrace{\nabla_{\vz}\pi}_{V\times V}$.
For each dimension of $\nabla_{\vz}\mathcal{L}_l$,
we have $\frac{\partial{\mathcal{L}}}{\pi_i}=0$ if $\pi_i\neq y_l$ and $\frac{\partial{\mathcal{L}}}{\pi_i}=-\frac{1}{\pi_i}$ if $\pi_i=y_l$.
By writing it in vector form, we have $\nabla_{\vz}\mathcal{L}=-(\ve_{y_l}\oslash\pi)^\top\nabla_{\vz}\pi$.
For $\nabla_{\vz}\pi$, we have:
\[
    \nabla_{\vz}\pi
    = \left[
    \begin{matrix}
        \pi_1(1-\pi_1)  &-\pi_2\pi_1    &\cdots     &-\pi_V\pi_1 \\
        -\pi_1\pi_2   &1-\pi_2\pi_2   &\cdots     &-\pi_V\pi_2 \\
        \dots   & \dots   &\ddots     &\dots  \\
        -\pi_1\pi_V   &-\pi_2\pi_V    &\cdots     &1-\pi_V\pi_V
    \end{matrix}
    \right].
\]
Combining this matrix and the $1\times V$ vector $(\ve_{y_l}\oslash\pi)^\top$,
where the only non-zero term is $\frac{1}{\pi_k}$ at the $k=y_l$ position.
So, left multiplying by this vector is actually first selecting the $k$-th row of $\nabla_{\vz}\pi$, and then multiplying $\frac{1}{\pi_k}$ to it.
In summary, we have:
\begin{align}\nonumber
    \nabla_{\vz}\mathcal{L}
    &=-\frac{1}{\pi_k}[-\pi_k\pi_1, -\pi_k\pi_2, \dots, -\pi_k(1-\pi_k),\dots, -\pi_k\pi_V]^\top \\\nonumber
    &=[\pi_1, \pi_2, \dots, \pi_k-1,\dots, \pi_V]^\top \\\nonumber
    &=\pi-\ve_k
\end{align}
By stacking the terms with different $l\in [L]$, we can get
\begin{equation}
    \mathcal{G}^t_\text{SFT}(\chi_u)=\nabla_{\vz}\mathcal{L}_\text{SFT}(\chi_u)|_{\vz^t}=\pi_{\theta^t}(\vy\mid\chi_u) - \ve_{\vy_u}
\end{equation}

\section{The Inner Product Calculation in GRPO (Proof of \cref{eq:GRPO_alpha}}
\label{sec:app2:grpo}

This appendix provides detailed steps for getting \cref{eq:GRPO_alpha}, i.e.,
\[
    \left<\nabla_{\theta}\mathcal{L}_{{o,m}},\nabla_\theta\mathcal{L}_{{n,l}}\right>
    \approx\alpha_{(o,m),(n,l)}
    \triangleq(\ve_{y_{o,m}} - \pi_{y_{o,m}})^\top (\ve_{y_{n,l}} - \pi_{y_{n,l}})\vh^\top_{n,l}\vh_{o,m}.
\]

Recall \cref{eq:ld_grpo} and UFM, in which $\theta\triangleq[\vw;\vh]$, where $\vw\in\mathbb{R}^{V\times d}$ and $\vh\in\mathbb{R}^{d\times 1}$ are both trainable parameters.
We use the same subscript system for both $\vh$, $\vchi$, and $\vy$.
So, the feature vector of $\vchi^+_{o,<m}$ is hence $\vh^+_{o,<m}$.
Then, $\nabla_\theta\mathcal{L}_{o,m}=\nabla_{\theta}\log\pi\left( y\mid \vx_{o},\vy^+_{o,<m}\right)$ can be written as $\nabla_{[\vw;\vh]}\log \sigma\left(\vw\vh^+_{o,<m}\right)$, where $\sigma(\cdot)$ is the $\mathsf{Softmax}$ function.
Our inner product then becomes:

\begin{equation}
    \left< \nabla_{[\vw;\vh]}\log\sigma_{y_{o,m}}(\vw\vh_{o,m}),    \nabla_{[\vw;\vh]}\log\sigma_{y_{n,l}}(\vw\vh_{n,l}) \right>.
    \label{eq:ld_grpo_inner_product}
\end{equation}

To make the derivation easy to follow, we suggest keeping the shape of these matrices in mind during the derivation.
First, $\nabla_{[\vw;\vh]}\log\sigma_{y_{o,m}}(\vw\vh_{o,m})$ is taking the gradient of one scalar (i.e., the $y_{o,m}$-th token of the $\log\pi$) to all the parameters, so it should be a very long $(V*d+d)\times 1$ vector:
\begin{align}\nonumber
    \nabla_{[\vw;\vh]}\log\pi_{i}&\triangleq 
     \left[
        \frac{\partial\log\pi_{i}}{\partial\hat{\vw}_1}, \dots, 
        \frac{\partial\log\pi_{i}}{\partial\hat{\vw}_{V*d}},
        \frac{\partial\log\pi_{i}}{\partial\vh_1}, \dots ,
        \frac{\partial\log\pi_{i}}{\partial\vh_d}
    \right]     \\\nonumber
   &=
    \left[
        \nabla_{\hat{\vw}}\log\pi_{i}; \nabla_{\vh}\log\pi_{i}
    \right]
\end{align}
where $\hat{\vw}\in\mathbb{R}^{V*d\times 1}$ is the column-wise flatten version of $\vw$.
Using this decomposition, the inner product in \cref{eq:ld_grpo_inner_product} can be computed by first evaluating the inner products between the gradient of each component separately, and then summing them together, i.e.,
\begin{equation}
    \left< \nabla_{[\vw;\vh]}\log\pi_{i},    \nabla_{[\vw;\vh]}\log\pi_{j} \right> = 
    \left< \nabla_{\hat{\vw}}\log\pi_{i},    \nabla_{\hat{\vw}}\log\pi_{j} \right> + 
    \left< \nabla_{\vh}\log\pi_{i},    \nabla_{\vh}\log\pi_{j} \right>. \label{eq:ld_grpo_inner_product_wh}    
\end{equation}

Let's start with $\nabla_{\vw}\log\sigma_{y_{o,m}}(\vw\vh_{o,m})$.
Note that taking the gradient of $\vw$ is not a typo here, since $\nabla_{\hat{\vw}}$ is just a column wise rearrange of $\nabla_{\vw}$, and we already have the matrix representation of $\nabla_{\vw}\log$ in our previous sections.
Specifically, we have
\[
    \underbrace{\nabla_{\vw}\log\sigma_{y_{o,m}}(\vw\vh_{o,m})}_{V\times d} = 
     (\underbrace{\ve_{y_{o,m}}-\pi_{o,m}}_{V\times1})
     \underbrace{\vh_{o,m}^\top}_{1\times d},
\]
Each column of the above matrix is $(\ve_{y_o}-\pi_o)h_{om,i}$, where $h_{om,i}$ is the $i$-th element of $\vh_{o,m}$.
The inner product of two gradients can be calculated as:
\begin{align}\nonumber
    \left< \nabla_{\hat{\vw}}\log\pi_{y_{o,m}},    \nabla_{\hat{\vw}}\log\pi_{y_{n,l}} \right>
    &= \sum_{i=1}^d(\ve_{y_{o,m}}-\pi_{y_{o,m}})^\top(\ve_{y_{n,l}}-\pi_{y_{n,l}})h_{om,i}h_{nl,i}\\\nonumber
    &= (\ve_{y_{o,m}}-\pi_{y_{o,m}})^\top(\ve_{y_{n,l}}-\pi_{y_{n,l}})\sum_{i=1}^d h_{om,i}h_{nl,i}\\
    &= \underbrace{(\ve_{y_{o,m}}-\pi_{y_{o,m}})^\top}_{1\times V}
       \underbrace{(\ve_{y_{n,l}}-\pi_{y_{n,l}})}_{V\times 1}
       \underbrace{\vh_{o,m}^\top\vh_{n,l}}_{1\times 1}\label{eq:ld_grpo_inner_product_w}
\end{align}

Similarly, the gradient with respect to $\vh$ is
\[
    \underbrace{\nabla_{\vh}\log\sigma_{y_{o,m}}(\vw\vh_{o,m})}_{d\times 1} = 
     \underbrace{\vw^\top}_{d\times V}
     (\underbrace{\ve_{y_{o,m}}-\pi_{y_{o,m}}}_{V\times1}).
\]
The inner product is then calculated as:
\begin{align}\nonumber
    \left< \nabla_{\vh}\log\pi_{y_{o,m}},    \nabla_{\vh}\log\pi_{y_{n,l}} \right> &= 
    \left(
        \vw^\top(\ve_{y_{o,m}} - \pi_{y_{o,m}})
    \right)^\top
    \left(
    \vw^\top(\ve_{y_{n,l}} - \pi_{y_{n,l}})
    \right) \\
    &= \underbrace{(\ve_{y_{o,m}} - \pi_{y_{o,m}})^\top}_{1\times V} 
       \underbrace{\vw\vw^\top}_{V\times V} 
       \underbrace{(\ve_{y_{n,l}} - \pi_{y_{n,l}})}_{V\times 1}.
    \label{eq:ld_grpo_inner_product_h}
\end{align}

Now, substituting \cref{eq:ld_grpo_inner_product_h} and \cref{eq:ld_grpo_inner_product_w} back to \cref{eq:ld_grpo_inner_product_wh}, we can get the decomposition of the learning dynamics between two tokens:
\begin{align}
    \left<\nabla_\theta\mathcal{L}_{o,m} ,\nabla_\theta\mathcal{L}_{n,l}\right>=& 
    (\ve_{y_{o,m}} - \pi_{y_{o,m}})^\top \left(
    \vh^\top_{n,l}\vh_{o,m}\cdot I_V + \vw\vw^\top\right) (\ve_{y_{n,l}} - \pi_{y_{n,l}}) \nonumber\\
    =&(\ve_{y_{o,m}} - \pi_{y_{o,m}})^\top (\ve_{y_{n,l}} - \pi_{y_{n,l}})\vh^\top_{n,l}\vh_{o,m}  \nonumber\\
    & +(\ve_{y_{o,m}} - \pi_{y_{o,m}})^\top\vw\vw^\top(\ve_{y_{n,l}} - \pi_{y_{n,l}}),
    \label{eq:ld_grpo_inner_product_detail}
\end{align}
where $I_V$ is a $V\times V$ identity matrix.
This decomposition also resembles a $\mathcal{AKG}$ style: the difference is that $\ve_{y_{o,m}} - \pi_{y_{o,m}}$ is only one column of $\mathcal{A}$.
The $\mathcal{G}$ term is still starting from $\pi$ and pointing to $\ve$ (note that we still have $-\eta$ outside the inner product).
The good news is that our $\mathcal{K}$ term is more interpretable, which sheds more light on the token-level mutual influence.

Ideally, we would compute every term in \cref{eq:ld_grpo_inner_product_detail} during training in order to accurately evaluate the inner product. 
However, in practice, calculating the term involving $\vw\vw^\top$ is significantly more computationally intensive than computing $\vh^\top\vh$, due to the dimensionality differences: $\vw \in \mathbb{R}^{V \times d}$ whereas $\vh \in \mathbb{R}^{d \times 1}$.
Moreover, we hypothesize that the hidden representations $\vh$, which tend to be relatively stable throughout training, already capture sufficient information about token-level similarity.
Therefore, we propose a simplified approximation by retaining only the first term on the right-hand side of \cref{eq:ld_grpo_inner_product_detail}. 
This approximation forms the basis of our influence score, denoted by $\alpha$, which serves as a practical and efficient surrogate for measuring token-level influence during GRPO training:
\begin{equation}
    \left<\nabla_{\theta}\mathcal{L}_{{o,m}},\nabla_\theta\mathcal{L}_{{n,l}}\right>
    \approx\alpha_{(o,m),(n,l)}
    \triangleq(\ve_{y_{o,m}} - \pi_{y_{o,m}})^\top (\ve_{y_{n,l}} - \pi_{y_{n,l}})\vh^\top_{n,l}\vh_{o,m}.
    \label{eq:app:GRPO_alpha}
\end{equation}

\section{Proof of Kolmogorov Complexity Bound in \cref{sec:case4_02}}
\label{sec:app2:k_complexity}

\KC*

\begin{proof}
    Recall the fact that any bijection from $\vz$ to $\vG$ can be represented by an element in the symmetry group $S_{v^m}$.
    From the definition of the symmetry group,
    we know each element in $S_{v^m}$ can be represented by a permutation matrix of size $v^m$.
    As there is only one $1$ in each row and column of a permutation matrix,
    any permutation matrix can be uniquely represented by a permuted sequence of length $v^m$.
    Specifically, assume we have a sequence of natural numbers $\{1,2,...,v^m\}$,
    each permuted sequence $\mathsf{Perm}(\{1,2,...,v^m\})$ represents a distinct permutation matrix,
    and hence represents a distinct bijection from $\vz$ to $\vG$.
    In other words,
    we can encode one bijection from $\vz$ to $\vG$ using a sequence of length $v^m$, 
    i.e., $\mathsf{Perm}(\{1,2,...,v^m\})$,
    and bound the corresponding Kolmogorov complexity (in bits) as
    \begin{equation}
        \mathsf{KC}(\text{bijection}) \le v^m\cdot \log_2 v^m=v^m\cdot m\cdot \log_2 v,
        \label{eq:app_KC_bijection}
    \end{equation}
    As an arbitrary bijection from $\vz$ to $\vG$ doesn't have any extra information to improve the coding efficiency,
    \cref{eq:app_KC_bijection} provides an upper bound of the minimal Kolmogorov complexity.

    On the contrary,
    as each compositional mapping can be represented by an element in $S_v^m\rtimes S_m$,
    we can encode the mappings more efficiently.
    Specifically, we need to first use $m$ sequences with length $v$,
    i.e., $\mathsf{Perm}(\{1,2,...,v\})$,
    to represent the assignment of ``words'' for each $z_i$.
    After that, we need a sequence of length $m$,
    i.e., $\mathsf{Perm}(\{1,2,...,m\})$ to encode the assignment between $z_i$ and $G_j$.
    The corresponding Kolmogorov complexity is then bounded as
    \begin{equation}
        \mathsf{KC}(\text{comp}) \le v\cdot\log_2 v + m\cdot\log_2 m,
        \label{eq:app_KC_comp}
    \end{equation}    
    Although this is only an upper bound, by a counting argument \emph{most} such mappings must have a complexity no less than, say, a constant multiple of that bound.

    To compare their Kolmogorov complexity,
    we can define a ratio as $\gamma\triangleq \frac{\mathsf{KC}(\text{bijection})}{\mathsf{KC}(\text{comp})}$.
    Obviously, when $m\leq v$,
    $\gamma\geq\frac{m*v^{m-1}}{2}$,
    which is larger than 1 as long as $m,v\geq 2$.
    When $m>v$,
    $\gamma\geq\frac{v^m\log_2 v}{2\log_2 m}$,
    which is also larger than 1 when $m,v\geq 2$.
\end{proof}

Actually, there might be some mappings that are not purely compositional or holistic.
For example, we can have a mapping with $z_{i\leq 10}$ sharing the reused rules while other $z_{i>10}$ doesn't.
Then this type of mapping can be represented by an element in $S_v^{10}\rtimes S_{10}\rtimes S_{v^{m-10}}$.
As a mapping in this subset shares 10 common rules,
its Kolmogorov complexity is between $\mathsf{KC}(\text{bijection})$ and $\mathsf{KC}(\text{comp})$.
Intuitively,
\emph{for all bijections},
smaller $\mathsf{KC}(\cdot)$ means higher compressibility and higher compositionality.

\chapter{Other Metrics Used in the Thesis}
\label{sec:app3}

\section{Calculation of Expected Calibration Error}
\label{sec:app3_ece}
Expected calibration error is a measurement about how well the predicted confidence represents the true correctness likelihood \citep{guo2017calibration}. 
For example, if our model gives 100 predictions, each with confidence, say, $p(y=v\mid x)\in[0.7,0.8]$). Then we might expect there are $70\sim80$ correct predictions among those 100 ones. 
To calculate this, we first uniformly divide $[0,1]$ into $M$ bins, with each bin represented by $I_m\in\left(\frac{m-1}{M},\frac{m}{M}\right]$ and $B_m$ is all the $\vx$ samples whose confidence falls into $I_m$. 
Then, ECE is calculated as:

\begin{equation}
    \text{ECE}=\sum_{m=1}^M\frac{|B_m|}{n}\left|acc(B_m)-conf(B_m)\right|,
\end{equation}

where $acc(B_m)=\frac{1}{|B_m|}\sum_{i\in B_m}\mathbbm{1}(\hat{y}_i=v_i)$, $conf(B_m)=\frac{1}{|B_m|}\sum_{i\in B_m}p(y=v_i\mid x_i)$, $v_i$ is the true label and $\hat{y}_i=\text{argmax } p(y\mid x_i)$ is the model's prediction. All the ECEs mentioned in this thesis are calculated by setting $M=10$.

\chapter{More Details about the Experiments}
\label{sec:app4}

\section{The Probing Dataset in \cref{sec:case2}}
\label{sec:app4_chap5}

Besides the 9 types of responses demonstrated in \cref{fig:chap5_probing_responses} in the main context, we further extend our probing responses to 14 types, as follows.
Furthermore,
we also create another probing dataset (named $\mathcal{D}_\text{probtest}$) where all $\vx$ comes from the test set.
Compared with $\mathcal{D}_\text{probtrain}$ that we used in the main context,
all the prompts and responses in $\mathcal{D}_\text{probtest}$ are never exposed to the model during finetuning.
By comparing the learning curves of these two probing datasets,
we can figure out the difference between the model's prediction of those directly influenced responses ($\vy$ appears during training) and the indirectly influenced ones ($\vy$ that the model never sees during training).
Finally, we believe the level of the ``on-policy'' property (which is very important for the preference finetuning, as discussed in \cite{tajwar2024preference}) could also be introduced as the second axis in our 2-D plane.
We left the exploration of this interesting direction in our future work.

\begin{itemize}
  \item $\mathcal{Y}_\text{VF}$: valid feedback following the instruction $\vx_u$:
        \begin{itemize}
            \item[0.] $\vy_{\pi^0}$, response generated by feeding $\vx_u$ to LLM before SFT
            \item[1.] $\vy_u^+$, the chosen (i.e., the preferred) response to $\vx_u$
                \begin{itemize}
                    \item[1.1] $\vy_\text{selfr}^+$, rephrase $\vy_u^+$ using $\pi^0$, algorithm from \cite{yang2024self};
                    \item[1.2] $\vy_\text{gpts}^+$, rephrase $\vy_u^+$ using \texttt{ChatGPT}, keep the semantics while changing the format;
                    \item[1.3] $\vy_\text{gptf}^+$, rephrase $\vy_u^+$ using \texttt{ChatGPT}, keep the format while changing the semantics;
                \end{itemize}
            \item[2.] $\vy_u^-$, the rejected (i.e., the less preferred) response to $\vx_u$
                \begin{itemize}
                    \item[2.1] $\vy_\text{selfr}^-$, rephrase $\vy_u^-$ using $\pi^0$, algorithm from \cite{yang2024self};
                    \item[2.2] $\vy_\text{gpts}^-$, rephrase $\vy_u^-$ using \texttt{ChatGPT}, keep the semantics while changing the format;
                    \item[2.3] $\vy_\text{gptf}^-$, rephrase $\vy_u^-$ using \texttt{ChatGPT}, keep the format while changing the semantics;
                \end{itemize}            
                    \end{itemize}
  \item $\mathcal{Y}_\text{IF}$: irrelevant feedback to $\vx_u$ that are still recognizably human language (in these datasets, roughly ``internet-standard'' English):
        \begin{itemize}
            \item[3.] $\vy_{j\neq u}^+$, chosen response for a different question $\vx_{j\neq u}$ selected from the training set;
            \item[4.] $\vy_\text{test}^+$, chosen response of a question $\vx_\text{test}$ selected from test set;
            \item[5.] $\vy_\text{hum}$, a ``random'' English sentence generated by \texttt{ChatGPT} with as many words as $\vy_u^+$;
        \end{itemize}
  \item $\mathcal{Y}_\text{Un-gram}$: ungrammatical sequences:
        \begin{itemize}
            \item[6.] $\vy_\text{rnd}^+$, a random permutation of the words of $\vy_u^+$, or a randomly generated sequence;
            \item[7.] $\vy'_\text{rnd}$, a random permutation of the words of a generated sentence as in $\vy_\text{hum}$.
        \end{itemize}
\end{itemize}

\begin{figure}[h]
    \begin{center}
    \centerline{\includegraphics[width=1\columnwidth,trim=20 0 20 0, clip]{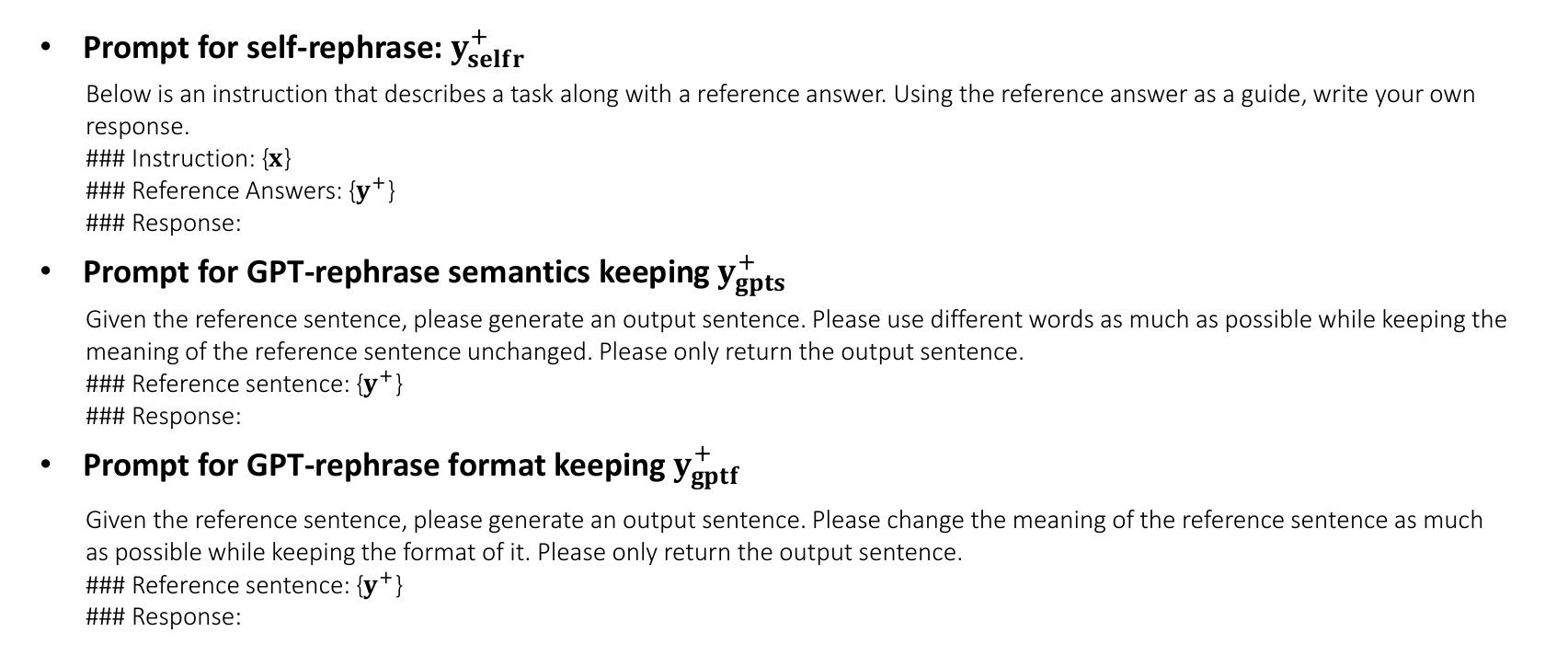}}
    \caption{The prompts used to generate $\vy_\text{selfr}^+$, $\vy_\text{gpts}^+$, and $\vy_\text{gptf}^+$. 
             The rephrases of rejected samples are generated similarly.
             The self-rephrase template comes from \cite{yang2024self}.}
    \label{fig:app4_chap5:app_prompt_design}
    \end{center}
\vskip -0in
\end{figure}

\begin{figure}[h]
    \begin{center}
    \centerline{\includegraphics[width=1\columnwidth,trim=0 0 0 0, clip]{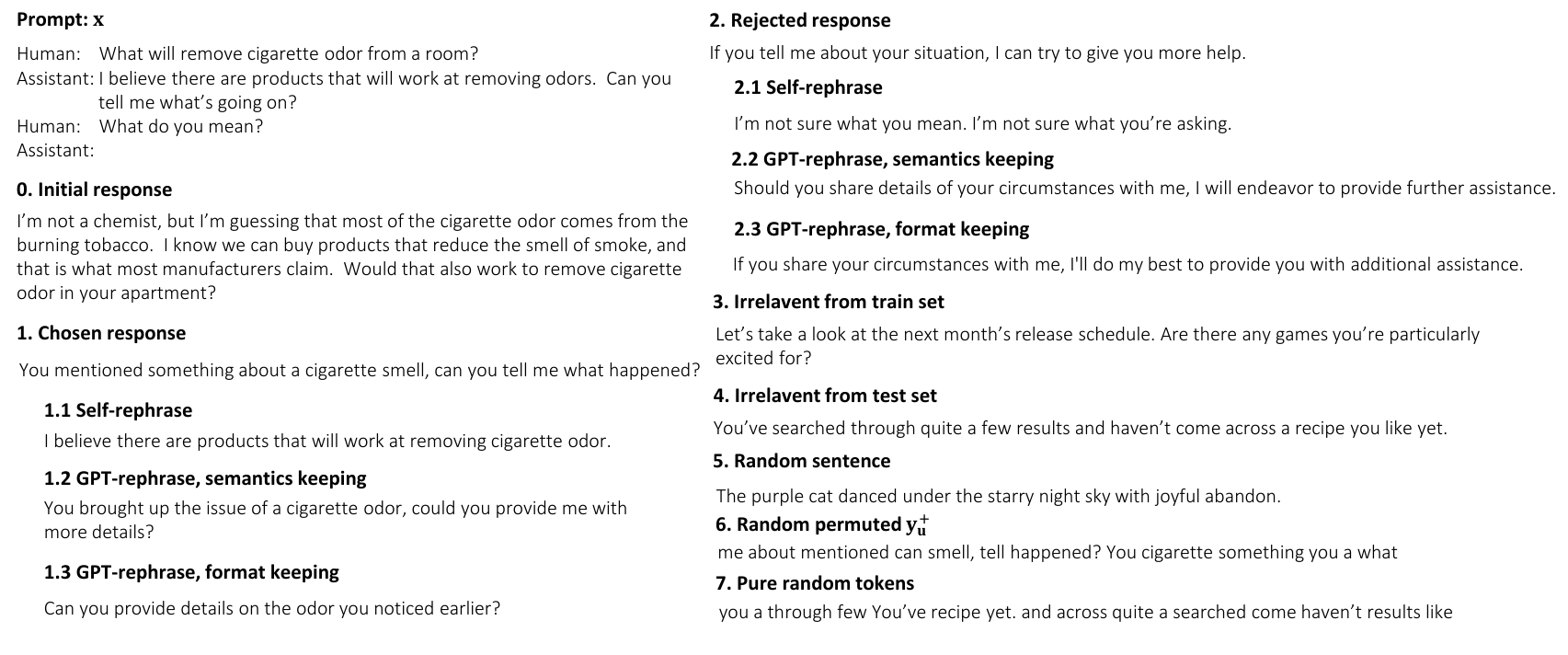}}
    \caption{Example of all possible responses for one $\vx$ in our probing dataset.
             Note that the pure random token is generated by first creating a random sentence,
             then randomly permuting its tokens.}
    \label{fig:app_response_examples}
    \end{center}
\vskip -0in
\end{figure}

\end{document}